\documentclass{article} 
\usepackage{iclr2027_conference,times}

\usepackage{amsmath,amsfonts,bm}

\def\eqref#1{equation~\ref{#1}}

\def\1{\bm{1}}

\DeclareMathAlphabet{\mathsfit}{\encodingdefault}{\sfdefault}{m}{sl}
\SetMathAlphabet{\mathsfit}{bold}{\encodingdefault}{\sfdefault}{bx}{n}

\usepackage{microtype}
\usepackage{amsmath}
\usepackage{amssymb}
\usepackage{amsfonts}
\usepackage{amsthm}

\usepackage{xcolor}
\usepackage{graphicx}
\usepackage{float}
\usepackage{wrapfig}
\usepackage{subcaption}
\usepackage{colortbl}

\usepackage{booktabs}
\usepackage{array}
\usepackage{tabularx}
\usepackage{threeparttable}
\usepackage{multirow}
\usepackage{multicol}
\usepackage{makecell}
\usepackage{longtable}
\usepackage{ragged2e}
\usepackage{adjustbox}
\usepackage{caption}
\usepackage{calc}
\usepackage{tikz}
\usepackage{algorithm}

\usepackage{algpseudocode}
\usepackage{listings}
\definecolor{lightblue}{RGB}{225,240,250}
\definecolor{softred}{RGB}{192,0,0}
\definecolor{darkgreen}{RGB}{0,100,70}
\definecolor{lightgray}{RGB}{242,242,242}

\definecolor{ttclboxbg}{RGB}{238,246,252}
\definecolor{ttclboxframe}{RGB}{120,170,210}
\definecolor{ttclboxtitle}{RGB}{170,205,232}

\newcommand{\repeatvalue}[1]{%
\begin{tikzpicture}[baseline=(val.base)]
    \draw[line width=0.5pt] (0,0.65) -- (0,-0.65);
    
    \draw[line width=0.5pt] (-0.35,0.65) -- (0.35,0.65);
    \draw[line width=0.5pt] (-0.35,-0.65) -- (0.35,-0.65);
    
    \node[fill=white, inner sep=1.5pt] (val) at (0,0) {#1};
\end{tikzpicture}%
}

\usepackage{enumitem}
\usepackage{nicefrac}
\definecolor{citeblue}{RGB}{70,110,180}
\usepackage[colorlinks=true, citecolor=citeblue]{hyperref}

\usepackage[most]{tcolorbox}

\newtheorem{theorem}{Theorem}[section]
\newtheorem{proposition}[theorem]{Proposition}

\title{TTCL: Test-time Calibration Learning for Large Language Model Reasoning}

\author{
\begin{tabular}{@{}l@{}}
\textbf{Zizhuo Zhang}$^{1}$\thanks{Equal contribution.}  \quad
\textbf{Xiong Peng}$^{1*}$ \quad
\textbf{Jingwei Sun}$^{1*}$ \quad
\textbf{Rong Yao}$^{2}$ \quad
\textbf{Shixiong Kai}$^{2}$ \\
\textbf{Mingxuan Yuan}$^{2}$ \quad
\textbf{Bo Han}$^{1}$\thanks{Correspondence to Bo Han (bhanml@comp.hkbu.edu.hk).}
\end{tabular}
\vspace{0mm} \\
$^{1}$TMLR Group, Department of Computer Science, Hong Kong Baptist University \\ 
$^{2}$Huawei Noah's Ark Lab.
}

\iclrfinalcopy

\begin{document}

\maketitle

\lhead{}

\begin{abstract}
Reliable large language models (LLMs) must not only produce accurate answers but also express confidence that faithfully reflects their probability of being correct. Such calibration is essential for identifying uncertain predictions and supporting reliable decision-making in real-world deployment. Recent studies incorporate calibration learning into reinforcement learning (RL), jointly optimizing answer correctness and verbalized confidence using ground-truth correctness supervision. However, their reliance on labeled data limits their applicability in practical test-time settings, where ground-truth labels are unavailable and calibration may need to adapt to newly encountered target tasks. To address this challenge, we propose \textit{Test-Time Calibration Learning (TTCL)}, a label-free framework that jointly adapts reasoning accuracy and verbalized confidence directly on unlabeled target-task data. Specifically, TTCL derives self-supervision signals for both correctness and calibration from multiple model-generated responses, enabling calibration learning at test time without ground-truth labels. Theoretical analysis further establishes TTCL as a bounded surrogate for the ideal calibration objective. Extensive experiments on mathematical reasoning and factual question answering demonstrate that TTCL consistently improves both accuracy and calibration across diverse models and tasks. On base models, TTCL achieves an average relative accuracy improvement of $+40.13\%$ and an ECE reduction of $+70.80\%$ across eight benchmarks. Moreover, TTCL can further improve both accuracy and calibration for already calibrated models under domain shift, particularly when source-domain calibration transfers poorly to target tasks. In the math-to-factQA setting, TTCL achieves an average relative accuracy gain of $+20.35\%$ and reduces ECE by $+53.83\%$. The source code is released at~\url{https://github.com/tmlr-group/TTCL}.

\begin{figure}[h]
  \centering
  \vspace{-1mm}
  \includegraphics[width=\textwidth]{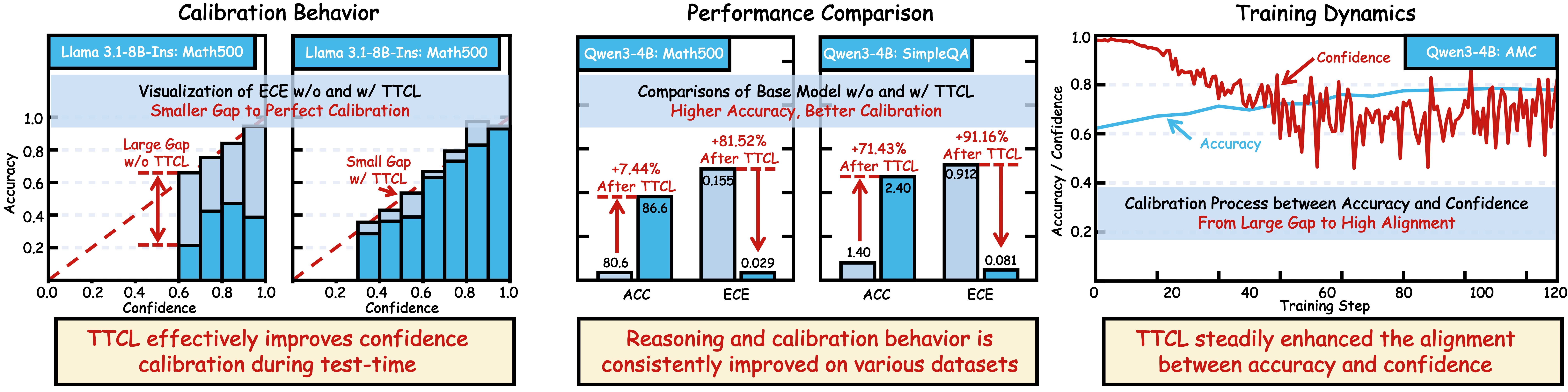}
  \vspace{-5mm}
  \caption{\textbf{Overview of our work}: TTCL enables label-free calibration learning directly on unlabeled target-task data, consistently improving reasoning accuracy while reducing calibration error across models and tasks, with confidence progressively aligning with accuracy during test-time learning.}
  \label{Fig:overview}
  \vspace{-6mm}
\end{figure} 
\end{abstract}

\section{Introduction}
\label{sec:introduction}
\vspace{-1mm}
Large language models (LLMs) are increasingly deployed in user-facing and decision-support systems~\citep{zhang2026co,sun2026trimem}, where the outputs may directly influence real-world decisions. Despite their remarkable capabilities, LLMs often assign high confidence to incorrect responses. Such miscalibration makes errors difficult to detect and can lead to severe consequences in high-stakes settings~\citep{ji2023hallucination}. Ideally, a well-calibrated model should produce confidence estimates that faithfully reflect its empirical likelihood of being correct. To achieve this, researchers have proposed calibration learning~\citep{damani2026RLCR,bani2026rewarding}, which combine calibration-oriented reward with answer correctness reward during reinforcement learning process to consistently improve LLM reasoning performance and confidence calibration behavior.

Despite the promising results achieved by previous research, we observe that existing calibration learning methods typically rely on large amounts of labeled data to provide supervision for calibration, incurring substantial annotation costs~\citep{damani2026RLCR,tan2026basecal,zollo2026unsupervised}. More importantly, a model calibrated on a source domain may not remain well-calibrated when testing on a different target domain. As shown in Fig.~\ref{Fig:Motivation}, a model calibrated by existing method RLCR~\citep{damani2026RLCR} on DAPO-Math-14K~\citep{yu2026dapo} transfers well to the math-related AMC benchmark~\citep{yang2024qwen2p5math}, but remains limited calibration behavior on the factual NQ benchmark~\citep{kwiatkowski2019natural}, highlighting the sensitivity of learned calibration to domain shift. These limitations pose a key challenge: \textit{how can we adapt a model's calibration behavior directly to an unseen target task at test time, without relying on ground-truth labels?}

To address this, we propose \textit{\textbf{T}est-\textbf{T}ime \textbf{C}alibration \textbf{L}earning} (TTCL), which provides a label-free surrogate for calibration learning. Specifically, TTCL derives two self-supervision signals from multiple sampled responses: a correctness reward based on the majority answer and a calibration reward based on the empirical support of each answer. The two rewards are separately applied to the answer and confidence tokens, allowing the model to improve answer correctness and confidence calibration without undesirable interference. To stabilize learning, we further introduce a target cache that maintains temporally lagged supervision targets, reducing the risk of reinforcing the model's own instantaneous predictions. To demonstrate the effectiveness of our method, we conduct extensive experiments on mathematical reasoning and factual question-answering benchmarks with multiple model scales and series. The results show that TTCL can consistently improve both reasoning accuracy and confidence on target task. Our main contributions are summarized as follows:

\begin{itemize}[leftmargin=1.3em, labelsep=0.5em, nosep]
\item Conceptually, we propose Test-Time Calibration Learning, a new calibration paradigm that adapts LLM confidence directly to the target task at test time without relying on ground-truth labels, enabling calibration of confidence estimation under unseen test distributions (Section~\ref{Sec:motivation_ttcl}).

\item Technically, we develop a label-free TTCL framework that constructs self-supervised signals for correctness and calibration rewards from multi-rollout frequency (Section~\ref{Sec:ttcl_1} \&~\ref{Sec:ttcl_2}).

\item Theoretically, we show that TTCL essentially serves as a bounded surrogate for the ideal calibration learning objective, theoretically grounding its label-free learning mechanism (Section~\ref{Sec:Theory}).

\item Experimentally, extensive experiments across models and benchmarks show consistent improvements of TTCL in both reasoning and calibration performance  (Section~\ref{sec:experiments}).

\end{itemize}

\section{Preliminary}
\label{sec:preliminary}
\vspace{-1mm}
In this section, we first introduce the problem setup of calibration learning and the corresponding calibration measurement, and then review Group Relative Policy Optimization (GRPO), a widely used RL algorithm in eliciting LLM reasoning capability.

\par
\textbf{Problem setup of calibration learning.} 
Let $\pi_\theta$ denote an LLM parameterized by $\theta$. Given a question $x$, the model generates a response $y\sim \pi_\theta(\cdot|x)$ containing a reasoning trajectory and an extracted final answer $a=\operatorname{Ans}(y)$. Given the ground-truth (GT) answer $a^*$, the correctness indicator is defined as $z(y,a^*)=\mathbb{I}[a \equiv a^\star]\in\{0,1\}.$ A confidence estimator $c(x,y)\in[0,1]$ quantifies the model's belief that the generated response $y$ is correct. Let $C$ and $Z$ denote the corresponding confidence and correctness random variables. A perfectly calibrated model satisfies:
\begin{equation}
\label{Eq:Def_Calibration}
    \operatorname{Pr}(Z=1|C=c)=c,
\end{equation}
indicating that predictions with confidence $c$ should be correct about $c$ fraction of the time. In practice, calibration learning typically optimizes $c(x,y)$ against observed correctness $z(y)$ using a proper scoring rule~\citep{gneiting2007strictlyproper}. A common choice is the Brier score~\citep{glenn1950Brier}, $\mathcal{L}_{\mathrm{Brier}}=\left(c(x,y)-z(y,a^*)\right)^2$, which encourages the predicted confidence to match the probability of correctness. Building on this, recent methods~\citep{damani2026RLCR,bani2026rewarding,ma2026decoupling} jointly optimize reasoning correctness and confidence calibration during RL. Specifically, RLCR~\citep{damani2026RLCR} trains the LLM to verbalize a confidence score using reward as:
\begin{equation}
    r_{\mathrm{RLCR}} = z(y,a^*) - \left(c(x,y)-z(y,a^*)\right)^2,
\end{equation}
where the first term is a binary correctness reward, while the second is a Brier-score calibration reward. However, both objectives require the ground-truth answer $a^*$, which limits their applicability when annotated labels are not available. In this work, we focus on test-time calibration learning, where the model adapts its confidence estimation on unlabeled test data.

\par
\textbf{Calibration measurement.}
Expected Calibration Error (ECE) is a widely used metric for evaluating calibration by measuring the discrepancy between predicted confidence and empirical accuracy~\citep{guo2017ECE}. Given $N$ predictions partitioned into $M$ confidence bins $\{B_m\}_{m=1}^M$, ECE is defined as:
\begin{equation}
\label{Eq:ECE}
    \mathrm{ECE}=\sum_{m=1}^{M}\frac{|B_m|}{N}\left|\operatorname{Acc}(B_m) -\operatorname{Conf}(B_m)\right|,
\end{equation}
where $\operatorname{Acc}(B_m)$ and $\operatorname{Conf}(B_m)$ denote the average correctness and confidence within bin $B_m$, respectively. A lower ECE indicates better calibration, with confidence better aligned with accuracy.

\vspace{-2mm}

\begin{figure}[t]
  \centering
  \vspace{-2mm}
  \includegraphics[width=0.98\textwidth]{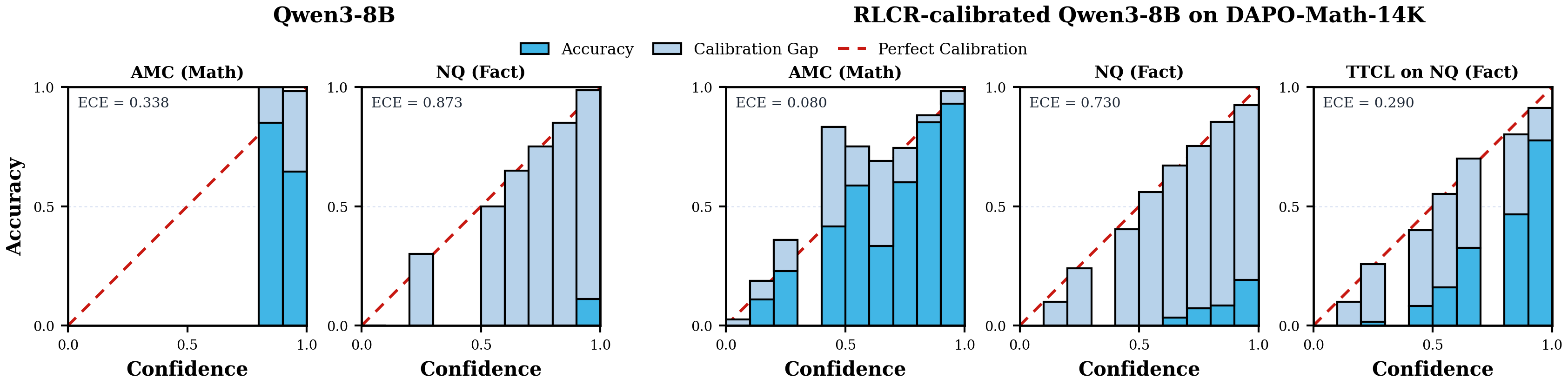}
  \vspace{-3mm}
  \caption{\textbf{Calibration performance of Qwen3-8B and its RLCR-calibrated model on DAPO-Math-14K.} \textit{Left:} Base Qwen3-8B exhibits over-confidence on both AMC and NQ, with predicted confidence exceeding empirical accuracy. \textit{Right:} RLCR on DAPO-Math-14K transfers well to AMC but remains limited on NQ, where further TTCL adaptation markedly improves calibration.}
  \label{Fig:Motivation}
  \vspace{-5mm}
\end{figure}

\section{TTCL: Test-time Calibration Learning}
\vspace{-1mm}
In this section, we first present the motivation of the need for TTCL (Sec.~\ref{Sec:motivation_ttcl}), and then introduce the technical details of our proposed TTCL method (Sec.~\ref{Sec:ttcl_1} and Sec.~\ref{Sec:ttcl_2}).
\vspace{-1mm}

\subsection{Why Test-time Calibration Learning (TTCL)?}
\label{Sec:motivation_ttcl}
Despite their strong reasoning capabilities, current LLMs often remain over-confident on downstream tasks. As illustrated in the left part of Fig.~\ref{Fig:Motivation}, Qwen3-8B exhibits clear over-confidence on both the mathematical benchmark AMC~\citep{yang2024qwen2p5math} and the factual benchmark NQ~\citep{kwiatkowski2019natural}, with estimated confidence systematically exceeding empirical accuracy and resulting in larger ECE values. Moreover, calibration learned on one task may not transfer well to others. As shown in the right part of Fig.~\ref{Fig:Motivation}, Qwen3-8B calibrated with RLCR on DAPO-Math-14K~\citep{yu2026dapo} achieves strong calibration on AMC, yet remains poorly calibrated on NQ. These observations motivate direct test-time calibration of LLMs and lead to two key principles of TTCL:

\begin{tcolorbox}[
    enhanced,
    colback=ttclboxbg,
    colframe=ttclboxframe,
    boxrule=0.5pt,
    arc=2mm,
    left=2mm,
    right=2mm,
    top=4mm,
    bottom=1mm,
    title={\small \textbf{Key Principles of TTCL}},
    attach boxed title to top left={
        xshift=4mm,
        yshift=-4mm
    },
    boxed title style={
        colback=ttclboxtitle,
        colframe=ttclboxframe,
        boxrule=0.5pt,
        arc=1mm,
        left=1.5mm,
        right=1.5mm,
        top=0.5mm,
        bottom=0.5mm
    },
    coltitle=black
]

\begin{itemize}[leftmargin=1.2em, nosep]
    \item \textbf{Label-free calibration learning:} derive supervision signals for calibration learning from unlabeled test-time data, removing the dependence on annotated ground-truth labels.
    \item \textbf{Target-task adaptation:} adapt confidence estimation directly to newly encountered tasks at test time, alleviating the limited cross-task generalization of calibration learning.
\end{itemize}

\end{tcolorbox}

Compared with existing calibration learning methods~\citep{damani2026RLCR,bani2026rewarding}, the first principle, \textit{label-free calibration learning}, removes the dependence on annotated correctness labels, which are often costly to obtain or unavailable at test time. The second principle, \textit{target-task adaptation}, enables the model to recalibrate its confidence directly on the deployment task, mitigating the limited cross-task transferability of calibration learned from source tasks. The right part of Fig.~\ref{Fig:Motivation} further illustrates that TTCL can adapt calibration directly to a new target task at test time. Starting from the RLCR-calibrated Qwen3-8B, applying TTCL on unlabeled NQ data reduces the ECE from 0.730 to 0.290, without requiring additional ground-truth annotations. We next introduce how TTCL constructs label-free supervision for both correctness and calibration from unlabeled data.

\begin{figure}[t]
  \centering
  \vspace{-2mm}
  \includegraphics[width=0.98\textwidth]{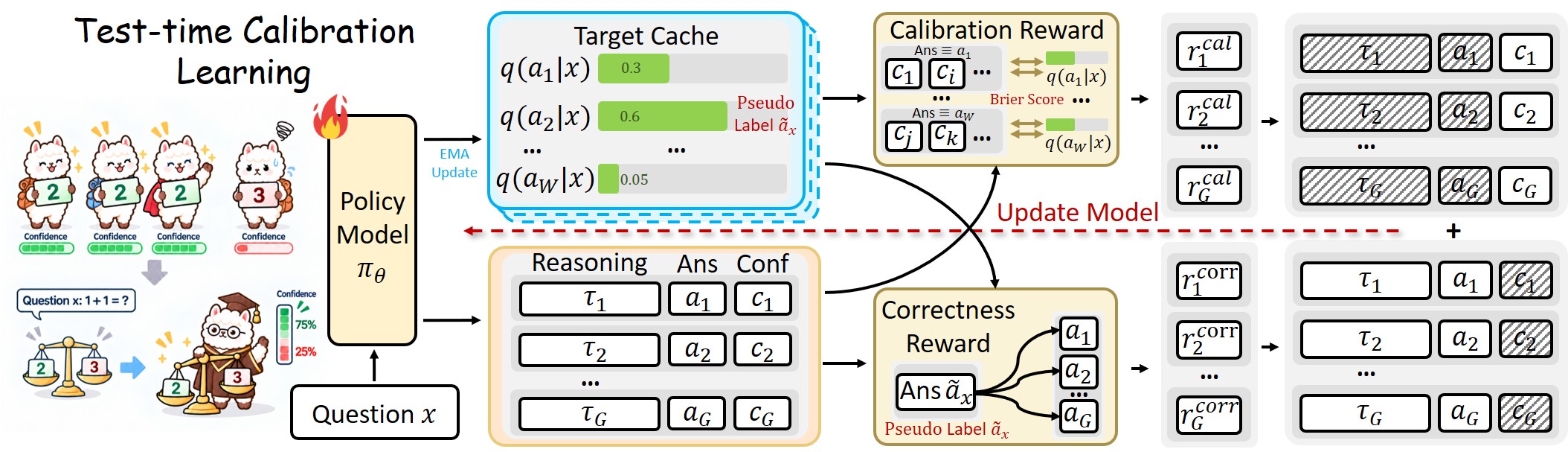}
  \vspace{-2mm}
  \caption{\textbf{Illustration of \textit{TTCL} method}: Given an unlabeled question, TTCL samples multiple rollouts to construct label-free supervision: the most supported answer serves as the pseudo label, while the support of each answer provides its confidence target. A target cache stabilizes these signals and is updated via EMA across epochs. Correctness and calibration rewards are optimized separately, with the former updating reasoning and answer tokens and the latter updating only confidence tokens.}
  \label{Fig:TTCL}
  \vspace{-4mm}
\end{figure} 

\subsection{TTCL Self-Supervision Signals and Optimization Objective}
\label{Sec:ttcl_1}
Given a test-time question $x$, TTCL samples a group of $G$ responses $\{y_i\}_{i=1}^G$. where each response is denoted as $y_i=(\tau_i, a_i, c_i)$, with $\tau_i$ denoting the reasoning trajectory, $a_i$ the final answer, and $c_i\in[0,1]$ the verbalized confidence. As illustrated in Fig.~\ref{Fig:TTCL}, our label-free supervision is built on a simple assumption: answers receiving stronger consensus across repeated sampled rollouts are more likely to be correct and should be associated with higher confidence. Based on this assumption, TTCL derives two self-supervision signals for correctness and calibration from the answer distribution.

\textbf{Correctness self-supervision signal.}
Let $q(a|x)$ denote the empirical answer distribution induced by the rollout group. TTCL selects the most supported answer as the pseudo correctness target:
\begin{equation}
    \tilde{a}_x = \arg\max_{a} q(a \mid x), \quad q(a \mid x)=\frac{1}{G}\sum_{j=1}^{G}\mathbb{I}[a_j \equiv a].
\end{equation}
The correctness reward for response $y_i$ is then defined as:
\begin{equation}
    r_i^{\mathrm{corr}}=\mathbb{I}[a_i \equiv \tilde{a}_x] \in \{0,1\}.
\end{equation}
This majority-based pseudo-labeling strategy follows prior self-supervised RL methods~\citep{zuo2026ttrl,zhang2026co,shafayat2025can}, which use the most frequent answer as a pseudo correctness target to encourage consensus and improve LLM reasoning capability.

\par
\textbf{Calibration self-supervision signal.}
Unlike the correctness signal that only uses the most supported answer, calibration requires finer-grained supervision to reflect the model's uncertainty over different candidate answers. TTCL therefore uses the empirical support of each generated answer as its confidence target. For response $y_i$, the calibration target is defined as:
\begin{equation}
    \tilde{p}_i = q(a_i \mid x).
\end{equation}
Accordingly, responses whose answers receive stronger consensus across the rollout group are assigned higher confidence targets, while less-supported answers receive lower confidence targets. Given the predicted confidence $c_i$, the calibration reward is defined using the negative Brier score:
\begin{equation}
    r_i^{\mathrm{cal}}=-\left(c_i-\tilde{p}_i\right)^2
\end{equation}
In this way, TTCL converts the empirical answer frequency into a soft probability target for confidence calibration, enabling calibration learning without access to ground-truth correctness labels.

\par
\textbf{TTCL optimization objective.} 
As the response $y_i$ consists of the tokens of reasoning trajectory, final answer, and confidence prediction, naively combining the correctness and calibration rewards over the entire response can introduce undesirable interference between correctness and calibration~\citep{ma2026decoupling}. For example, correct reasoning and answer tokens may still be penalized due to poorly calibrated confidence. TTCL computes separate group-relative advantages for the two rewards:
\begin{align}
    \hat{A}_i^{\mathrm{corr}} &= r_i^{\mathrm{corr}} - \frac{1}{G}\sum_{j=1}^{G} r_j^{\mathrm{corr}}, \quad \hat{A}_i^{\mathrm{cal}} = r_i^{\mathrm{cal}} - \frac{1}{G}\sum_{j=1}^{G} r_j^{\mathrm{cal}}.
\end{align}
The correctness advantage is applied only to the reasoning and answer tokens, while the calibration advantage is applied only to the confidence tokens. Accordingly, the TTCL objective is defined as:
\begin{equation}
\label{Eq:TTCL}
    \mathcal{J}_{\mathrm{TTCL}}(\theta) = \mathbb{E}\left[\frac{1}{G}\sum_{i=1}^{G} \left[ \frac{1}{|o_i^{\mathrm{corr}}|}\sum_{y_t\in o_i^{\mathrm{corr}}}\left(\rho_{i,t}(\theta)\hat{A}_i^{\mathrm{corr}} \right) + \frac{1}{|o_i^{\mathrm{cal}}|}\sum_{y_t\in o_i^{\mathrm{cal}}}\left(\rho_{i,t}(\theta)\hat{A}_i^{\mathrm{cal}} \right) \right] \right],
\end{equation}
where $o_i^{\mathrm{corr}} = (\tau_i,a_i)$ denotes the tokens of the reasoning trajectory and final answer, and $o_i^{\mathrm{cal}} = (c_i)$ denotes the confidence tokens. The $\rho_{i,t}(\theta)$ denotes the PPO-style clipped importance ratio.

\subsection{Target Cache for Stable Learning}
\label{Sec:ttcl_2}
The self-supervision signals above are constructed from the empirical answer distribution $q(a\mid x)$ of the current rollout group. However, directly using this instantaneous distribution can tightly couple the policy with its self-generated targets, leading to reward hacking. As majority-based supervision encourages answer consistency, their empirical frequencies rise, further increasing the associated confidence targets. Without stabilization, this feedback can drive the model toward a degenerate solution where rollouts collapse to the same incorrect answer with confidence approaching $1.0$.

\par
To mitigate this issue, TTCL introduces a \textit{target cache} that maintains a temporally lagged answer distribution for each question. Let $q_{\mathrm{cache}}^{(e)}(a \mid x)$ denote the cached answer distribution at epoch $e$. During epoch $e$, the correctness and calibration targets are constructed from the frozen cache rather than the current policy rollout distribution:
\begin{equation}
    \tilde{a}_x^{(e)}=\arg\max_{a} q_{\mathrm{cache}}^{(e)}(a\mid x), \quad \tilde{p}_i^{(e)}=q_{\mathrm{cache}}^{(e)}(a_i\mid x).
\end{equation}
The corresponding rewards are therefore:
\begin{equation}
    r_i^{\mathrm{corr}}=\mathbb{I}\!\left[a_i\equiv \tilde{a}_x^{(e)}\right], \quad
    r_i^{\mathrm{cal}}=-\left(c_i-\tilde{p}_i^{(e)}\right)^2.
\end{equation}
Before training, the cache is initialized from the answer distribution of the initial model. During training, the newly observed rollout distribution is not used for immediate supervision, but instead updates the cache for subsequent epochs. Let $q^{(e)}(a \mid x)$ denote the empirical answer distribution observed during epoch $e$, and the cache is updated using an exponential moving average (EMA):
\begin{equation}
    q_{\mathrm{cache}}^{(e+1)}(a\mid x) = (1-\eta)\,q_{\mathrm{cache}}^{(e)}(a\mid x) + \eta\,q^{(e)}(a\mid x).
\end{equation}
where $\eta$ controls the update rate. By disentangling supervision targets from the current policy, the target cache blocks the positive feedback between answer consistency, confidence targets, and policy updates, stabilizing TTCL training and mitigating collapse toward a single high-confidence answer.
\vspace{-2mm}

\vspace{-2mm}
\section{Theoretical Analysis of TTCL}
\label{Sec:Theory}
In this section, we theoretically connect TTCL to the calibration definition in Eq.~(\ref{Eq:Def_Calibration}).

\definecolor{lightpurple}{RGB}{245,240,252}
\definecolor{lightpurpleframe}{RGB}{210,195,230}
\begin{tcolorbox}[
    enhanced,
    colback=lightpurple,
    colframe=lightpurpleframe,
    boxrule=0.5pt,
    arc=2mm,
    left=2mm,
    right=2mm,
    top=1.5mm,
    bottom=1.5mm,
    coltitle=black
]

\textbf{Analysis:} \textit{TTCL provides a label-free surrogate for calibration learning, with its deviation from the ideal calibration objective bounded by the frequency-based estimation error.}

\end{tcolorbox}

Let $p^\star(x,y)=\Pr(Z=1\mid X=x,Y=y)$ denote the conditional probability that response $y$ is correct for question $x$. For TTCL, we denote its frequency-based confidence target and the corresponding estimation error as $s(x,y)=q\!\left(\operatorname{Ans}(y)\mid x\right)$ and $\epsilon=\mathbb{E}\!\left[\left|s(X,Y)-p^\star(X,Y)\right|\right]$ respectively.

\begin{theorem}[\textbf{TTCL as a Surrogate for the Ideal Calibration Objective}]
\label{Theory:TTCL_boundary}
\leavevmode\\
Let $\mathcal{J}_{\mathrm{TTCL}}(c)=\mathbb{E}\!\left[(c(X,Y)-s(X,Y))^2\right]$ denote the calibration objective optimized by TTCL, and let $\mathcal{J}_{\mathrm{ideal}}(c)=\mathbb{E}\!\left[(c(X,Y)-p^\star(X,Y))^2\right]$ denote the ideal calibration objective. Then, for any confidence estimator $c(X,Y)\in[0,1]$, $\left|\mathcal{J}_{\mathrm{TTCL}}(c) - \mathcal{J}_{\mathrm{ideal}}(c) \right| \le 2\epsilon.$
\end{theorem}
Theorem~\ref{Theory:TTCL_boundary} demonstrates that TTCL optimizes a bounded surrogate of the ideal calibration objective. The objective gap is controlled by the estimation error between the frequency-based target and the latent conditional correctness probability. A detailed proof is provided in Appendix~\ref{App:Proof_Theorem}.

\begin{table}[t]
    \centering
    \caption{\textbf{Performance of TTCL for base models on mathematical reasoning benchmarks.} Improved results after TTCL are highlighted in \textbf{bold}, with positive relative improvements shown in \textcolor{softred}{red} color. AUROC is not reported for Llama-3.1-8B-Instruct on AIME25 because its accuracy is 0.}
    \label{Tble:TTCL_Base_Math}
    \vspace{-3mm}
    \resizebox{\textwidth}{!}{
    \begin{tabular}{l|ccc|ccc|ccc|ccc}
    \toprule[1.6pt]
        \multirow{2}{*}{\textbf{Methods}} & \multicolumn{3}{c}{\textbf{MATH500}} & \multicolumn{3}{c}{\textbf{AMC}} & \multicolumn{3}{c}{\textbf{AIME24}} & \multicolumn{3}{c}{\textbf{AIME25}} \\ \cmidrule(lr){2-4} \cmidrule(lr){5-7} \cmidrule(lr){8-10} \cmidrule(lr){11-13} 
        ~ & \textbf{ACC $\uparrow$} & \textbf{ECE $\downarrow$} & \textbf{AUROC $\uparrow$} & \textbf{ACC $\uparrow$} & \textbf{ECE $\downarrow$} & \textbf{AUROC $\uparrow$} & \textbf{ACC $\uparrow$} & \textbf{ECE $\downarrow$} & \textbf{AUROC $\uparrow$} & \textbf{ACC $\uparrow$} & \textbf{ECE $\downarrow$} & \textbf{AUROC $\uparrow$} \\ 
        \midrule
            \multicolumn{13}{c}{\textit{\textbf{Qwen3-4B}}} \\
        \midrule
        Base Model & 80.6 & 0.155 & 0.670 & 59.38 & 0.383 & 0.703 & 20.42 & 0.739 & 0.652 & 16.25 & 0.807 & 0.692 \\
        \rowcolor{lightblue} w/ \textbf{TTCL} & \textbf{86.6} & \textbf{0.029} & \textbf{0.866} & \textbf{77.81} & \textbf{0.091} & \textbf{0.910} & \textbf{28.33} & \textbf{0.065} & \textbf{0.856} & \textbf{20.42} & \textbf{0.261} & \textbf{0.878} \\
        \# Improv & \color{softred}+7.44\% & \color{softred}+81.52\% & \color{softred}+29.28\% & \color{softred}+31.05\% & \color{softred}+76.25\% & \color{softred}+29.52\% & \color{softred}+38.78\% & \color{softred}+91.26\% & \color{softred}+31.15\% & \color{softred}+25.64\% & \color{softred}+67.68\% & \color{softred}+27.00\% \\
        \specialrule{1pt}{0.4ex}{0.4ex}
            \multicolumn{13}{c}{\textit{\textbf{Qwen3-8B}}} \\
        \midrule
        Base Model & 84.4 & 0.183 & 0.717 & 62.81 & 0.338 & 0.651 & 21.67 & 0.753 & 0.642 & 18.75 & 0.789 & 0.620 \\
        \rowcolor{lightblue} w/ \textbf{TTCL} & \textbf{89.0} & \textbf{0.131} & \textbf{0.830} & \textbf{80.63} & \textbf{0.168} & \textbf{0.844} & \textbf{30.83} & \textbf{0.120} & \textbf{0.882} & \textbf{26.67} & \textbf{0.145} & \textbf{0.908} \\
        \# Improv & \color{softred}+5.45\% & \color{softred}+28.17\% & \color{softred}+15.72\% & \color{softred}+28.36\% & \color{softred}+50.31\% & \color{softred}+29.60\% & \color{softred}+42.31\% & \color{softred}+84.09\% & \color{softred}+37.37\% & \color{softred}+42.22\% & \color{softred}+81.58\% & \color{softred}+46.37\% \\
        \specialrule{1pt}{0.4ex}{0.4ex}
            \multicolumn{13}{c}{\textit{\textbf{Llama-3.1-8B-Instruct}}} \\
        \midrule
        Base Model & 39.6 & 0.436 & 0.479 & 21.88 & 0.611 & 0.518 & 4.58 & 0.807 & 0.365 & 0.00 & 0.853 & - \\
        \rowcolor{lightblue} w/ \textbf{TTCL} & \textbf{58.0} & \textbf{0.094} & \textbf{0.770} & \textbf{40.94} & \textbf{0.128} & \textbf{0.617} & \textbf{9.17} & \textbf{0.223} & \textbf{0.878} & 0.00 & \textbf{0.114} & - \\
        \# Improv & \color{softred}+46.46\% & \color{softred}+78.50\% & \color{softred}+60.99\% & \color{softred}+87.14\% & \color{softred}+79.12\% & \color{softred}+19.03\% & \color{softred}+100.00\% & \color{softred}+72.37\% & \color{softred}+140.72\% & 0.00\% & \color{softred}+86.59\% & - \\
        \bottomrule[1.6pt]
    \end{tabular}}
    \vspace{-4mm}
\end{table}

\begin{table}[t!]
    \centering
    \caption{\textbf{Performance of TTCL for base models on factual QA benchmarks.} Improved results after TTCL are highlighted in \textbf{bold}, with positive relative improvements shown in \textcolor{softred}{red}.}
    \label{Tble:TTCL_Base_Fact}
    \vspace{-3mm}
    \resizebox{\textwidth}{!}{
    \begin{tabular}{l|ccc|ccc|ccc|ccc}
    \toprule[1.6pt]
        \multirow{2}{*}{\textbf{Methods}} & \multicolumn{3}{c}{\textbf{SimpleQA}} & \multicolumn{3}{c}{\textbf{NQ}} & \multicolumn{3}{c}{\textbf{HotpotQA}} & \multicolumn{3}{c}{\textbf{TriviaQA}} \\ \cmidrule(lr){2-4} \cmidrule(lr){5-7} \cmidrule(lr){8-10} \cmidrule(lr){11-13}
        ~ & \textbf{ACC $\uparrow$} & \textbf{ECE $\downarrow$} & \textbf{AUROC $\uparrow$} & \textbf{ACC $\uparrow$} & \textbf{ECE $\downarrow$} & \textbf{AUROC $\uparrow$} & \textbf{ACC $\uparrow$} & \textbf{ECE $\downarrow$} & \textbf{AUROC $\uparrow$} & \textbf{ACC $\uparrow$} & \textbf{ECE $\downarrow$} & \textbf{AUROC $\uparrow$} \\
        \midrule
        \multicolumn{13}{c}{\textit{\textbf{Qwen3-4B}}} \\
        \midrule
        Base Model & 1.40 & 0.912 & 0.481 & 8.63 & 0.859 & 0.705 & 9.56 & 0.831 & 0.703 & 15.63 & 0.789 & 0.660  \\
        \rowcolor{lightblue} w/ \textbf{TTCL} & \textbf{2.40} & \textbf{0.081} & \textbf{0.731} & \textbf{10.63} & \textbf{0.155} & \textbf{0.848} & \textbf{10.06} & \textbf{0.367} & \textbf{0.816} & \textbf{17.00} & \textbf{0.257} & \textbf{0.766} \\
        \# Improv & \color{softred}+71.43\% & \color{softred}+91.16\% & \color{softred}+51.96\% & \color{softred}+23.25\% & \color{softred}+81.95\% & \color{softred}+20.32\% & \color{softred}+5.20\% & \color{softred}+55.82\% & \color{softred}+16.13\% & \color{softred}+8.80\% & \color{softred}+67.41\% & \color{softred}+16.08\% \\
        \specialrule{1pt}{0.4ex}{0.4ex}
        \multicolumn{13}{c}{\textit{\textbf{Qwen3-8B}}} \\
        \midrule
        Base Model & 2.20 & 0.906 & 0.564 & 9.63 & 0.873 & 0.624 & 12.75 & 0.822 & 0.596 & 24.13 & 0.735 & 0.563 \\
        \rowcolor{lightblue} w/ \textbf{TTCL} & \textbf{3.50} & \textbf{0.352} & \textbf{0.623} & \textbf{14.81} & \textbf{0.190} & \textbf{0.752} & \textbf{15.31} & \textbf{0.330} & \textbf{0.798} & \textbf{26.50} & \textbf{0.226} & \textbf{0.799} \\
        \# Improv & \color{softred}+59.09\% & \color{softred}+61.17\% & \color{softred}+10.45\% & \color{softred}+53.87\% & \color{softred}+78.23\% & \color{softred}+20.56\% & \color{softred}+20.08\% & \color{softred}+59.85\% & \color{softred}+33.78\% & \color{softred}+9.84\% & \color{softred}+69.23\% & \color{softred}+42.02\% \\
        \specialrule{1pt}{0.4ex}{0.4ex}
        \multicolumn{13}{c}{\textit{\textbf{Llama-3.1-8B-Instruct}}} \\
        \midrule
        Base Model & 3.20 & 0.812 & 0.535 & 14.69 & 0.670 & 0.516 & 12.31 & 0.714 & 0.514 & 36.88 & 0.451 & 0.522 \\
        \rowcolor{lightblue} w/ \textbf{TTCL} & \textbf{5.80} & \textbf{0.219} & \textbf{0.716} & \textbf{23.31} & \textbf{0.174} & \textbf{0.704} & \textbf{20.13} & \textbf{0.107} & \textbf{0.666} & \textbf{50.19} & \textbf{0.335} & \textbf{0.639} \\
        \# Improv & \color{softred}+81.25\% & \color{softred}+72.98\% & \color{softred}+33.82\% & \color{softred}+40.23\% & \color{softred}+73.83\% & \color{softred}+35.36\% & \color{softred}+64.29\% & \color{softred}+85.10\% & \color{softred}+23.79\% & \color{softred}+30.78\% & \color{softred}+25.00\% & \color{softred}+24.08\% \\
        \bottomrule[1.6pt]
    \end{tabular}}
    \vspace{-5mm}
\end{table}

\par
We next present two supporting results: the empirical probability $q\!\left(\operatorname{Ans}(y)\mid x\right)$ of the answer frequency can serve as a proxy for the latent conditional correctness probability $p^\star(x,y)$, while $p^\star(x,y)$ constitutes the ideal target for calibration learning.

\begin{proposition}[\textbf{Answer Frequency as a Proxy for Conditional Correctness}]
\label{Theory:frequency_proxy}
\leavevmode\\
Let $q_{\pi_\theta}(a\mid x)$ denote the population probability of answer $a$ under the model. The empirical answer probability estimated from $G$ independent responses is an unbiased and consistent estimator of $q_{\pi_\theta}(a\mid x)$. When answer distribution reflects model's uncertainty over the correct answer, $q_{\pi_\theta}(a\mid x)$ serves as a proxy for the probability that response $y$ is correct, i.e., $p^\star(x,y)$, where $a=\operatorname{Ans}(y)$.
\end{proposition}
Proposition~\ref{Theory:frequency_proxy} provides the statistical justification for using answer frequency as the label-free confidence target in TTCL. A detailed proof is provided in Appendix~\ref{App:Proof_Prop1}.

We next show that the conditional correctness probability is the ideal target for calibration learning.
\begin{proposition}[\textbf{Conditional Correctness as the Ideal Calibration Target}]
\label{Theory:ideal_calibration_target}
\leavevmode\\
If the confidence estimator satisfies $c(X,Y)=p^\star(X,Y)$, then the resulting confidence is perfectly calibrated, i.e., $\Pr(Z=1\mid C=c)=c$. Minimizing the expected Brier score yields the optimal confidence estimator $c^\star = \arg\min_c \mathcal{J}_{\mathrm{Brier}}(c) = p^\star$.
\end{proposition}
This proposition connects Brier score minimization to perfect calibration through the conditional correctness probability. A detailed proof is provided in Appendix~\ref{App:Proof_Prop2}. Together with Proposition~\ref{Theory:frequency_proxy}, this forms a theoretical bridge from answer frequency to conditional correctness probability and ultimately to calibrated confidence, supporting the surrogate guarantee in Theorem~\ref{Theory:TTCL_boundary}.

\begin{table}[t]
    \centering
    \caption{\textbf{In-domain performance of TTCL on RLCR-calibrated models.} Improved results are highlighted in \textbf{bold}, with positive and negative relative gains shown in \textcolor{softred}{red} and \textcolor{darkgreen}{green} respectively.}
    \label{Tble:TTCL_InDomain}
    \vspace{-3mm}
    \resizebox{\textwidth}{!}{
    \begin{tabular}{l|ccc|ccc|ccc|ccc|ccc}
    \toprule[1.6pt]
        \multirow{2}{*}{\textbf{Math-to-Math}} & \multicolumn{3}{c}{\textbf{MATH500}} & \multicolumn{3}{c}{\textbf{AMC}} & \multicolumn{3}{c}{\textbf{AIME24}} & \multicolumn{3}{c}{\textbf{AIME25}} & \multicolumn{3}{c}{\textbf{AIME26}} \\ 
        \cmidrule(lr){2-4} \cmidrule(lr){5-7} \cmidrule(lr){8-10} \cmidrule(lr){11-13} \cmidrule(lr){14-16}
        ~ & \textbf{ACC $\uparrow$} & \textbf{ECE $\downarrow$} & \textbf{AUROC $\uparrow$} & \textbf{ACC $\uparrow$} & \textbf{ECE $\downarrow$} & \textbf{AUROC $\uparrow$} & \textbf{ACC $\uparrow$} & \textbf{ECE $\downarrow$} & \textbf{AUROC $\uparrow$} & \textbf{ACC $\uparrow$} & \textbf{ECE $\downarrow$} & \textbf{AUROC $\uparrow$} & \textbf{ACC $\uparrow$} & \textbf{ECE $\downarrow$} & \textbf{AUROC $\uparrow$} \\ 
        \midrule
        \multicolumn{16}{c}{\textit{\textbf{RLCR-calibrated Qwen3-4B on DAPO-Math-14K}}} \\
        \midrule
        Base Model & 80.6 & 0.155 & 0.670 & 59.38 & 0.383 & 0.703 & 20.42 & 0.739 & 0.652 & 16.25 & 0.807 & 0.692 & 15.42 & 0.813 & 0.606  \\
        \rowcolor{lightgray} RLCR Model & 89.0 & 0.158 & 0.867 & 79.38 & 0.131 & 0.865 & 29.17 & 0.291 & 0.881 & 24.17 & 0.311 & 0.888 & 23.33 & 0.348 & 0.823 \\
        \rowcolor{lightblue} w/ \textbf{TTCL} & \textbf{90.4} & \textbf{0.026} & \textbf{0.924} & \textbf{88.13} & \textbf{0.101} & \textbf{0.902} & \textbf{42.08} & \textbf{0.096} & \textbf{0.933} & \textbf{25.00} & \textbf{0.122} & \textbf{0.910} & \textbf{24.58} & \textbf{0.189} & \textbf{0.839} \\
        \# Improv & \color{softred}+1.57\% & \color{softred}+83.37\% & \color{softred}+6.67\% & \color{softred}+11.02\% & \color{softred}+22.76\% & \color{softred}+4.31\% & \color{softred}+44.29\% & \color{softred}+67.18\% & \color{softred}+5.93\% & \color{softred}+3.45\% & \color{softred}+60.57\% & \color{softred}+2.52\% & \color{softred}+5.36\% & \color{softred}+45.52\% & \color{softred}+2.02\% \\
        \specialrule{1pt}{0.4ex}{0.4ex}
        \multicolumn{16}{c}{\textit{\textbf{RLCR-calibrated Qwen3-8B on DAPO-Math-14K}}} \\
        \midrule
        Base Model & 84.4 & 0.183 & 0.717 & 62.81 & 0.338 & 0.651 & 21.67 & 0.753 & 0.642 & 18.75 & 0.789 & 0.620 & 13.75 & 0.839 & 0.597 \\
        \rowcolor{lightgray} RLCR Model  & 90.0 & 0.054 & 0.899 & 80.63 & 0.080 & 0.866 & 35.42 & 0.100 & 0.936 & 28.33 & 0.141 & 0.975 & 27.50 & 0.157 & 0.928 \\
        \rowcolor{lightblue} w/ \textbf{TTCL} & \textbf{91.6} & \textbf{0.042} & \textbf{0.967} & \textbf{89.38} & \textbf{0.065} & \textbf{0.881} & \textbf{45.42} & \textbf{0.074} & 0.907 & \textbf{30.42} & \textbf{0.127} & 0.959 & \textbf{28.33} & \textbf{0.122} & 0.902 \\
        \# Improv & \color{softred}+1.78\% & \color{softred}+23.14\% & \color{softred}+7.58\% & \color{softred}+10.85\% & \color{softred}+18.75\% & \color{softred}+1.68\% & \color{softred}+28.24\% & \color{softred}+26.06\% & \color{darkgreen}-3.14\% & \color{softred}+7.35\% & \color{softred}+10.35\% & \color{darkgreen}-1.62\% & \color{softred}+3.03\% & \color{softred}+22.37\% & \color{darkgreen}-2.76\% \\
        \bottomrule[1.6pt]
    \end{tabular}}
    \vspace{-4mm}
\end{table}

\begin{table}[t!]
    \centering
    \caption{\textbf{Out-of-domain adaptation of TTCL on RLCR-calibrated models.} Improved results after TTCL are highlighted in \textbf{bold}, with positive relative improvements shown in \textcolor{softred}{red}.}
    \label{Tble:TTCL_OutDomain}
    \vspace{-3mm}
    \resizebox{\textwidth}{!}{
    \begin{tabular}{l|ccc|ccc|ccc|ccc|ccc}
    \toprule[1.6pt]
        \multirow{2}{*}{\textbf{Math-to-FactQA}} & \multicolumn{3}{c}{\textbf{SimpleQA}} & \multicolumn{3}{c}{\textbf{ChineseSimpleQA}} & \multicolumn{3}{c}{\textbf{NQ}} & \multicolumn{3}{c}{\textbf{HotpotQA}} & \multicolumn{3}{c}{\textbf{TriviaQA}} \\ \cmidrule(lr){2-4} \cmidrule(lr){5-7} \cmidrule(lr){8-10} \cmidrule(lr){11-13} \cmidrule(lr){14-16}
        ~ & \textbf{ACC $\uparrow$} & \textbf{ECE $\downarrow$} & \textbf{AUROC $\uparrow$} & \textbf{ACC $\uparrow$} & \textbf{ECE $\downarrow$} & \textbf{AUROC $\uparrow$} & \textbf{ACC $\uparrow$} & \textbf{ECE $\downarrow$} & \textbf{AUROC $\uparrow$} & \textbf{ACC $\uparrow$} & \textbf{ECE $\downarrow$} & \textbf{AUROC $\uparrow$} & \textbf{ACC $\uparrow$} & \textbf{ECE $\downarrow$} & \textbf{AUROC $\uparrow$} \\ 
        \midrule
            \multicolumn{16}{c}{\textit{\textbf{RLCR-calibrated Qwen3-4B on DAPO-Math-14K}}} \\
        \midrule
        Base Model & 1.40 & 0.912 & 0.481 & 12.13 & 0.819 & 0.658 & 8.63 & 0.859 & 0.705 & 9.56 & 0.831 & 0.703 & 15.63 & 0.789 & 0.660 \\
        \rowcolor{lightgray} RLCR Model & 2.00 & 0.623 & 0.662 & 13.00 & 0.600 & 0.730 & 11.00 & 0.651 & 0.732 & 11.13 & 0.560 & 0.753 & 19.38 & 0.551 & 0.713 \\
        \rowcolor{lightblue} w/ \textbf{TTCL} & \textbf{2.70} & \textbf{0.277} & \textbf{0.747} & \textbf{14.03} & \textbf{0.307} & \textbf{0.798} & \textbf{12.81} & \textbf{0.228} & \textbf{0.867} & \textbf{13.50} & \textbf{0.293} & \textbf{0.850} & \textbf{20.38} & \textbf{0.290} & \textbf{0.812} \\
        \# Improv & \color{softred}+35.00\% & \color{softred}+55.61\% & \color{softred}+12.84\% & \color{softred}+7.95\% & \color{softred}+48.83\% & \color{softred}+9.18\% & \color{softred}+16.48\% & \color{softred}+65.01\% & \color{softred}+18.47\% & \color{softred}+21.35\% & \color{softred}+47.62\% & \color{softred}+12.87\% & \color{softred}+5.16\% & \color{softred}+47.34\% & \color{softred}+13.91\% \\
        \specialrule{1pt}{0.4ex}{0.4ex}
            \multicolumn{16}{c}{\textit{\textbf{RLCR-calibrated Qwen3-8B on DAPO-Math-14K}}} \\
        \midrule
        Base Model & 2.20 & 0.906 & 0.564 & 15.77 & 0.796 & 0.579 & 9.63 & 0.873 & 0.624 & 12.75 & 0.822 & 0.596 & 24.13 & 0.735 & 0.563  \\
        \rowcolor{lightgray} RLCR Model & 2.80 & 0.768 & 0.591 & 18.17 & 0.640 & 0.676 & 12.56 & 0.730 & 0.658 & 14.13 & 0.633 & 0.696 & 25.69 & 0.599 & 0.691 \\
        \rowcolor{lightblue} w/ \textbf{TTCL} & \textbf{3.90} & \textbf{0.270} & \textbf{0.712} & \textbf{20.83} & \textbf{0.225} & \textbf{0.771} & \textbf{15.00} & \textbf{0.290} & \textbf{0.863} & \textbf{18.38} & \textbf{0.292} & \textbf{0.772} & \textbf{29.31} & \textbf{0.426} & \textbf{0.823} \\
        \# Improv & \color{softred}+39.29\% & \color{softred}+64.85\% & \color{softred}+20.46\% & \color{softred}+14.68\% & \color{softred}+64.92\% & \color{softred}+14.07\% & \color{softred}+19.40\% & \color{softred}+60.30\% & \color{softred}+31.12\% & \color{softred}+30.09\% & \color{softred}+53.85\% & \color{softred}+10.90\% & \color{softred}+14.11\% & \color{softred}+28.98\% & \color{softred}+19.15\% \\
        \toprule[1.6pt]
        \multirow{2}{*}{\textbf{FactQA-to-Math}} & \multicolumn{3}{c}{\textbf{MATH500}} & \multicolumn{3}{c}{\textbf{AMC}} & \multicolumn{3}{c}{\textbf{AIME24}} & \multicolumn{3}{c}{\textbf{AIME25}} & \multicolumn{3}{c}{\textbf{AIME26}} \\ \cmidrule(lr){2-4} \cmidrule(lr){5-7} \cmidrule(lr){8-10} \cmidrule(lr){11-13} \cmidrule(lr){14-16}
        ~ & \textbf{ACC $\uparrow$} & \textbf{ECE $\downarrow$} & \textbf{AUROC $\uparrow$} & \textbf{ACC $\uparrow$} & \textbf{ECE $\downarrow$} & \textbf{AUROC $\uparrow$} & \textbf{ACC $\uparrow$} & \textbf{ECE $\downarrow$} & \textbf{AUROC $\uparrow$} & \textbf{ACC $\uparrow$} & \textbf{ECE $\downarrow$} & \textbf{AUROC $\uparrow$} & \textbf{ACC $\uparrow$} & \textbf{ECE $\downarrow$} & \textbf{AUROC $\uparrow$} \\ 
        \midrule
            \multicolumn{16}{c}{\textit{\textbf{RLCR-calibrated Qwen3-4B on FactQA-10K}}} \\
        \midrule
        Base Model & 80.6 & 0.155 & 0.670 & 59.38 & 0.383 & 0.703 & 20.42 & 0.739 & 0.652 & 16.25 & 0.807 & 0.692 & 15.42 & 0.813 & 0.606 \\
        \rowcolor{lightgray} RLCR Model & 79.8 & 0.114 & 0.818 & 64.38 & 0.103 & 0.779 & 23.75 & 0.286 & 0.814 & 18.75 & 0.322 & 0.822 & 14.58 & 0.334 & 0.772 \\
        \rowcolor{lightblue} w/ \textbf{TTCL} & \textbf{85.2} & \textbf{0.059} & \textbf{0.920} & \textbf{80.00} & \textbf{0.055} & \textbf{0.859} & \textbf{25.00} & \textbf{0.087} & \textbf{0.847} & \textbf{24.17} & \textbf{0.123} & \textbf{0.855} & \textbf{16.67} & \textbf{0.202} & 0.743 \\
        \# Improv & \color{softred}+6.77\% & \color{softred}+48.33\% & \color{softred}+12.48\% & \color{softred}+24.27\% & \color{softred}+46.37\% & \color{softred}+10.22\% & \color{softred}+5.26\% & \color{softred}+69.46\% & \color{softred}+3.96\% & \color{softred}+28.89\% & \color{softred}+61.68\% & \color{softred}+3.97\% & \color{softred}+14.29\% & \color{softred}+39.44\% & \color{darkgreen}-3.69\% \\
        \bottomrule[1.6pt]
    \end{tabular}}
    \vspace{-4mm}
\end{table}

\section{Experiments}
\label{sec:experiments}
\vspace{-1mm}
\subsection{Experimental Setup}
\textbf{Models.} 
We conduct experiments on LLMs from different families and scales, including Qwen3-4B/8B~\citep{yang2025qwen3} and Llama-3.1-8B-Instruct~\citep{grattafiori2024llama}, where Qwen3 is used in non-thinking mode. Besides applying TTCL directly to base models, we further evaluate it on models pre-calibrated with RLCR~\citep{damani2026RLCR}. Specifically, Qwen3-4B/8B are trained with RLCR for 3 epochs on DAPO-Math-14K and FactQA-10K, and the resulting models serve as calibrated starting points for evaluating TTCL adaptation to new target tasks.

\textbf{Benchmarks.}
We evaluate TTCL on mathematical reasoning and factual QA benchmarks. The mathematics includes MATH500~\citep{lightman2024math500}, AMC~\citep{yang2024qwen2p5math}, AIME24~\citep{aime24}, AIME25~\citep{aime25} and AIME26~\citep{aime26}, while the factual QA includes SimpleQA~\citep{wei2024simpleqa}, ChineseSimpleQA~\citep{he2025chinese}, NQ~\citep{kwiatkowski2019natural}, HotpotQA~\citep{yang2018hotpotqa} and TriviaQA~\citep{joshi2017triviaqa}. We apply TTCL to each benchmark independently without access its ground-truth labels, and evaluate the resulting model on the same benchmark. More details are provided in Appendix~\ref{App:Benchmarks}.

\textbf{Implementation details.}
We implement TTCL based on the VERL framework~\citep{sheng2025Verl}. For each question, we sample $G=32$ rollouts with a temperature of $0.7$. We train TTCL with the learning rate of $5\times10^{-7}$. The target cache is updated after each epoch with an EMA rate of $\eta=0.01$. Following prior studies~\citep{damani2026RLCR}, we evaluate answer correctness using Accuracy (ACC), calibration using Expected Calibration Error (ECE), and confidence discrimination using AUROC. More implementation and metric details refer to Appendix~\ref{App:Implement_detail} and Appendix~\ref{App:Metrics}.

\textbf{Baselines.}
Besides the corresponding starting models, we compare TTCL with several label-free calibration methods, including P(True)~\citep{kadavath2022PTrue}, Verbalization and Elicitation~\citep{tian2023calibration}, and BaseCal~\citep{tan2026basecal}. These methods mainly rely on post-hoc confidence estimation strategies to obtain more accurate confidence scores for given responses, without training the model. More details of the baselines are introduced in Appendix~\ref{App:Baselines}.

\vspace{-1mm}
\subsection{Main Experimental Results}
\vspace{-1mm}
\label{Sec:Main_Results}
\textbf{TTCL consistently improves both accuracy and calibration of base models.} 
As shown in Table~\ref{Tble:TTCL_Base_Math} and Table~\ref{Tble:TTCL_Base_Fact}, we observe that TTCL consistently improves performance across mathematical reasoning and factual QA benchmarks. On mathematical reasoning, TTCL achieves average relative improvements of $+41.35\%$ in ACC and $+73.12\%$ in ECE; on factual QA, the corresponding gains are $+31.02\%$ and $+66.37\%$. These results suggest that TTCL substantially improves calibration while enhancing answer accuracy across diverse models and tasks. The AUROC gains further indicate improved confidence discrimination between correct and incorrect responses.

\begin{table}[t]
    \centering
    \caption{\textbf{Performance comparison with existing label-free calibration baselines.} The best results among label-free methods are highlighted in \textbf{bold}. RLCR is included as a label-leakage reference. AUROC is not reported for Llama-3.1-8B-Instruct on AIME25 because its accuracy is 0.}
    \label{Tble:TTCL_Baseline}
    \vspace{-3mm}
    \resizebox{\textwidth}{!}{
    \begin{tabular}{l|c|ccc|ccc|ccc|ccc|ccc}
    \toprule[1.6pt]
        \multirow{2}{*}{\textbf{Methods}} & \multirow{2}{*}{\makecell{\textbf{Label} \\ \textbf{Leakage}}} & \multicolumn{3}{c}{\textbf{MATH500}} & \multicolumn{3}{c}{\textbf{AMC}} & \multicolumn{3}{c}{\textbf{AIME24}} & \multicolumn{3}{c}{\textbf{AIME25}} & \multicolumn{3}{c}{\textbf{Average}} \\ 
        \cmidrule(lr){3-5} \cmidrule(lr){6-8} \cmidrule(lr){9-11} \cmidrule(lr){12-14} \cmidrule(lr){15-17}
        ~ & ~ & \textbf{ACC $\uparrow$} & \textbf{ECE $\downarrow$} & \textbf{AUROC $\uparrow$} & \textbf{ACC $\uparrow$} & \textbf{ECE $\downarrow$} & \textbf{AUROC $\uparrow$} & \textbf{ACC $\uparrow$} & \textbf{ECE $\downarrow$} & \textbf{AUROC $\uparrow$} & \textbf{ACC $\uparrow$} & \textbf{ECE $\downarrow$} & \textbf{AUROC $\uparrow$} & \textbf{ACC $\uparrow$} & \textbf{ECE $\downarrow$} & \textbf{AUROC $\uparrow$} \\ 
        \midrule
            \multicolumn{17}{c}{\textit{\textbf{Qwen3-4B}}} \\
        \midrule
        RLCR & \color{softred}Yes & 88.8 & 0.054 & 0.713 & 89.38 & 0.054 & 0.803 & 47.92 & 0.004 & 0.996 & 39.58 & 0.000 & 1.000 & 66.42 & 0.028 & 0.878\\
        \midrule
        P(True) & No & \multirow{4}{*}{\repeatvalue{80.6}} & 0.155 & 0.786 & \multirow{4}{*}{\repeatvalue{59.38}} & 0.315 & 0.697 & \multirow{4}{*}{\repeatvalue{20.42}} & 0.214 & 0.865 & \multirow{4}{*}{\repeatvalue{16.25}} & 0.306 & 0.701 & \multirow{4}{*}{\repeatvalue{44.16}} & 0.248 & 0.762\\
        Verbalization & No &  & 0.266 & 0.602 &  & 0.326 & 0.706 &  & 0.485 & 0.750 &  & 0.636 & 0.543 & & 0.428 & 0.650\\
        Elicitation & No &  & 0.096 & 0.754 &  & 0.281 & 0.718 &  & 0.486 & 0.800 &  & 0.592 & 0.779 & & 0.364 & 0.763\\
        BaseCal & No &  & 0.062 & 0.598 &  & 0.270 & 0.618 &  & 0.636 & 0.675 &  & 0.667 & 0.717 & & 0.409 & 0.652\\
        \rowcolor{lightblue} \textbf{TTCL} & No & \textbf{86.6} & \textbf{0.029} & \textbf{0.866} & \textbf{77.81} & \textbf{0.091} & \textbf{0.910} & \textbf{28.33} & \textbf{0.065} & \textbf{0.856} & \textbf{20.42} & \textbf{0.261} & \textbf{0.878} & \textbf{53.29} & \textbf{0.112} & \textbf{0.878} \\

        \specialrule{1pt}{0.4ex}{0.4ex}
            \multicolumn{17}{c}{\textit{\textbf{Qwen3-8B}}} \\
        \midrule
        RLCR & \color{softred}Yes & 91.2 & 0.070 & 0.727 & 95.94 & 0.028 & 0.807 & 61.67 & 0.074 & 0.959 & 47.08 & 0.095 & 0.994 & 73.97 & 0.067 & 0.872\\
        \midrule
        P(True) & No & \multirow{4}{*}{\repeatvalue{84.4}} & 0.170 & 0.823 & \multirow{4}{*}{\repeatvalue{62.81}} & 0.275 & 0.819 & \multirow{4}{*}{\repeatvalue{21.67}} & 0.534 & 0.911 & \multirow{4}{*}{\repeatvalue{18.75}} & 0.526 & 0.701 & \multirow{4}{*}{\repeatvalue{46.91}} & 0.376 & 0.814\\
        Verbalization & No &  & 0.197 & 0.672 &  & 0.280 & 0.669 &  & 0.706 & 0.632 &  & 0.683 & 0.706 & & 0.467 & 0.670\\
        Elicitation & No &  & 0.165 & 0.728 &  & 0.286 & 0.693 &  & 0.599 & 0.758 &  & 0.619 & 0.722 & & 0.417 & 0.725\\
        BaseCal & No &  & \textbf{0.096} & 0.599 &  & 0.236 & 0.593 &  & 0.633 & 0.704 &  & 0.643 & 0.716 & & 0.402 & 0.653\\
        \rowcolor{lightblue} \textbf{TTCL} & No & \textbf{89.0} & 0.131 & \textbf{0.830} & \textbf{80.63} & \textbf{0.168} & \textbf{0.844} & \textbf{30.83} & \textbf{0.120} & \textbf{0.882} & \textbf{26.67} & \textbf{0.145} & \textbf{0.908} & \textbf{56.78} & \textbf{0.141} & \textbf{0.866} \\

        \specialrule{1pt}{0.4ex}{0.4ex}
            \multicolumn{17}{c}{\textit{\textbf{Llama-3.1-8B-Instruct}}} \\
        \midrule
        RLCR & \color{softred}Yes & 71.0 & 0.100 & 0.688 & 82.19 & 0.109 & 0.791 & 39.58 & 0.040 & 0.994 & 0.00 & 0.010 & - & 48.19 & 0.065 & 0.824\\
        \midrule
        P(True) & No & \multirow{4}{*}{\repeatvalue{39.6}} & 0.095 & 0.755 & \multirow{4}{*}{\repeatvalue{21.88}} & 0.163 & \textbf{0.620} & \multirow{4}{*}{\repeatvalue{4.58}} & 0.291 & 0.439 & \multirow{4}{*}{\repeatvalue{0.00}} & 0.326 & - & \multirow{4}{*}{\repeatvalue{16.52}} & 0.219 & 0.605\\
        Verbalization & No &  & 0.436 & 0.479 &  & 0.611 & 0.518 &  & 0.807 & 0.365 &  & 0.853 & - & & 0.677 & 0.454\\
        Elicitation & No &  & 0.450 & 0.599 &  & 0.612 & 0.524 &  & 0.770 & 0.364 &  & 0.837 & - & & 0.667 & 0.496\\
        BaseCal & No &  & 0.372 & 0.385 &  & 0.575 & 0.362 &  & 0.730 & 0.519 &  & 0.774 & - & & 0.613 & 0.422\\
        \rowcolor{lightblue} \textbf{TTCL} & No & \textbf{58.0} & \textbf{0.094} & \textbf{0.770} & \textbf{40.94} & \textbf{0.128} & 0.617 & \textbf{9.17} & \textbf{0.223} & \textbf{0.878} & 0.00 & \textbf{0.114} & - & \textbf{27.03} & \textbf{0.140} & \textbf{0.755} \\
        
        \bottomrule[1.6pt]
    \end{tabular}}
\end{table}

\newlength{\rightimgheight}
\setlength{\rightimgheight}{0.1\textheight} 

\begin{figure}[t!]
  \centering
  \vspace{-4mm}
  \begin{minipage}[t]{0.58\textwidth}
    \vspace{-22mm}
    \centering 
    \captionof{table}{\textbf{Performance of TTCL models on other domains.} The results show that TTCL improves calibration without sacrificing task performance on other domains. More results of other TTCL models refer to Appendix~\ref{App:Results_Other_Domains}.}
    \label{Tble:TTCL_Other_Bench} 
    \vspace{-2mm} 
    \resizebox{\linewidth}{!}{
    \begin{tabular}{l|ccccccc}
    \toprule[1.6pt]
        \multirow{2}{*}{\textbf{TTCL Target}} & \multicolumn{2}{c}{\textbf{ARC}} & \multirow{2}{*}{\textbf{GPQA-D}} & \multirow{2}{*}{\textbf{IFEval}} & \multirow{2}{*}{\textbf{MMLU-Pro}} & \multicolumn{2}{c}{\textbf{TruthfulQA}} \\
        \cmidrule(lr){2-3} \cmidrule(lr){7-8}
        ~ & Challenge & Easy & ~ & ~ & ~ & MC1 & MC2 \\
        \midrule
        \multicolumn{8}{c}{\textit{\textbf{Qwen3-8B}}} \\
        \midrule
        \textbf{Base Model} & 55.20 & 75.34 & 50.00 & 84.62 & 63.70 & 32.80 & 51.60 \\
        \textbf{MATH500} & 55.29 & 74.83 & 51.52 & 85.47 & 64.82 & 33.29 & 51.65 \\
        \textbf{AMC} & 55.29 & 74.83 & 45.96 & 85.83 & 63.87 & 33.05 & 51.35 \\
        \textbf{AIME24} & 54.52 & 74.07 & 48.99 & 86.20 & 62.66 & 32.68 & 50.95 \\
        \textbf{AIME25} & 54.52 & 74.12 & 51.01 & 84.98 & 63.16 & 32.80 & 50.82 \\
        \specialrule{1pt}{0.4ex}{0.4ex}
        \multicolumn{8}{c}{\textit{\textbf{RLCR-Qwen3-8B on DAPO-Math-14K}}} \\
        \midrule
        \textbf{RLCR Model} & 55.03 & 74.62 & 52.53 & 84.63 & 63.85 & 32.44 & 51.48 \\
        \textbf{AMC} & 54.61 & 74.41 & 57.07 & 84.58 & 64.50 & 32.80 & 51.61 \\
        \textbf{SimpleQA} & 55.03 & 74.03 & 52.53 & 84.28 & 64.07 & 32.93 & 50.49 \\
        \textbf{ChineseSimpleQA} & 54.95 & 74.16 & 55.05 & 84.10 & 64.54 & 32.56 & 50.83  \\
        \textbf{HotpotQA} & 54.86 & 74.33 & 52.53 & 85.38 & 64.40 & 32.93 & 51.12  \\
        \textbf{TriviaQA} & 54.86 & 74.49 & 53.03 & 84.07 & 63.92 & 32.56 & 50.71  \\
    \bottomrule[1.6pt]
    \end{tabular}}
  \end{minipage}\hfill
  \begin{minipage}[t]{0.41\textwidth}
    \centering
    \includegraphics[width=0.49\linewidth,height=\rightimgheight]{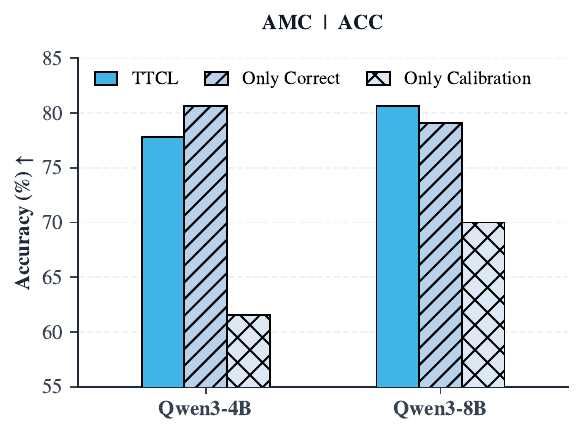}
    \includegraphics[width=0.49\linewidth,height=\rightimgheight]{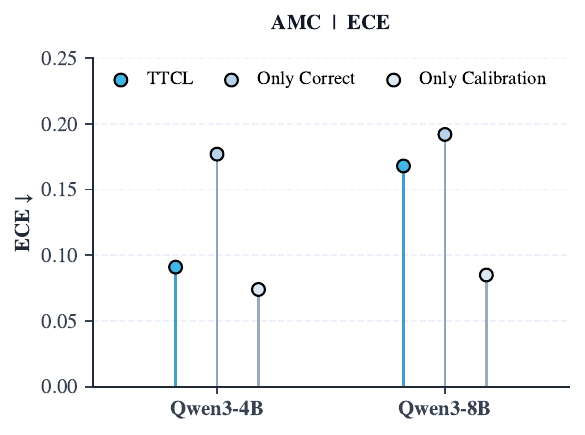} \\
    \includegraphics[width=0.49\linewidth,height=\rightimgheight]{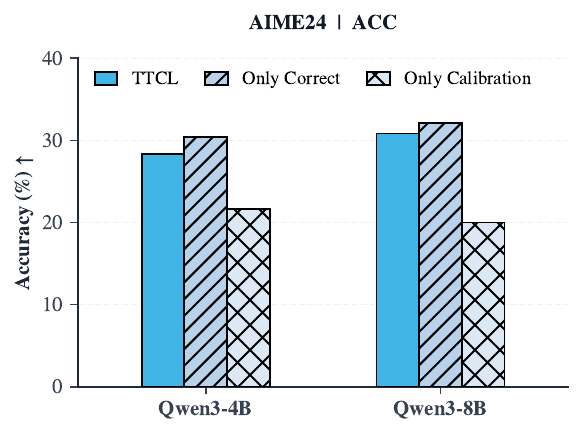}
    \includegraphics[width=0.49\linewidth,height=\rightimgheight]{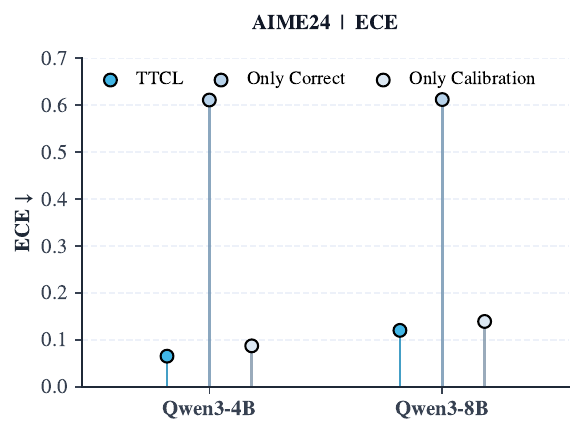} \\
    \vspace{-3mm}
    \captionof{figure}{\textbf{Ablation study.} ``Only Correct" and ``Only Calibration" retain correctness and calibration rewards, respectively.}
    \label{Fig.Ablation}
  \end{minipage}
  \vspace{-5mm}
\end{figure}

\textbf{TTCL can further adapt calibrated models to new target tasks.}
From Tables~\ref{Tble:TTCL_InDomain} and~\ref{Tble:TTCL_OutDomain}, TTCL remains effective on RLCR-calibrated models. In the in-domain math-to-math setting, it further improves both accuracy and calibration. More importantly, RLCR calibration transfers imperfectly across domains: math-calibrated models remain poorly calibrated on factual QA tasks, and vice versa. In both math-to-FactQA and FactQA-to-math settings, TTCL consistently improves ACC and ECE, showing that it can refine calibration directly on unlabeled target-task data when transfer is limited.

\textbf{TTCL consistently outperforms existing label-free calibration strategies.}
From Table~\ref{Tble:TTCL_Baseline}, TTCL generally achieves higher ACC and AUROC and lower ECE than existing label-free baselines. As these methods mainly perform post-hoc confidence estimation without model training, their ACC remains unchanged, while TTCL further improves accuracy through target-task adaptation. RLCR performs strongly under label-leakage settings with access to ground-truth correctness labels, reflecting the benefit of explicit supervision. Notably, TTCL matches or surpasses RLCR in calibration in several cases without such labels, such as Llama-3.1-8B-Instruct on MATH500.

\vspace{-2mm}
\subsection{Further Analysis}
\label{Sec:Furhter_Results}
\vspace{-2mm}
\textbf{TTCL improves calibration without sacrificing performance on other domains.}
Table~\ref{Tble:TTCL_Other_Bench} evaluates TTCL-trained models on ARC~\citep{clark2018arc}, GPQA-Diamond~\citep{rein2023gpqa}, IFEval~\citep{zhou2023ifeval}, MMLU-Pro~\citep{wang2024mmlupro}, and TruthfulQA~\citep{lin2022truthfulqa}, covering scientific reasoning, instruction following, multi-domain knowledge, and truthfulness. TTCL-trained models maintain performance comparable to their pre-adaptation counterparts, indicating that TTCL improves target-task calibration while largely preserving general capabilities of the models.

\par
\textbf{Each component of TTCL contributes to its effectiveness.} As shown in the ablation study in Fig.~\ref{Fig.Ablation}, optimizing only the correctness reward preserves accuracy but leads to substantially worse calibration, whereas optimizing only the calibration reward achieves low ECE at the cost of clear accuracy degradation. These results indicate that the two objectives are complementary, and combining them enables TTCL to achieve a better balance between accuracy and calibration.

\begin{table}[t]
    \centering
    \caption{\textbf{Performance of binary pseudo-correctness as confidence supervision vs. TTCL confidence supervision.} Pseudo-RLCR uses the majority-vote pseudo-label for both correctness and confidence supervision signals, whereas TTCL uses answer frequency as the confidence supervision.}
    \label{Tble:TTCL_PseudoRLCR}
    \vspace{-3mm}
    \resizebox{1.0\textwidth}{!}{
    \begin{tabular}{ll|ccc|ccc|ccc|ccc}
    \toprule[1.6pt]
        \multirow{2}{*}{\textbf{Model}} & \multirow{2}{*}{\textbf{Method}} & \multicolumn{3}{c|}{\textbf{AMC23}} & \multicolumn{3}{c|}{\textbf{AIME24}} & \multicolumn{3}{c|}{\textbf{NQ}} & \multicolumn{3}{c}{\textbf{HotpotQA}} \\
        \cmidrule(lr){3-5} \cmidrule(lr){6-8} \cmidrule(lr){9-11} \cmidrule(lr){12-14}
        & & \textbf{ACC $\uparrow$} & \textbf{ECE $\downarrow$} & \textbf{AUROC $\uparrow$} & \textbf{ACC $\uparrow$} & \textbf{ECE $\downarrow$} & \textbf{AUROC $\uparrow$} & \textbf{ACC $\uparrow$} & \textbf{ECE $\downarrow$} & \textbf{AUROC $\uparrow$} & \textbf{ACC $\uparrow$} & \textbf{ECE $\downarrow$} & \textbf{AUROC $\uparrow$} \\
    \midrule
        \multirow{2}{*}{Qwen3-4B}
        & Pseudo-RLCR & 74.38 & 0.158 & 0.698 & 26.67 & 0.595 & 0.574 & \textbf{10.94} & 0.836 & 0.688 & \textbf{11.00} & 0.853 & 0.554 \\
        & \textbf{TTCL} & \textbf{77.81} & \textbf{0.091} & \textbf{0.910}
        & \textbf{28.33} & \textbf{0.065} & \textbf{0.856} & 10.63 & \textbf{0.155} & \textbf{0.848} & 10.50 & \textbf{0.244} & \textbf{0.744} \\
    \midrule
        \multirow{2}{*}{Qwen3-8B}
        & Pseudo-RLCR & 72.50 & 0.249 & 0.735 & \textbf{31.25} & 0.450 & 0.744 & \textbf{15.50} & 0.751 & 0.553 & 14.00 & 0.847 & 0.553 \\
        & \textbf{TTCL} & \textbf{80.63} & \textbf{0.168} & \textbf{0.844}
        & 30.83 & \textbf{0.120} & \textbf{0.882} & 14.81 & \textbf{0.190} & \textbf{0.752} & \textbf{15.31} & \textbf{0.330} & \textbf{0.798} \\
    \bottomrule[1.6pt]
    \end{tabular}}
    \vspace{-4mm}
\end{table}

\definecolor{accuracyblue}{HTML}{41B6E6}
\definecolor{calibrationgapblue}{HTML}{B7D2EA}

\begin{figure}[t!]
  \centering
  \begin{minipage}[t]{0.49\textwidth}
    \centering
    \includegraphics[width=0.99\linewidth]{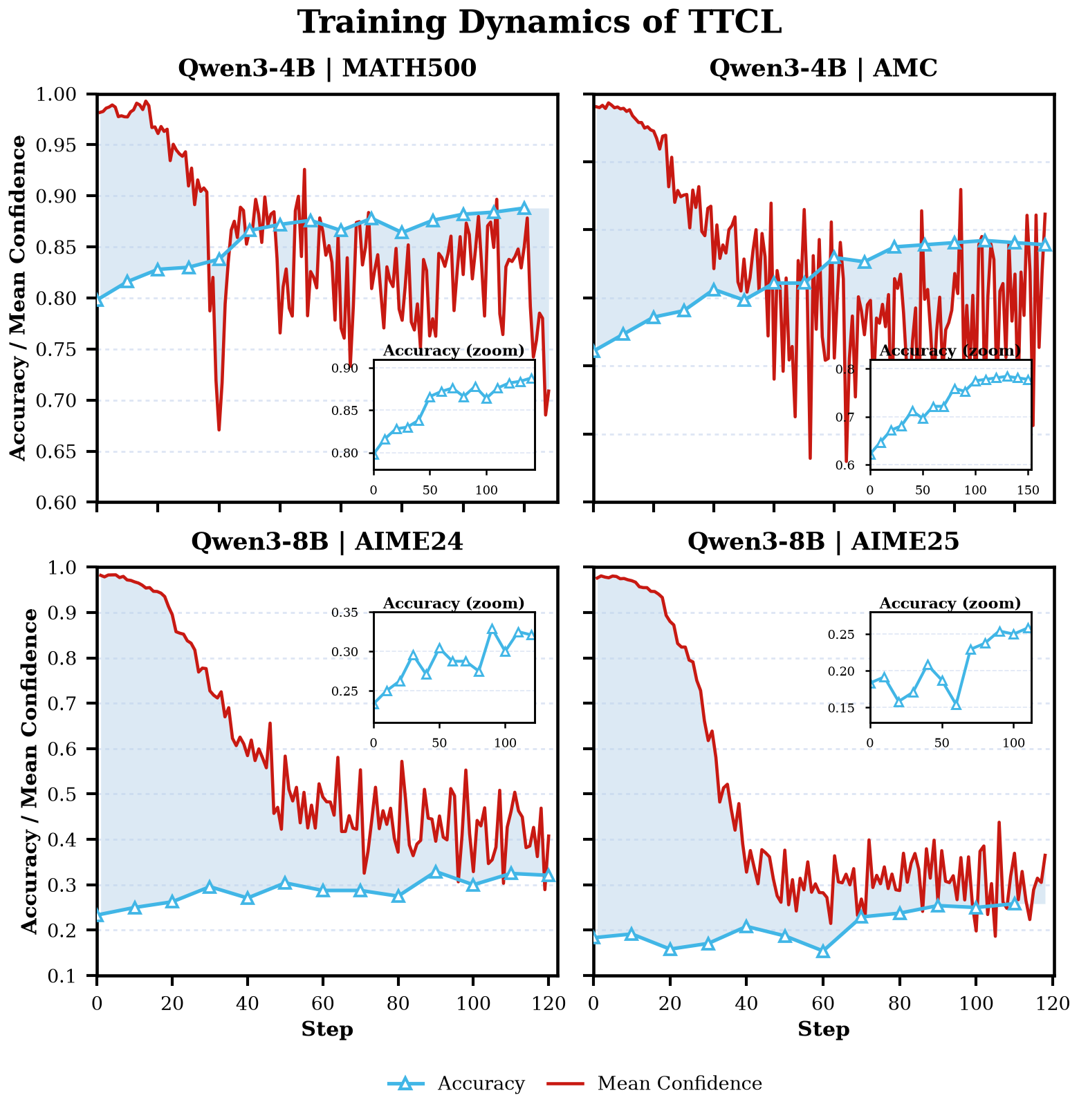}
    \vspace{-8mm}
    \caption{\textbf{Training dynamics of TTCL.} The dynamics illustrate the calibration process between accuracy and mean confidence during TTCL training. More dynamics refer to Appendix~\ref{App:Training_Dynamics_TTCL}.
    }
    \label{Fig:training_dynamics}
  \end{minipage}\hfill
  \begin{minipage}[t]{0.49\textwidth}
    \centering
    \includegraphics[width=0.99\linewidth]{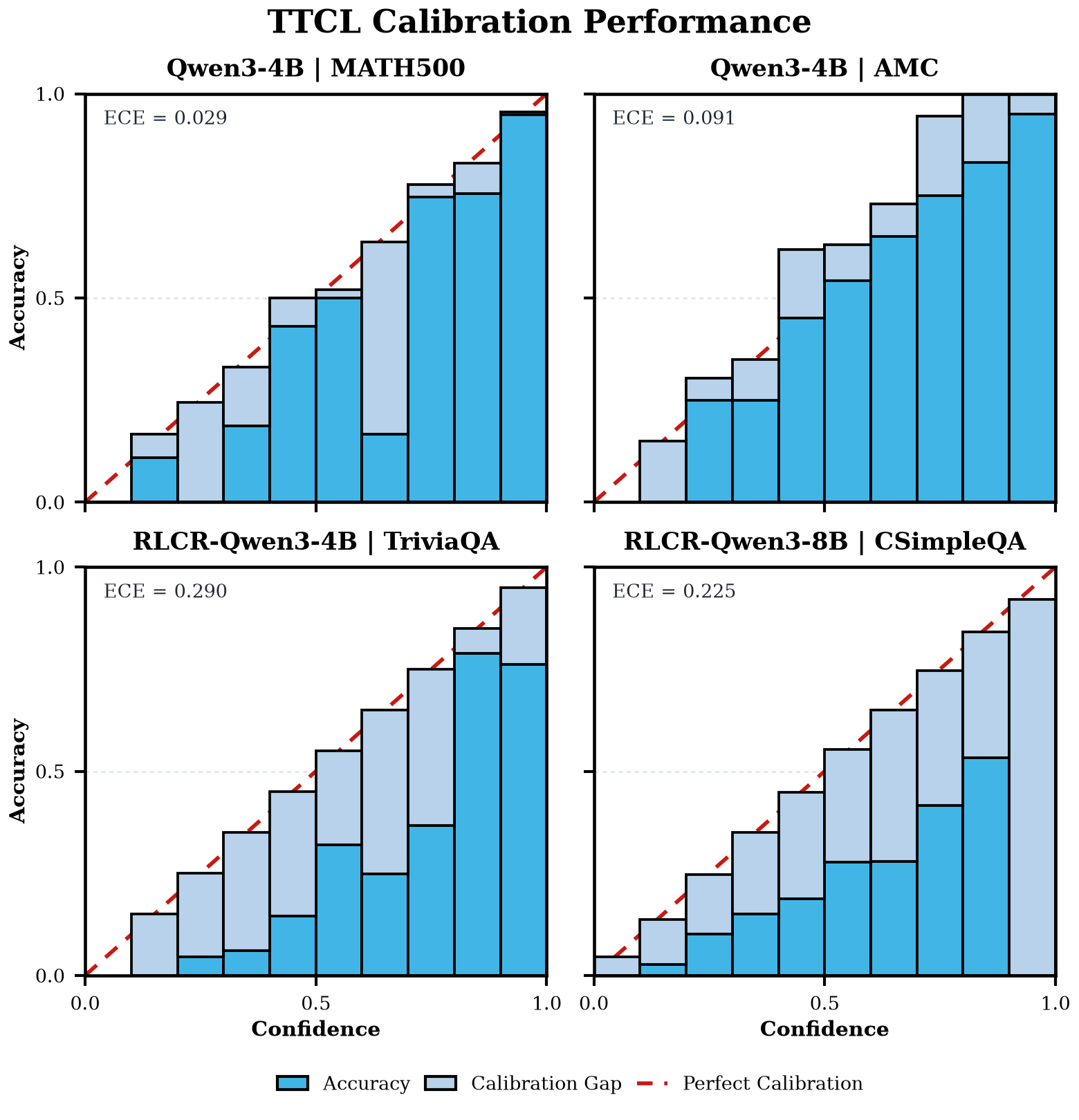}
    \vspace{-8mm}
    \caption{\textbf{Visulization of TTCL calibration performance across different settings.} Additional visualizations comparing calibration before and after TTCL are provided in Appendix~\ref{App:Visual_Calibrate}.
    }
    \label{Fig:calibration_bar}
  \end{minipage}
  \vspace{-7mm}
\end{figure}

\par
\textbf{Frequency-based confidence supervision is more effective than binary pseudo-correctness supervision.} 
Table~\ref{Tble:TTCL_PseudoRLCR} compares TTCL with Pseudo-RLCR, which uses the majority-vote pseudo-label for both correctness and confidence supervision. We observe that TTCL achieves lower ECE and higher ACC and AUROC in most cases, showing that binary pseudo-labels are insufficient for confidence learning and that frequency-based probability provides a more informative target.

\par
\textbf{TTCL progressively improves calibration during training.}
Fig.~\ref{Fig:training_dynamics} shows the evolution of accuracy and mean confidence during TTCL training. The initial models are generally over-confident, with mean confidence exceeding empirical accuracy. As training proceeds, TTCL narrows this gap while improving accuracy, bringing the two increasingly into alignment. This training behavior empirically supports our theoretical interpretation of TTCL as a surrogate for the ideal calibration objective.

\par
\textbf{TTCL effectively aligns confidence with empirical accuracy.}
Beyond the scalar ECE metric, Fig.~\ref{Fig:calibration_bar} visualizes confidence against empirical accuracy. Accuracy across confidence bins becomes closer to the ideal calibration diagonal after TTCL, indicating better alignment between predicted confidence and correctness likelihood. This pattern holds across mathematical reasoning and factual QA, as well as base and pre-calibrated models, showing that reliable confidence estimates of TTCL.

\vspace{-2mm}
\section{Conclusion}
\vspace{-2mm}

In this work, we introduce \textit{Test-Time Calibration Learning (TTCL)}, a label-free framework for adapting LLM calibration directly to target tasks at test time. Unlike existing methods relying on labeled supervision, TTCL derives correctness and calibration self-supervision from multiple model-generated responses, enabling calibration learning without access to ground-truth labels. Theoretical analysis shows that TTCL provides a label-free surrogate for the ideal calibration objective. Extensive experiments on mathematical reasoning and factual QA demonstrate consistent improvements in accuracy and calibration. TTCL can further adapt pre-calibrated models under domain shift, improving calibration when source-domain transfer is limited. We hope this work inspires further research on test-time calibration for reliable confidence estimation in unlabeled deployment settings.

\bibliography{iclr2027_conference}
\bibliographystyle{iclr2027_conference}

\newpage

\appendix

\section{Group Relative Policy Optimization (GRPO)}
\textbf{Group Relative Policy Optimization (GRPO).} GRPO is a widely used RL algorithm for LLM reasoning. Given a question $x$, GRPO samples a group of $G$ responses $\{y_i\}_{i=1}^G$, each associated with a reward $r_i$, and optimizes the policy LLM using the group-relative advantages~\citep{shao2024deepseekmath}:
\begin{equation}
\label{Eq:GRPO}
    \mathcal{J}_{\mathrm{GRPO}}(\theta) = \mathbb{E}\left[\frac{1}{G}\sum_{i=1}^{G}\frac{1}{|y_i|}\sum_{t=1}^{|y_i|}\left(\rho_{i,t}(\theta)\hat{A}_i\right) \right], \quad \hat{A}_i = r_i - \operatorname{Mean}\left(\{r_i\}_{i=1}^G\right),
\end{equation}
where $\rho_{i,t}(\theta)$ denotes the PPO-style clipped importance ratio between the current policy $\pi_\theta$ and the rollout policy $\pi_{\theta_{\mathrm{old}}}$. Different from standard GRPO, we use mean-only advantage normalization without dividing by the reward standard deviation. As the Brier calibration reward is continuous, standard-deviation normalization can amplify small reward differences and lead to unstable or extreme calibration updates~\citep{bereket2025uncalibrated}. While several improved variants of GRPO have been proposed, such as DAPO~\citep{yu2026dapo} and Dr.GRPO~\citep{liu2025DrGRPO}, our work primarily focuses on test-time calibration learning of confidence estimation. We adopt GRPO as our policy optimization algorithm throughout our experiments.

\section{Related Work}
\label{App:RelatedWork}
In this section, we mainly review related work from three perspectives: LLM reasoning, label-free RL methods for LLM reasoning, and existing LLM calibration methods, to better position TTCL within the broader research landscape and demonstrate its distinctions from prior studies.

\par
\textbf{Large language model (LLM) reasoning.} 
The reasoning capabilities of LLMs have advanced rapidly in recent years. Early studies primarily elicit reasoning capability through prompting and inference-time strategies, such as Chain-of-Thought (CoT) prompting~\citep{wei2022CoT}, Tree-of-Thoughts (ToT)~\citep{yao2023treeofthought}, or self-consistency~\citep{wang2023selfconsistency}. More recently, reinforcement learning (RL) has emerged as an effective paradigm for further eliciting and strengthening LLM reasoning capabilities~\citep{shao2024deepseekmath,openai2024reasoningo1}, with reinforcement learning with verifiable rewards (RLVR) becoming a widely adopted approach in modern LLM post-training pipelines~\citep{guo2025deepseekr1,team2026kimiK3,xiao2026mimov2,yang2025qwen3}. However, improvements in reasoning accuracy do not necessarily translate into reliable uncertainty estimation. Existing studies have shown that LLMs can exhibit substantial over-confidence~\citep{groot2024overconfidence,leng2025taming}, while RL training that mainly rewards answer correctness may further exacerbate such miscalibration~\citep{kalai2025language,bereket2025uncalibrated}. This poses substantial risks for deploying LLMs in high-risk scenarios, where confidently incorrect predictions may lead to severe consequences. Therefore, beyond improving reasoning accuracy, reliable reasoning models should also faithfully reflect the uncertainty of their predictions through well-calibrated confidence estimates.

\par
\textbf{Label-free RL methods for LLM reasoning.} 
Recently, emerging studies have explored RL without annotated ground-truth answers to improve LLM reasoning capabilities, reducing reliance on labeled supervision~\citep{yang2026scaling}. According to the source of self-supervised signals, these methods can be broadly divided into two categories. One line of work mainly encourages self-consistency among model outputs~\citep{zuo2026ttrl,zhang2026consistent}, for example, using majority-voted answers as pseudo labels to supervise RL optimization~\citep{zhang2026co,shafayat2025can}. Another line constructs rewards from the model's internal signals~\citep{zhang2025no,wen2026verifier}, such as minimizing prediction entropy~\citep{zhang2026empo,prabhudesai2025maximizing} or maximizing self-certainty~\citep{zhao2026learning}, to sharpen the model's predictive distribution. Although these methods enable label-free improvement of reasoning performance, they primarily focus on answer correctness and lack an explicit optimization objective for calibration learning. As a result, the resulting models may still fail to express confidence that faithfully reflects uncertainty over their own predictions. In contrast, our work further explores label-free learning for LLM calibration, enabling models to adapt both reasoning capability and confidence estimation on target tasks at test time.

\par
\textbf{Existing LLM calibration methods.} 
Existing approaches to LLM calibration can be broadly categorized into post-hoc calibration and calibration learning. Post-hoc methods estimate confidence for model-generated responses using signals from multiple perspectives, such as token probabilities~\citep{tan2026basecal}, verbalized confidence~\citep{tian2023calibration}, self-evaluation~\citep{kadavath2022PTrue}, sampling consistency~\citep{lyu2025calibrating}, or semantic uncertainty~\citep{farquhar2024detecting,kuhn2023semantic}. While effective in many settings, these methods primarily derive confidence through external signals or post-processing, rather than explicitly training the LLM to develop the capability of estimating the confidence of its own generated responses. In contrast, recent calibration learning studies directly train LLMs to express their calibrated confidence, typically by incorporating calibration-aware objectives into existing LLM training objectives~\citep{ma2026decoupling,seo2026advice,finlay2026caliber}. For example, RLCR~\citep{damani2026RLCR}, BCRL~\citep{wu2025BCRL}, and Rewarding Doubt~\citep{bani2026rewarding} leverage proper scoring rules~\citep{gneiting2007strictlyproper} to introduce calibration rewards into RL training, jointly optimizing answer correctness and confidence estimation. 

These methods explicitly encourage calibration to become an intrinsic capability of the model. However, their calibration supervision generally relies on ground-truth correctness labels during training~\citep{tan2026basecal}, incurring additional annotation costs, while the learned calibration behavior may transfer poorly when the target distribution differs from the training distribution~\citep{jangconfidence2026}. In this work, we instead introduce \emph{Test-Time Calibration Learning (TTCL)}, which directly adapts model calibration on the target task at test time using label-free supervision without requiring ground-truth labels, thereby enabling calibration learning under unseen test distributions.

\section{Pseudocode of TTCL}
To provide an intuitive overview of the TTCL pipeline, we summarize its pseudocode in Algorithm~\ref{Alg:TTCL}. Specifically, the target cache is initialized once using the initial policy model and remains fixed throughout each epoch. It provides the pseudo-correctness and confidence targets and is updated only at the end of each epoch using the newly observed answer distribution.

\begin{algorithm}[t]
\caption{Test-Time Calibration Learning (TTCL)}
\label{Alg:TTCL}
\begin{algorithmic}[1]
\State \textbf{Input:} Unlabeled target set $\mathcal{D}_{\mathrm{test}}$; initial policy model $\pi_{\theta_0}$; number of training epochs $E$; learning rate $\alpha$; cache update rate $\eta$.
\State \textbf{Output:} Calibration-adapted policy model $\pi_{\theta}$
\State Initialize $\theta \gets \theta_0$
\State Initialize an empty target cache $q_{\mathrm{cache}}^{(0)}$
\ForAll{$x \in \mathcal{D}_{\mathrm{test}}$}
    \State Sample $G$ responses from $\pi_{\theta_0}(\cdot\mid x)$
    \State Set $q_{\mathrm{cache}}^{(0)}(\cdot\mid x)$ to their empirical
    answer distribution
\EndFor
\For{$e=0,\ldots,E-1$}
    \State Freeze $q_{\mathrm{cache}}^{(e)}$ for supervision in epoch $e$
    \ForAll{minibatches $\mathcal{B}\subset\mathcal{D}_{\mathrm{test}}$}
        \ForAll{input question $x\in\mathcal{B}$}
            \State Sample $\{y_i=(\tau_i,a_i,c_i)\}_{i=1}^{G}$ from the rollout policy $\pi_{\theta_{\mathrm{old}}}(\cdot\mid x)$.
            \State $\tilde a_x^{(e)} \gets \arg\max_a q_{\mathrm{cache}}^{(e)}(a\mid x)$. \# \textit{Read the pseudo-correctness target from the cache}
            \For{$i=1,\ldots,G$}
                \State $\tilde p_i^{(e)} \gets q_{\mathrm{cache}}^{(e)}(a_i\mid x)$. \# \textit{Read the confidence target from the cache}
                \State $r_i^{\mathrm{corr}} \gets \mathbb{I}[a_i\equiv\tilde a_x^{(e)}]$. \# \textit{Compute correctness reward}
                \State $r_i^{\mathrm{cal}} \gets -(c_i-\tilde p_i^{(e)})^2$. \# \textit{Compute calibration reward}
            \EndFor
            \State Compute mean-centered advantages $\{\hat A_i^{\mathrm{corr}},\hat A_i^{\mathrm{cal}}\}_{i=1}^{G}$.
        \EndFor
        \State Compute TTCL objective in Eq.~(\ref{Eq:TTCL}).
        \State Update $\theta \leftarrow \theta + \alpha \nabla_\theta \mathcal{J}_{\mathrm{TTCL}}(\theta)$.
    \EndFor
    \State \# \textit{EMA updates the target cache after each epoch}
    \ForAll{$x \in \mathcal{D}_{\mathrm{test}}$}
        \State Let $q^{(e)}(\cdot\mid x)$ be the empirical answer distribution observed in epoch $e$.
        \State Update $q_{\mathrm{cache}}^{(e+1)}(\cdot\mid x) \gets (1-\eta)q_{\mathrm{cache}}^{(e)}(\cdot\mid x)+\eta q^{(e)}(\cdot\mid x)$ using EMA.
    \EndFor
\EndFor
\State \Return Calibrated policy model $\pi_{\theta}$
\end{algorithmic}
\end{algorithm}

\section{Further Theoretical Analysis of TTCL}
We provide the detailed proofs for the theoretical analysis presented in Section~\ref{Sec:Theory}. We first establish the statistical property of the frequency-based probability confidence target, then show that the conditional correctness probability is the ideal target for calibration learning, and finally prove the main analysis that the TTCL objective provides a bounded surrogate for the ideal calibration objective.
\subsection{Proof of Proposition~\ref{Theory:frequency_proxy}}
\label{App:Proof_Prop1}
Proof of \textbf{Proposition~\ref{Theory:frequency_proxy} (Answer Frequency as a Proxy for Conditional Correctness)}.
\begin{proof}
For a fixed question $x$ and candidate answer $a$, consider $G$ sampled responses
\begin{equation}
    \{Y_1,\ldots,Y_G\} \sim \pi_\theta(\cdot\mid x),
\end{equation}
and let $A_j=\operatorname{Ans}(Y_j)$ denote the extracted answer of the $j$-th response. Define the empirical answer probability as
\begin{equation}
    \hat q_G(a\mid x)=\frac{1}{G}\sum_{j=1}^{G}\mathbb{I}[A_j \equiv a].
\end{equation}
Since
\begin{equation}
    \mathbb{E}\!\left[\mathbb{I}[A_j \equiv a]\right] = \Pr(A_j=a\mid X=x) = q_{\pi_\theta}(a\mid x),
\end{equation}
We obtain
\begin{equation}
    \mathbb{E}\!\left[\hat q_G(a\mid x)\right] = \frac{1}{G} \sum_{j=1}^{G} q_{\pi_\theta}(a\mid x) = q_{\pi_\theta}(a\mid x).
\end{equation}
Therefore, $\hat q_G(a\mid x)$ is an unbiased estimator of the population answer probability. Moreover, each indicator variable is Bernoulli distributed with parameter $q_{\pi_\theta}(a\mid x)$. Hence,
\begin{equation}
    \operatorname{Var}\!\left[\hat q_G(a\mid x)\right] = \frac{q_{\pi_\theta}(a\mid x) \left(1-q_{\pi_\theta}(a\mid x)\right)}{G}.
\end{equation}
As $G\rightarrow\infty$,
\begin{equation}
    \operatorname{Var}\!\left[\hat q_G(a\mid x)\right] \rightarrow 0,
\end{equation}
and by the law of large numbers,
\begin{equation}
    \hat q_G(a\mid x) \xrightarrow[]{p} q_{\pi_\theta}(a\mid x).
\end{equation}
Thus, the empirical answer probability is both unbiased and consistent.
\par
For a response $y$ whose extracted answer is $a=\operatorname{Ans}(y)$, when the model-induced answer distribution reflects the model's uncertainty over which answer is correct, answers with larger $q_{\pi_\theta}(a\mid x)$ correspond to larger probabilities of the associated response being correct. Therefore,
\begin{equation}
    q_{\pi_\theta}(a\mid x)
\end{equation}
provides a frequency-based proxy for
\begin{equation}
    p^\star(x,y)=\Pr(Z=1\mid X=x,Y=y).
\end{equation}
Since $\hat q_G(a\mid x)$ consistently estimates $q_{\pi_\theta}(a\mid x)$, repeated responses provide a statistically well-defined way to estimate this proxy without access to ground-truth correctness labels.
\end{proof}

\subsection{Proof of Proposition~\ref{Theory:ideal_calibration_target}}
\label{App:Proof_Prop2}
Proof of \textbf{Proposition~\ref{Theory:ideal_calibration_target} (Conditional Correctness as the Ideal Calibration Target)}.
\begin{proof}
We first establish the connection between conditional correctness probability and perfect calibration. By definition,
\begin{equation}
    p^\star(X,Y)=\Pr(Z=1\mid X,Y).
\end{equation}
Since $Z\in\{0,1\}$,
\begin{equation}
    \mathbb{E}[Z\mid X,Y]=p^\star(X,Y).
\end{equation}
Suppose the confidence estimator satisfies
\begin{equation}
    C=c(X,Y)=p^\star(X,Y).
\end{equation}
Using the law of iterated expectations,
\begin{equation}
    \mathbb{E}[Z\mid C] = \mathbb{E}\!\left[ \mathbb{E}[Z\mid X,Y] \mid C \right].
\end{equation}
Substituting the definition of $p^\star$,
\begin{equation}
    \mathbb{E}[Z\mid C] = \mathbb{E}\!\left[ p^\star(X,Y) \mid C \right].
\end{equation}
Since $C=p^\star(X,Y)$,
\begin{equation}
    \mathbb{E}[Z\mid C] = C.
\end{equation}
Conditioning on $C=c$ therefore gives
\begin{equation}
    \Pr(Z=1\mid C=c) = \mathbb{E}[Z\mid C=c] = c,    
\end{equation}
which is exactly the perfect calibration condition in Eq.~(\ref{Eq:Def_Calibration}).
\par
We next show that the expected Brier score recovers the same target. Consider
\begin{equation}
    \mathcal{J}_{\mathrm{Brier}}(c) = \mathbb{E}\!\left[(c(X,Y)-Z)^2\right].
\end{equation}
Conditioning on $(X,Y)$, and writing $c=c(X,Y)$ and $p^\star=p^\star(X,Y)$ for simplicity, we have
\begin{equation}
    \mathbb{E}\!\left[(c-Z)^2\mid X,Y\right].
\end{equation}
Adding and subtracting $p^\star$,
\begin{equation}   
    c-Z=(c-p^\star)+(p^\star-Z).
\end{equation}
Therefore,
\begin{equation}
    \mathbb{E}\!\left[(c-Z)^2\mid X,Y\right] = (c-p^\star)^2 +2(c-p^\star)\, \mathbb{E}\!\left[p^\star-Z\mid X,Y\right] +\mathbb{E}\!\left[(p^\star-Z)^2\mid X,Y\right].
\end{equation}
Since
\begin{equation}
    \mathbb{E}[Z\mid X,Y]=p^\star,    
\end{equation}
the cross term vanishes:
\begin{equation}
    \mathbb{E}\!\left[p^\star-Z\mid X,Y\right]=0.
\end{equation}
Hence,
\begin{equation}
    \mathbb{E}\!\left[(c-Z)^2\mid X,Y\right] = (c-p^\star)^2 + \mathbb{E}\!\left[(p^\star-Z)^2\mid X,Y\right]. 
\end{equation}
Because $Z$ is Bernoulli with conditional mean $p^\star$,
\begin{equation}
    \mathbb{E}\!\left[(p^\star-Z)^2\mid X,Y\right] = p^\star(1-p^\star).
\end{equation}
Taking expectation over $(X,Y)$ yields
\begin{equation}
    \mathcal{J}_{\mathrm{Brier}}(c) = \mathbb{E}\!\left[ (c(X,Y)-p^\star(X,Y))^2 \right] + \mathbb{E}\!\left[ p^\star(X,Y)(1-p^\star(X,Y)) \right].  
\end{equation}
The second term is independent of the confidence estimator $c$. Therefore, minimizing the expected Brier score is equivalent to minimizing
\begin{equation}    
    \mathbb{E}\!\left[ (c(X,Y)-p^\star(X,Y))^2 \right],
\end{equation}
whose minimum is attained at
\begin{equation}    
    c^\star(X,Y)=p^\star(X,Y).
\end{equation}
Thus,
\begin{equation}    
    c^\star = \arg\min_c \mathcal{J}_{\mathrm{Brier}}(c) = p^\star.
\end{equation}
Therefore, the conditional correctness probability is both sufficient for perfect calibration and the optimal target under Brier-score minimization.
\end{proof}

\subsection{Proof of Theorem~\ref{Theory:TTCL_boundary}}
\label{App:Proof_Theorem}
With the frequency-based proxy and the ideal calibration target established above, we now prove the surrogate bound in Theorem~\ref{Theory:TTCL_boundary}.
\par
Proof of \textbf{Theorem~\ref{Theory:TTCL_boundary} (TTCL as a Surrogate for the Ideal Calibration Objective)}.
\begin{proof}
Recall that
\begin{equation}
    \mathcal{J}_{\mathrm{TTCL}}(c) = \mathbb{E}\!\left[ (c(X,Y)-s(X,Y))^2 \right],
\end{equation}
and
\begin{equation}    
    \mathcal{J}_{\mathrm{ideal}}(c) = \mathbb{E}\!\left[ (c(X,Y)-p^\star(X,Y))^2 \right].
\end{equation}
The estimation error is defined as
\begin{equation}    
    \epsilon = \mathbb{E}\!\left[ \left|s(X,Y)-p^\star(X,Y)\right| \right].
\end{equation}
For notational simplicity, in this proof, let
\begin{equation}
    c=c(X,Y), \qquad s=s(X,Y), \qquad p^\star=p^\star(X,Y).    
\end{equation}
We first consider the pointwise difference between the two squared-error objectives:
\begin{equation}    
    \left|(c-s)^2-(c-p^\star)^2\right|.
\end{equation}
Using the difference-of-squares identity,
\begin{align}
\left|(c-s)^2-(c-p^\star)^2\right| &= \left|(p^\star-s)(2c-s-p^\star)\right| \\
&= |s-p^\star|\,|2c-s-p^\star|.
\end{align}
Since
\begin{equation}    
    c,s,p^\star\in[0,1],
\end{equation}
we have
\begin{equation} 
    |2c-s-p^\star|\le 2.
\end{equation}
Therefore,
\begin{equation}    
    \left|(c-s)^2-(c-p^\star)^2\right| \le 2|s-p^\star|.
\end{equation}
Taking expectation on both sides and applying
\begin{equation}    
    |\mathbb{E}[U]| \le \mathbb{E}[|U|],
\end{equation}
we obtain
\begin{align}
    \left| \mathcal{J}_{\mathrm{TTCL}}(c) - \mathcal{J}_{\mathrm{ideal}}(c) \right| &= \left| \mathbb{E}\!\left[ (c-s)^2-(c-p^\star)^2 \right] \right| \\
    &\le \mathbb{E}\!\left[ \left| (c-s)^2-(c-p^\star)^2 \right| \right] \\
    &\le 2\mathbb{E}\!\left[ |s-p^\star| \right] \\ &= 2\epsilon.
\end{align}
Therefore,
\begin{equation}
    \left| \mathcal{J}_{\mathrm{TTCL}}(c) - \mathcal{J}_{\mathrm{ideal}}(c) \right| \le 2\epsilon,
\end{equation}
for any confidence estimator $c(X,Y)\in[0,1]$. This establishes that TTCL provides a bounded surrogate for the ideal calibration learning objective.
\end{proof}

\section{Additional Experimental Details}
\subsection{Details of Baselines}
\label{App:Baselines}
We compare TTCL with several existing label-free calibration strategies and additionally include RLCR, a calibration learning method with access to ground-truth labels, as a label-leakage reference. We briefly introduce these methods below:
\begin{itemize}[left=2pt]
    \item \textbf{RLCR}~\citep{damani2026RLCR}: Reinforcement Learning with Calibration Rewards which jointly optimizes answer
    correctness and verbalized confidence by augmenting the binary correctness reward with a Brier-score-based calibration reward. Unlike TTCL, RLCR requires ground-truth answer labels to construct its training reward for both correctness and calibration.
    \item \textbf{P(True)}~\citep{kadavath2022PTrue}: P(True) estimates confidence through self-evaluation. Given a question and
    the generated answer, the model is prompted to judge whether the
    proposed answer is true or false, and the predicted probability assigned to ``True'' is used as the confidence score.
    \item \textbf{Verbalization}~\citep{tian2023calibration}: Verbalization directly prompts the model to report a numerical confidence
    together with its answer, using the model's instruction-following ability
    to express uncertainty without accessing internal logits or additional
    labeled data.
    \item \textbf{Elicitation}~\citep{tian2023calibration}: Elicitation improves confidence verbalization by obtaining multiple
    confidence estimates through repeated prompting and sampling, and
    aggregating these estimates to produce the final confidence score.
    \item \textbf{BaseCal}~\citep{tan2026basecal}: BaseCal exploits the observation that the corresponding pre-trained base
    model can provide better-calibrated signals than its post-trained
    counterpart. Following BaseCal-ReEval, we re-evaluate the response
    generated by the post-trained model using the base model and aggregate the token probabilities as the calibrated confidence.
\end{itemize}
For a fair comparison, we use the same evaluation settings for all methods.

\begin{table}[t]
\caption{\textbf{Detailed implementation configurations.}}
\vspace{-2mm}
\begin{center}
\resizebox{1.0\columnwidth}{!}{
    \begin{tabular}{l|cccccc}
    \toprule[1.5pt]
    \textbf{Setting} & \textbf{MATH500} & \textbf{AMC} & \textbf{AIME24/AIME25/AIME26} & \textbf{SimpleQA} & \textbf{ChineseSimpleQA} & \textbf{NQ/HotpotQA/TriviaQA} \\
    \midrule
    Training epochs & 10 & 40 & 40 & 10 & 10 & 10 \\
    Training batch size & 32 & 8 & 8 & 64 & 128 & 32 \\
    \midrule
    Training Temperature & \multicolumn{6}{c}{0.7} \\
    Rollouts $G$ & \multicolumn{6}{c}{32} \\
    Max prompt length & \multicolumn{6}{c}{1,024} \\
    Max response length & \multicolumn{6}{c}{4,096} \\
    Learning rate & \multicolumn{6}{c}{5e-7} \\
    KL coefficient & \multicolumn{6}{c}{0.001} \\
    Optimizer & \multicolumn{6}{c}{AdamW ($\beta_1=0.9,\beta_2=0.999$)} \\
    Clip ratio & \multicolumn{6}{c}{0.2} \\
    Warmup style & \multicolumn{6}{c}{Cosine} \\
    Warmup step ratio & \multicolumn{6}{c}{0.1} \\
    \midrule
    Evaluation Temperature & \multicolumn{6}{c}{0.7} \\
    \bottomrule[1.5pt]
    \end{tabular}
}
\label{Tble:implementation_details}
\end{center}
\end{table}

\begin{table}[t!]
\caption{\textbf{Statistics of benchmarks} used in our experiments.}
\vspace{-2mm}
\begin{center}
\resizebox{0.7\columnwidth}{!}{
    \begin{tabular}{l|lc}
    \toprule[1.5pt]
    \textbf{Task} & \textbf{Dataset Name} & \textbf{\# Data Size} \\
    \midrule
    \multirow{6}{*}{\textbf{Mathematics}} & MATH500~\citep{lightman2024math500} & 500 \\
    ~ & AMC~\citep{yang2024qwen2p5math} & 40 \\
    ~ & AIME24~\citep{aime24} & 30 \\
    ~ & AIME25~\citep{aime25} & 30 \\
    ~ & AIME26~\citep{aime26} & 30 \\
    ~ & DAPO-Math-14K~\citep{yu2026dapo} & 14,109 \\
    \midrule
    \multirow{6}{*}{\textbf{Factual QA}} & SimpleQA~\citep{wei2024simpleqa,haas2025simpleqaVerified} & 1,000 \\
    ~ & ChineseSimpleQA~\citep{he2025chinese} & 3,000 \\
    ~ & NQ~\citep{kwiatkowski2019natural} & 200 \\
    ~ & HotpotQA~\citep{yang2018hotpotqa} & 200 \\
    ~ & TriviaQA~\citep{joshi2017triviaqa} & 200 \\
    ~ & FactQA-10K~\citep{yang2018hotpotqa,joshi2017triviaqa} & 10,000 \\ 
    \bottomrule[1.5pt]
    \end{tabular}}
\label{Tble:dataset_details}
\end{center}
\end{table}

\subsection{More Implementation Details}
\label{App:Implement_detail}
For both training and evaluation, we use the same system prompt as follows:
\begin{lstlisting}[title={\textbf{System prompt for mathematical reasoning}}]
Let's think step by step and output the final answer within \boxed{}. 
Also output your confidence for this answer after <conf> as a floating-point number between 0 (complete uncertainty) and 1 (full confidence).
Example: <conf> Confidence: <float value>
Ensure that the confidence output appears on a single line and strictly follows the specified format.
\end{lstlisting}

\begin{lstlisting}[title={\textbf{System prompt for factual question answering}}]
Let's think step by step and output the final answer within <answer> </answer> tags. 
Also output your confidence for this answer after <conf> as a floating-point number between 0 (complete uncertainty) and 1 (full confidence).
Example: <conf> Confidence: <float value>
Ensure that the confidence output appears on a single line and strictly follows the specified format.
\end{lstlisting}
The first prompt is used for mathematical reasoning, while the second is used for factual question answering; they differ only in the requirement for final answer extraction.

\par
Table~\ref{Tble:implementation_details} summarizes the implementation details of TTCL for each benchmark. Following a similar principle adopted in existing test-time training methods~\citep{zuo2026ttrl}, we use smaller batch sizes and more training epochs for smaller benchmarks to ensure a sufficient number of training steps. Additionally, all our experiments are conducted on $8\times$ Ascend 910B NPUs.

\begin{table}[t]
    \centering
    \caption{\textbf{Brier score and mean confidence of TTCL for base models on mathematical reasoning benchmarks.} Improved results after TTCL are highlighted in \textbf{bold}, with positive relative improvements shown in \textcolor{softred}{red}. As mean confidence is not a metric for which higher or lower values are better, we do not report its relative improvement, and it only serves as a reference alongside other metrics.}
    \label{Tble:TTCL_Base_Math_BrierConf}
    \vspace{-3mm}
    \resizebox{\textwidth}{!}{
    \begin{tabular}{l|ccc|ccc|ccc|ccc}
    \toprule[1.6pt]
        \multirow{2}{*}{\textbf{Methods}} & \multicolumn{3}{c}{\textbf{MATH500}} & \multicolumn{3}{c}{\textbf{AMC}} & \multicolumn{3}{c}{\textbf{AIME24}} & \multicolumn{3}{c}{\textbf{AIME25}} \\ \cmidrule(lr){2-4} \cmidrule(lr){5-7} \cmidrule(lr){8-10} \cmidrule(lr){11-13} 
        ~ & \textbf{ACC $\uparrow$} & \textbf{Brier $\downarrow$} & \textbf{Conf} & \textbf{ACC $\uparrow$} & \textbf{Brier $\downarrow$} & \textbf{Conf} & \textbf{ACC $\uparrow$} & \textbf{Brier $\downarrow$} & \textbf{Conf} & \textbf{ACC $\uparrow$} & \textbf{Brier $\downarrow$} & \textbf{Conf} \\ 
        \midrule
            \multicolumn{13}{c}{\textit{\textbf{Qwen3-4B}}} \\
        \midrule
        Base Model  & 80.6 & 0.184 & 0.984 & 59.38 & 0.387 & 0.984 & 20.42 & 0.731 & 0.951 & 16.25 & 0.790 & 0.979 \\
        \rowcolor{lightblue} w/ \textbf{TTCL} & \textbf{86.6} & \textbf{0.071} & 0.881 & \textbf{77.81} & \textbf{0.131} & 0.709 & \textbf{28.33} & \textbf{0.144} & 0.325 & \textbf{20.42} & \textbf{0.185} & 0.470 \\
        \# Improv & \color{softred}+7.44\% & \color{softred}+61.48\% & - & \color{softred}+31.05\% & \color{softred}+66.15\% & - & \color{softred}+38.78\% & \color{softred}+80.29\% & - & \color{softred}+25.64\% & \color{softred}+76.58\% & - \\
        \specialrule{1pt}{0.4ex}{0.4ex}
            \multicolumn{13}{c}{\textit{\textbf{Qwen3-8B}}} \\
        \midrule
        Base Model & 84.4 & 0.219 & 0.987 & 62.81 & 0.357 & 0.983 & 21.67 & 0.757 & 0.978 & 18.75 & 0.780 & 0.981 \\
        \rowcolor{lightblue} w/ \textbf{TTCL} & \textbf{89.0} & \textbf{0.090} & 0.769 & \textbf{80.63} & \textbf{0.145} & 0.844 & \textbf{30.83} & \textbf{0.177} & 0.392 & \textbf{26.67} & \textbf{0.113} & 0.303 \\
        \# Improv & \color{softred}+5.45\% & \color{softred}+59.17\% & - & \color{softred}+28.36\% & \color{softred}+59.30\% & - & \color{softred}+42.31\% & \color{softred}+76.63\% & - & \color{softred}+42.22\% & \color{softred}+85.53\% & - \\
        \specialrule{1pt}{0.4ex}{0.4ex}
            \multicolumn{13}{c}{\textit{\textbf{Llama-3.1-8B-Instruct}}} \\
        \midrule
        Base Model & 39.6 & 0.480 & 0.858 & 21.88 & 0.605 & 0.518 & 4.58 & 0.731 & 0.856 & 0.00 & 0.751 & 0.853 \\
        \rowcolor{lightblue} w/ \textbf{TTCL} & \textbf{58.0} & \textbf{0.208} & 0.622 & \textbf{40.94} & \textbf{0.261} & 0.617 & \textbf{9.17} & \textbf{0.112} & 0.284 & 0.00 & \textbf{0.036} & 0.114 \\
        \# Improv & \color{softred}+46.46\% & \color{softred}+56.68\% & - & \color{softred}+87.14\% & \color{softred}+56.80\% & - & \color{softred}+100.00\% & \color{softred}+84.69\% & - & 0.00\% & \color{softred}+95.25\% & - \\
        \bottomrule[1.6pt]
    \end{tabular}}
\end{table}

\begin{table}[t!]
    \centering
    \caption{\textbf{Brier score and mean confidence of TTCL for base models on factual QA benchmarks.} Improved results after TTCL are highlighted in \textbf{bold}, with positive relative improvements shown in \textcolor{softred}{red} color. As mean confidence is not a metric for which higher or lower values are better, we do not report its relative improvement and include it only as a reference alongside the other metrics.}
    \label{Tble:TTCL_Base_Fact_BrierConf}
    \vspace{-3mm}
    \resizebox{\textwidth}{!}{
    \begin{tabular}{l|ccc|ccc|ccc|ccc}
    \toprule[1.6pt]
        \multirow{2}{*}{\textbf{Methods}} & \multicolumn{3}{c}{\textbf{SimpleQA}} & \multicolumn{3}{c}{\textbf{NQ}} & \multicolumn{3}{c}{\textbf{HotpotQA}} & \multicolumn{3}{c}{\textbf{TriviaQA}} \\ \cmidrule(lr){2-4} \cmidrule(lr){5-7} \cmidrule(lr){8-10} \cmidrule(lr){11-13}
        ~ & \textbf{ACC $\uparrow$} & \textbf{Brier $\downarrow$} & \textbf{Conf} & \textbf{ACC $\uparrow$} & \textbf{Brier $\downarrow$} & \textbf{Conf} & \textbf{ACC $\uparrow$} & \textbf{Brier $\downarrow$} & \textbf{Conf} & \textbf{ACC $\uparrow$} & \textbf{Brier $\downarrow$} & \textbf{Conf} \\ 
        \midrule
        \multicolumn{13}{c}{\textit{\textbf{Qwen3-4B}}} \\
        \midrule
        Base Model & 1.40 & 0.875 & 0.927 & 8.63 & 0.855 & 0.961 & 9.56 & 0.795 & 0.927 & 15.63 & 0.772 & 0.954  \\
        \rowcolor{lightblue} w/ \textbf{TTCL} & \textbf{2.40} & \textbf{0.041} & 0.095 & \textbf{10.63} & \textbf{0.176} & 0.420 & \textbf{10.06} & \textbf{0.198} & 0.349 & \textbf{17.00} & \textbf{0.099} & 0.191 \\
        \# Improv & \color{softred}+71.43\% & \color{softred}+95.36\% & - & \color{softred}+23.25\% & \color{softred}+79.41\% & - & \color{softred}+5.20\% & \color{softred}+75.06\% & - & \color{softred}+8.80\% & \color{softred}+87.22\% & - \\
        \specialrule{1pt}{0.4ex}{0.4ex}
        \multicolumn{13}{c}{\textit{\textbf{Qwen3-8B}}} \\
        \midrule
        Base Model & 2.20 & 0.886 & 0.928 & 9.63 & 0.876 & 0.983 & 12.75 & 0.810 & 0.950 & 24.13 & 0.743 & 0.989 \\
        \rowcolor{lightblue} w/ \textbf{TTCL} & \textbf{3.50} & \textbf{0.184} & 0.387 & \textbf{14.81} & \textbf{0.154} & 0.338 & \textbf{15.31} & \textbf{0.240} & 0.483 & \textbf{26.50} & \textbf{0.216} & 0.492 \\
        \# Improv & \color{softred}+59.09\% & \color{softred}+79.29\% & - & \color{softred}+53.87\% & \color{softred}+82.39\% & - & \color{softred}+20.08\% & \color{softred}+70.36\% & - & \color{softred}+9.84\% & \color{softred}+70.91\% & - \\
        \specialrule{1pt}{0.4ex}{0.4ex}
        \multicolumn{13}{c}{\textit{\textbf{Llama-3.1-8B-Instruct}}} \\
        \midrule
        Base Model & 3.20 & 0.719 & 0.846 & 14.69 & 0.666 & 0.854 & 12.31 & 0.662 & 0.851 & 36.88 & 0.479 & 0.857 \\
        \rowcolor{lightblue} w/ \textbf{TTCL} & \textbf{5.80} & \textbf{0.111} & 0.278 & \textbf{23.31} & \textbf{0.203} & 0.065 & \textbf{20.13} & \textbf{0.163} & 0.098 & \textbf{50.19} & \textbf{0.352} & 0.169 \\
        \# Improv & \color{softred}+81.25\% & \color{softred}+84.51\% & - & \color{softred}+40.23\% & \color{softred}+69.55\% & - & \color{softred}+64.29\% & \color{softred}+75.34\% & - & \color{softred}+30.78\% & \color{softred}+26.44\% & - \\
        \bottomrule[1.6pt]
    \end{tabular}}
    \vspace{-3mm}
\end{table}

\subsection{More Benchmark Details}
\label{App:Benchmarks}
We conduct experiments on benchmarks mainly covering two tasks: mathematical reasoning and factual question answering. For mathematical reasoning, we include benchmarks below:
\begin{itemize}[left=2pt]
\item \textbf{MATH500}~\citep{lightman2024math500}: MATH500 is a commonly used benchmark for mathematics, which is a 500-question subset of the MATH dataset~\citep{hendrycksmath2021}, covering a broad range of topics including algebra, geometry, number theory, counting, and probability. 

\item \textbf{AMC}~\citep{yang2024qwen2p5math}: AMC contains problems drawn from the American Mathematics Competitions, which evaluate math problem-solving ability across topics such as algebra, geometry, combinatorics, and number theory. Compared with MATH500, AMC problems are typically shorter but still require non-trivial reasoning and careful derivation to obtain correct answer.

\item \textbf{AIME24/AIME25/AIME26}~\citep{aime24,aime25,aime26}: These benchmarks consist of problems from the 2024/2025/2026 American Invitational Mathematics Examination, respectively. These problems are substantially more challenging than standard school-level mathematics questions and typically require several intermediate reasoning steps, making the benchmark useful for evaluating both reasoning accuracy and confidence on difficult mathematical problems.

\end{itemize}

\begin{table}[t]
    \centering
    \caption{\textbf{In-domain performance of Brier score and mean confidence of TTCL on RLCR-calibrated models.} Improved results are highlighted in \textbf{bold}, with positive relative gains shown in \textcolor{softred}{red}. Mean confidence serves as a reference for analysis alongside other metrics. Since neither higher nor lower mean confidence is inherently better, we do not report its relative improvement.}
    \label{Tble:TTCL_InDomain_BrierConf}
    \vspace{-3mm}
    \resizebox{\textwidth}{!}{
    \begin{tabular}{l|ccc|ccc|ccc|ccc|ccc}
    \toprule[1.6pt]
        \multirow{2}{*}{\textbf{Math-to-Math}} & \multicolumn{3}{c}{\textbf{MATH500}} & \multicolumn{3}{c}{\textbf{AMC}} & \multicolumn{3}{c}{\textbf{AIME24}} & \multicolumn{3}{c}{\textbf{AIME25}} & \multicolumn{3}{c}{\textbf{AIME26}} \\ 
        \cmidrule(lr){2-4} \cmidrule(lr){5-7} \cmidrule(lr){8-10} \cmidrule(lr){11-13} \cmidrule(lr){14-16}
        ~ & \textbf{ACC $\uparrow$} & \textbf{Brier $\downarrow$} & \textbf{Conf} & \textbf{ACC $\uparrow$} & \textbf{Brier $\downarrow$} & \textbf{Conf} & \textbf{ACC $\uparrow$} & \textbf{Brier $\downarrow$} & \textbf{Conf} & \textbf{ACC $\uparrow$} & \textbf{Brier $\downarrow$} & \textbf{Conf} & \textbf{ACC $\uparrow$} & \textbf{Brier $\downarrow$} & \textbf{Conf}  \\ 
        \midrule
        \multicolumn{16}{c}{\textit{\textbf{RLCR-calibrated Qwen3-4B on DAPO-Math-14K}}} \\
        \midrule
        Base Model  & 80.6 & 0.184 & 0.984 & 59.38 & 0.387 & 0.984 & 20.42 & 0.731 & 0.951 & 16.25 & 0.790 & 0.979 & 15.42 & 0.803 & 0.974 \\
        \rowcolor{lightgray} RLCR Model & 89.0 & 0.090 & 0.772 & 79.38 & 0.134 & 0.701 & 29.17 & 0.228 & 0.557 & 24.17 & 0.238 & 0.566 & 23.33 & 0.257 & 0.568 \\
        \rowcolor{lightblue} w/ \textbf{TTCL} & \textbf{90.4} & \textbf{0.069} & 0.913 & \textbf{88.13} & \textbf{0.078 }& 0.804 & \textbf{42.08} & \textbf{0.160} & 0.453 & \textbf{25.00} & \textbf{0.112} & 0.360 & \textbf{24.58} & \textbf{0.152} & 0.397 \\
        \# Improv & \color{softred}+1.57\% & \color{softred}+23.76\% & - & \color{softred}+11.02\% & \color{softred}+41.66\% & - & \color{softred}+44.29\% & \color{softred}+29.60\% & - & \color{softred}+3.45\% & \color{softred}+52.98\% & - & \color{softred}+5.36\% & \color{softred}+40.71\% & - \\
        \specialrule{1pt}{0.4ex}{0.4ex}
        \multicolumn{16}{c}{\textit{\textbf{RLCR-calibrated Qwen3-8B on DAPO-Math-14K}}} \\
        \midrule
        Base Model & 84.4 & 0.219 & 0.987 & 62.81 & 0.357 & 0.983 & 21.67 & 0.757 & 0.978 & 18.75 & 0.780 & 0.981 & 13.75 & 0.834 & 0.981 \\
        \rowcolor{lightgray} RLCR Model & 90.0 & 0.052 & 0.860 & 80.63 & 0.092 & 0.754 & 35.42 & 0.106 & 0.441 & 28.33 & 0.069 & 0.388 & 27.50 & 0.127 & 0.416 \\
        \rowcolor{lightblue} w/ \textbf{TTCL} & \textbf{91.6} & \textbf{0.035} & 0.881 & \textbf{89.38} & \textbf{0.086} & 0.811 & \textbf{45.42} & \textbf{0.101} & 0.425 & \textbf{30.42} & \textbf{0.063} & 0.390 & \textbf{28.33} & \textbf{0.101} & 0.385 \\
        \# Improv & \color{softred}+1.78\% & \color{softred}+31.74\% & - & \color{softred}+10.85\% & \color{softred}+6.19\% & - & \color{softred}+28.24\% & \color{softred}+4.38\% & - & \color{softred}+7.35\% & \color{softred}+8.73\% & - & \color{softred}+3.03\% & \color{softred}+20.71\% & - \\
        \bottomrule[1.6pt]
    \end{tabular}}
    \vspace{-2mm}
\end{table}

\begin{table}[t!]
    \centering
    \caption{\textbf{Out-of-domain performance of Brier score and mean confidence of TTCL on RLCR-calibrated models.} Improved results are highlighted in \textbf{bold}, with positive relative gains shown in \textcolor{softred}{red}. Mean confidence serves as a reference for analysis alongside other metrics. Since neither higher nor lower mean confidence is inherently better, we do not report its relative improvement.}
    \label{Tble:TTCL_OutDomain_BrierConf}
    \vspace{-3mm}
    \resizebox{\textwidth}{!}{
    \begin{tabular}{l|ccc|ccc|ccc|ccc|ccc}
    \toprule[1.6pt]
        \multirow{2}{*}{\textbf{Math-to-FactQA}} & \multicolumn{3}{c}{\textbf{SimpleQA}} & \multicolumn{3}{c}{\textbf{ChineseSimpleQA}} & \multicolumn{3}{c}{\textbf{NQ}} & \multicolumn{3}{c}{\textbf{HotpotQA}} & \multicolumn{3}{c}{\textbf{TriviaQA}} \\ \cmidrule(lr){2-4} \cmidrule(lr){5-7} \cmidrule(lr){8-10} \cmidrule(lr){11-13} \cmidrule(lr){14-16}
        ~ & \textbf{ACC $\uparrow$} & \textbf{Brier $\downarrow$} & \textbf{Conf} & \textbf{ACC $\uparrow$} & \textbf{Brier $\downarrow$} & \textbf{Conf} & \textbf{ACC $\uparrow$} & \textbf{Brier $\downarrow$} & \textbf{Conf} & \textbf{ACC $\uparrow$} & \textbf{Brier $\downarrow$} & \textbf{Conf} & \textbf{ACC $\uparrow$} & \textbf{Brier $\downarrow$} & \textbf{Conf} \\  
        \midrule
            \multicolumn{16}{c}{\textit{\textbf{RLCR-calibrated Qwen3-4B on DAPO-Math-14K}}} \\
        \midrule
        Base Model & 1.40 & 0.875 & 0.927 & 12.13 & 0.801 & 0.950 & 8.63 & 0.855 & 0.961 & 9.56 & 0.795 & 0.927 & 15.63 & 0.772 & 0.954 \\
        \rowcolor{lightgray} RLCR Model & 2.00 & 0.431 & 0.643 & 13.00 & 0.472 & 0.730 & 11.00 & 0.518 & 0.762 & 11.13 & 0.413 & 0.671 & 19.38 & 0.456 & 0.745 \\
        \rowcolor{lightblue} w/ \textbf{TTCL} & \textbf{2.70} & \textbf{0.134} & 0.304 & \textbf{14.03} & \textbf{0.204} & 0.447 & \textbf{12.81} & \textbf{0.159} & 0.356 & \textbf{13.50} & \textbf{0.176} & 0.428 & \textbf{20.38} & \textbf{0.207} & 0.494 \\
        \# Improv & \color{softred}+35.00\% & \color{softred}+68.92\% & - & \color{softred}+7.95\% & \color{softred}+56.77\% & - & \color{softred}+16.48\% & \color{softred}+69.35\% & - & \color{softred}+21.35\% & \color{softred}+57.56\% & - & \color{softred}+5.16\% & \color{softred}+54.55\% & - \\
        \specialrule{1pt}{0.4ex}{0.4ex}
            \multicolumn{16}{c}{\textit{\textbf{RLCR-calibrated Qwen3-8B on DAPO-Math-14K}}} \\
        \midrule
        Base Model & 2.20 & 0.886 & 0.928 & 15.77 & 0.815 & 0.978 & 9.63 & 0.876 & 0.983 & 12.75 & 0.810 & 0.950 & 24.13 & 0.743 & 0.989  \\
        \rowcolor{lightgray} RLCR Model & 2.80 & 0.656 & 0.796 & 18.17 & 0.571 & 0.824 & 12.56 & 0.666 & 0.862 & 14.13 & 0.553 & 0.777 & 25.69 & 0.559 & 0.862 \\
        \rowcolor{lightblue} w/ \textbf{TTCL} & \textbf{3.90} & \textbf{0.136} & 0.309 & \textbf{20.83} & \textbf{0.201} & 0.434 & \textbf{15.00} & \textbf{0.183} & 0.440 & \textbf{18.38} & \textbf{0.219} & 0.476 & \textbf{29.31} & \textbf{0.344} & 0.720 \\
        \# Improv & \color{softred}+39.29\% & \color{softred}+79.28\% & - & \color{softred}+14.68\% & \color{softred}+64.75\% & - & \color{softred}+19.40\% & \color{softred}+72.59\% & - & \color{softred}+30.09\% & \color{softred}+60.46\% & - & \color{softred}+14.11\% & \color{softred}+38.44\% & - \\
        \toprule[1.6pt]
        \multirow{2}{*}{\textbf{FactQA-to-Math}} & \multicolumn{3}{c}{\textbf{MATH500}} & \multicolumn{3}{c}{\textbf{AMC}} & \multicolumn{3}{c}{\textbf{AIME24}} & \multicolumn{3}{c}{\textbf{AIME25}} & \multicolumn{3}{c}{\textbf{AIME26}} \\ \cmidrule(lr){2-4} \cmidrule(lr){5-7} \cmidrule(lr){8-10} \cmidrule(lr){11-13} \cmidrule(lr){14-16}
        ~ & \textbf{ACC $\uparrow$} & \textbf{Brier $\downarrow$} & \textbf{Conf} & \textbf{ACC $\uparrow$} & \textbf{Brier $\downarrow$} & \textbf{Conf} & \textbf{ACC $\uparrow$} & \textbf{Brier $\downarrow$} & \textbf{Conf} & \textbf{ACC $\uparrow$} & \textbf{Brier $\downarrow$} & \textbf{Conf} & \textbf{ACC $\uparrow$} & \textbf{Brier $\downarrow$} & \textbf{Conf} \\  
        \midrule
            \multicolumn{16}{c}{\textit{\textbf{RLCR-calibrated Qwen3-4B on FactQA-10K}}} \\
        \midrule
        Base Model & 80.6 & 0.184 & 0.984 & 59.38 & 0.387 & 0.984 & 20.42 & 0.731 & 0.951 & 16.25 & 0.790 & 0.979 & 15.42 & 0.803 & 0.974 \\
        \rowcolor{lightgray} RLCR Model & 79.8 & 0.126 & 0.726 & 64.38 & 0.182 & 0.661 & 23.75 & 0.268 & 0.536 & 18.75 & 0.235 & 0.511 & 14.58 & 0.263 & 0.485 \\
        \rowcolor{lightblue} w/ \textbf{TTCL} & \textbf{85.2} & \textbf{0.075} & 0.827 & \textbf{80.00} & \textbf{0.119} & 0.749 & \textbf{25.00} & \textbf{0.153} & 0.333 & \textbf{24.17} & \textbf{0.187} & 0.354 & \textbf{16.67} & \textbf{0.172} & 0.371 \\
        \# Improv & \color{softred}+6.77\% & \color{softred}+40.26\% & - & \color{softred}+24.27\% & \color{softred}+34.58\% & - & \color{softred}+5.26\% & \color{softred}+42.80\% & - & \color{softred}+28.89\% & \color{softred}+20.57\% & - & \color{softred}+14.29\% & \color{softred}+34.59\% & - \\
        \bottomrule[1.6pt]
    \end{tabular}}
\end{table}

For factual question answering, we include benchmarks below:
\begin{itemize}[left=2pt]
\item \textbf{SimpleQA}~\citep{wei2024simpleqa,haas2025simpleqaVerified}: SimpleQA is a challenging factual question answering benchmark proposed by OpenAI, consisting of short fact-seeking questions. In this work, we use the 1,000-question SimpleQA-Verified version introduced by~\citet{haas2025simpleqaVerified}.

\item \textbf{ChineseSimpleQA}~\citep{he2025chinese}: ChineseSimpleQA extends factual question answering evaluation to Chinese and covers a diverse range of knowledge domains. We use it to evaluate calibration and reasoning performance across different language settings.

\item \textbf{NQ}~\citep{kwiatkowski2019natural}: NQ is a large-scale open-domain question answering benchmark constructed from real user queries with Wikipedia-based answers. Its questions cover diverse factual topics and formats. We randomly sample 200 examples from the test set for evaluation.

\item \textbf{HotpotQA}~\citep{yang2018hotpotqa}: HotpotQA is a multi-hop question answering benchmark that requires combining information from multiple supporting facts or documents. We randomly sample 200 examples from the validation set for evaluation.

\item \textbf{TriviaQA}~\citep{joshi2017triviaqa}: TriviaQA is a large-scale factual question answering benchmark constructed from trivia questions paired with independently collected evidence documents. We randomly sample 200 examples from the test set as our evaluation set.
\end{itemize}

\begin{figure}[t]
  \centering
  \includegraphics[width=0.24\textwidth]{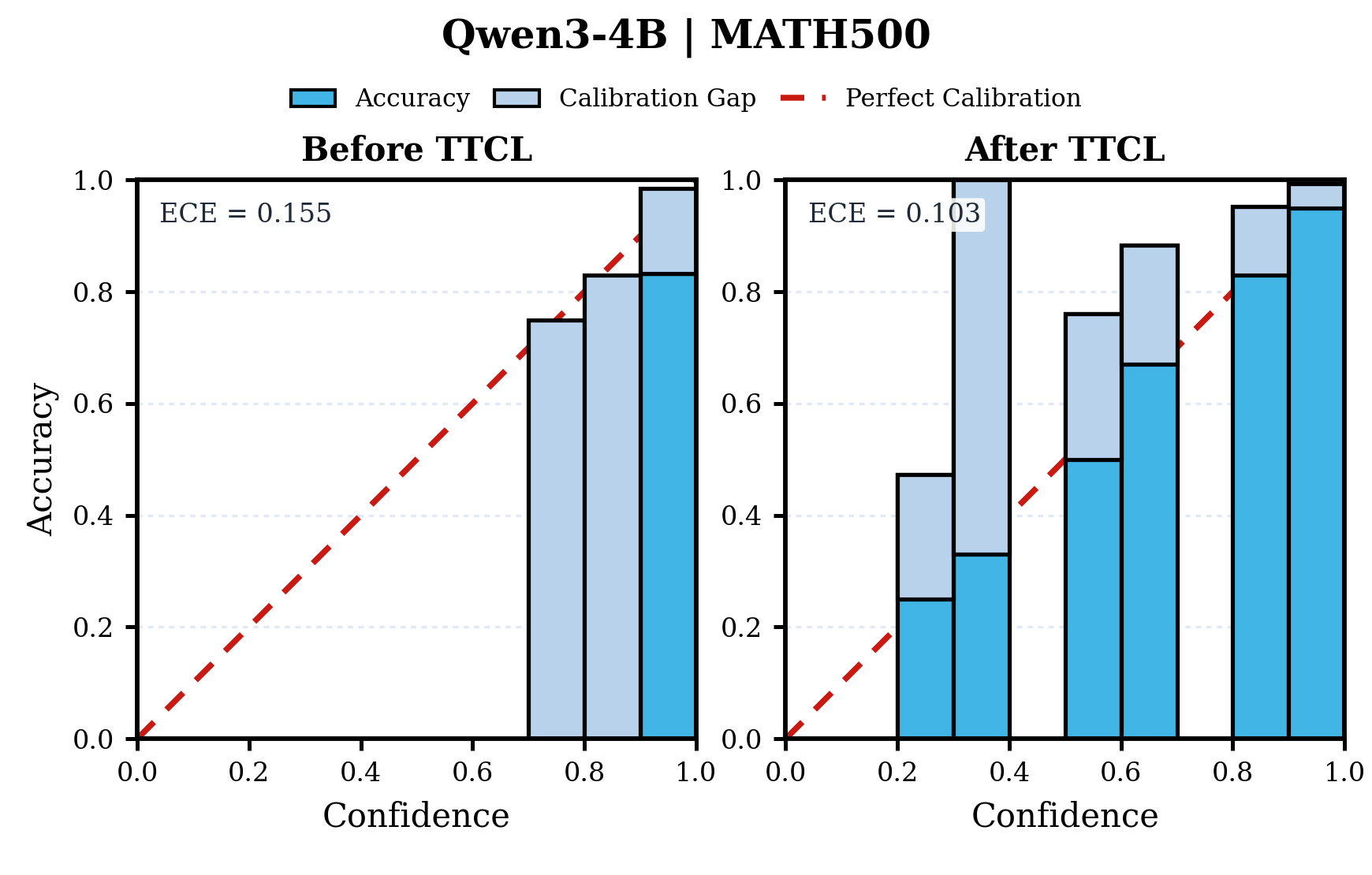}
  \includegraphics[width=0.24\textwidth]{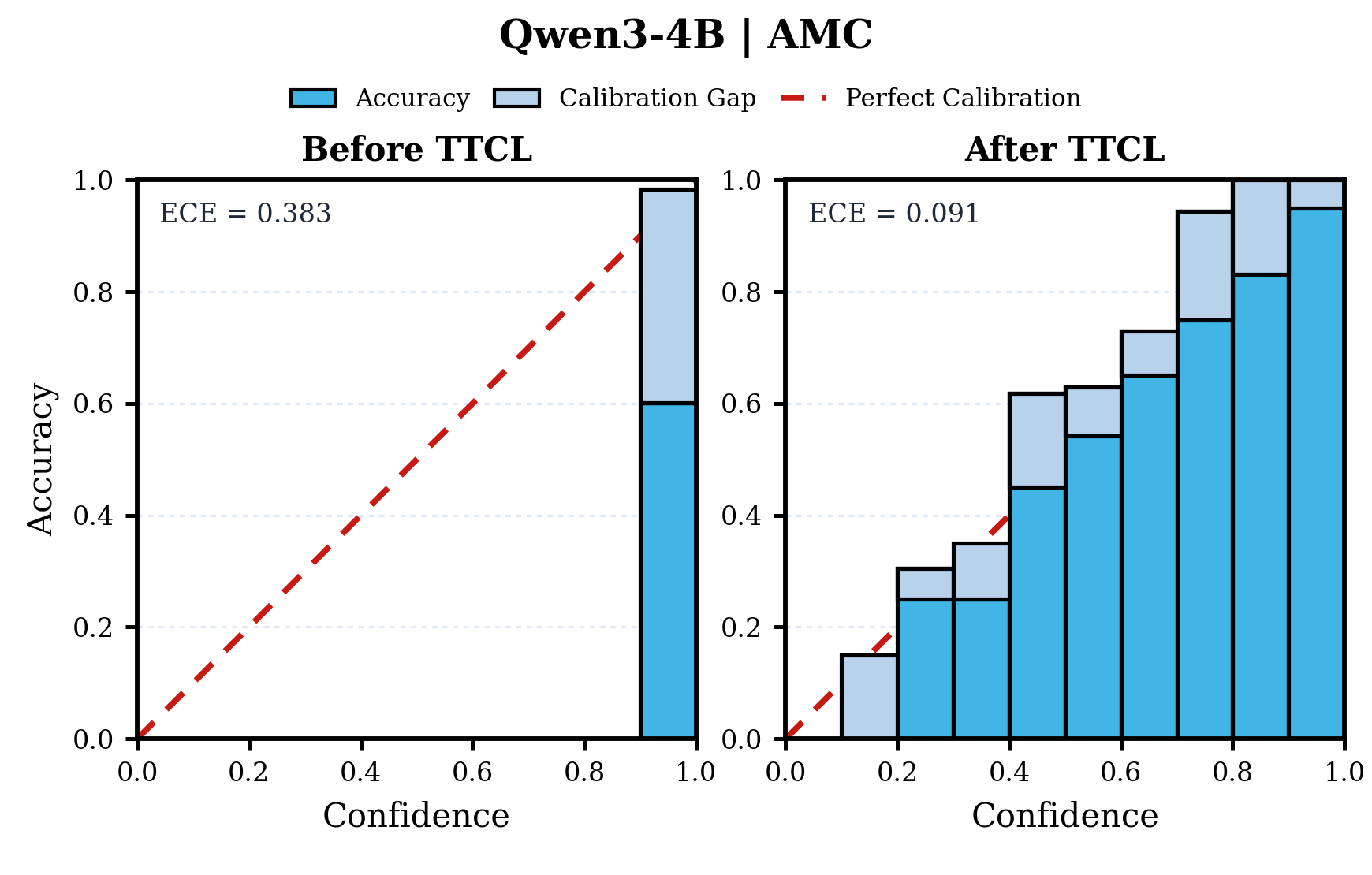}
  \includegraphics[width=0.24\textwidth]{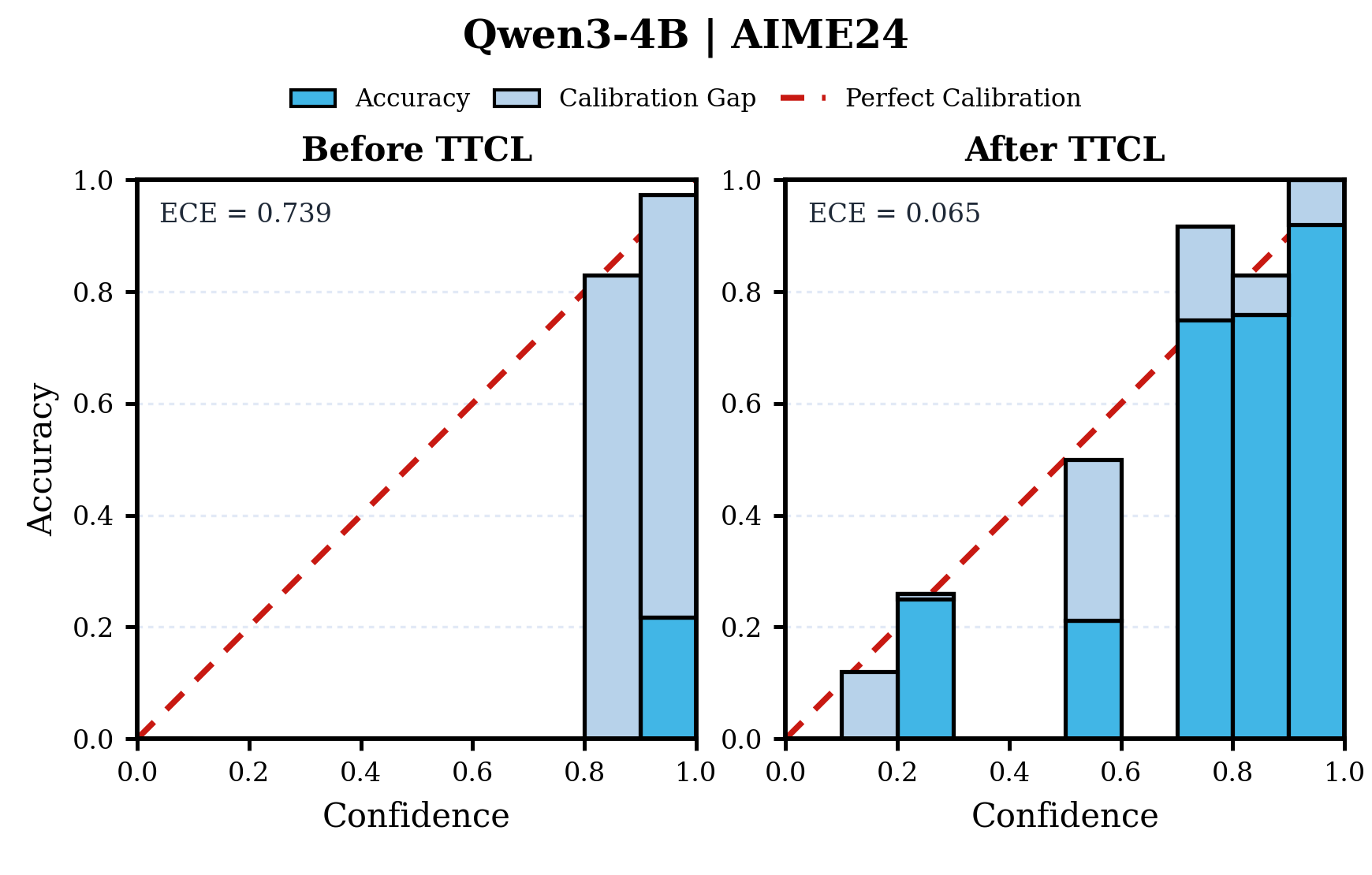} 
  \includegraphics[width=0.24\textwidth]{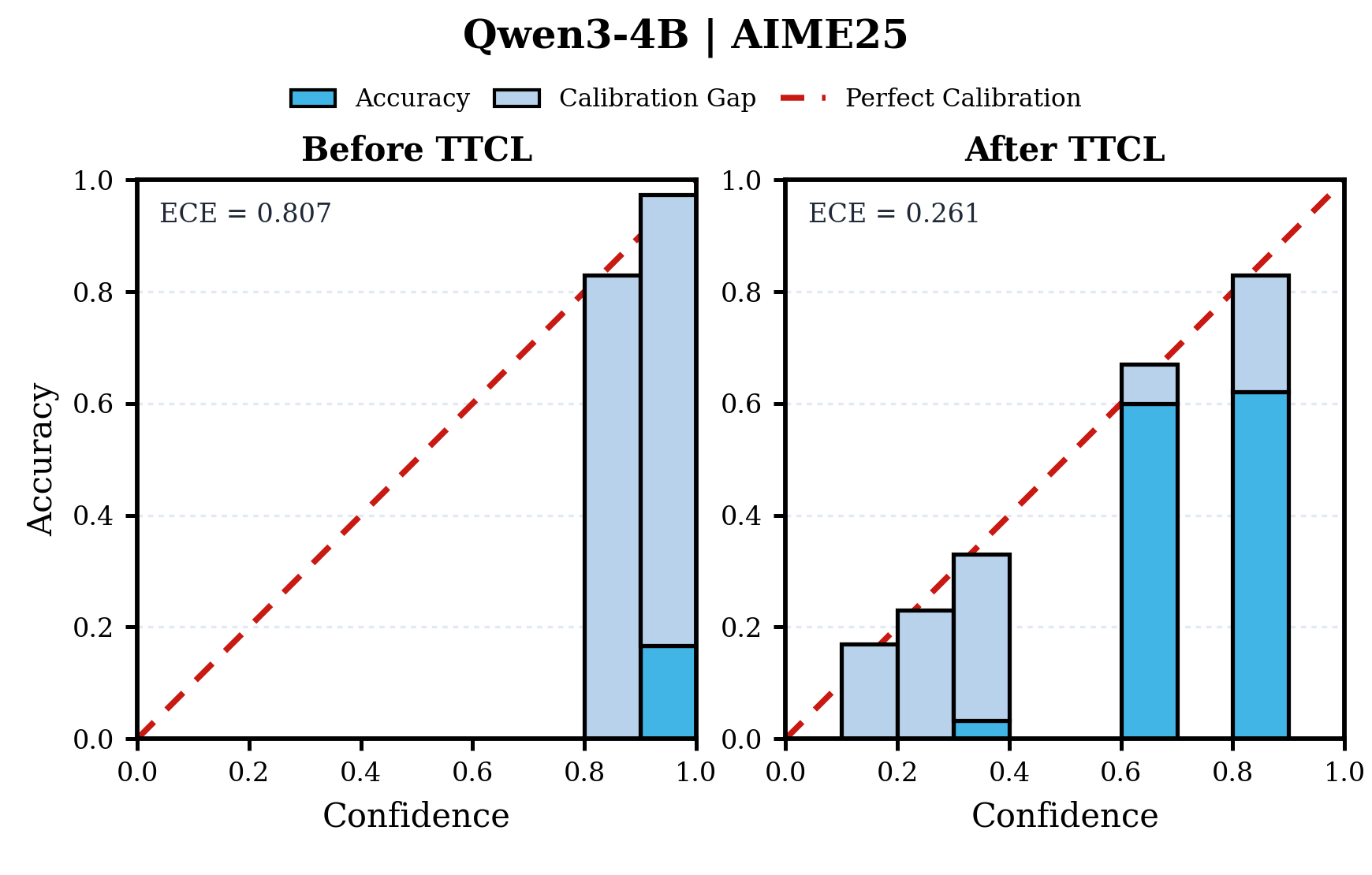} \\
  \includegraphics[width=0.24\textwidth]{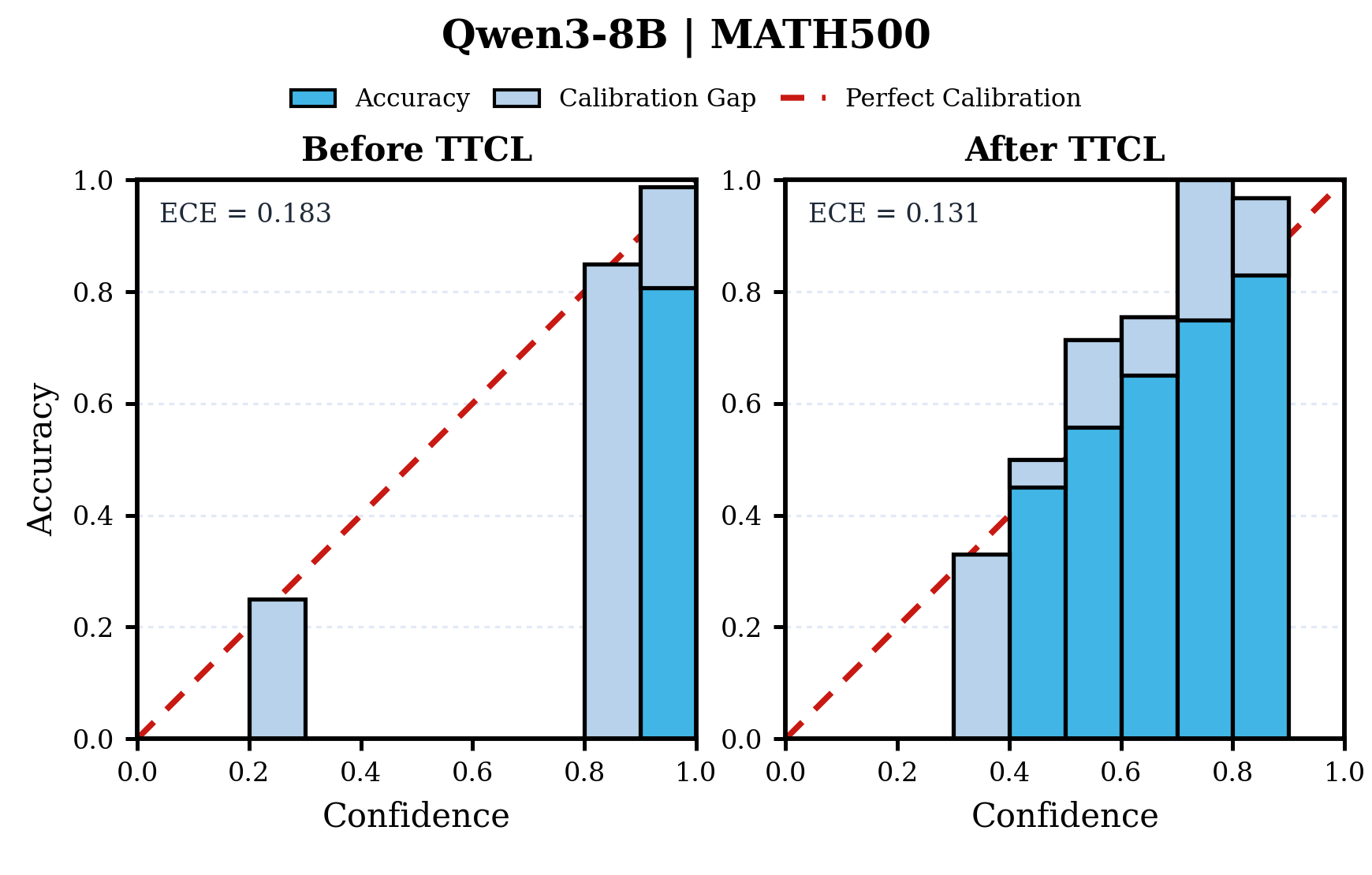}
  \includegraphics[width=0.24\textwidth]{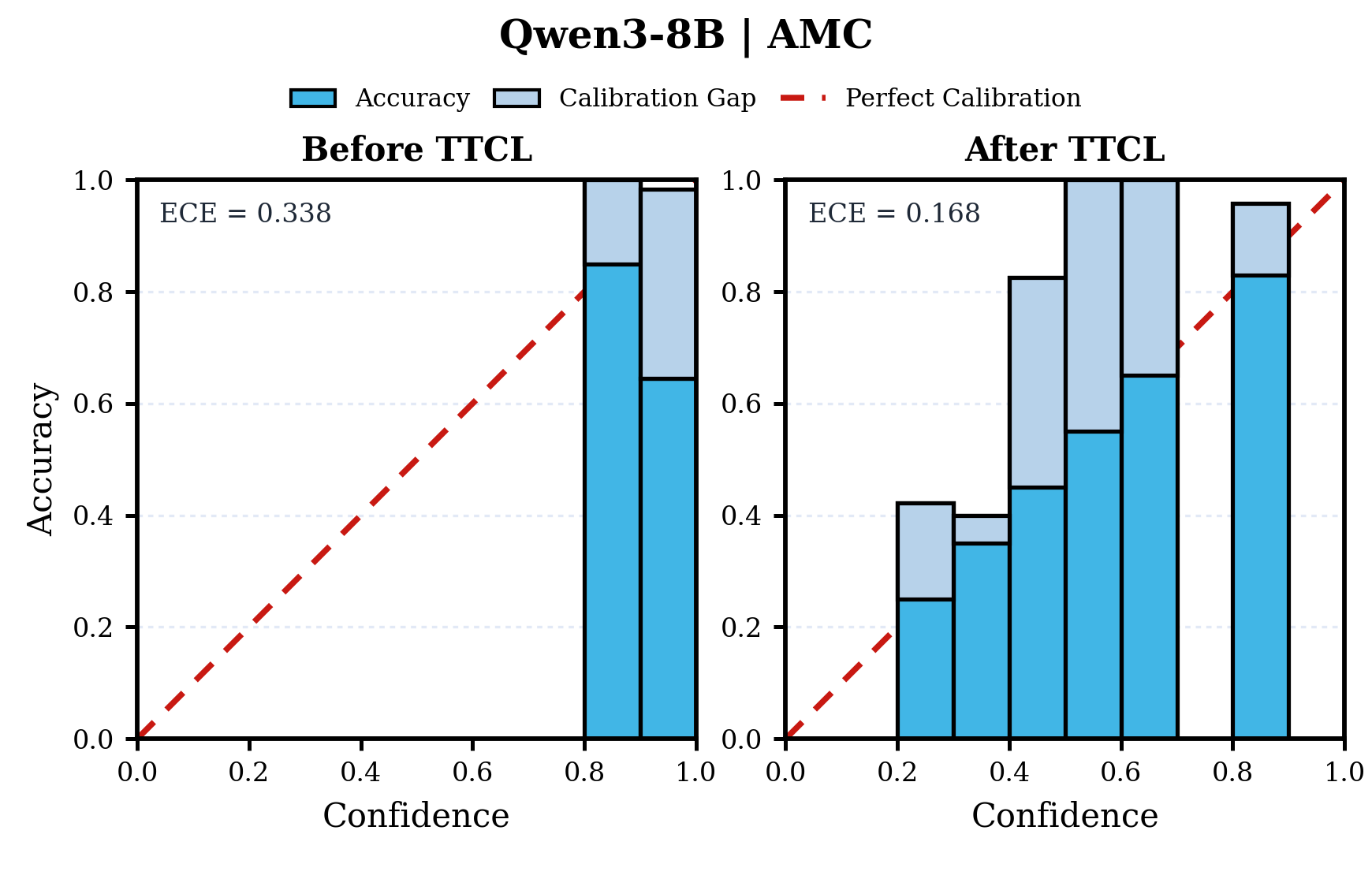} 
  \includegraphics[width=0.24\textwidth]{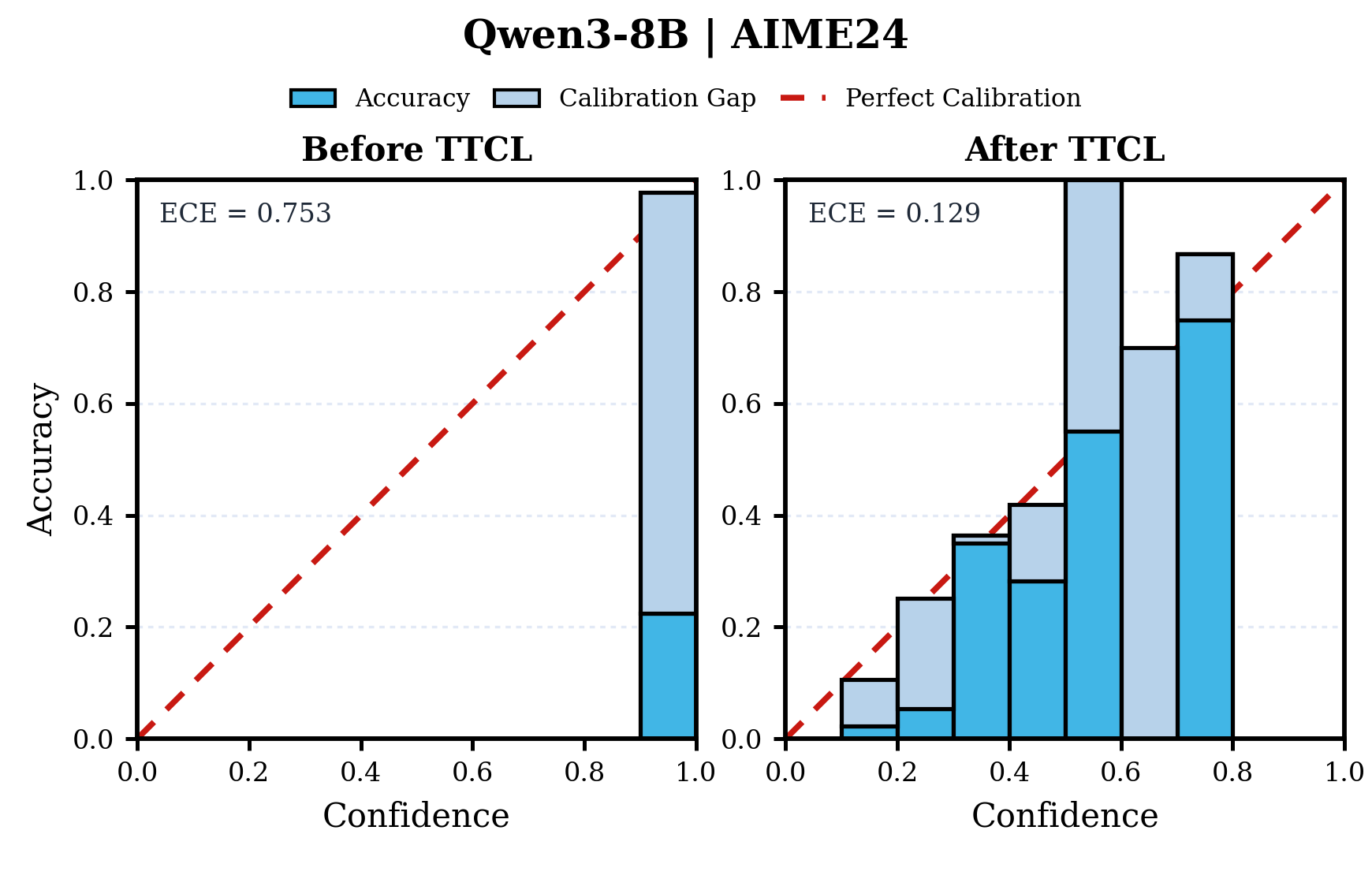}
  \includegraphics[width=0.24\textwidth]{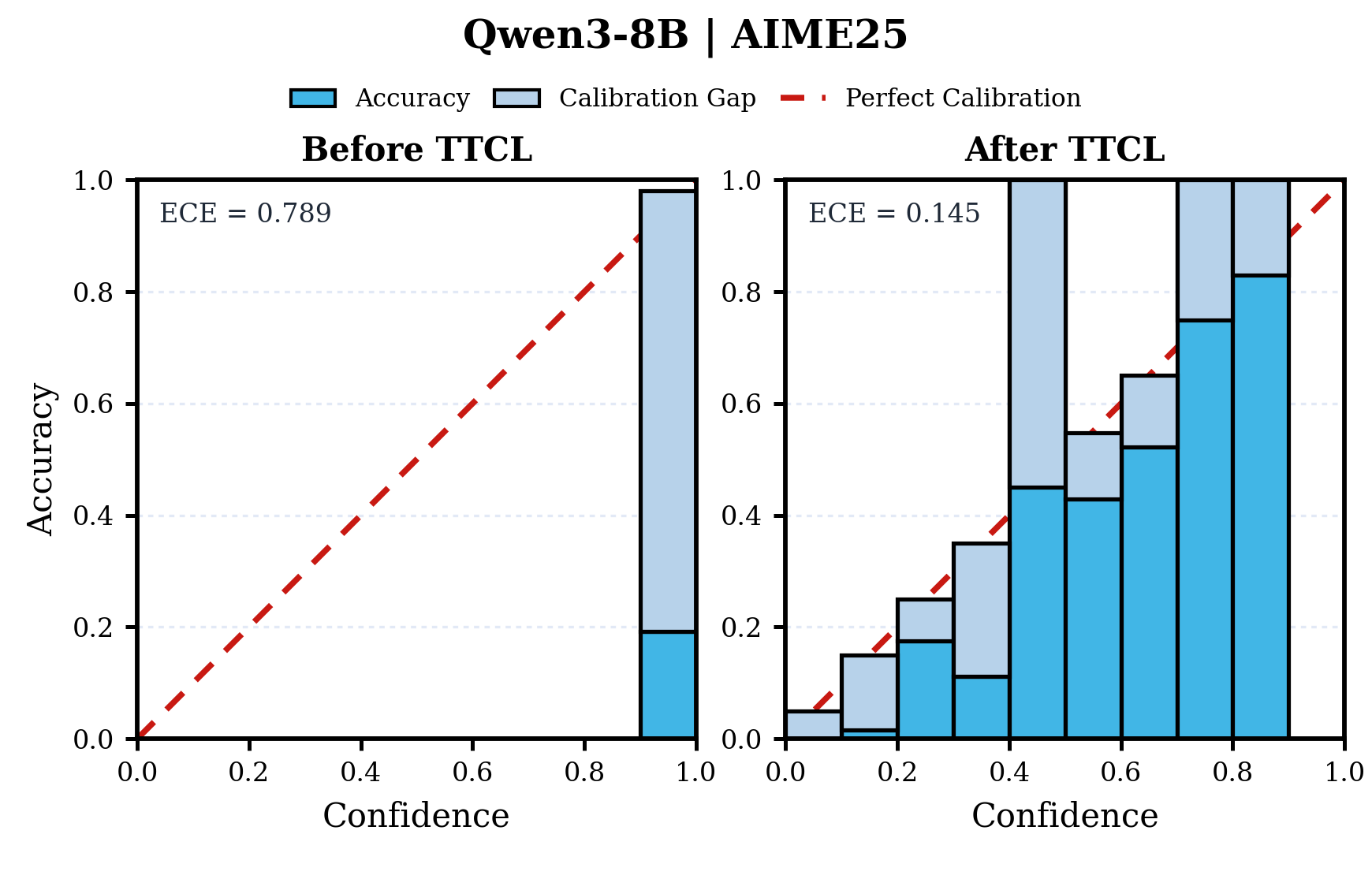} \\
  \includegraphics[width=0.24\textwidth]{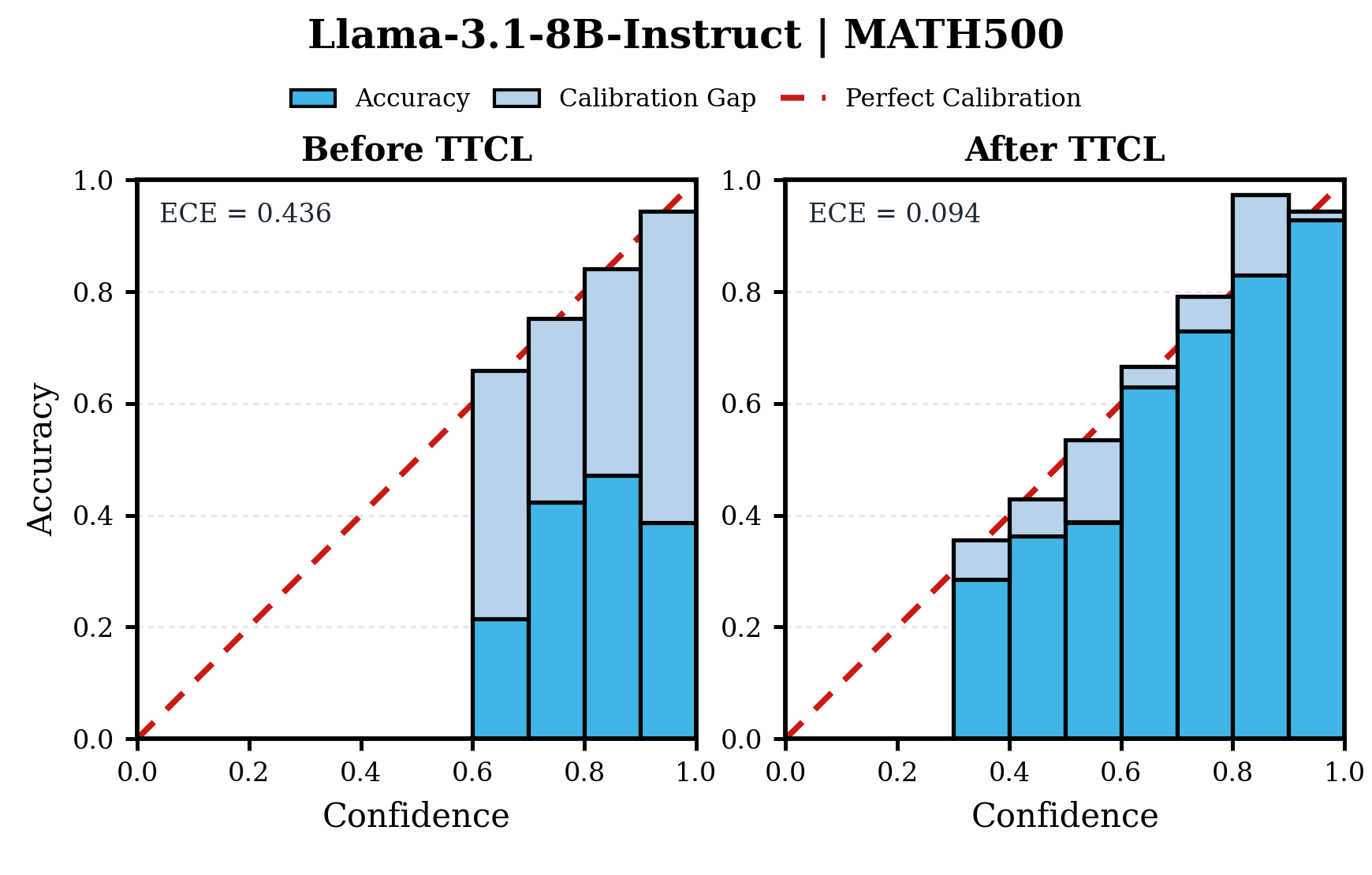}
  \includegraphics[width=0.24\textwidth]{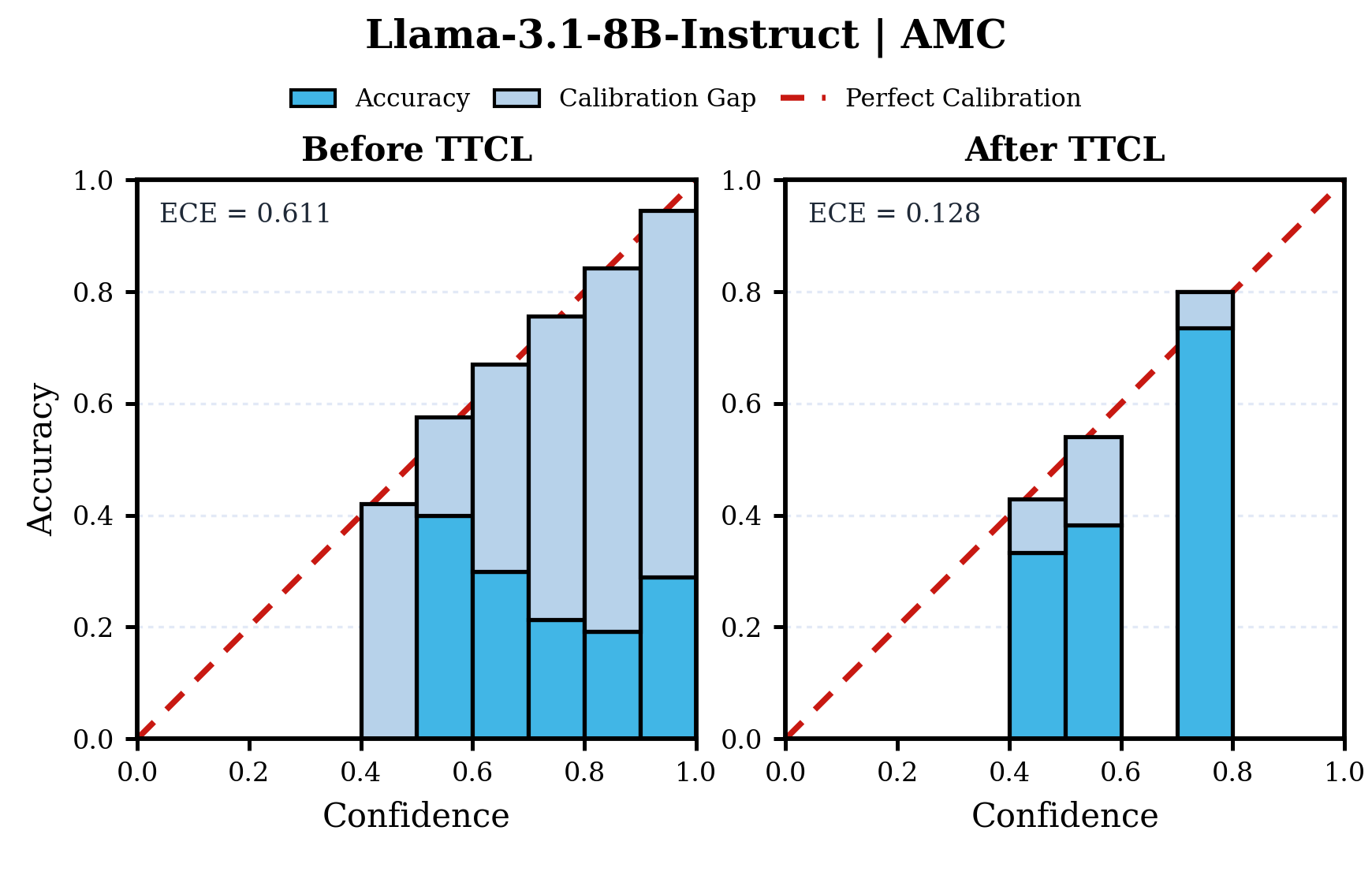}
  \includegraphics[width=0.24\textwidth]{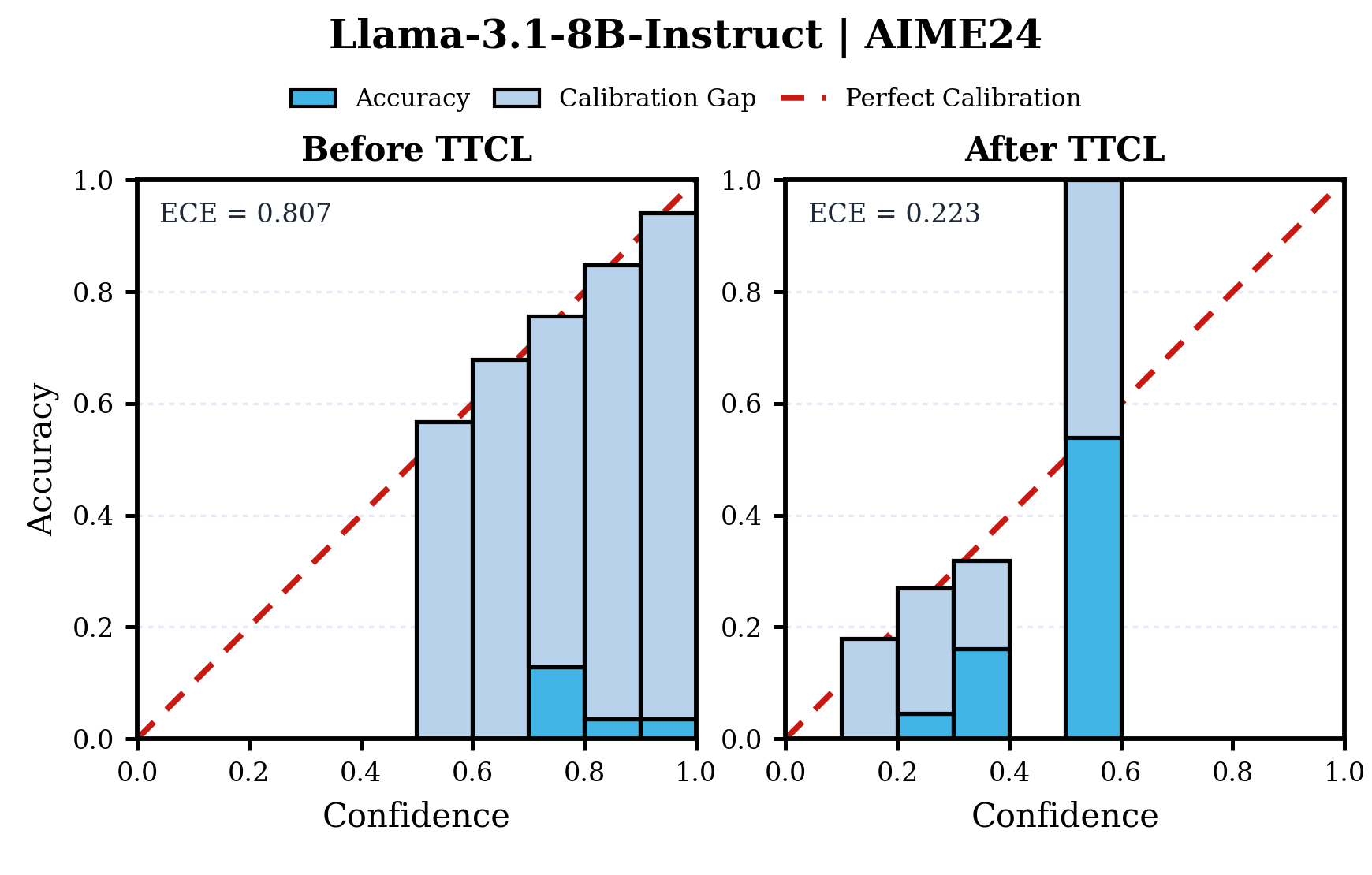}
  \includegraphics[width=0.24\textwidth]{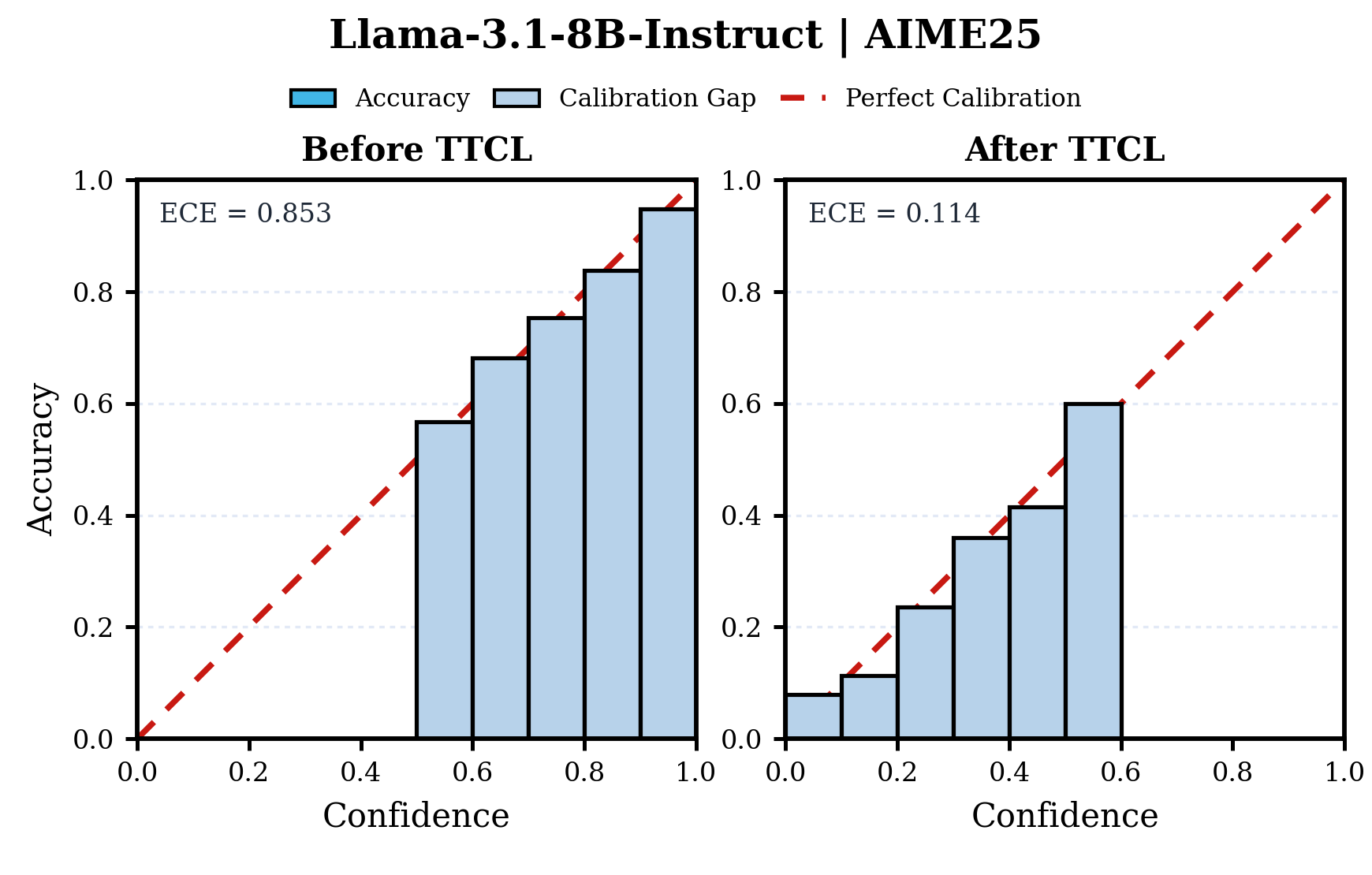}
  \vspace{-2mm}
  \caption{\textbf{Visualization of calibration effects before and after TTCL} on base models for mathematical benchmarks. TTCL effectively calibrates models to the target benchmarks, substantially narrowing the gap between confidence and empirical accuracy. \textcolor{accuracyblue}{Blue bars} denote empirical accuracy within each confidence bin, while \textcolor{calibrationgapblue}{light blue bars} indicate the deviation from perfect calibration. For Llama-3.1-8B-Instruct on AIME25, no \textcolor{accuracyblue}{accuracy bars} is shown because its empirical accuracy is 0.}
  \label{Fig:TTCL_BaseMath_Calibration_Visualize}
\end{figure}

\begin{figure}[t!]
  \centering
  \includegraphics[width=0.24\textwidth]{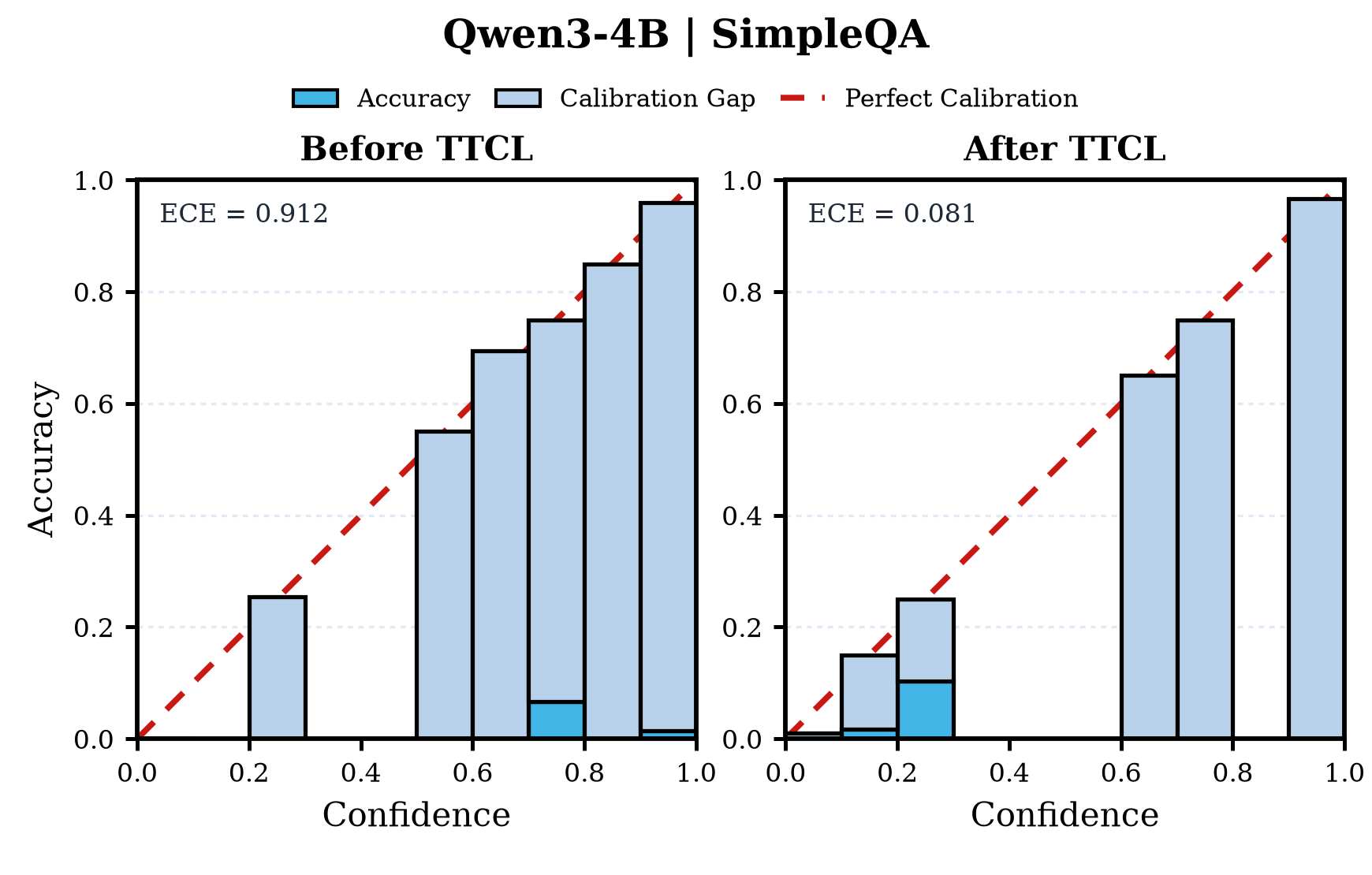}
  \includegraphics[width=0.24\textwidth]{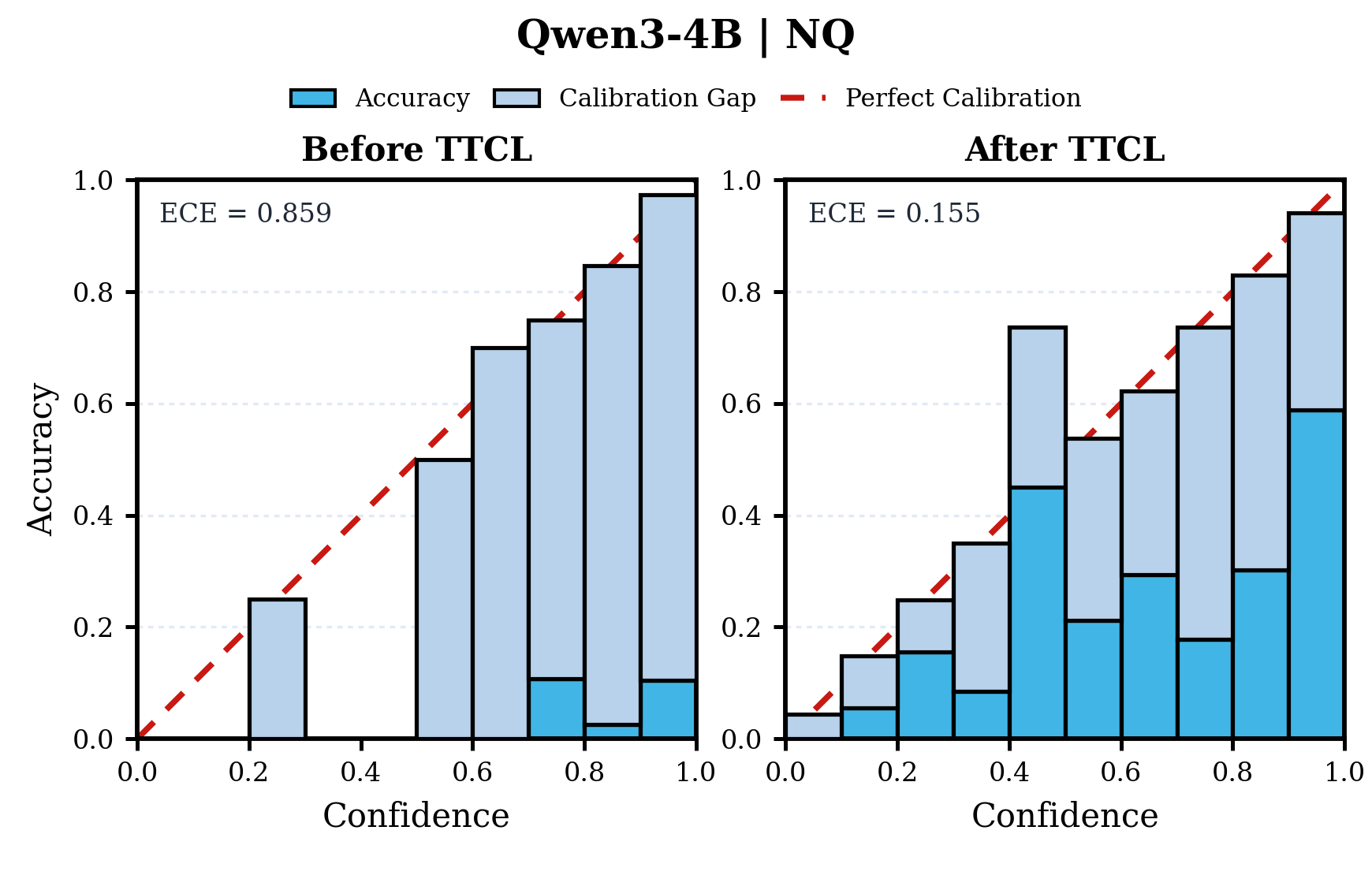}
  \includegraphics[width=0.24\textwidth]{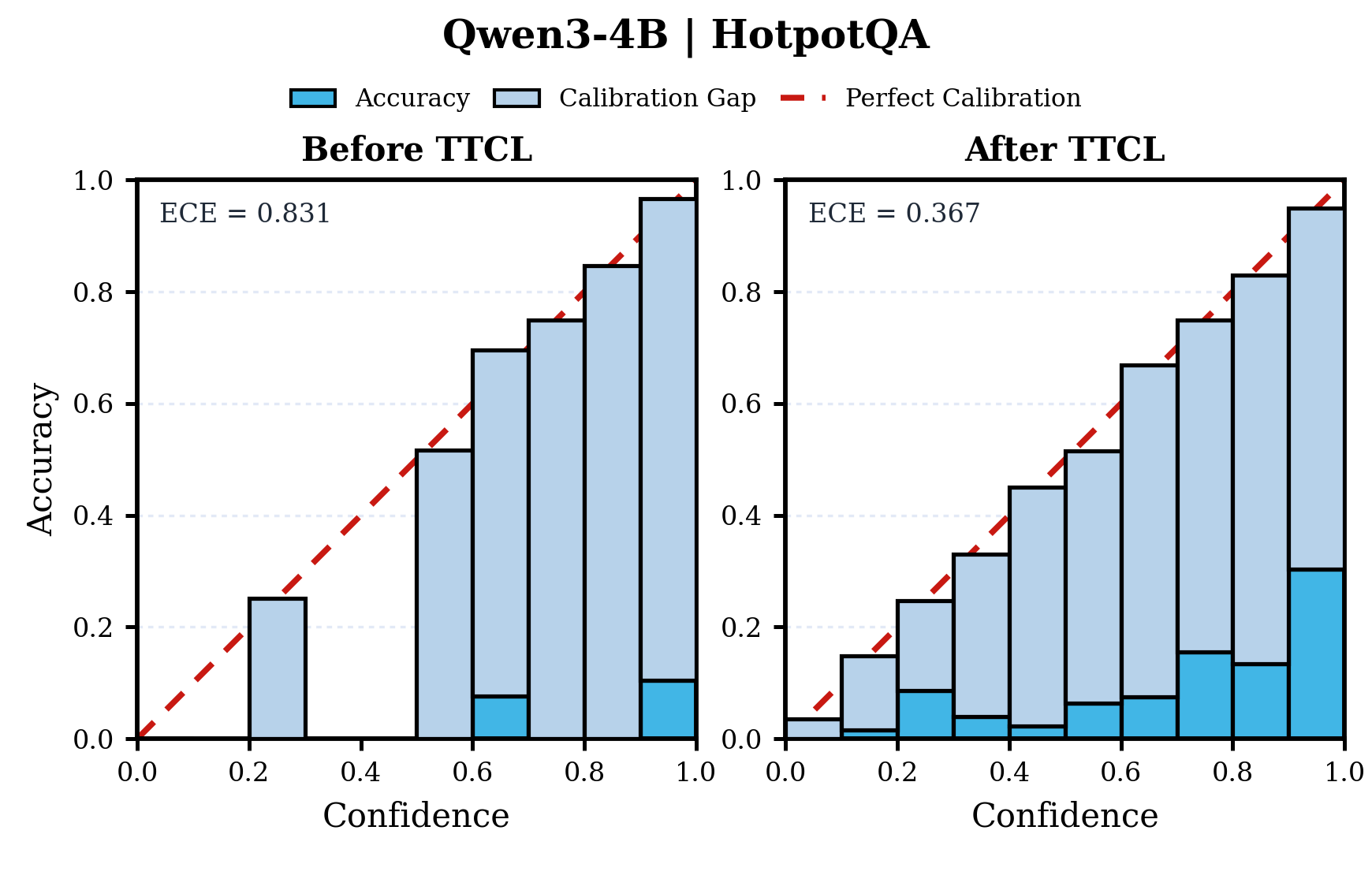}
  \includegraphics[width=0.24\textwidth]{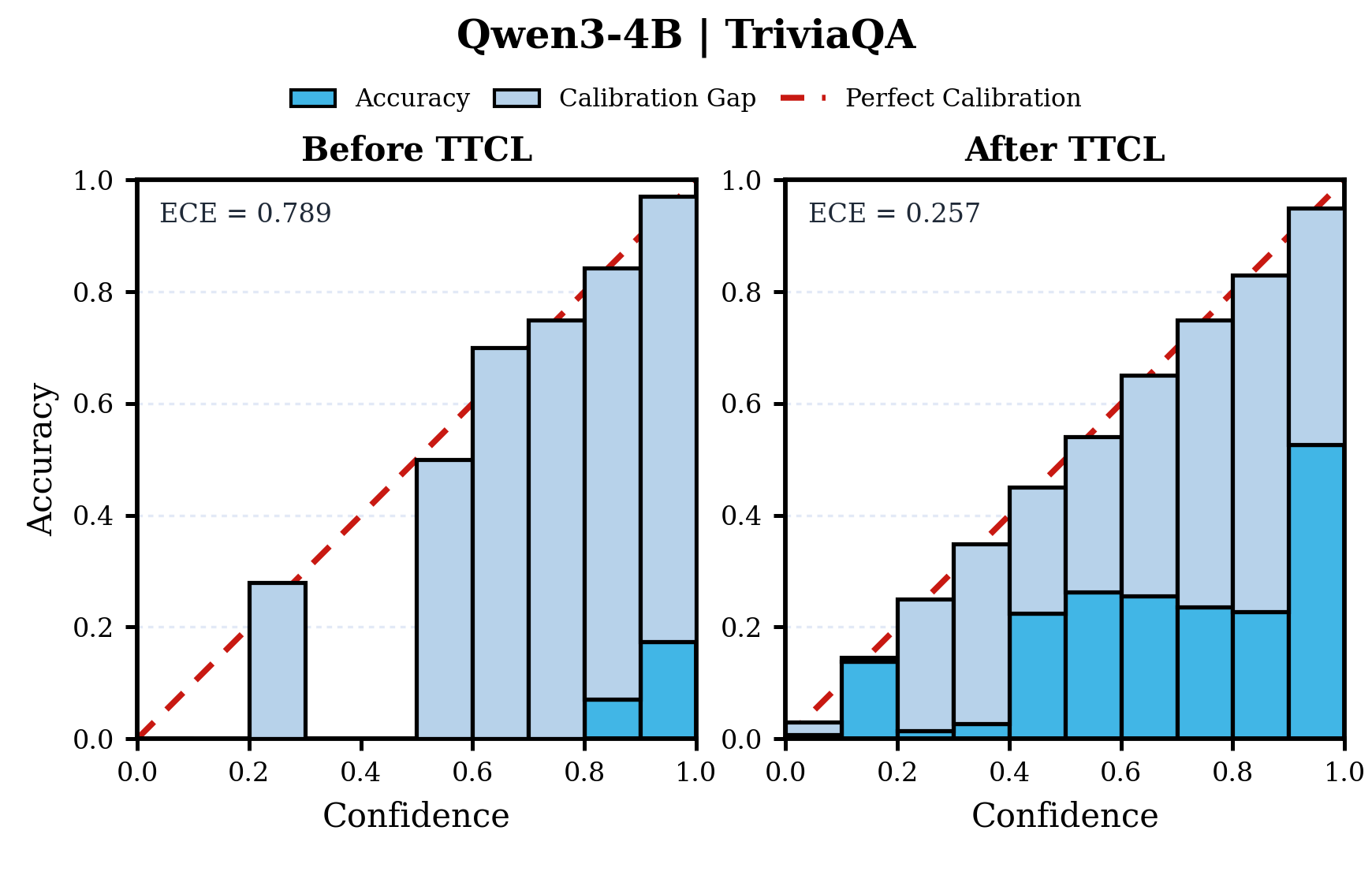} \\
  \includegraphics[width=0.24\textwidth]{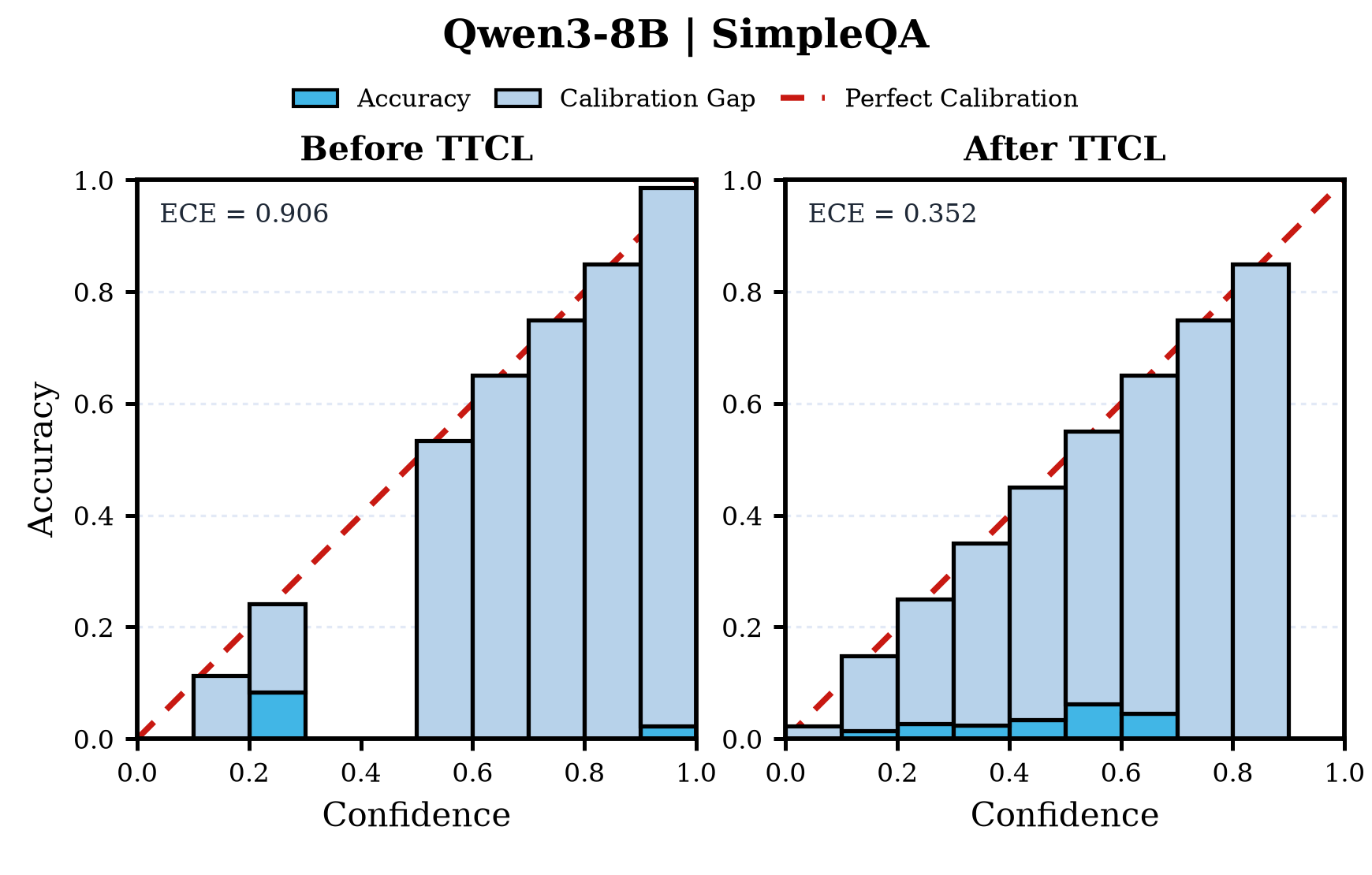}
  \includegraphics[width=0.24\textwidth]{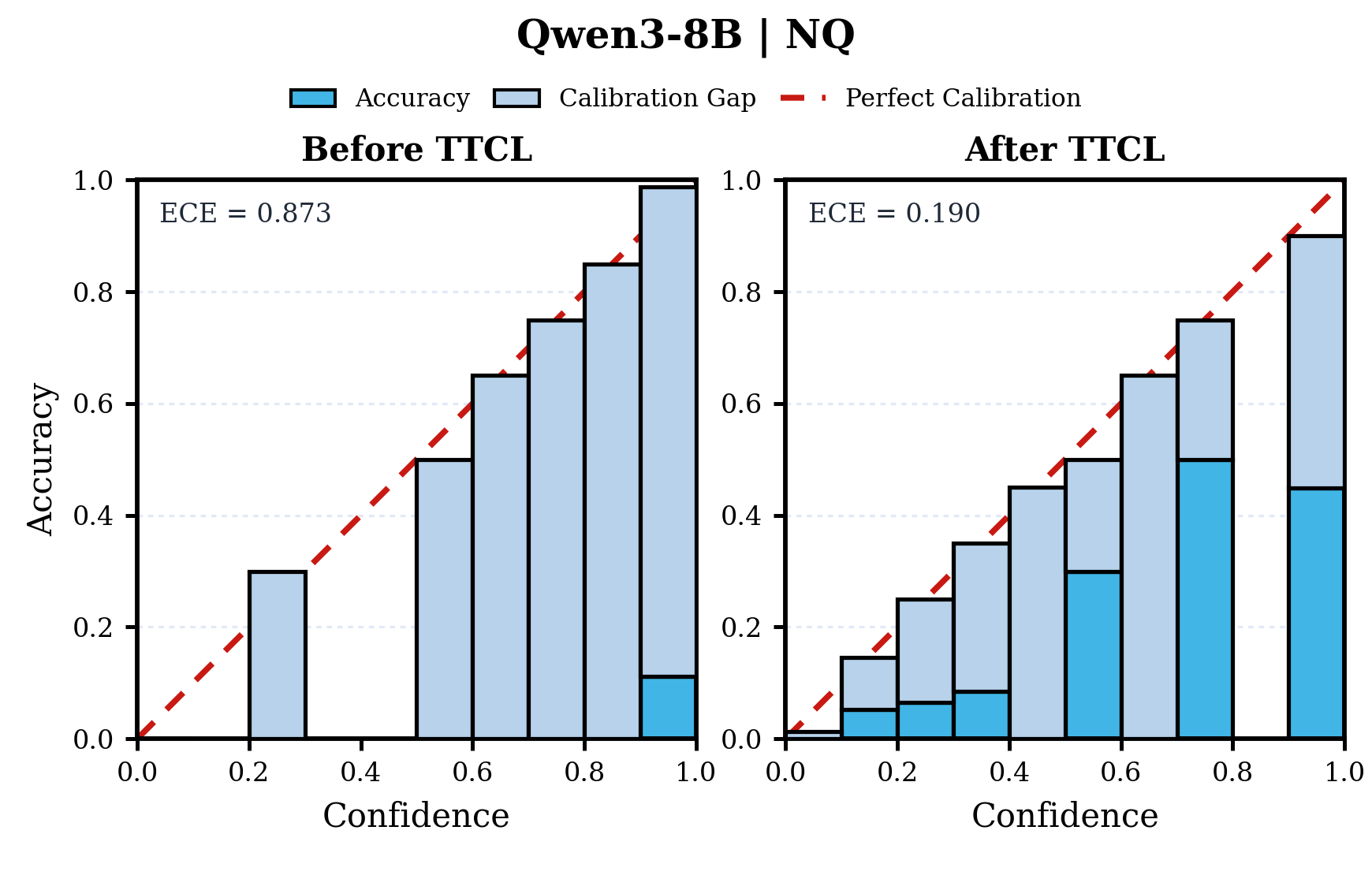}
  \includegraphics[width=0.24\textwidth]{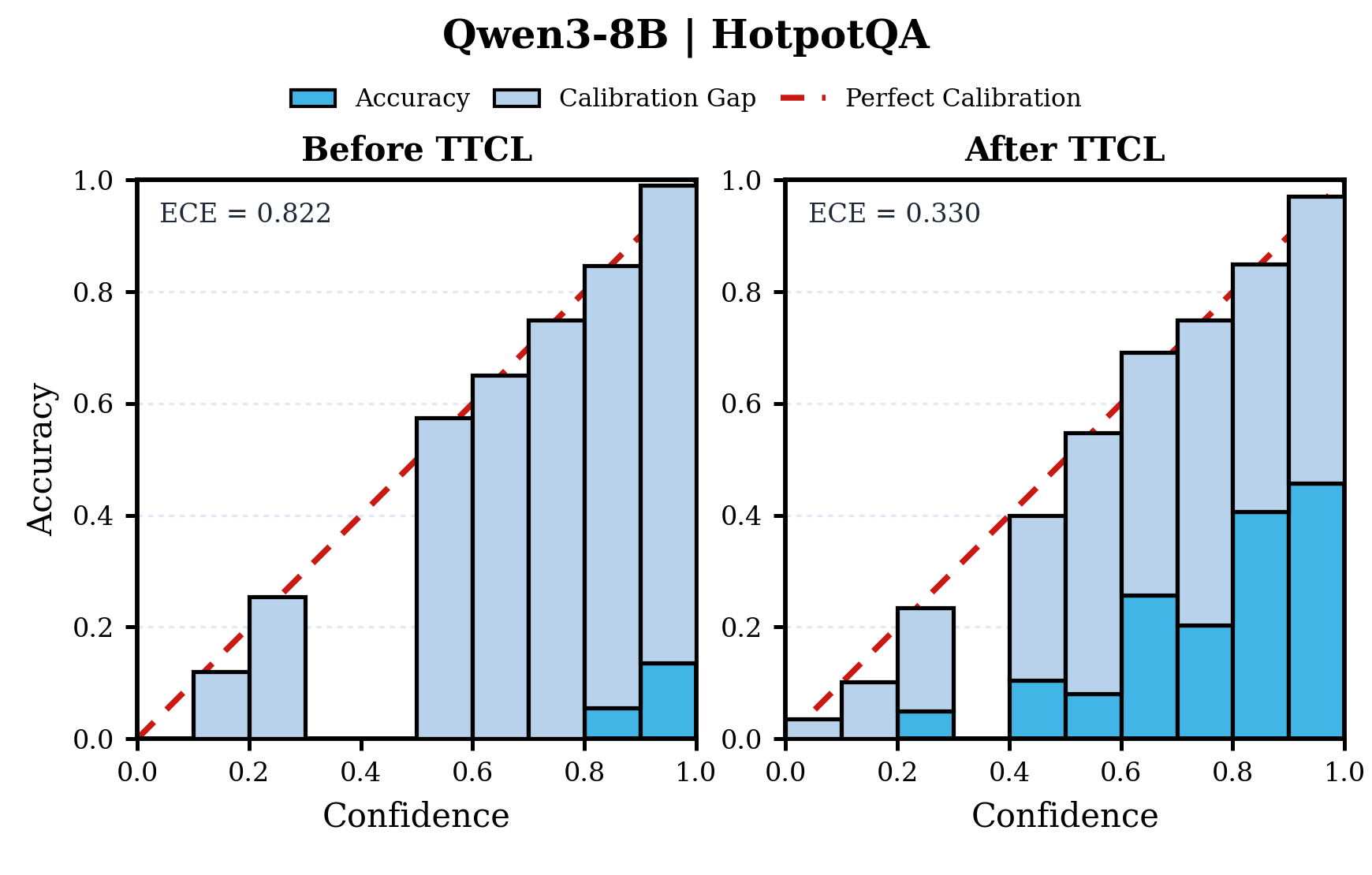}
  \includegraphics[width=0.24\textwidth]{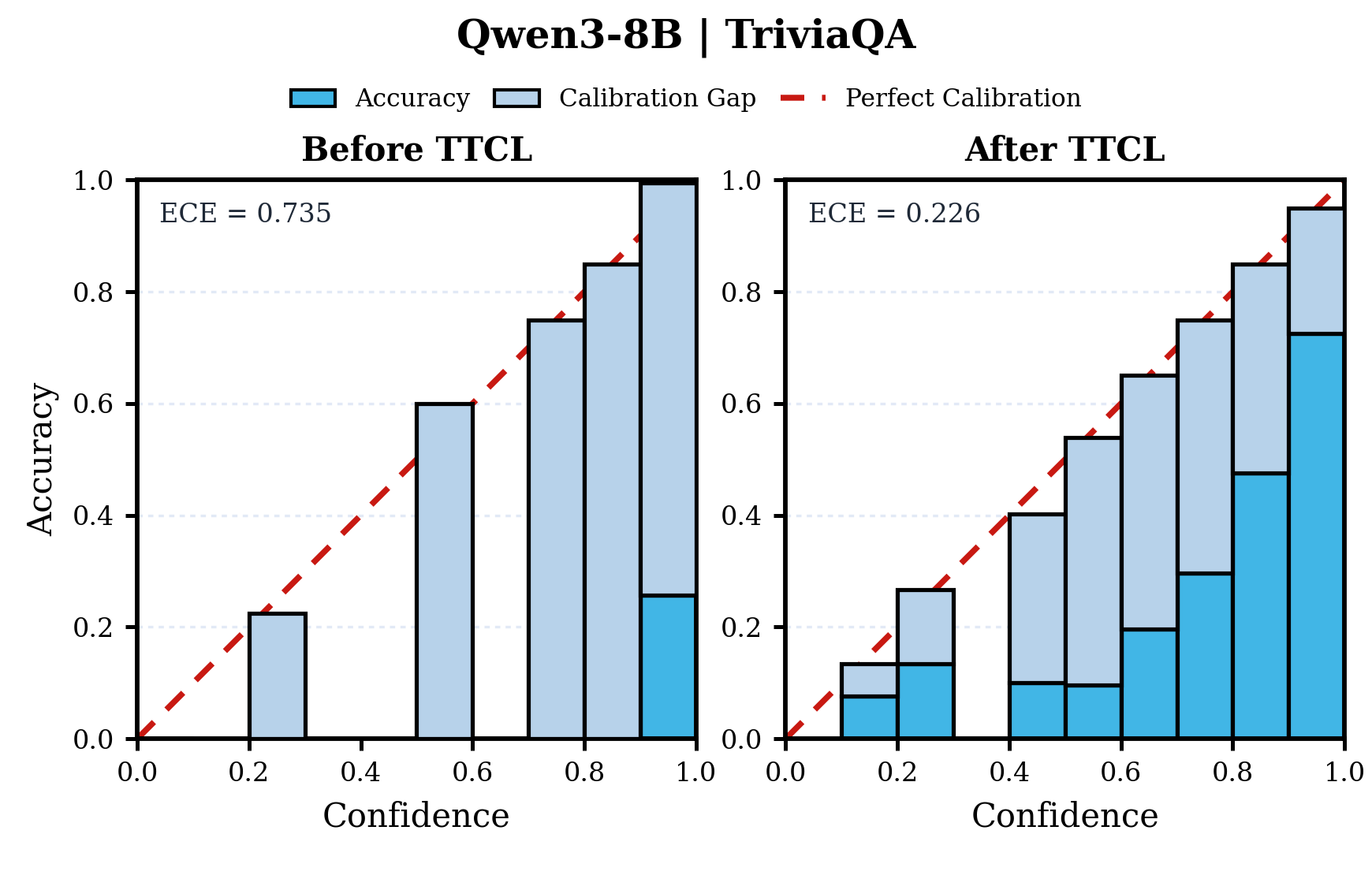} \\
  \includegraphics[width=0.24\textwidth]{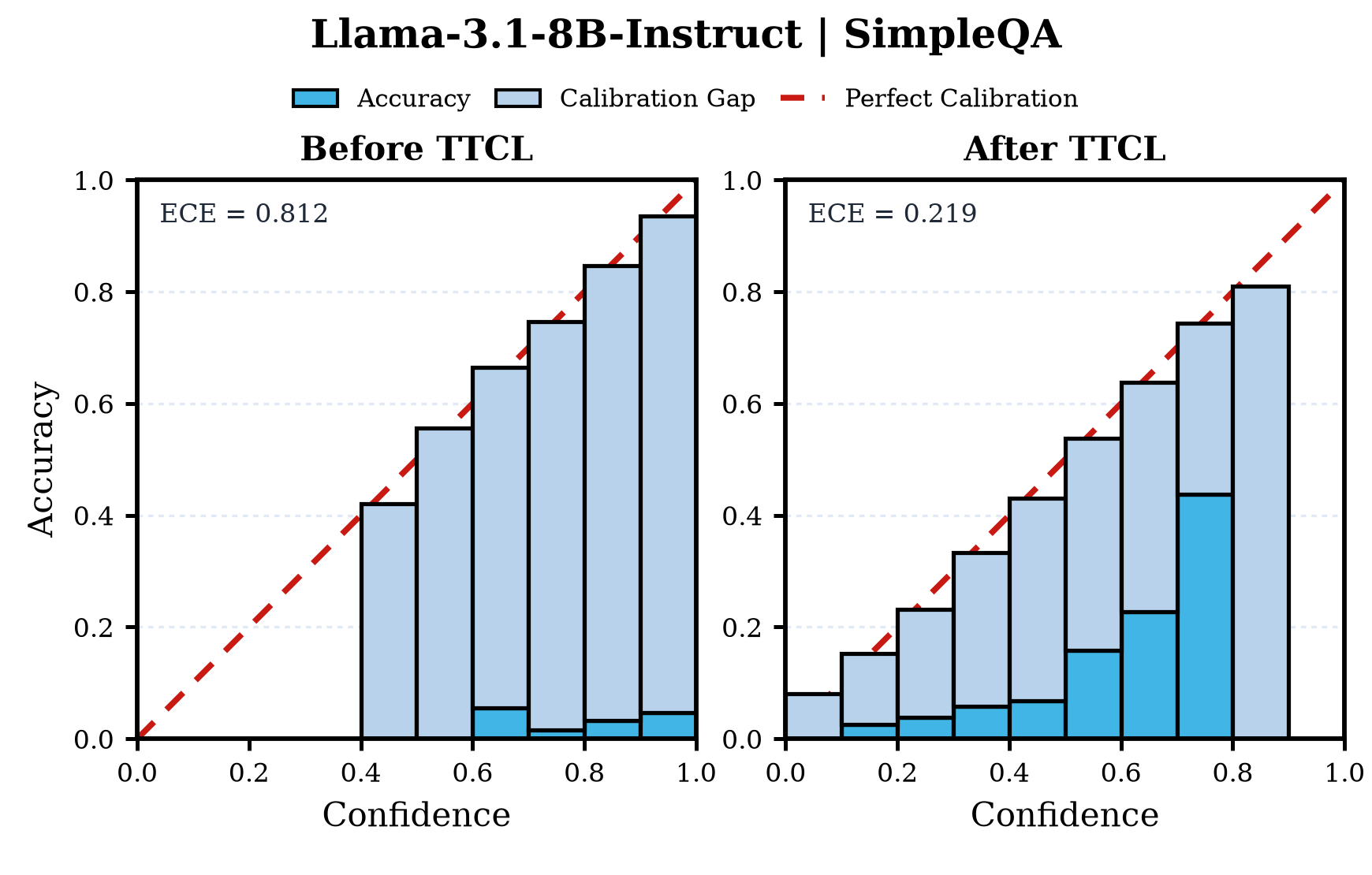}
  \includegraphics[width=0.24\textwidth]{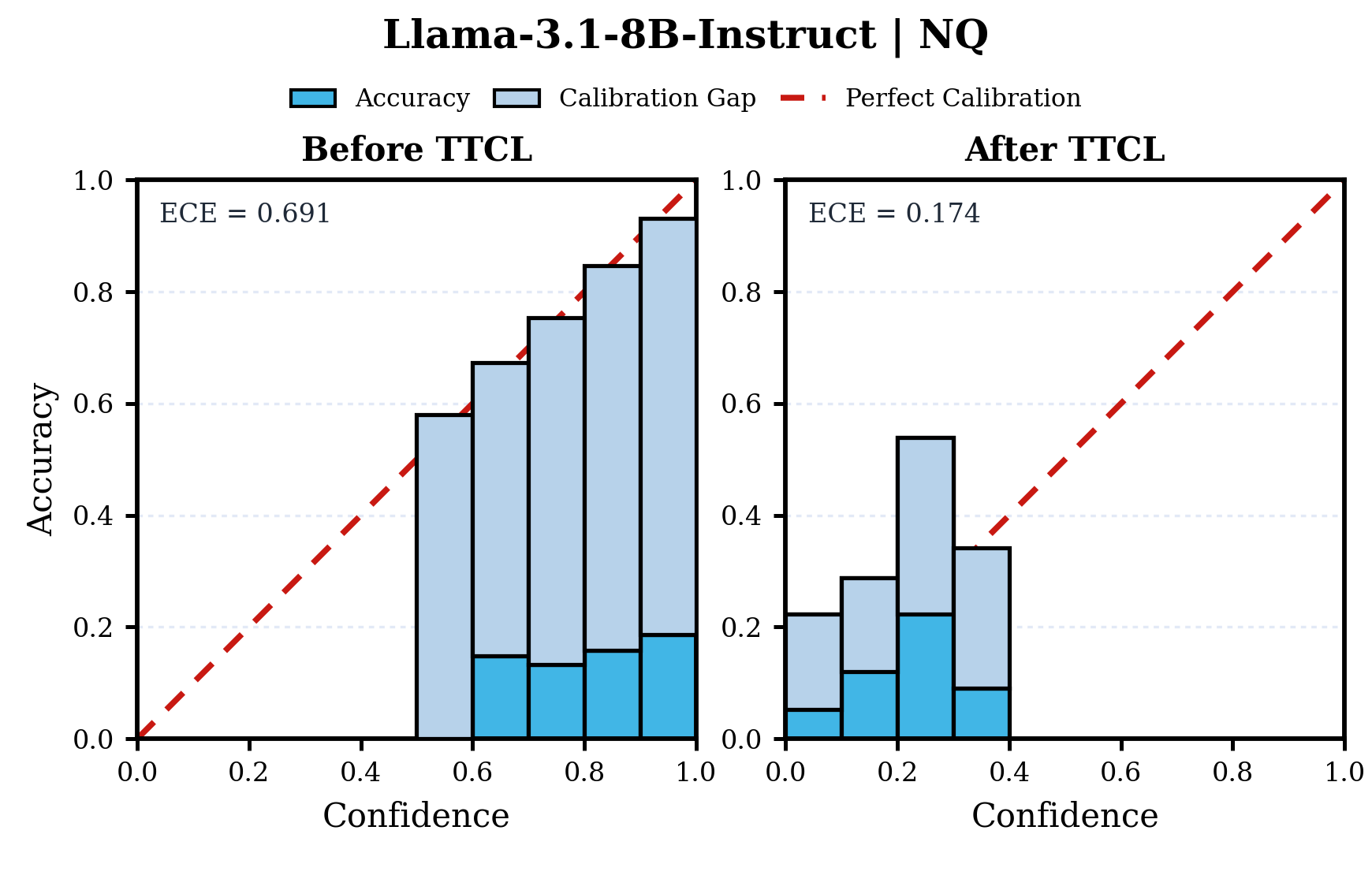}
  \includegraphics[width=0.24\textwidth]{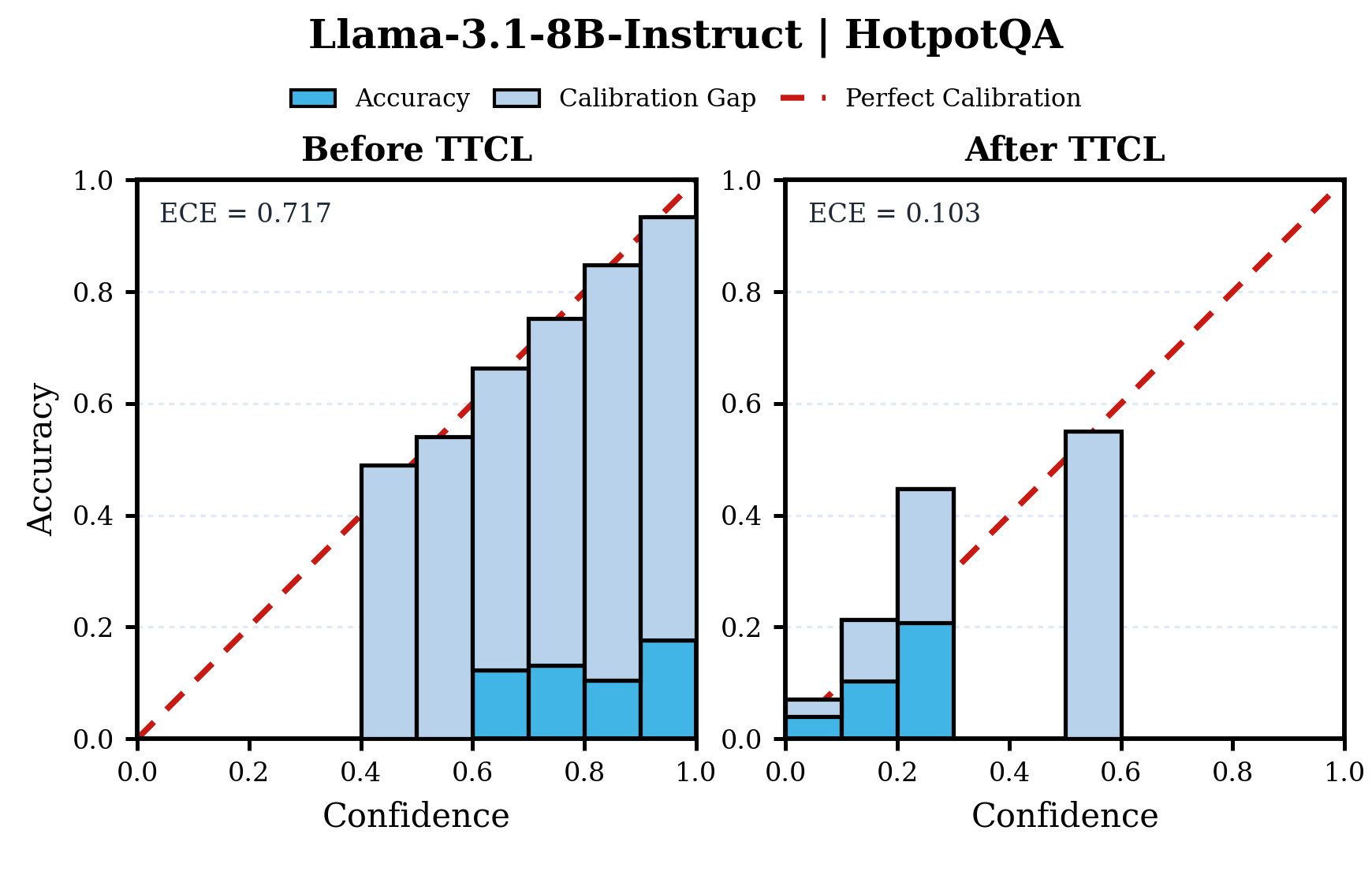}
  \includegraphics[width=0.24\textwidth]{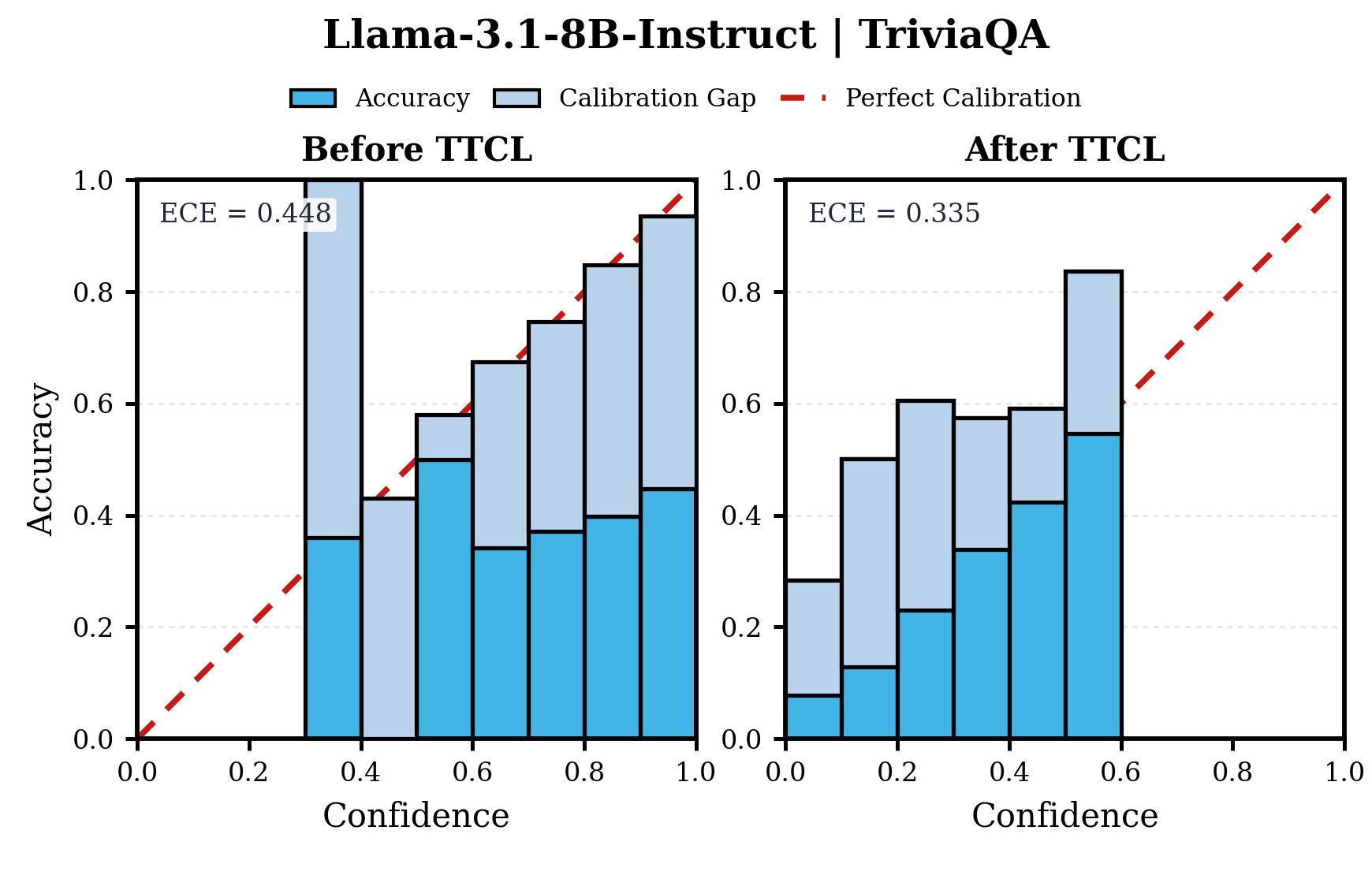}
  \vspace{-2mm}
  \caption{\textbf{Visualization of calibration effects before and after TTCL} on base models for factual QA benchmarks. TTCL effectively calibrates models to the target benchmarks, substantially narrowing the gap between confidence and empirical accuracy. \textcolor{accuracyblue}{Blue bars} denote empirical accuracy within each confidence bin, while \textcolor{calibrationgapblue}{light blue bars} indicate the deviation from perfect calibration.}
  \label{Fig:TTCL_BaseFact_Calibration_Visualize}
\end{figure}

Additionally, we train RLCR on two labeled datasets, DAPO-Math-14K and FactQA-10K, to obtain calibrated models. The two datasets are introduced as follows:
\begin{itemize}[left=2pt]
\item \textbf{DAPO-Math-14K}~\citep{yu2026dapo}: We construct DAPO-Math-14K using the 14.1K English samples from the DAPO-Math-17k-Processed dataset, a widely used training set for mathematical reasoning containing diverse and high-quality examples with annotated ground-truth answers.

\item \textbf{FactQA-10K}: We construct FactQA-10K by randomly sampling 5K training examples from HotpotQA~\citep{yang2018hotpotqa} and 5K training examples from TriviaQA~\citep{joshi2017triviaqa}, resulting in a total of 10K factual question answering samples.
\end{itemize}

The benchmark statistics are summarized in Table~\ref{Tble:dataset_details}. Additionally, we adopt several benchmarks from diverse domains, including ARC~\citep{clark2018arc}, GPQA-Diamond~\citep{rein2023gpqa}, IFEval~\citep{zhou2023ifeval}, MMLU-Pro~\citep{wang2024mmlupro}, and TruthfulQA~\citep{lin2022truthfulqa}, to examine whether TTCL affects the model's general capabilities beyond the target tasks. These benchmarks cover scientific reasoning, instruction following, multi-domain knowledge and truthfulness. We evaluate the general capabilities of the models on these benchmarks using LightEval~\citep{lighteval}, which is a widely used evaluation platform that supports a broad range of LLM benchmarks.

\begin{figure}[t]
  \centering
  \includegraphics[width=0.19\textwidth]{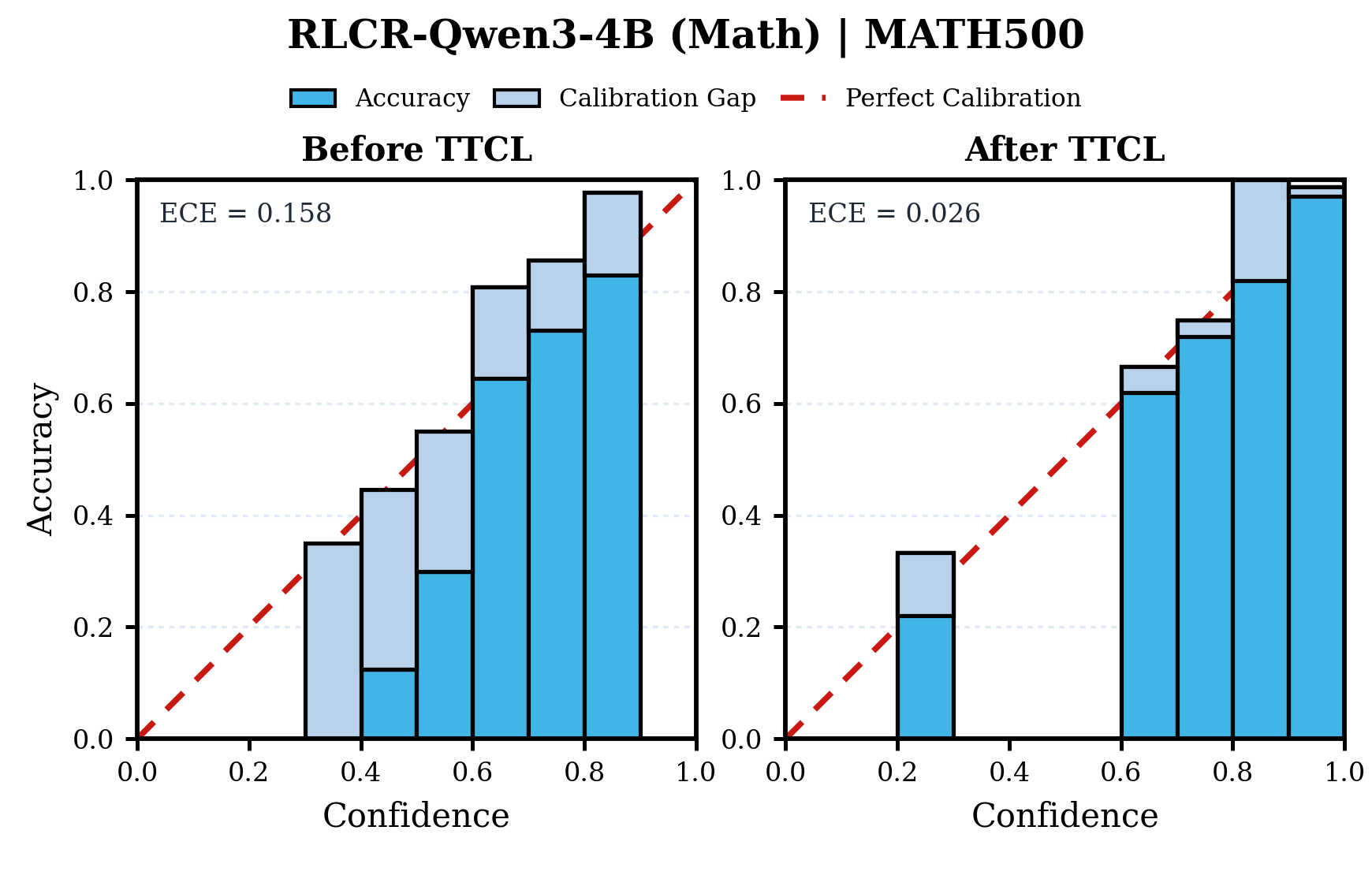}
  \includegraphics[width=0.19\textwidth]{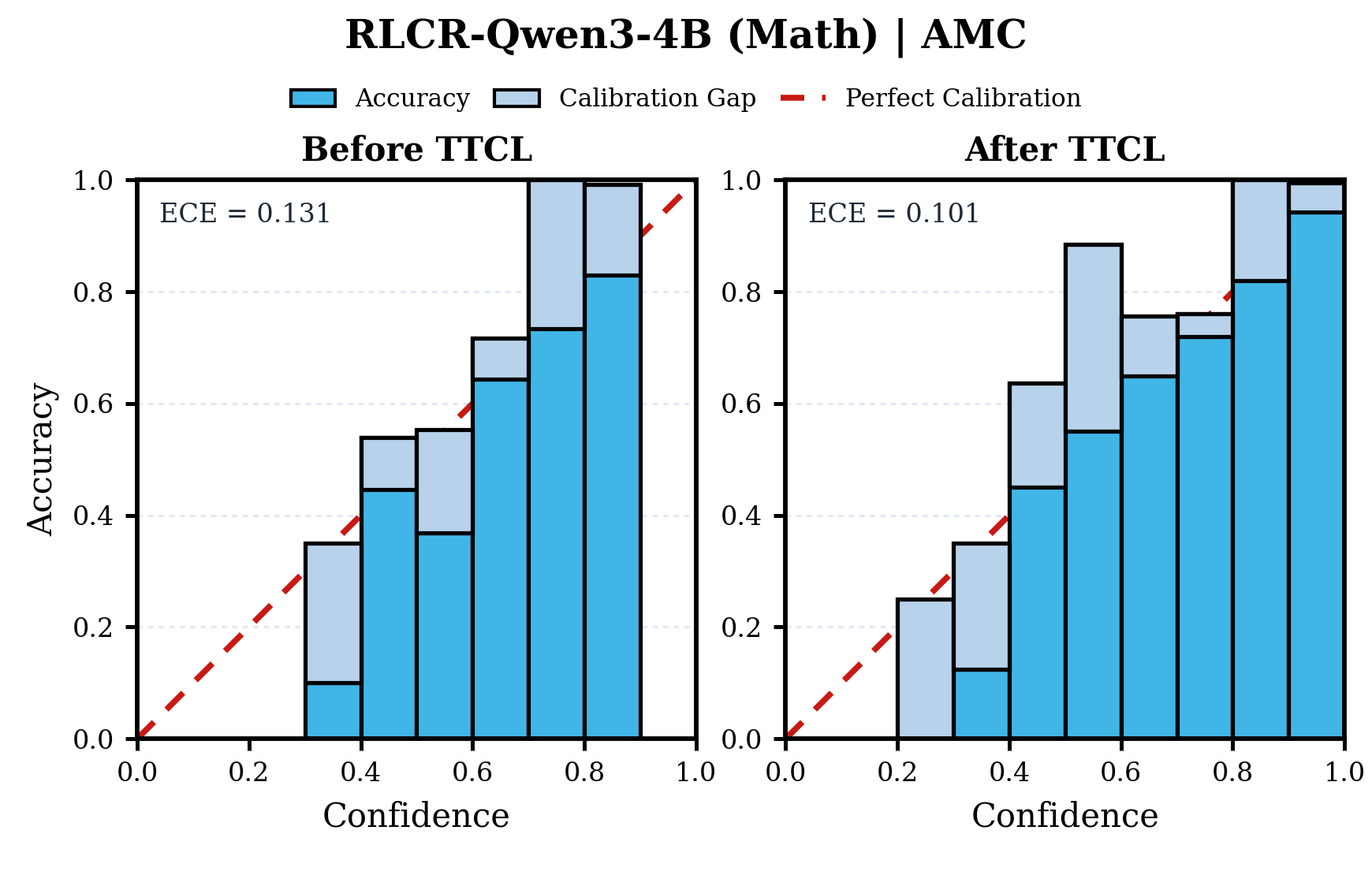}
  \includegraphics[width=0.19\textwidth]{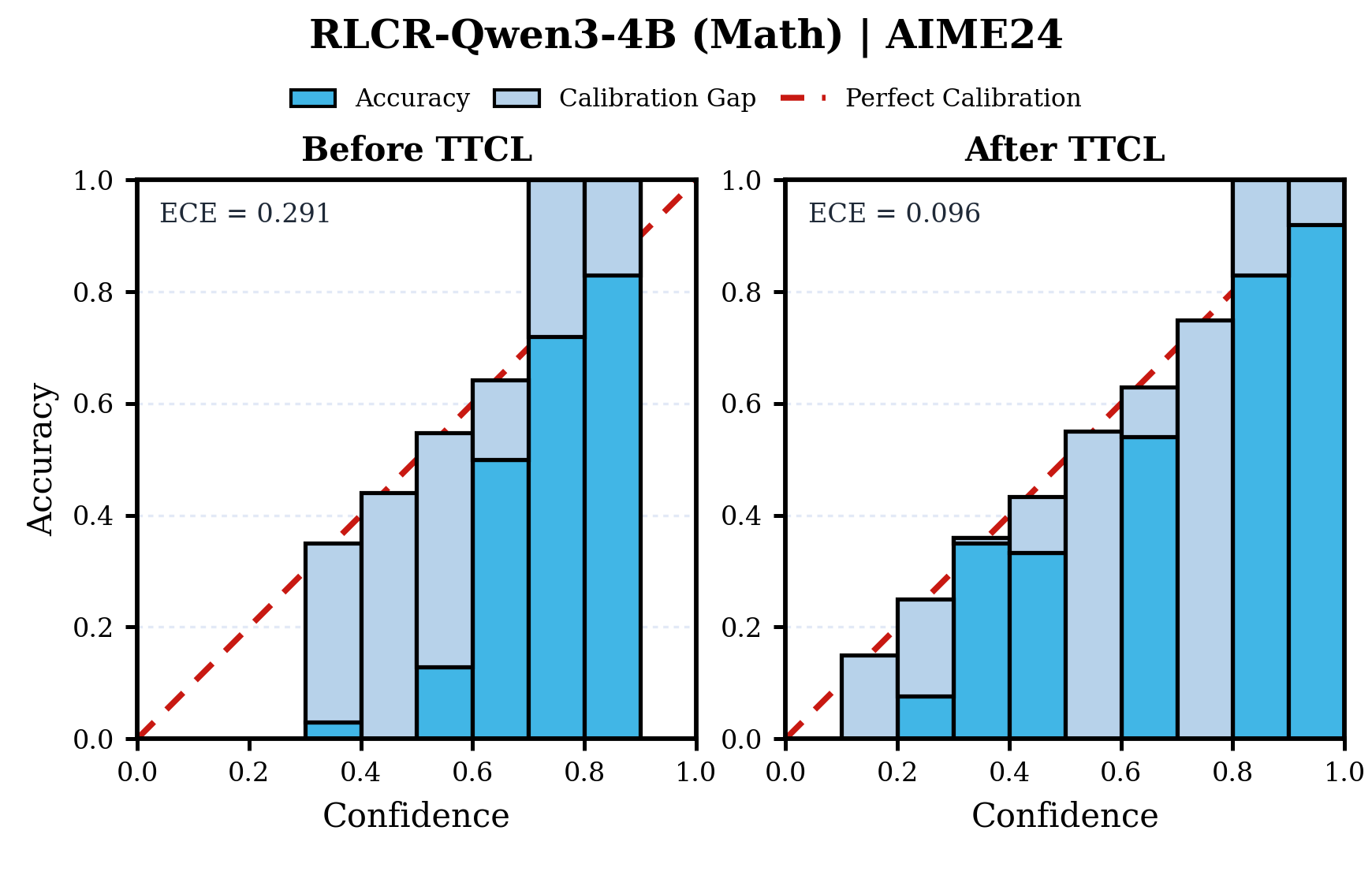} 
  \includegraphics[width=0.19\textwidth]{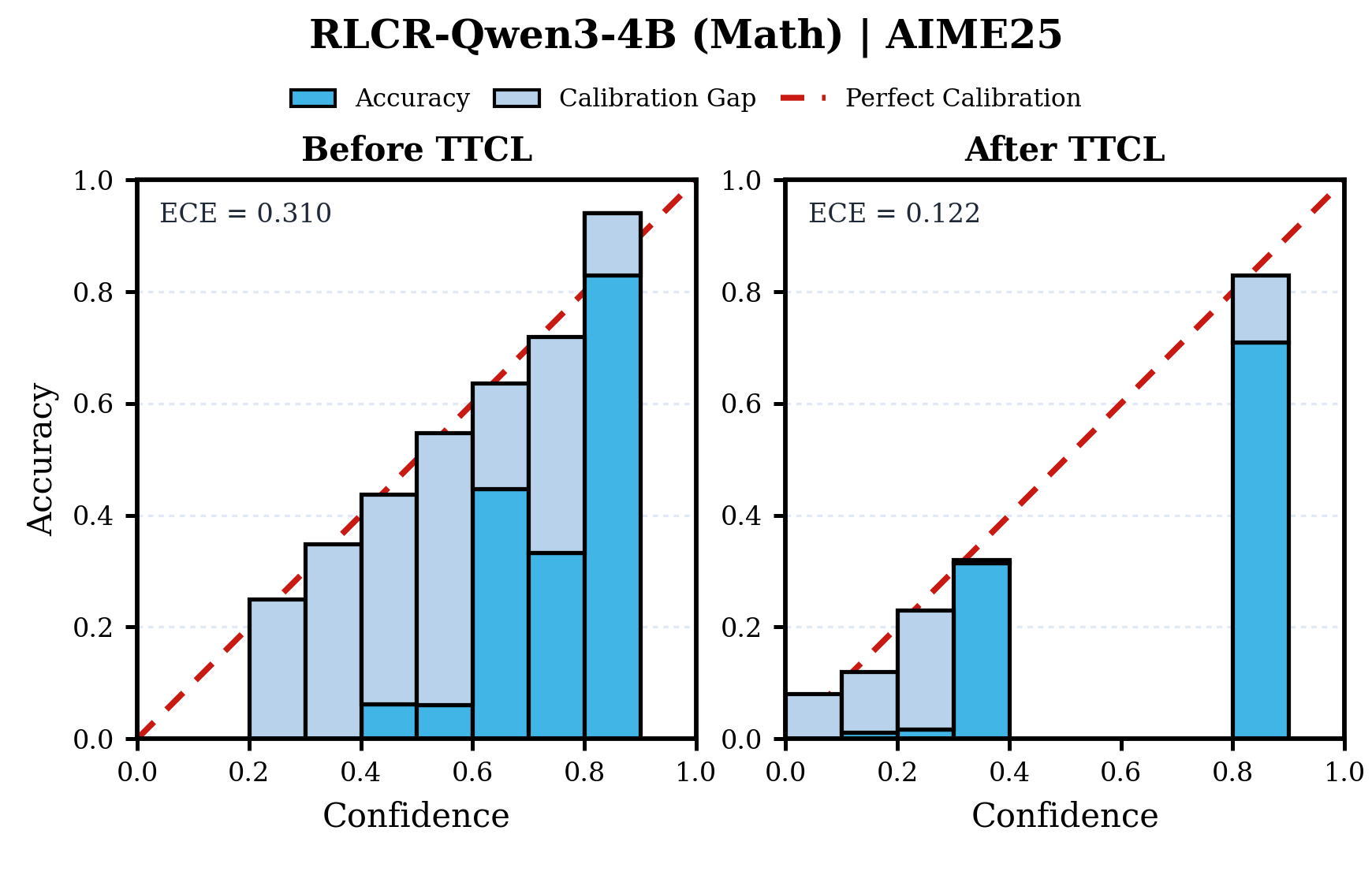}
  \includegraphics[width=0.19\textwidth]{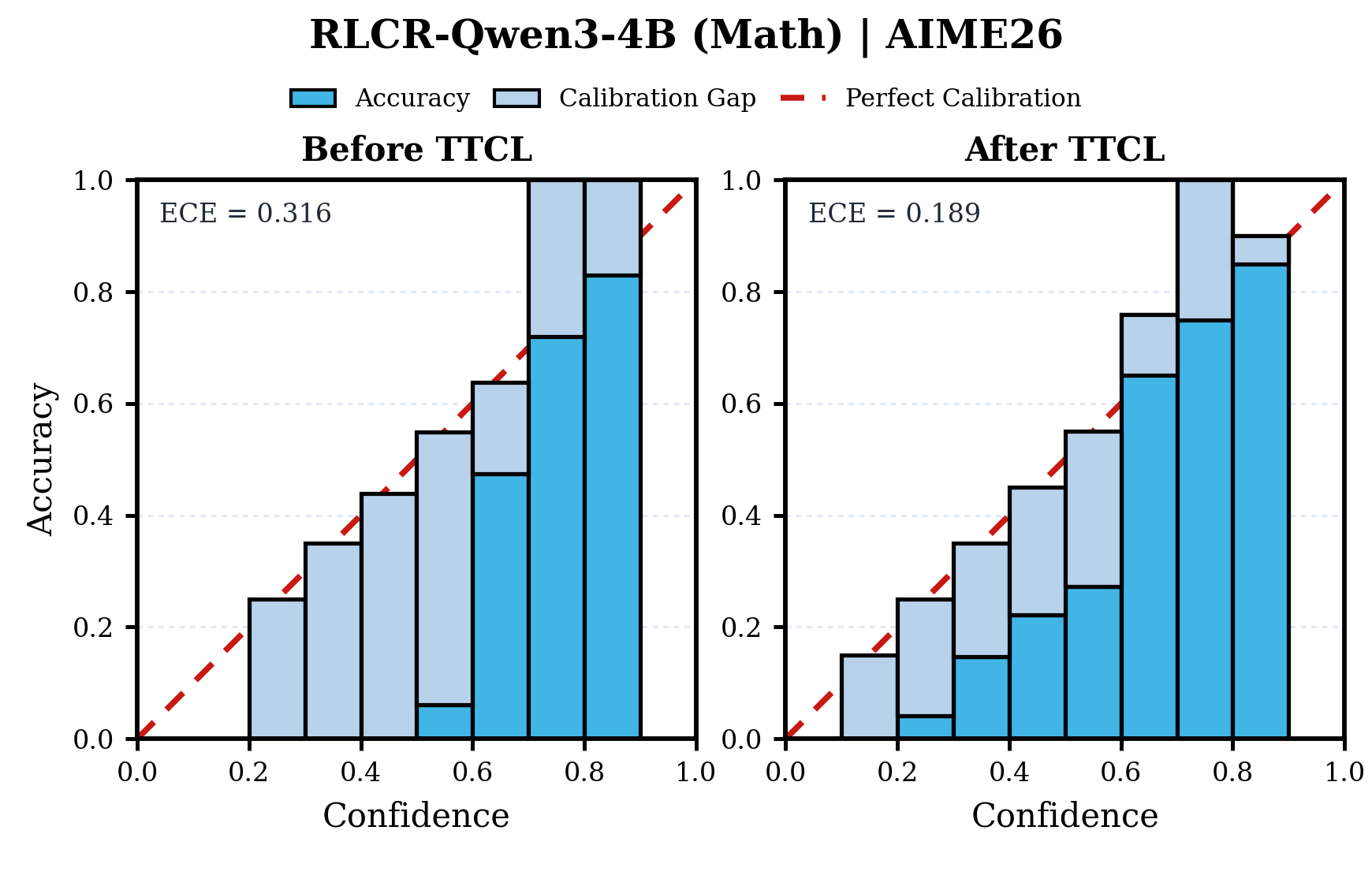} \\
  \includegraphics[width=0.19\textwidth]{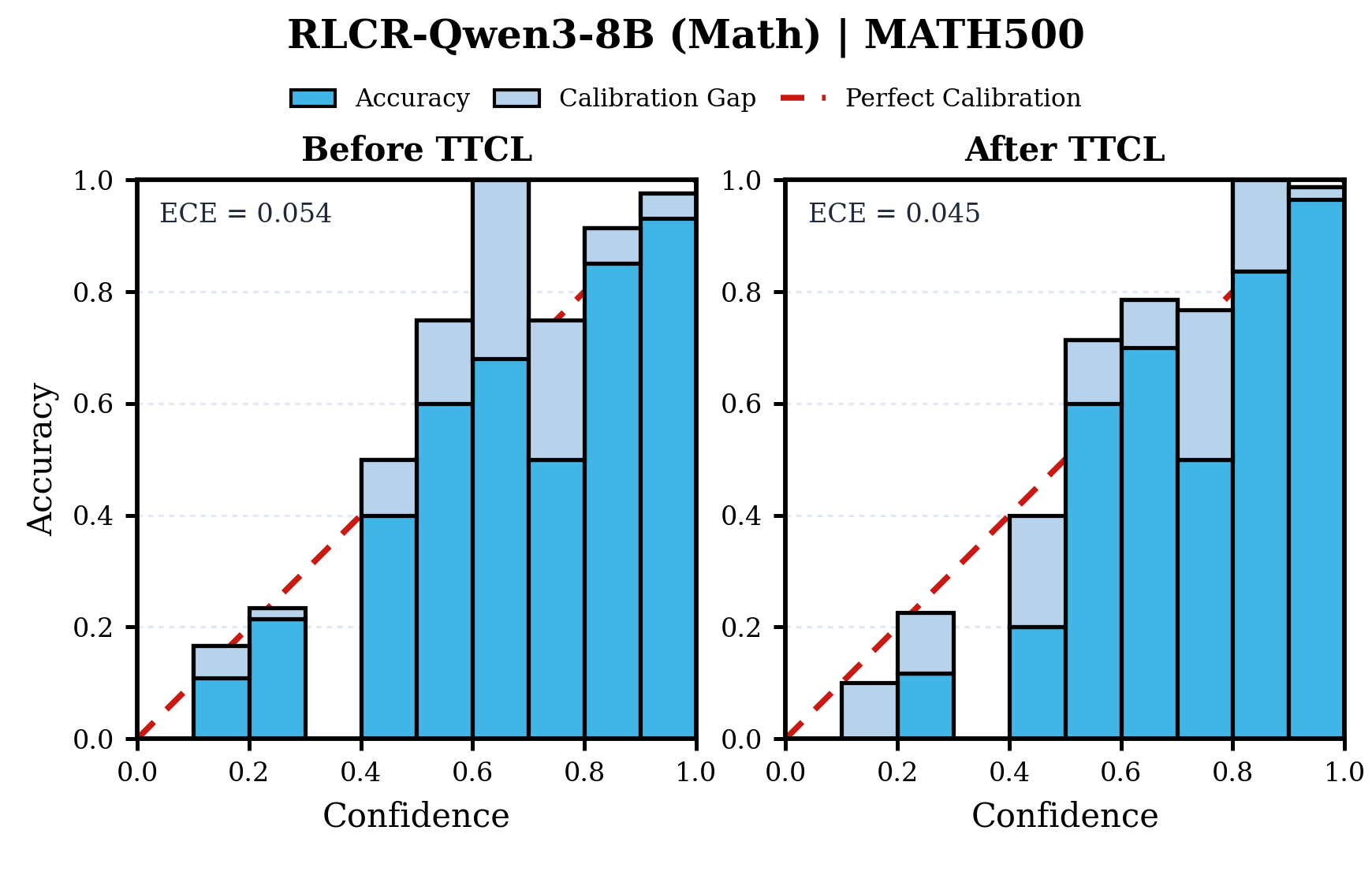}
  \includegraphics[width=0.19\textwidth]{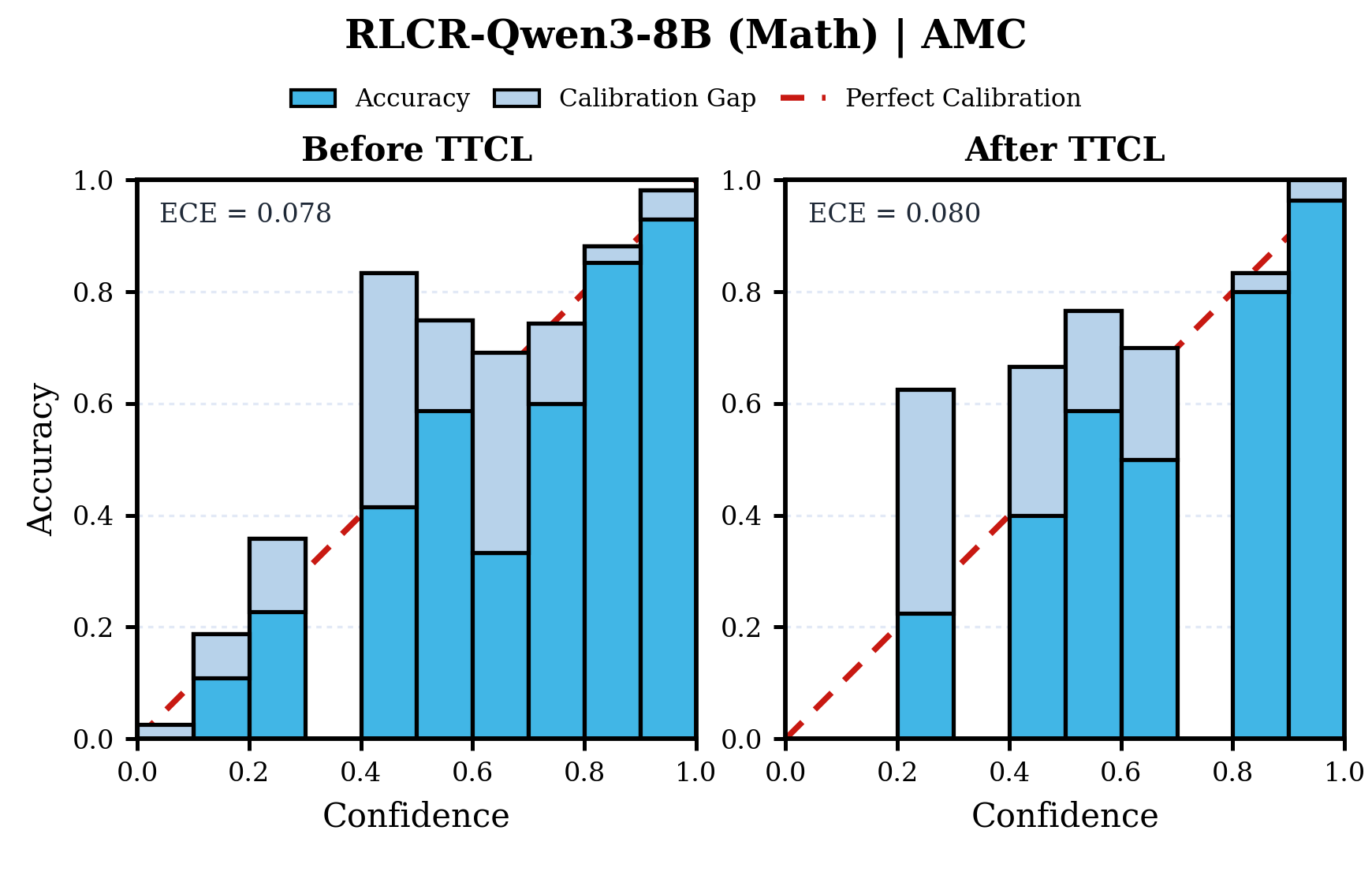}
  \includegraphics[width=0.19\textwidth]{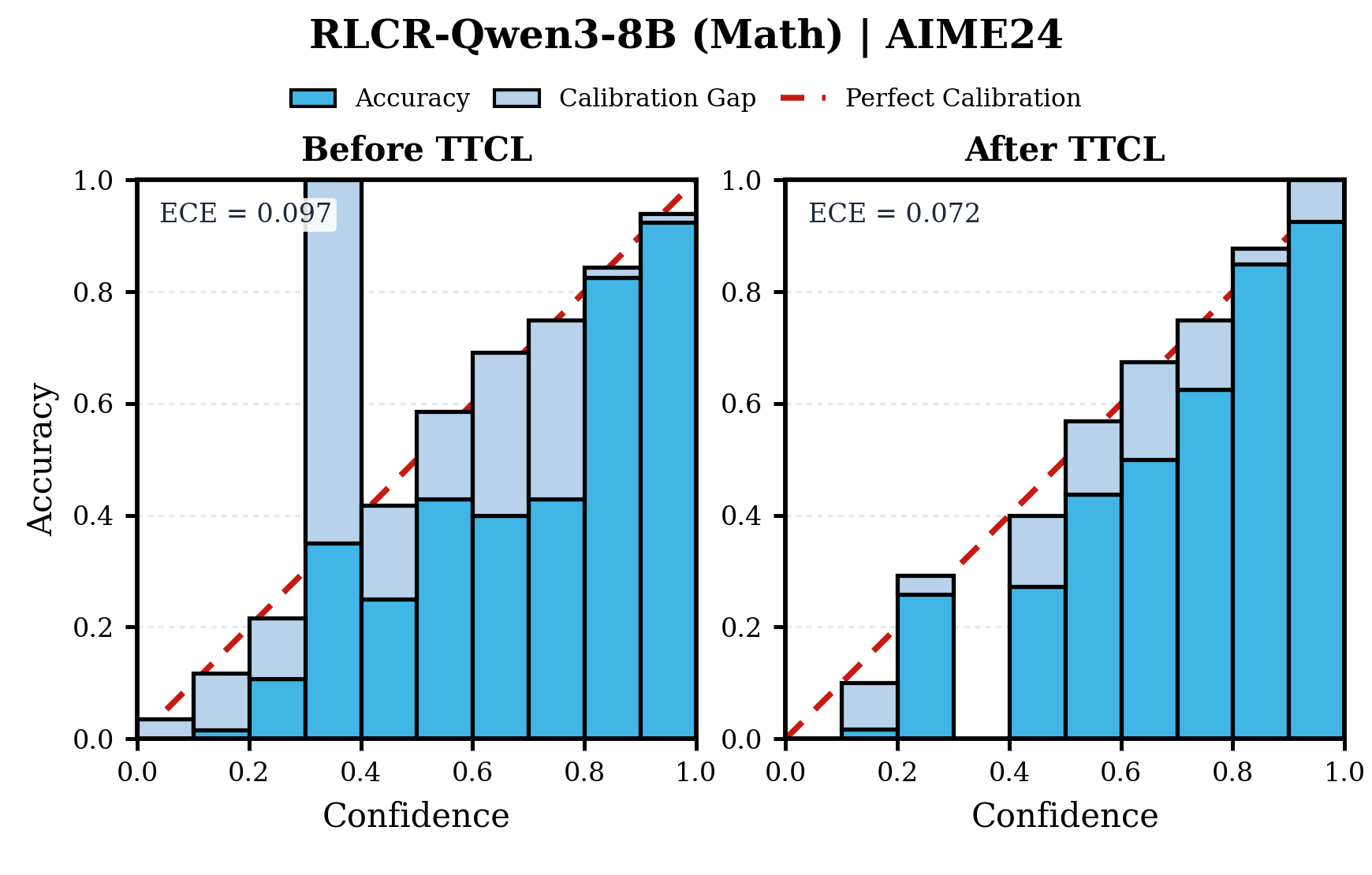} 
  \includegraphics[width=0.19\textwidth]{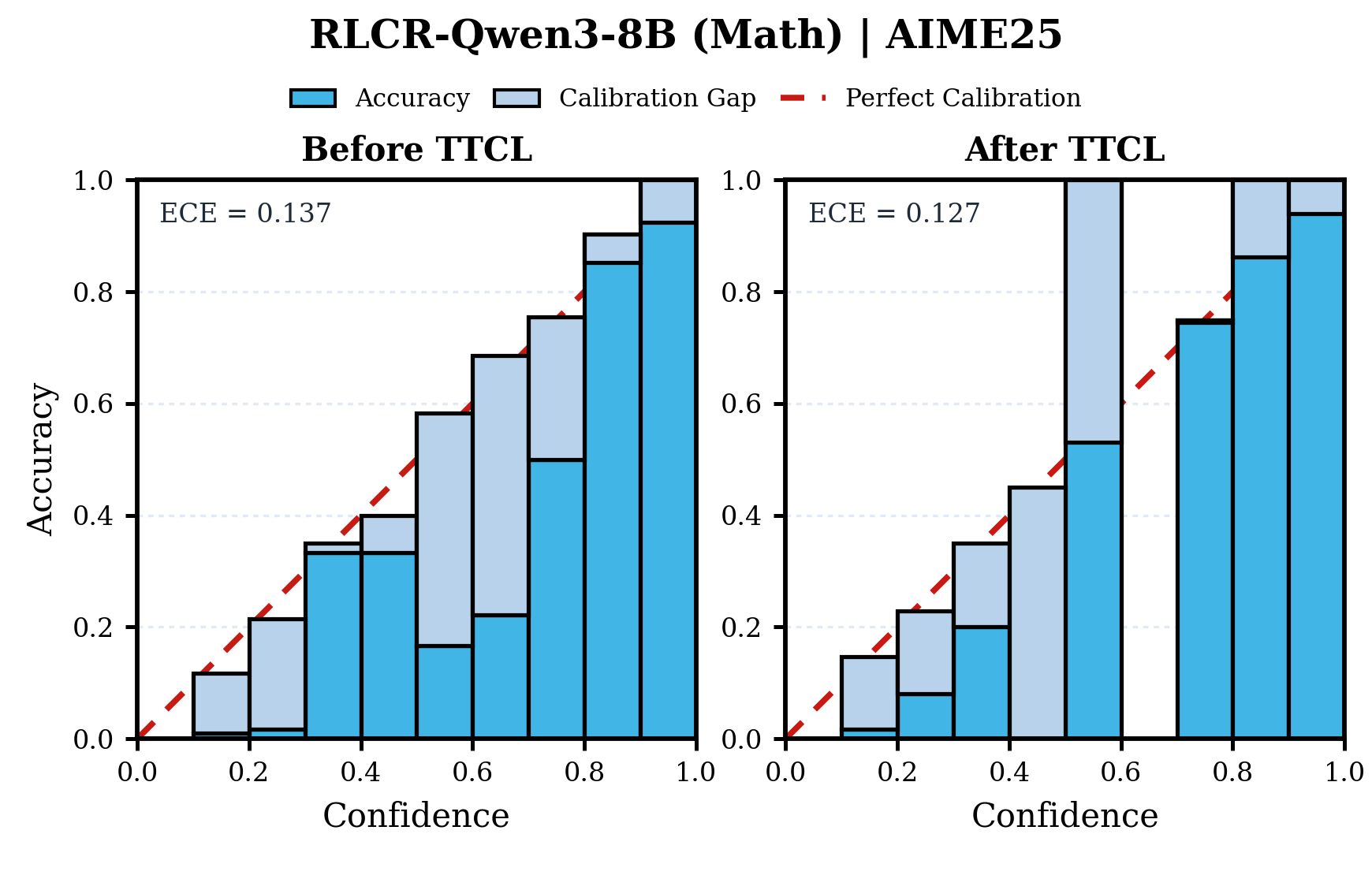}
  \includegraphics[width=0.19\textwidth]{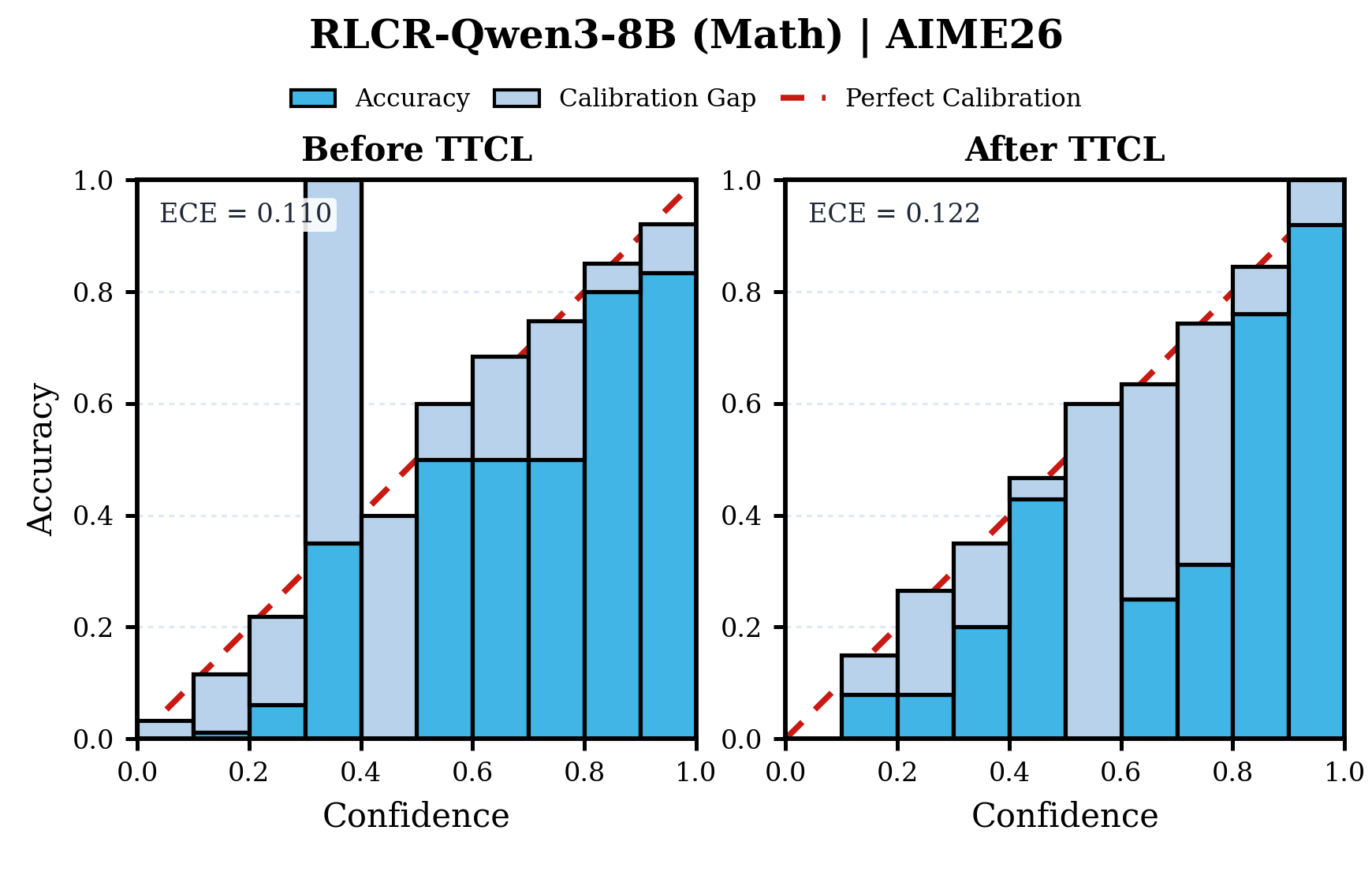}
  \vspace{-2mm}
  \caption{\textbf{Visualization of in-domain calibration effects before and after TTCL} on math-calibrated models for mathematical benchmarks. TTCL calibrates models to the target benchmarks, further narrowing the gap between confidence and empirical accuracy. \textcolor{accuracyblue}{Blue bars} denote empirical accuracy within each confidence bin, while \textcolor{calibrationgapblue}{light blue bars} indicate the deviation from perfect calibration.}
  \label{Fig:TTCL_RLCR_InDomain_Visualize}
  \vspace{-3mm}
\end{figure}

\begin{figure}[t!]
  \centering
  \includegraphics[width=0.19\textwidth]{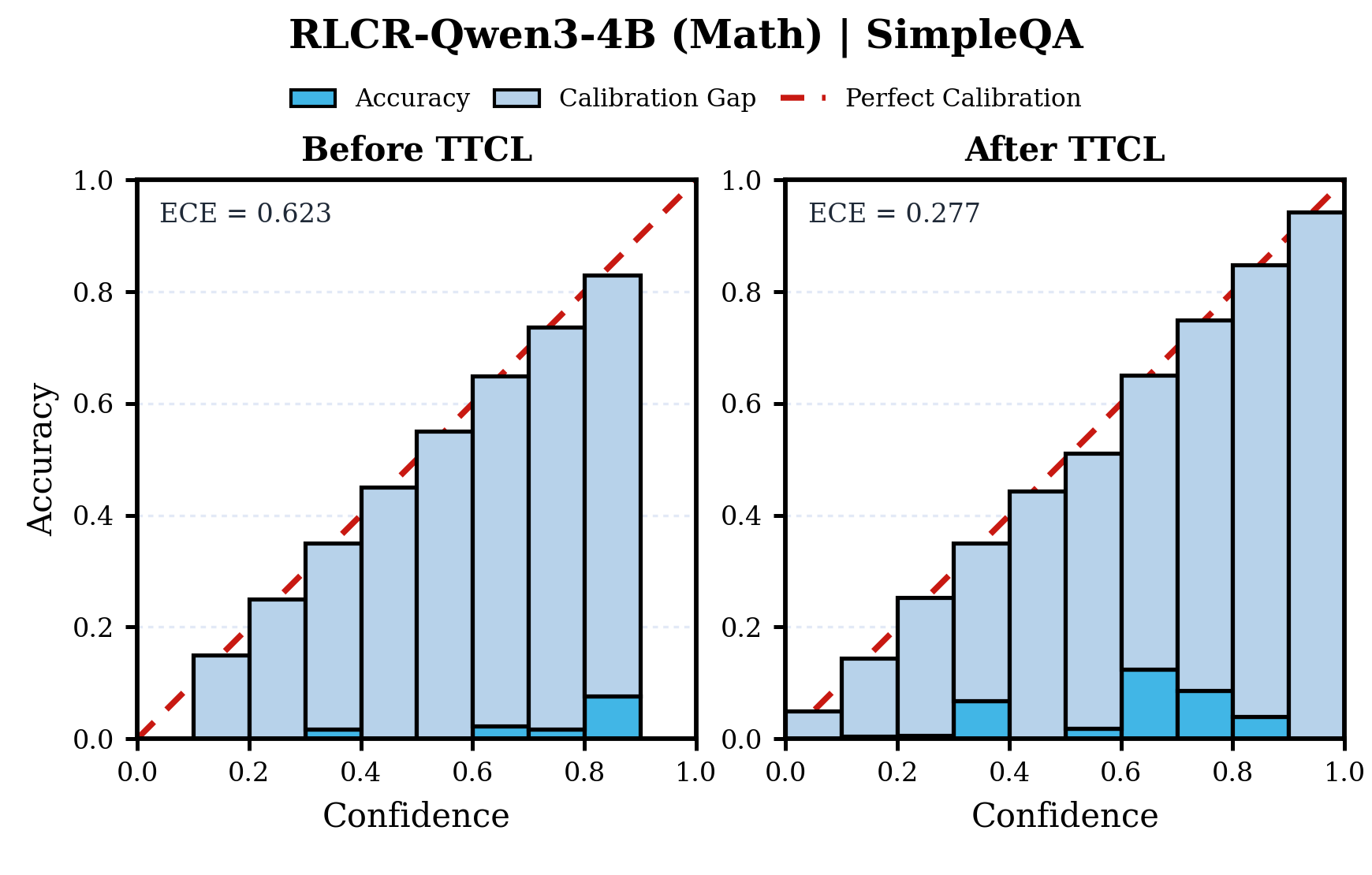}
  \includegraphics[width=0.19\textwidth]{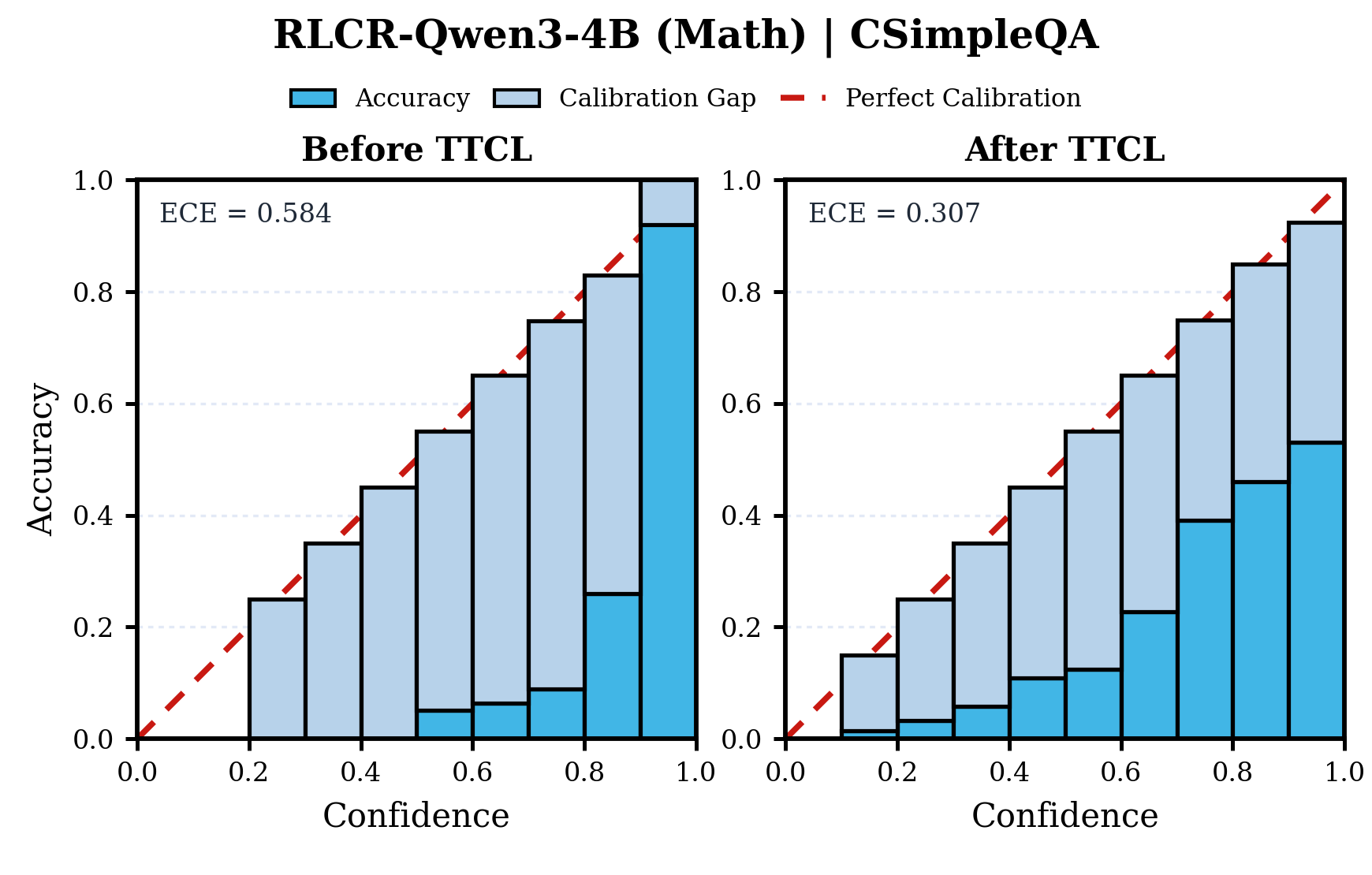}
  \includegraphics[width=0.19\textwidth]{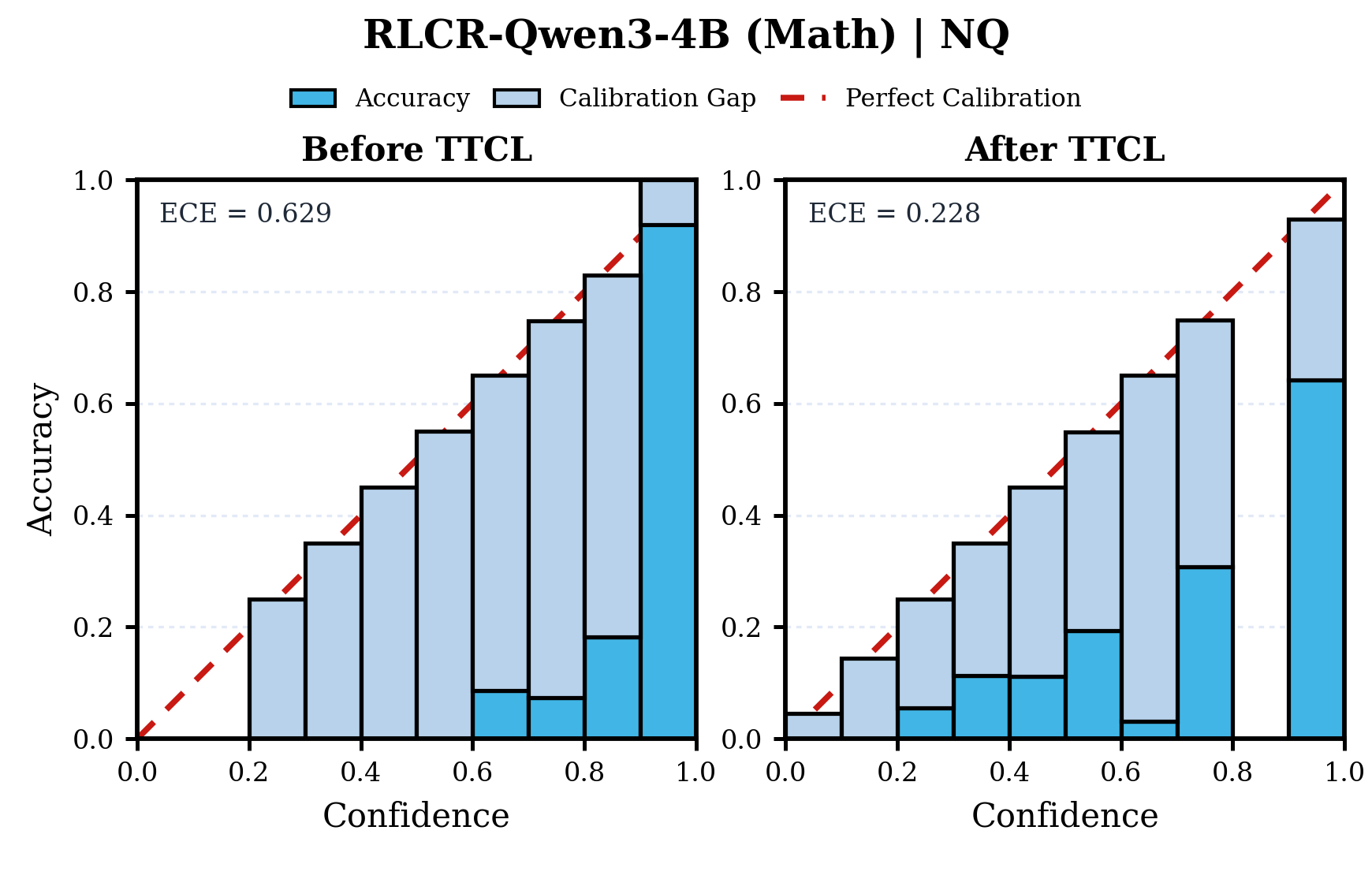} 
  \includegraphics[width=0.19\textwidth]{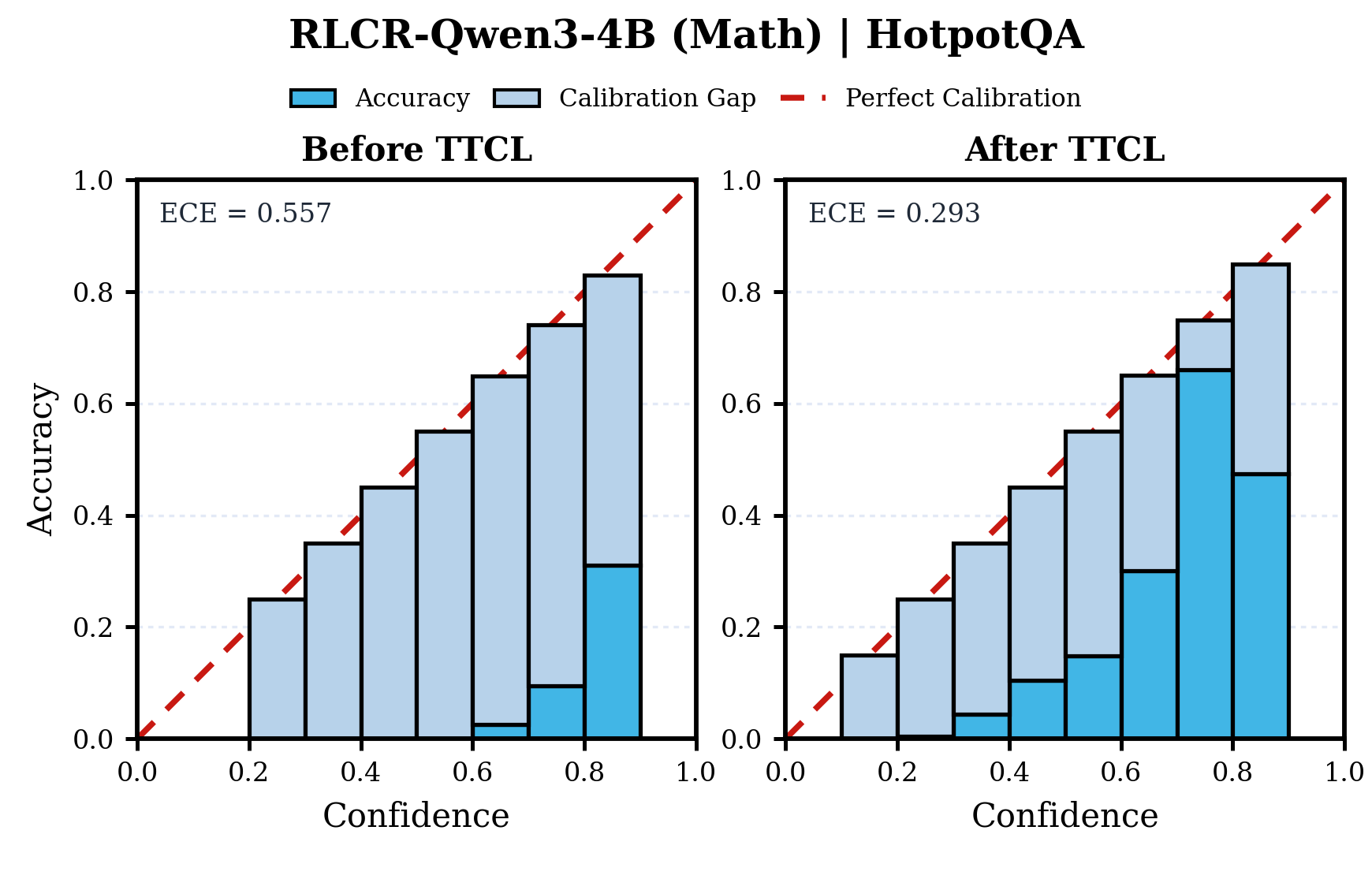}
  \includegraphics[width=0.19\textwidth]{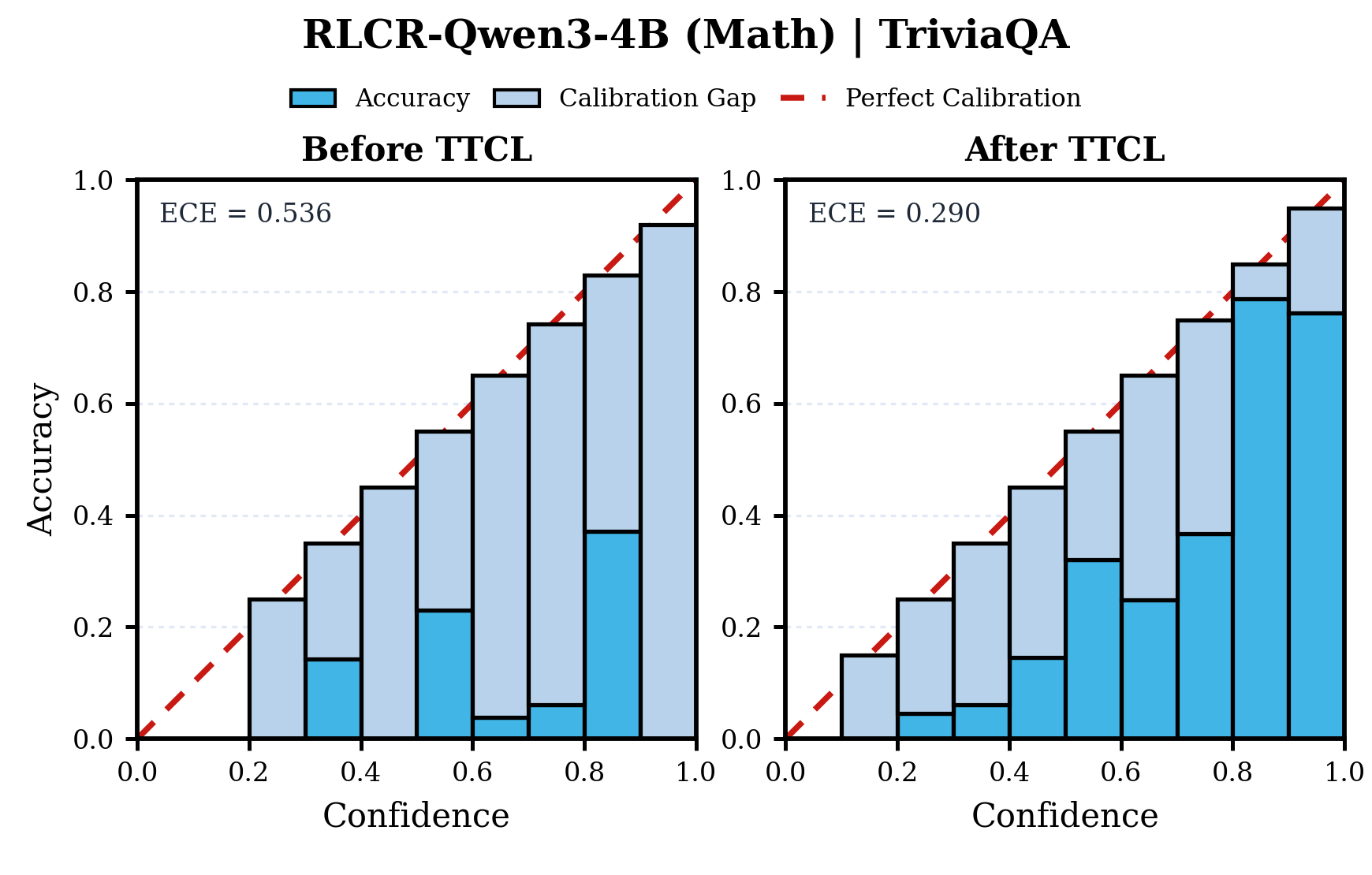} \\
  \includegraphics[width=0.19\textwidth]{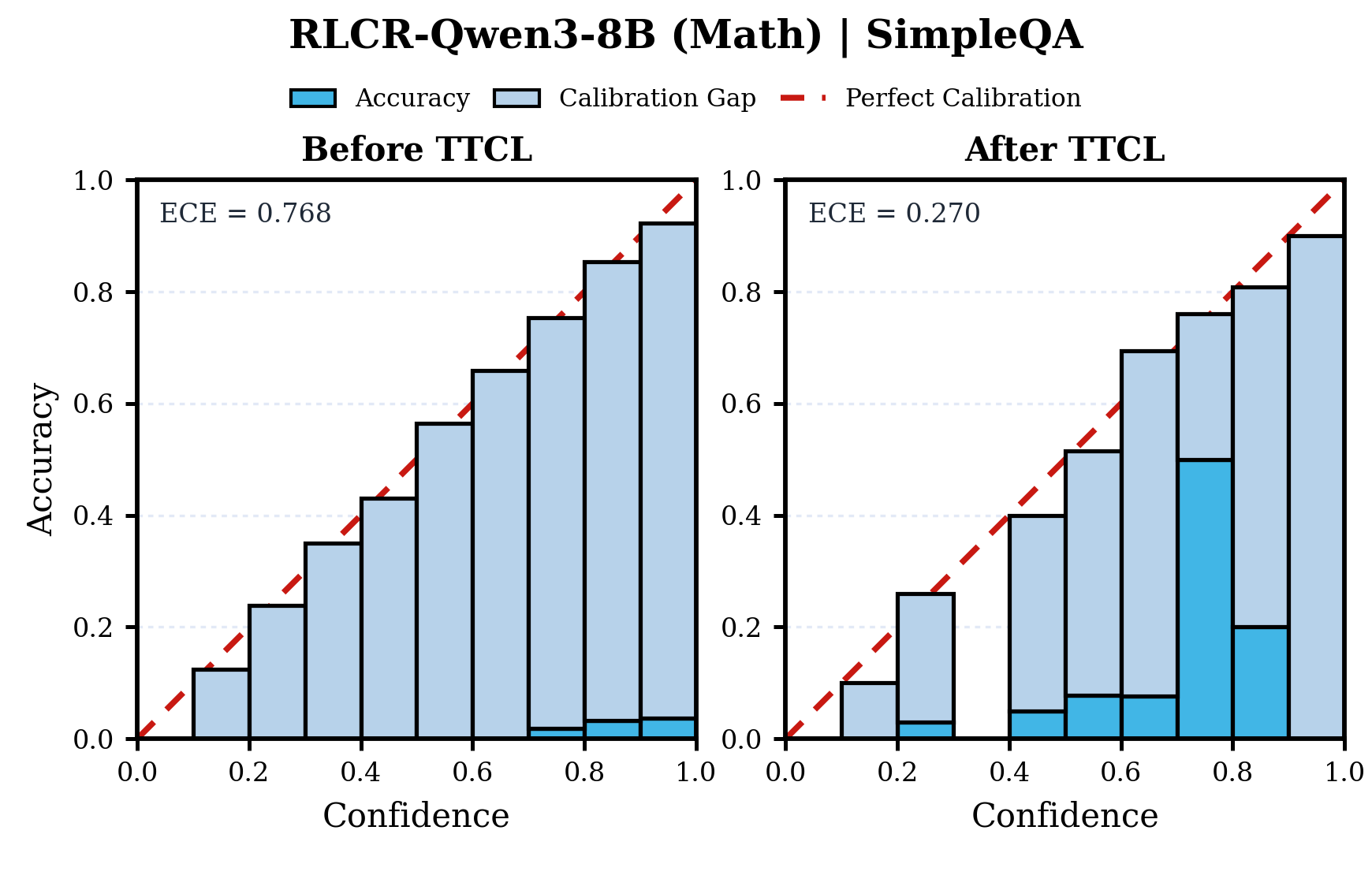}
  \includegraphics[width=0.19\textwidth]{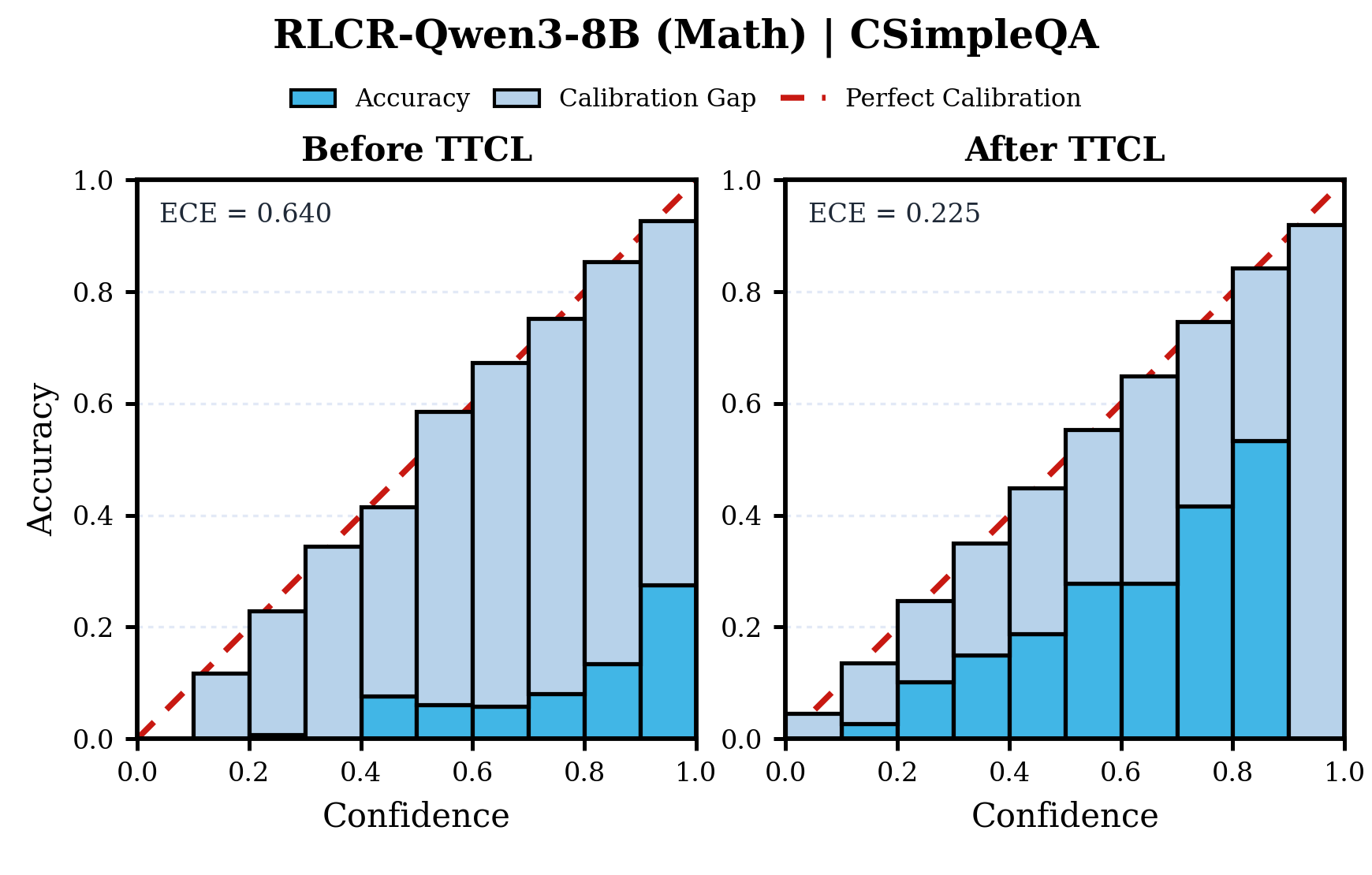}
  \includegraphics[width=0.19\textwidth]{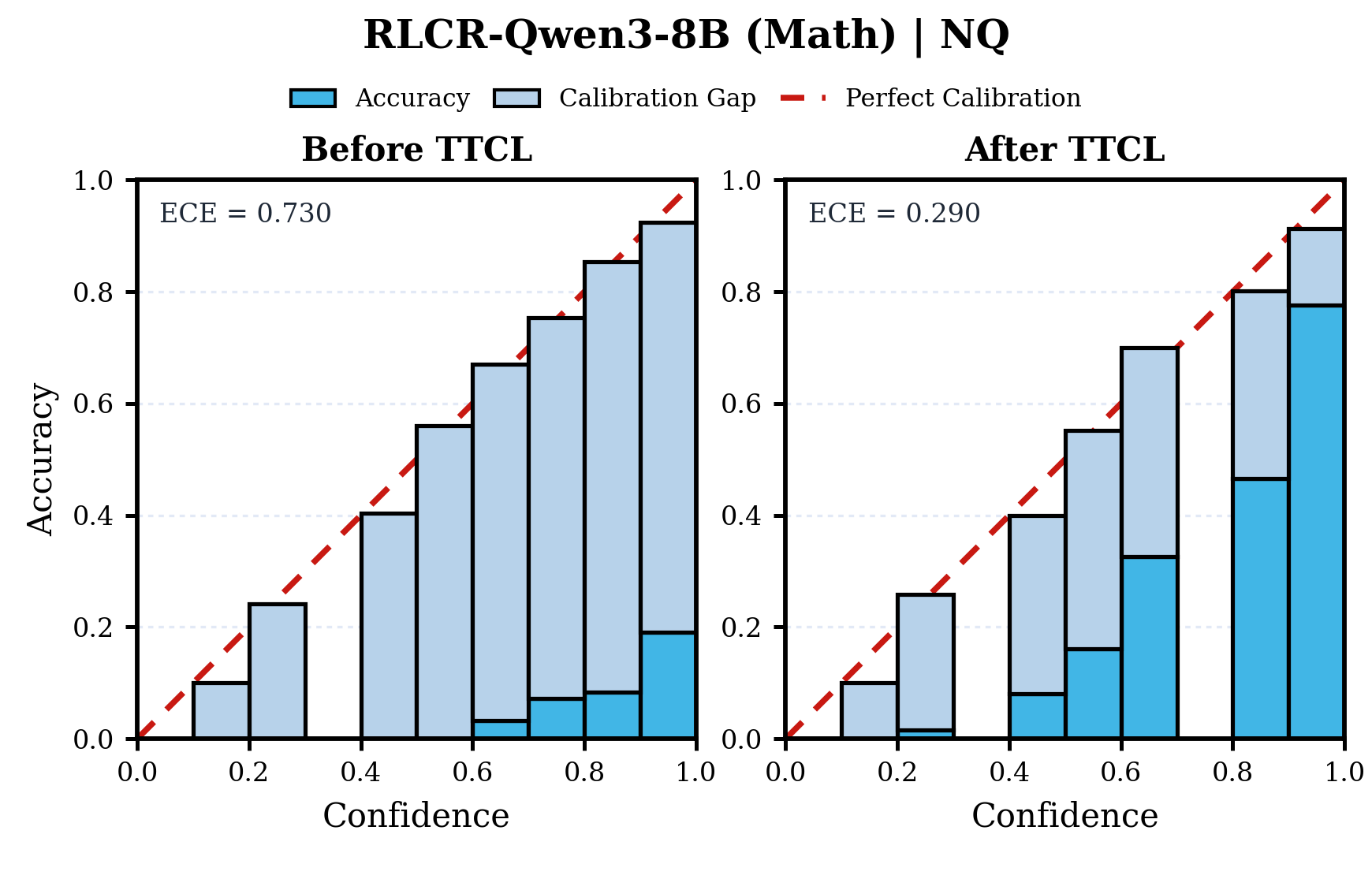} 
  \includegraphics[width=0.19\textwidth]{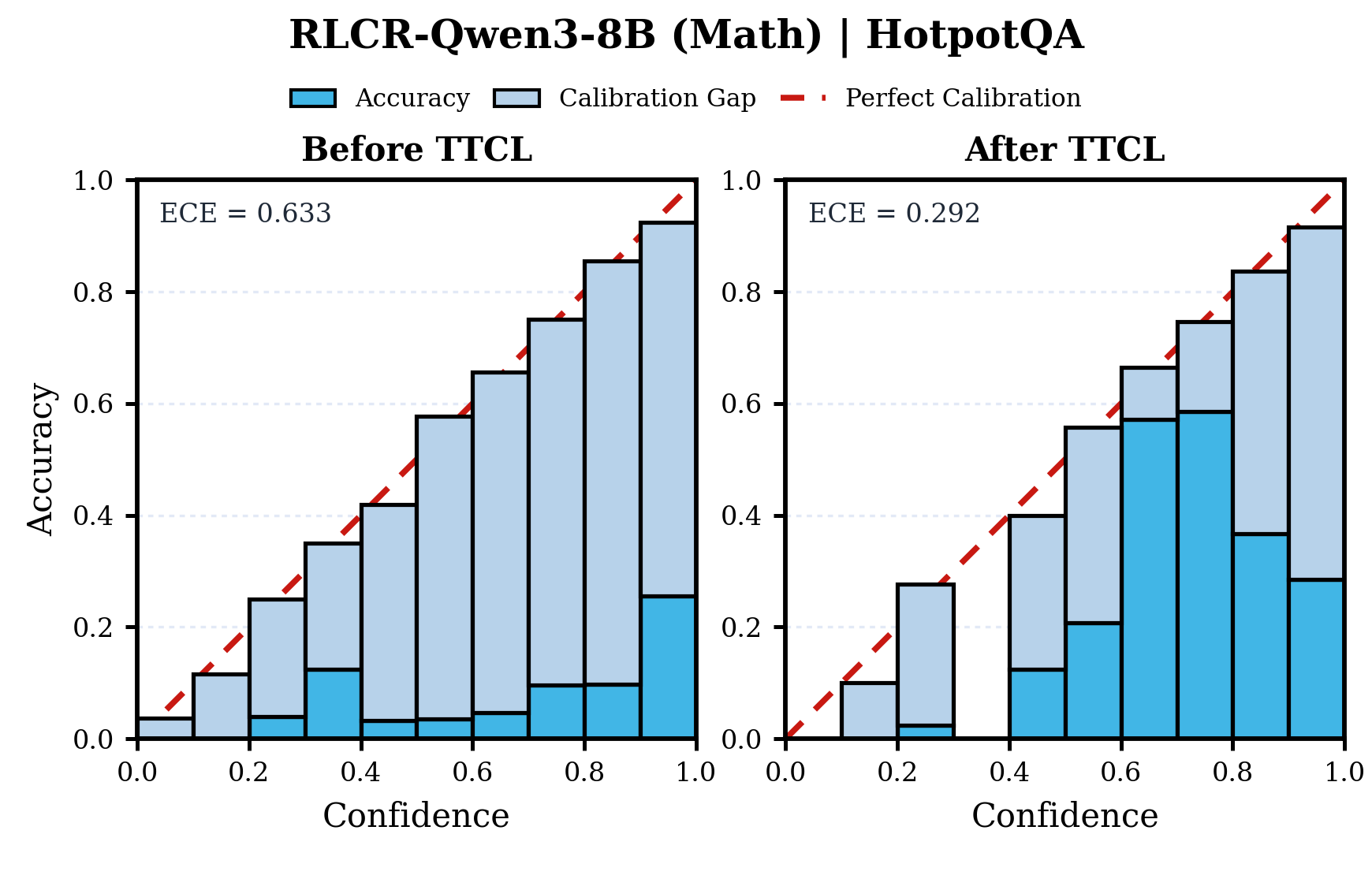}
  \includegraphics[width=0.19\textwidth]{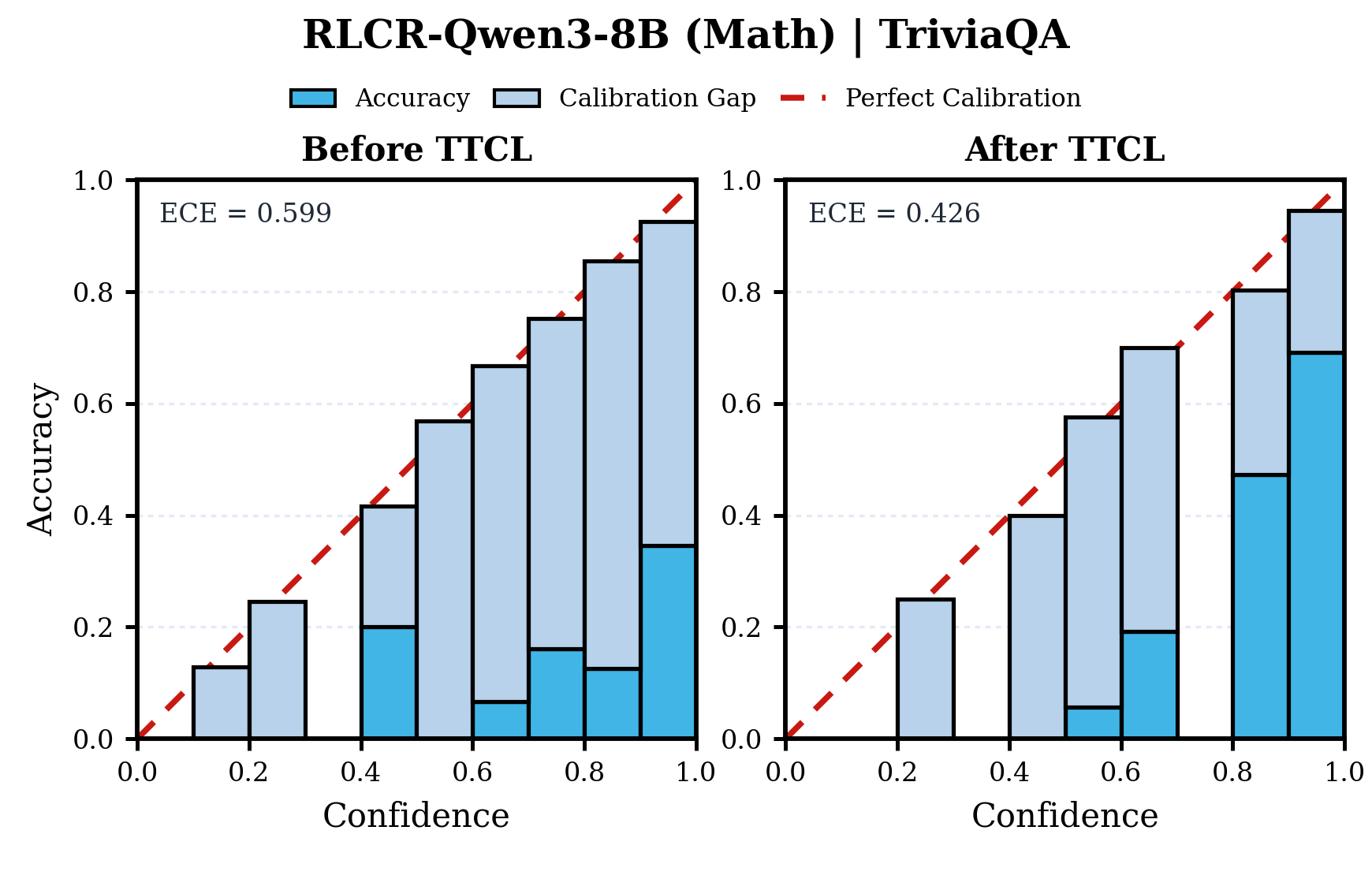} \\
  \includegraphics[width=0.19\textwidth]{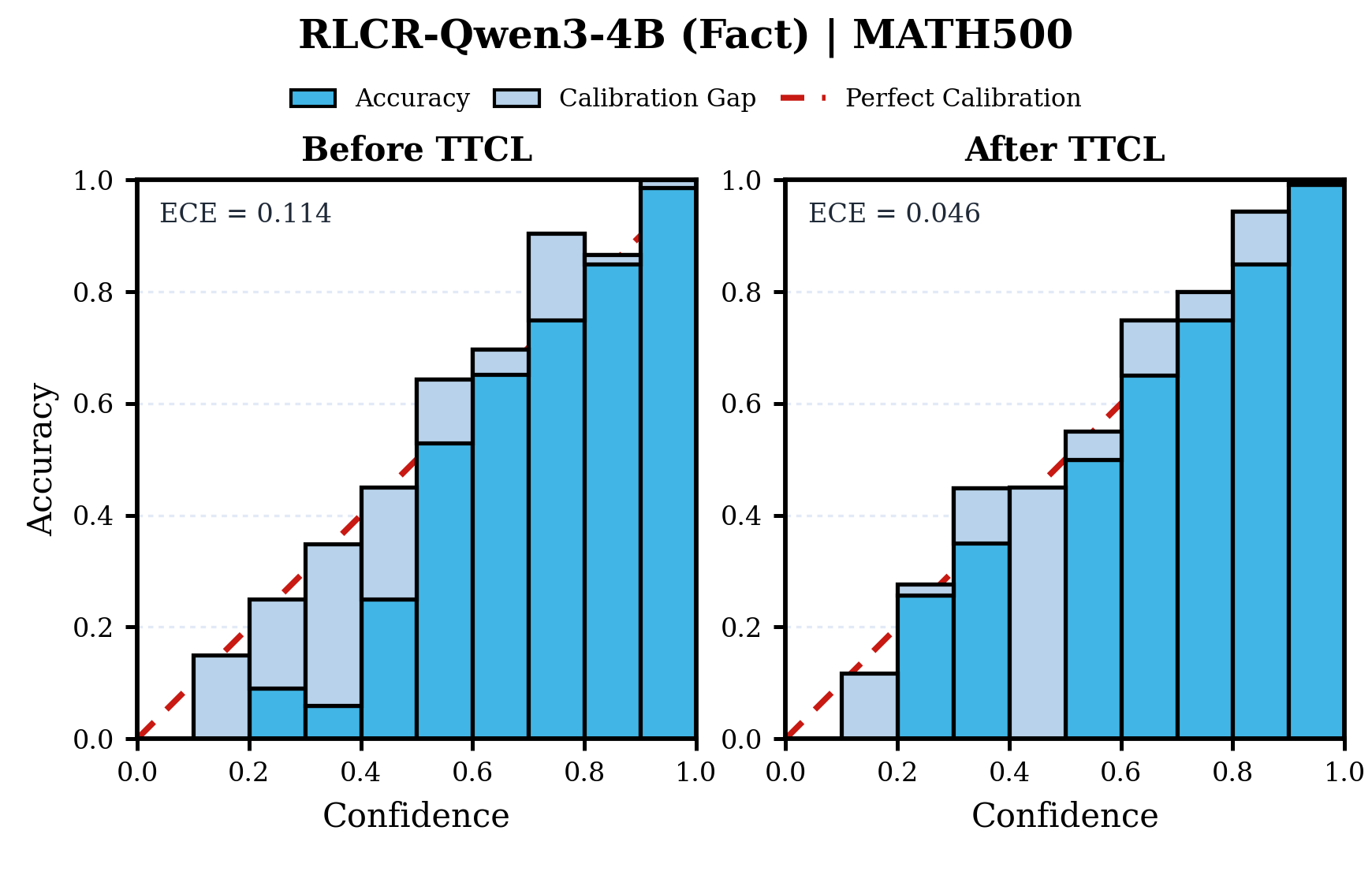}
  \includegraphics[width=0.19\textwidth]{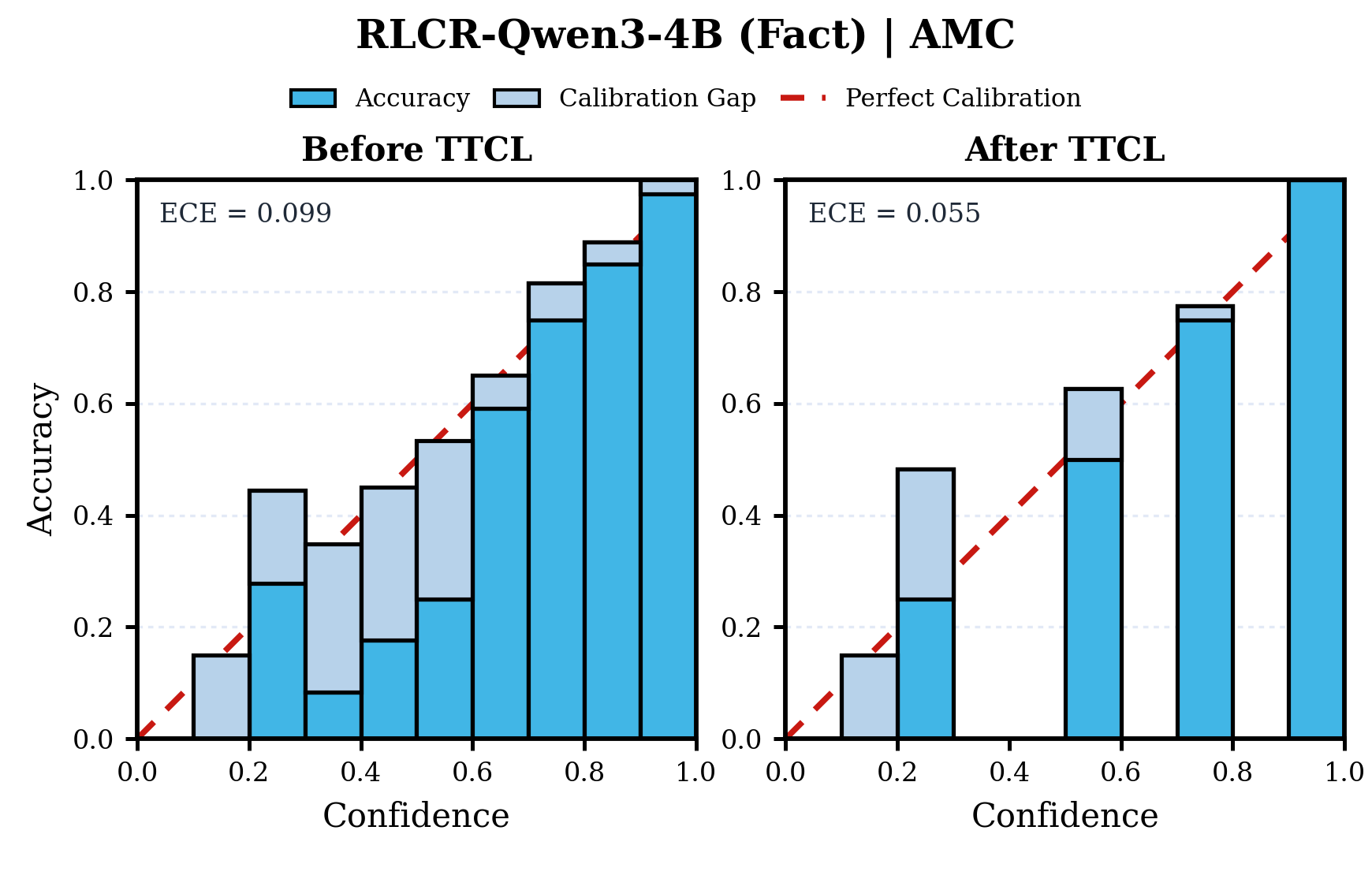}
  \includegraphics[width=0.19\textwidth]{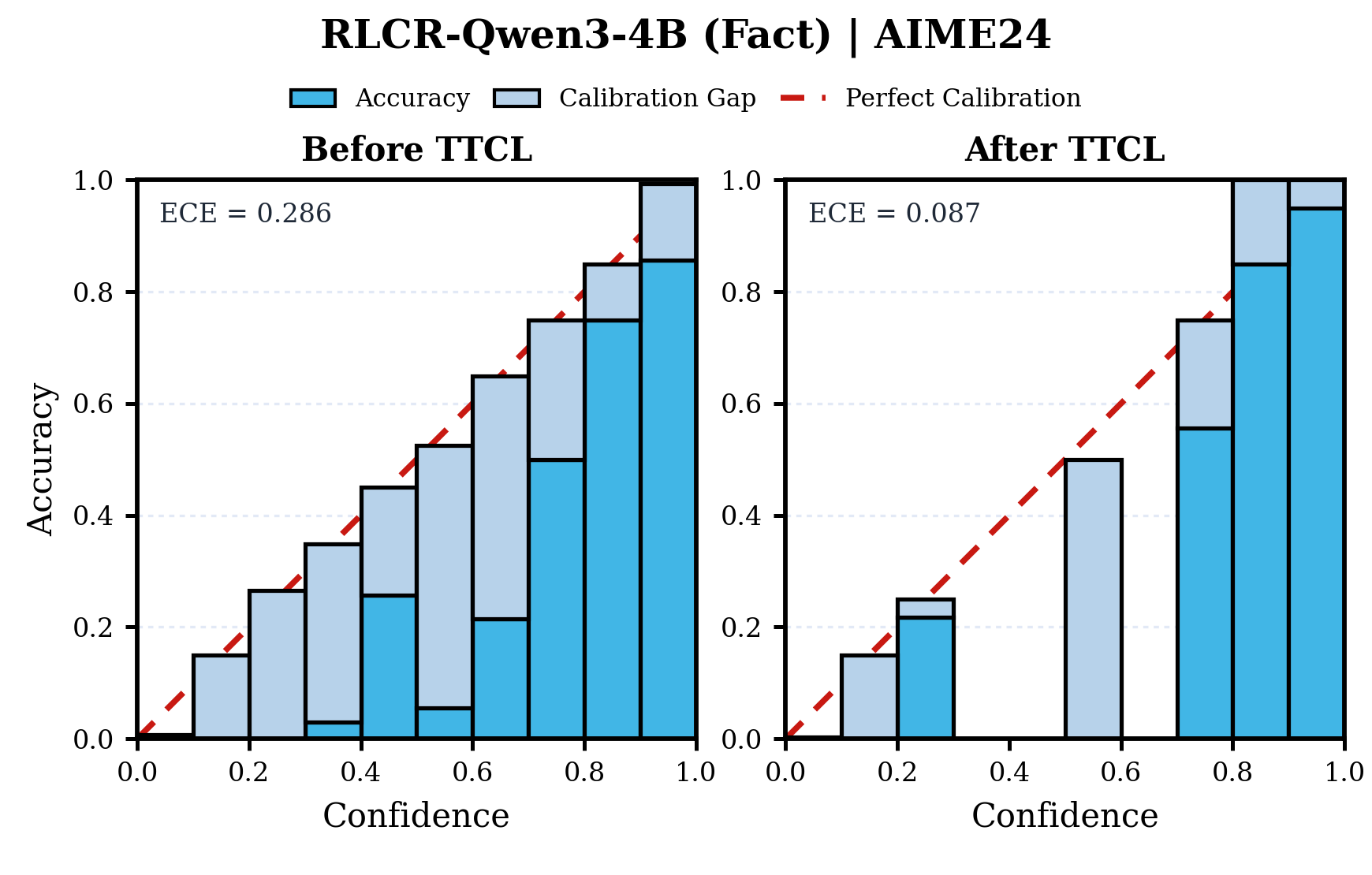} 
  \includegraphics[width=0.19\textwidth]{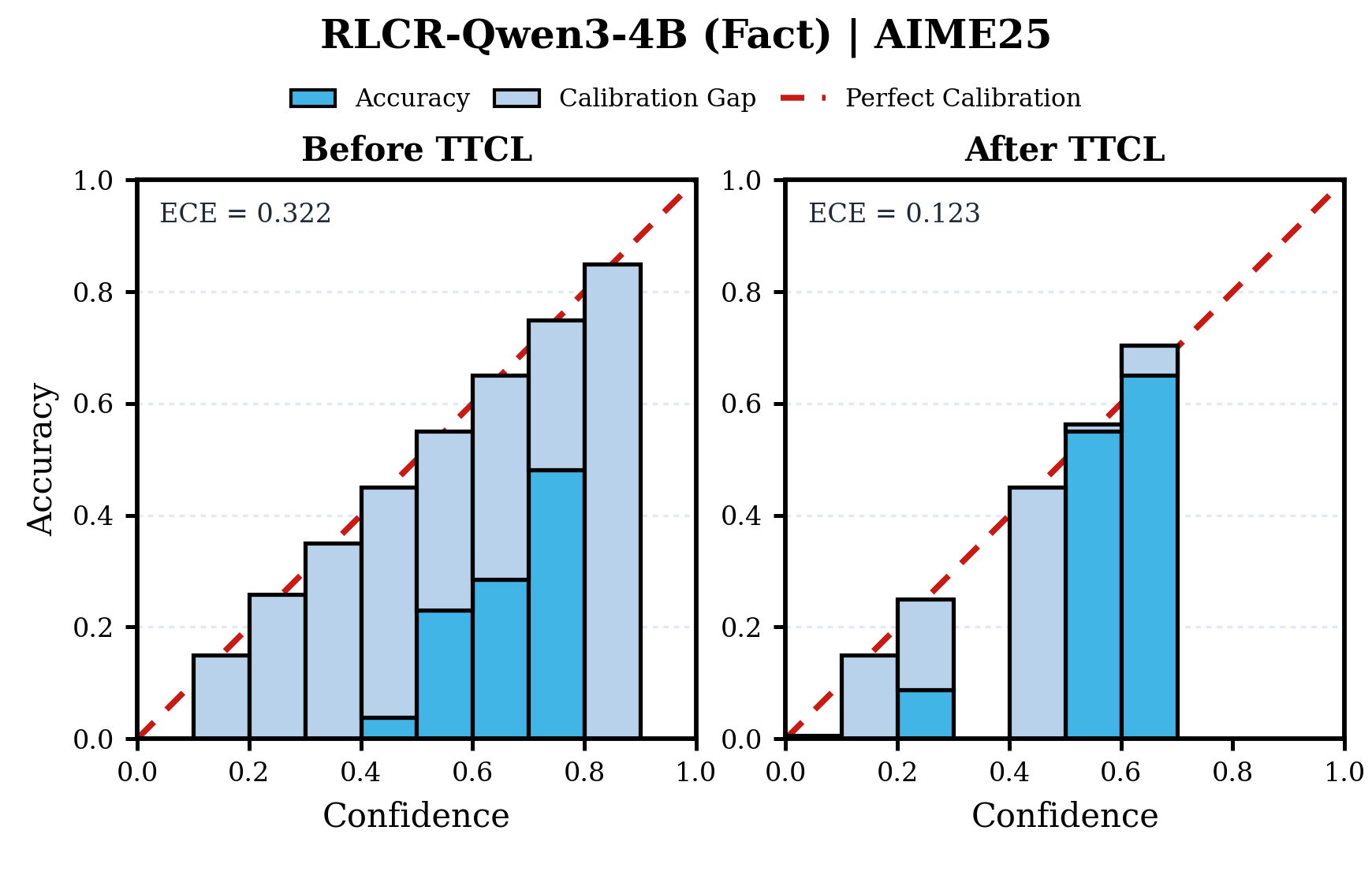}
  \includegraphics[width=0.19\textwidth]{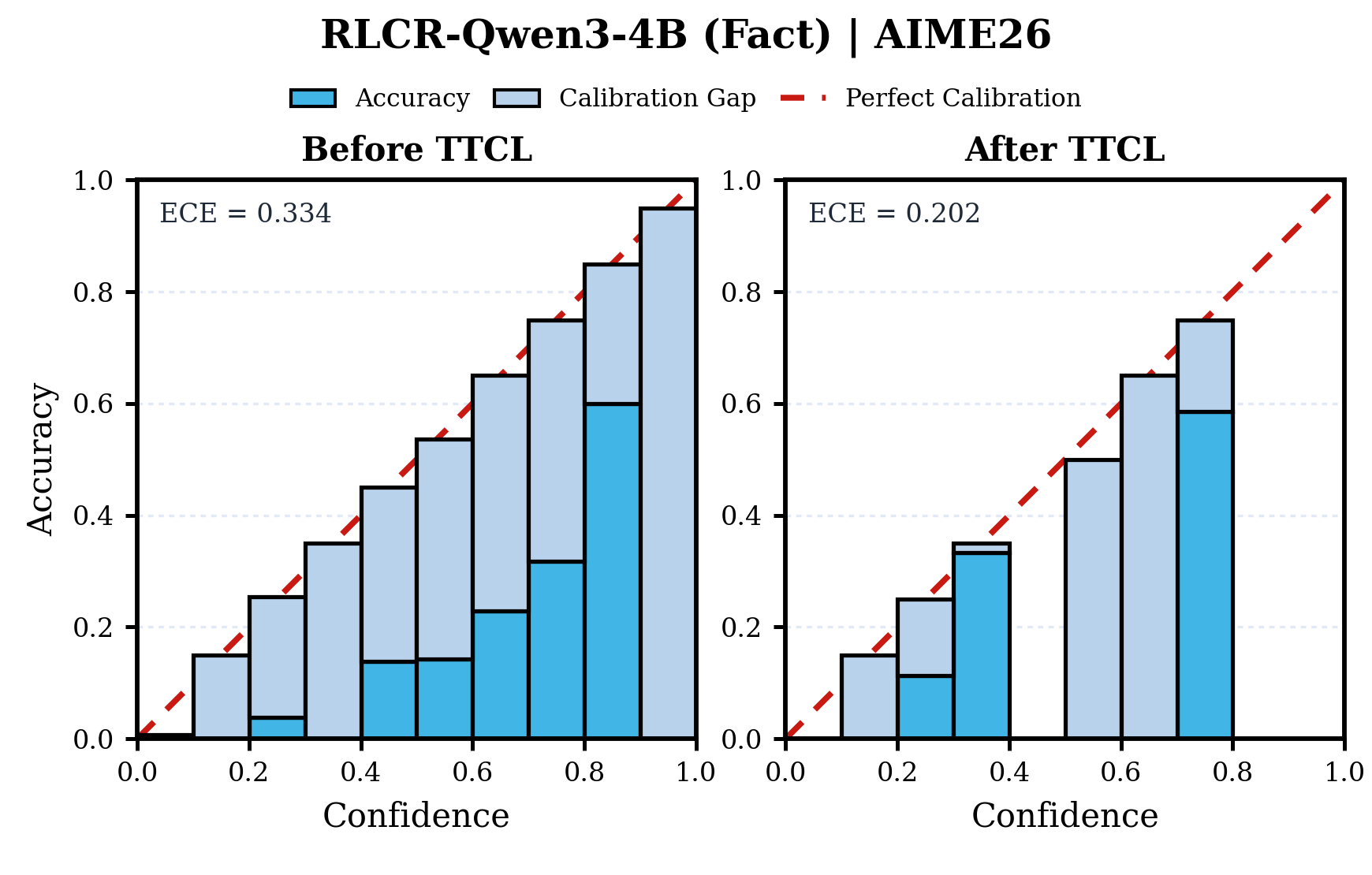}
  \vspace{-2mm}
    \caption{\textbf{Visualization of out-of-domain calibration effects before and after TTCL} on math-calibrated models for factual benchmarks. The calibration learned on the source domain transfers poorly to the target domain in both Math-to-FactQA and FactQA-to-Math settings, whereas TTCL effectively adapts the models to the target benchmarks and further narrows the gap between confidence and empirical accuracy. \textcolor{accuracyblue}{Blue bars} denote empirical accuracy within each confidence bin, while \textcolor{calibrationgapblue}{light blue bars} indicate the deviation from perfect calibration.}

  \label{Fig:TTCL_RLCR_OutDomain_Visualize}
\end{figure}

\subsection{More Details of Metrics}
\label{App:Metrics}
In this work, we evaluate both reasoning correctness and confidence calibration mainly using three metrics: Accuracy (ACC), Expected Calibration Error (ECE), and Area Under the Receiver Operating Characteristic Curve (AUROC). Their formulations and details are introduced as follows:
\begin{itemize}[left=2pt]
\item \textbf{ACC:} Accuracy measures the proportion of correctly answered examples and is used to evaluate the model's reasoning performance. Given \(N\) evaluation examples, it is computed as
\begin{equation}
    \mathrm{ACC} = \frac{1}{N}\sum_{i=1}^{N} \mathbb{I}\left[ \operatorname{Ans}(y_i) \equiv a_i^\star \right],
\end{equation}
where $\mathbb{I}\left[ \operatorname{Ans}(y_i) \equiv a_i^\star\right]$ indicates whether the answer of the $i$-th response is correct, and $a_i^\star$ is the ground-truth answer. A higher ACC indicates better correctness of the reasoning capability. 

\item \textbf{ECE:} ECE measures the discrepancy between the model's predicted confidence and its empirical correctness. The formulation of ECE is introduced in Eq.~(\ref{Eq:ECE}). A lower ECE indicates better alignment between confidence estimates and actual correctness. Following previous work~\citep{damani2026RLCR,ma2026decoupling}, we set the number of bins to 10 when computing ECE.

\item \textbf{AUROC:} AUROC measures how well the estimated confidence distinguishes correct responses from incorrect ones. By varying the confidence threshold, we obtain the true-positive rate (TPR) and false-positive rate (FPR), and compute AUROC as follows:
\begin{equation}
    \mathrm{AUROC} = \int_{0}^{1} \mathrm{TPR}(u)\, d\mathrm{FPR}(u).
\end{equation}
A higher AUROC indicates that correct responses tend to receive higher confidence than incorrect ones. Unlike ECE, which evaluates the absolute alignment between predicted confidence and empirical correctness, AUROC mainly reflects the ranking quality of confidence scores.
\end{itemize}

\begin{figure}[t]
  \centering
  \includegraphics[width=0.24\textwidth]{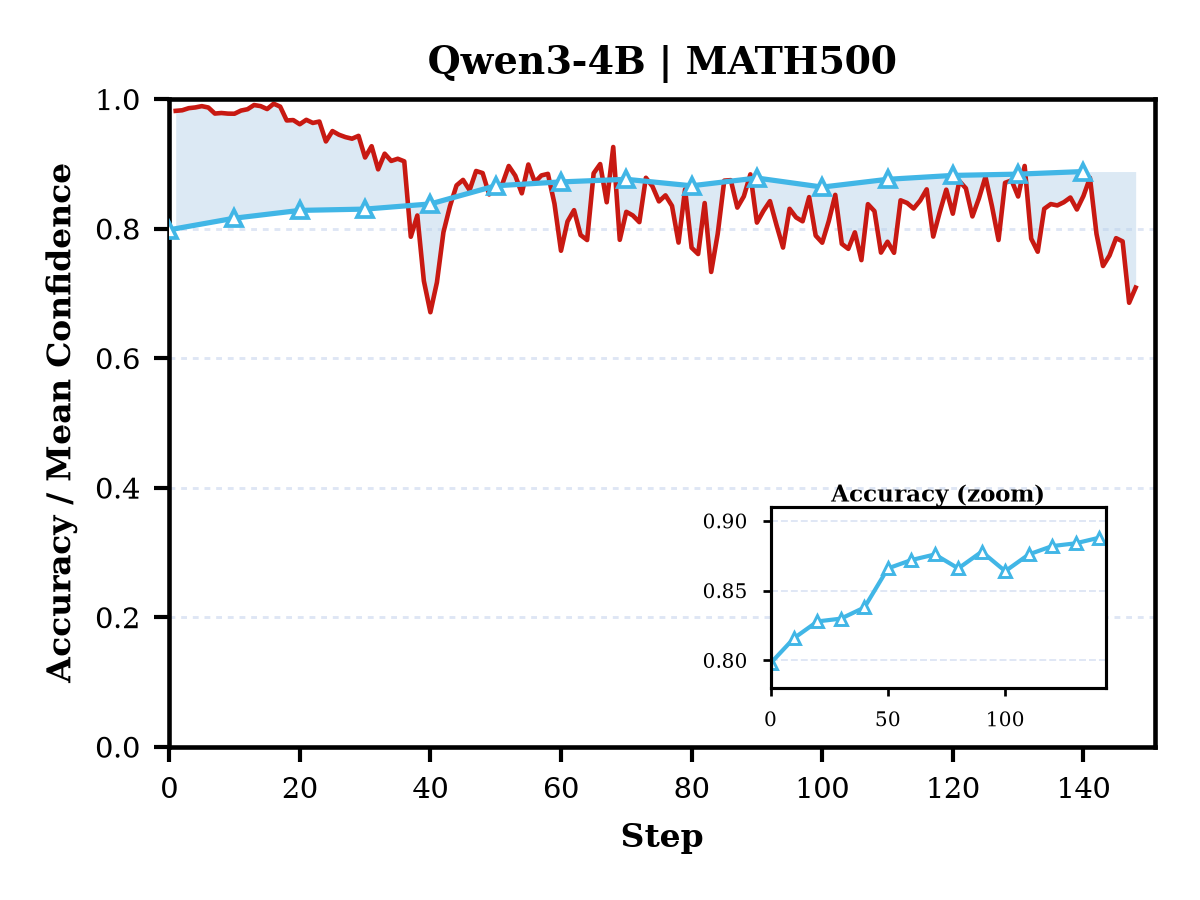}
  \includegraphics[width=0.24\textwidth]{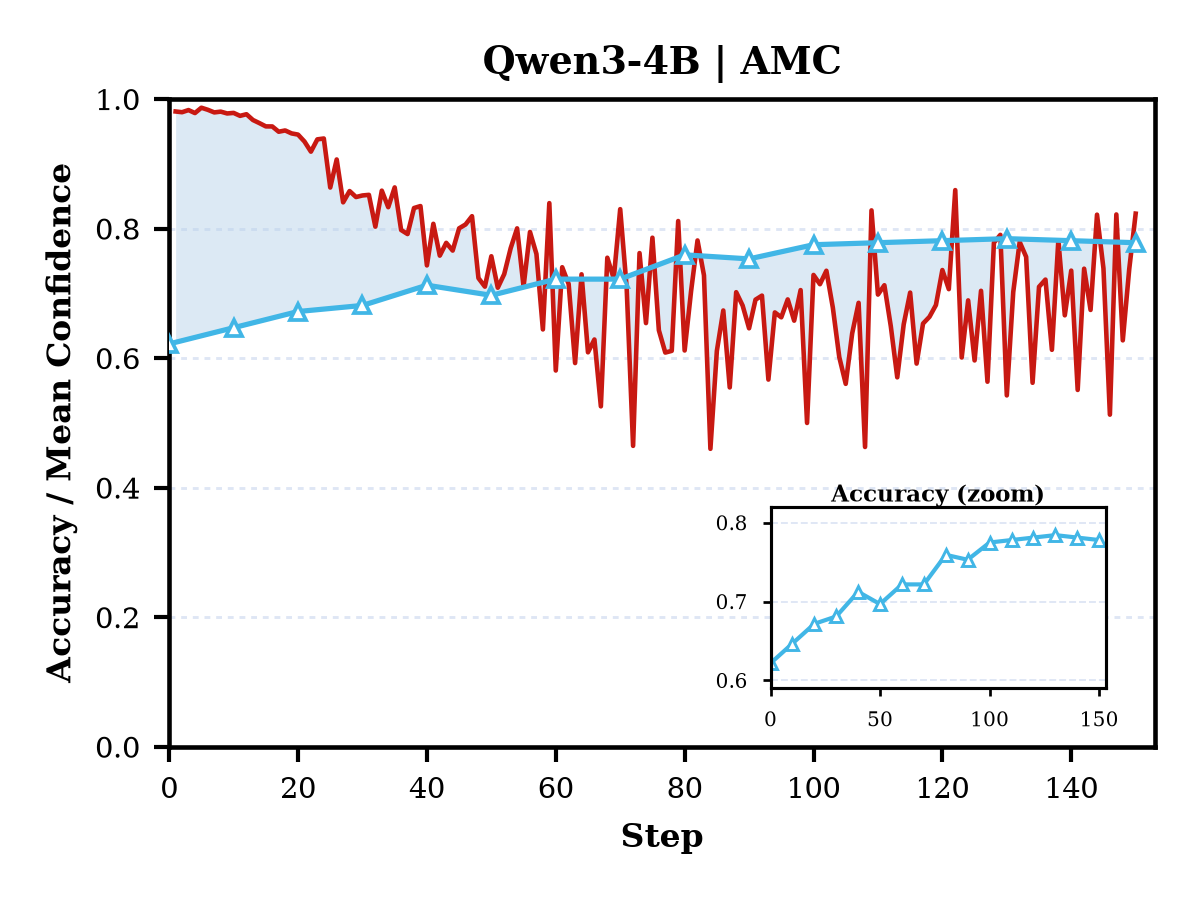}
  \includegraphics[width=0.24\textwidth]{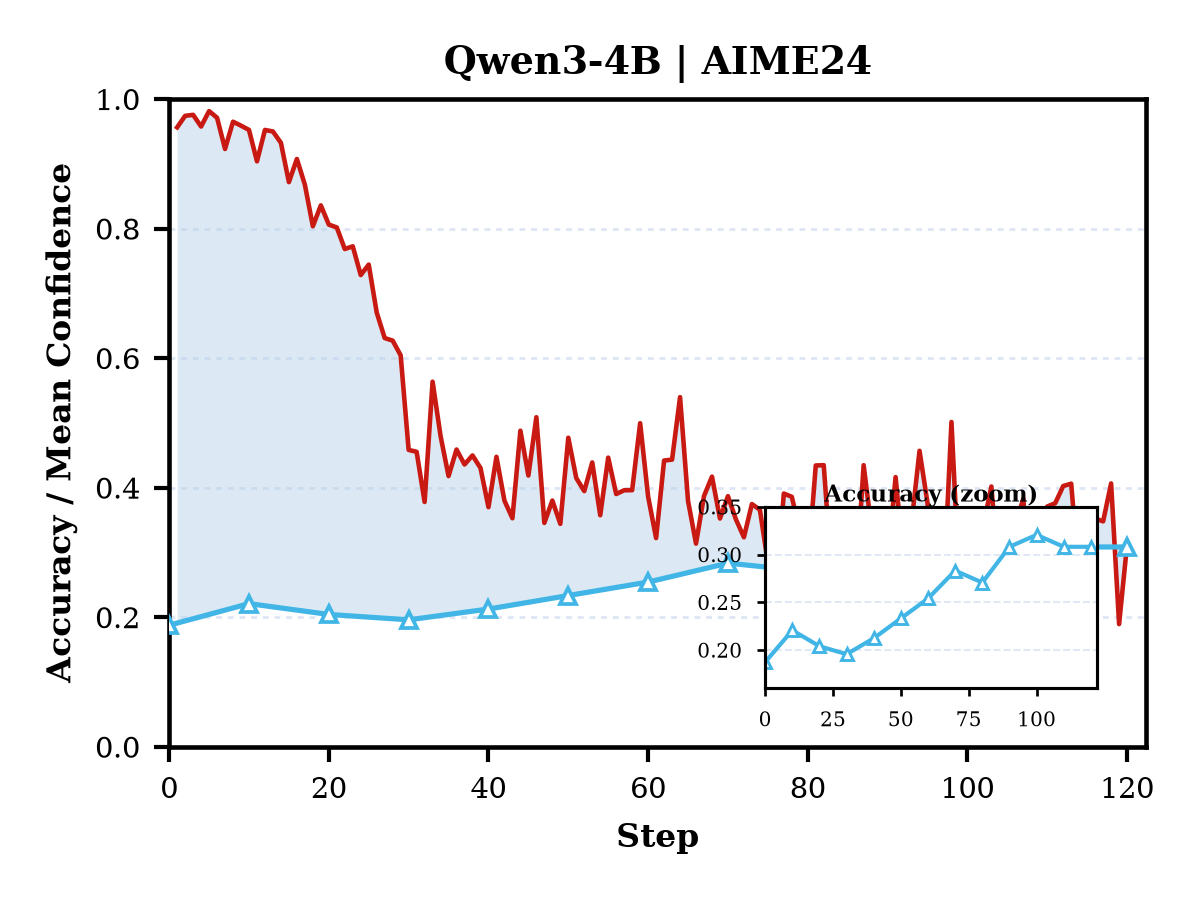}
  \includegraphics[width=0.24\textwidth]{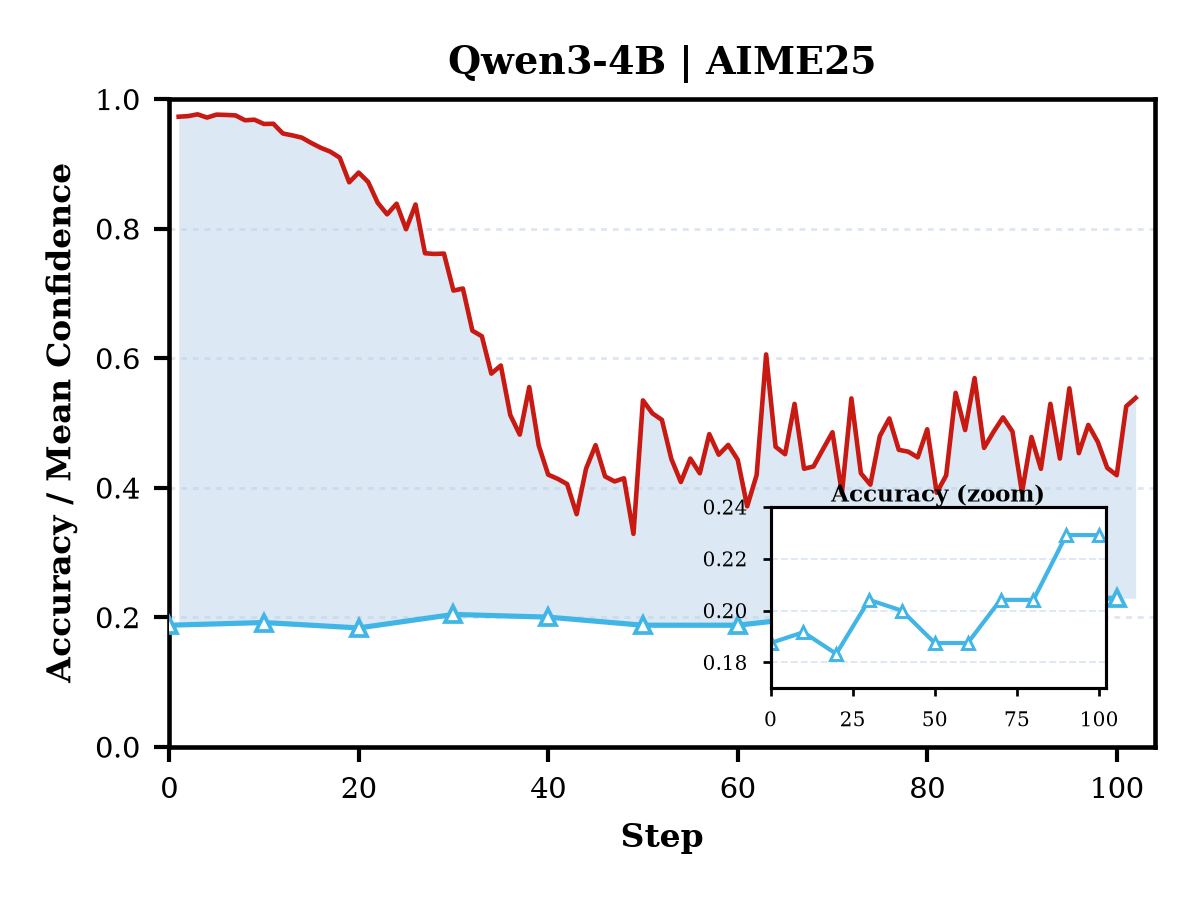} \\
  \includegraphics[width=0.24\textwidth]{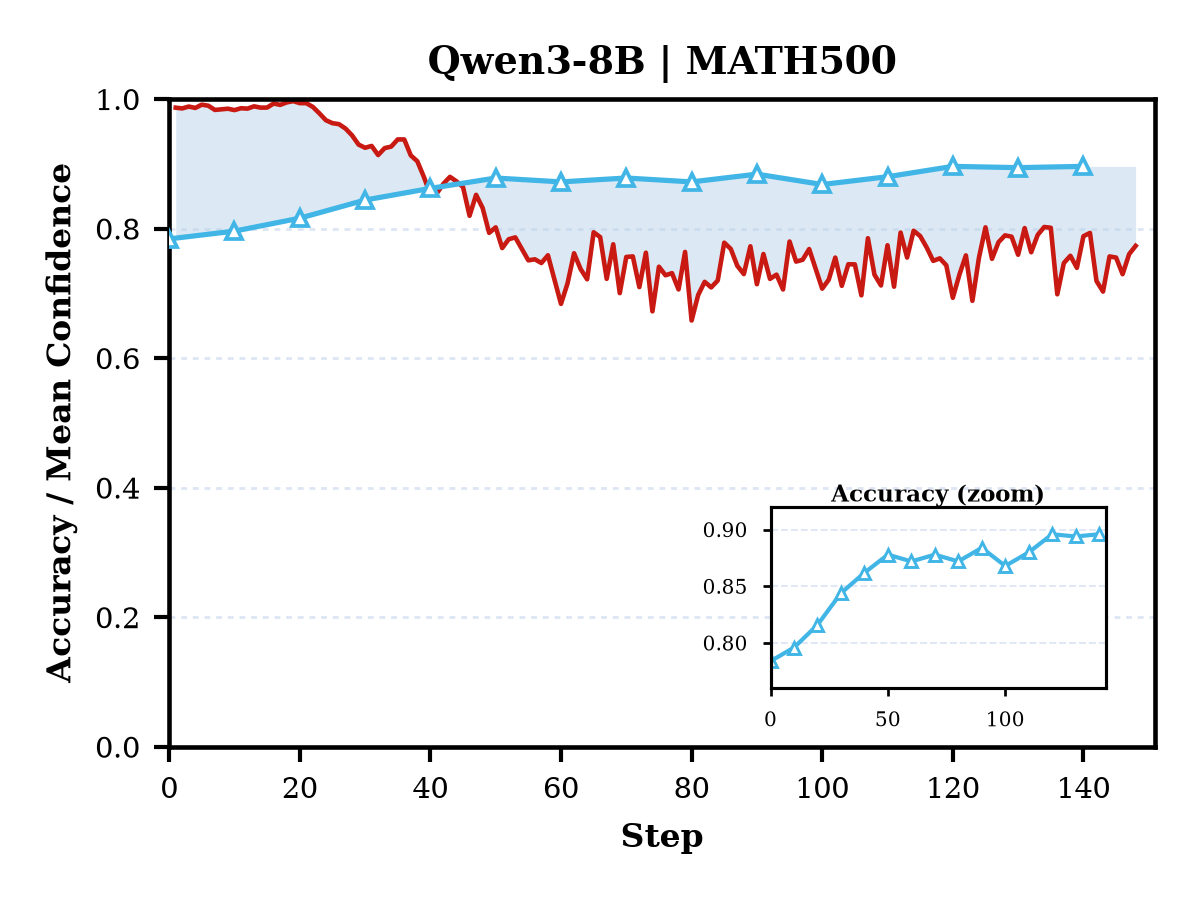}
  \includegraphics[width=0.24\textwidth]{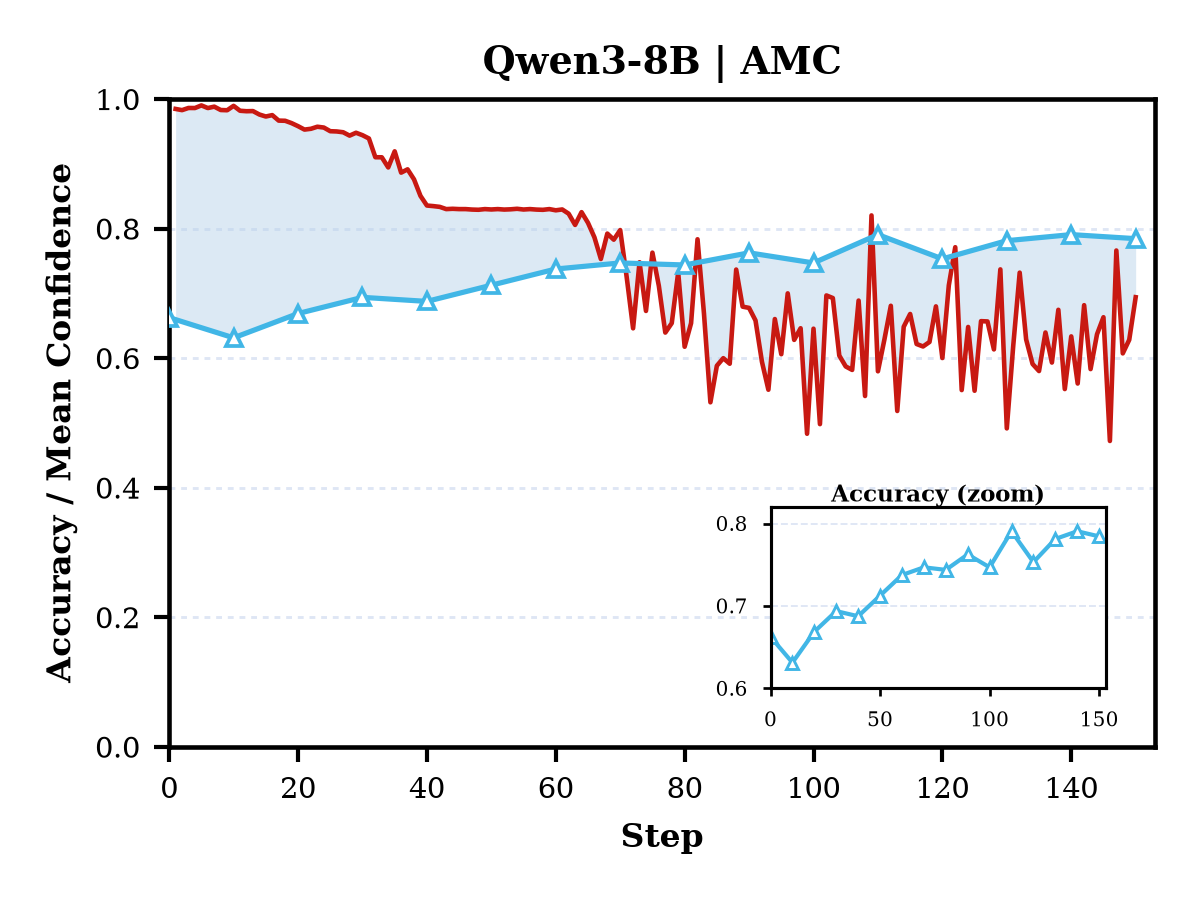}
  \includegraphics[width=0.24\textwidth]{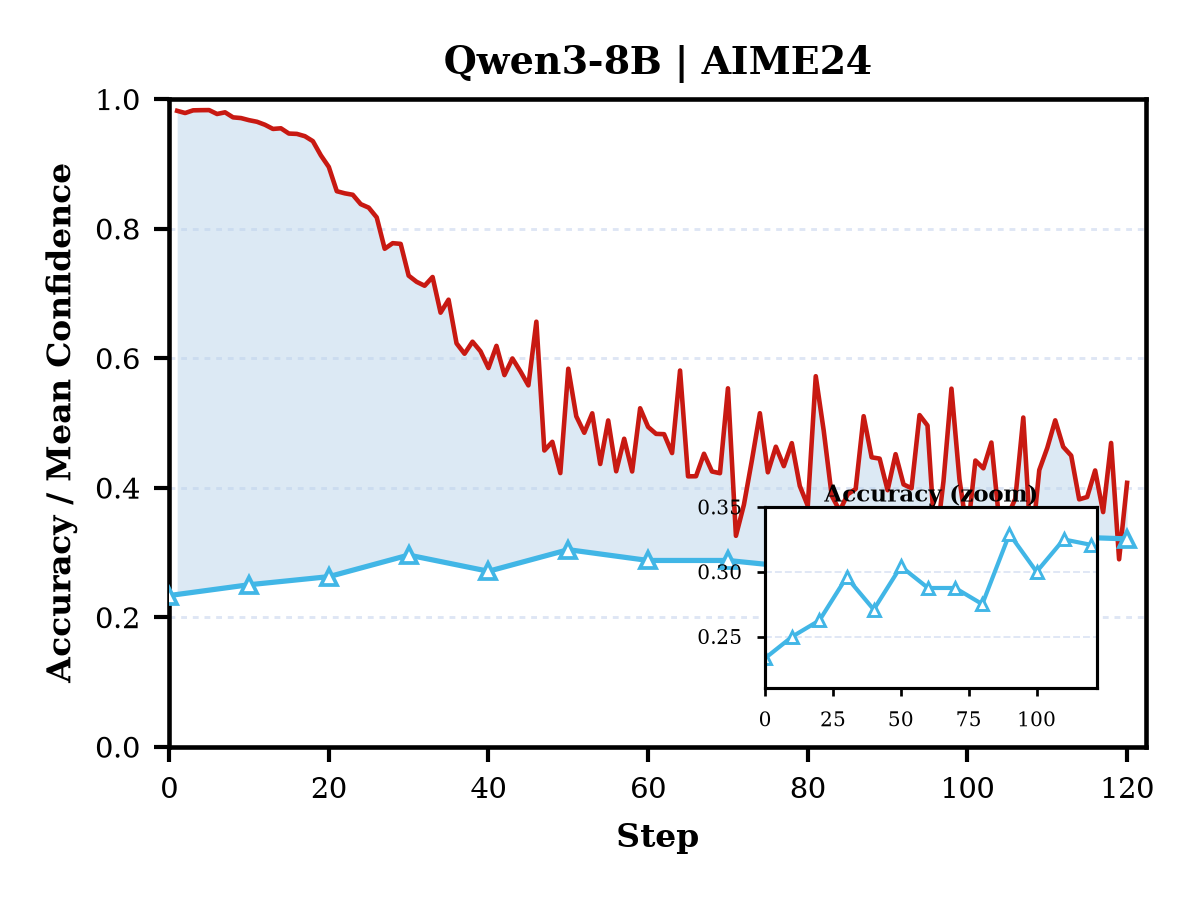}
  \includegraphics[width=0.24\textwidth]{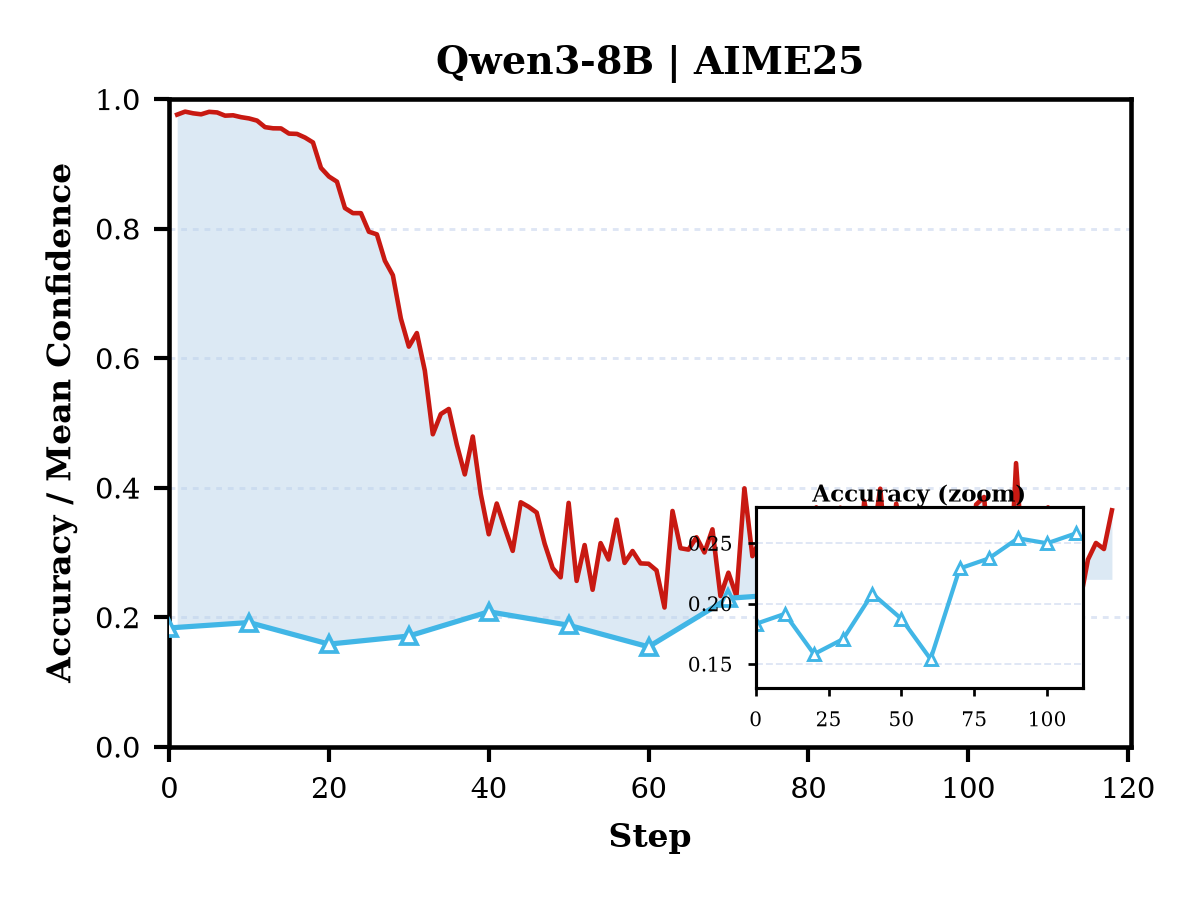} \\
  \includegraphics[width=0.24\textwidth]{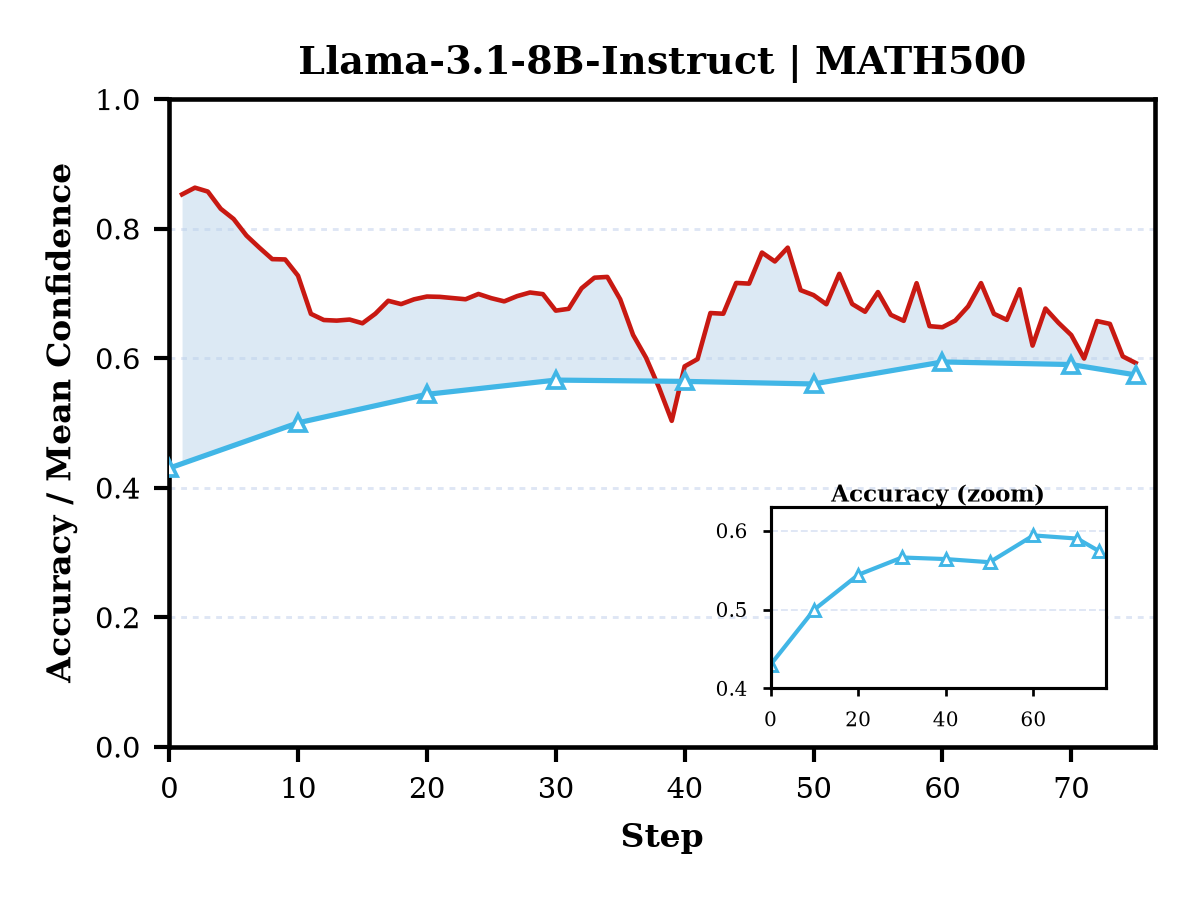}
  \includegraphics[width=0.24\textwidth]{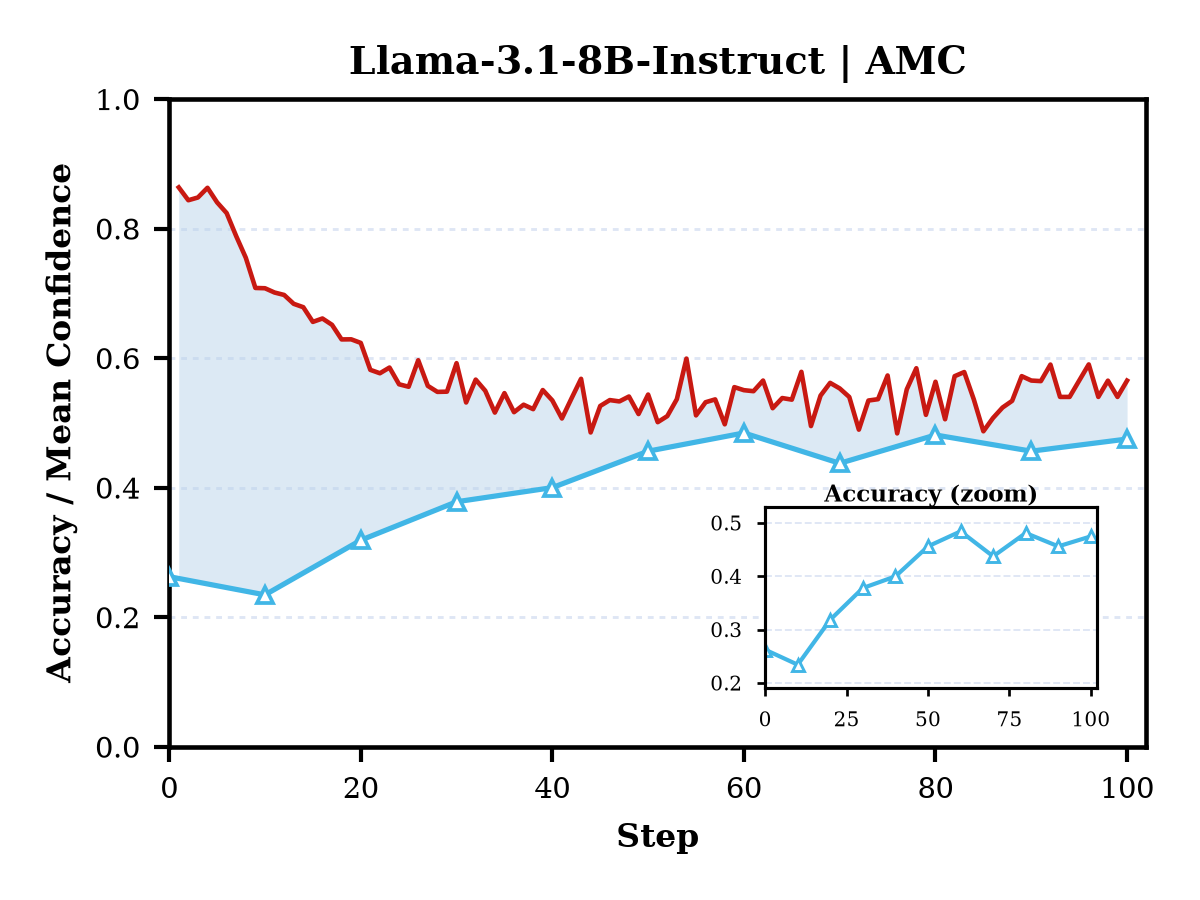}
  \includegraphics[width=0.24\textwidth]{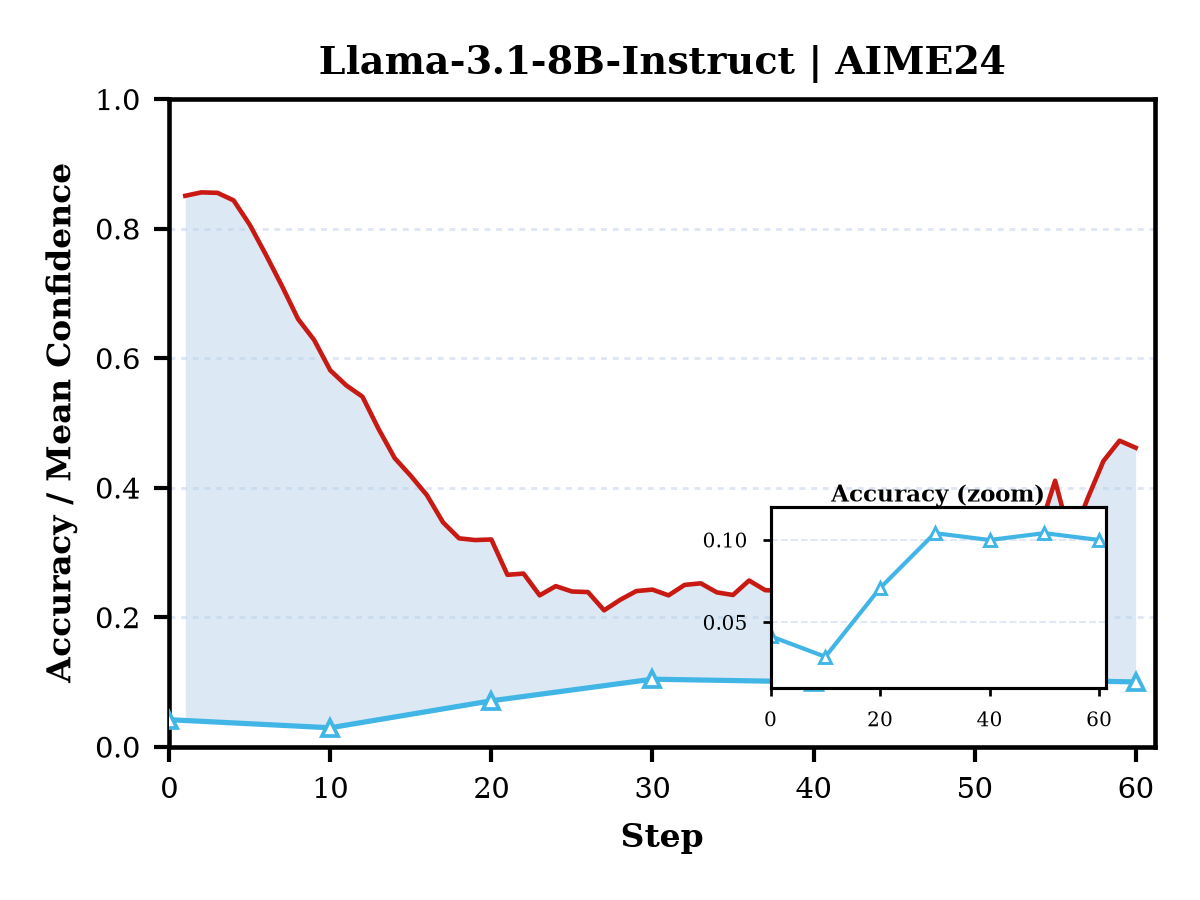}
  \includegraphics[width=0.24\textwidth]{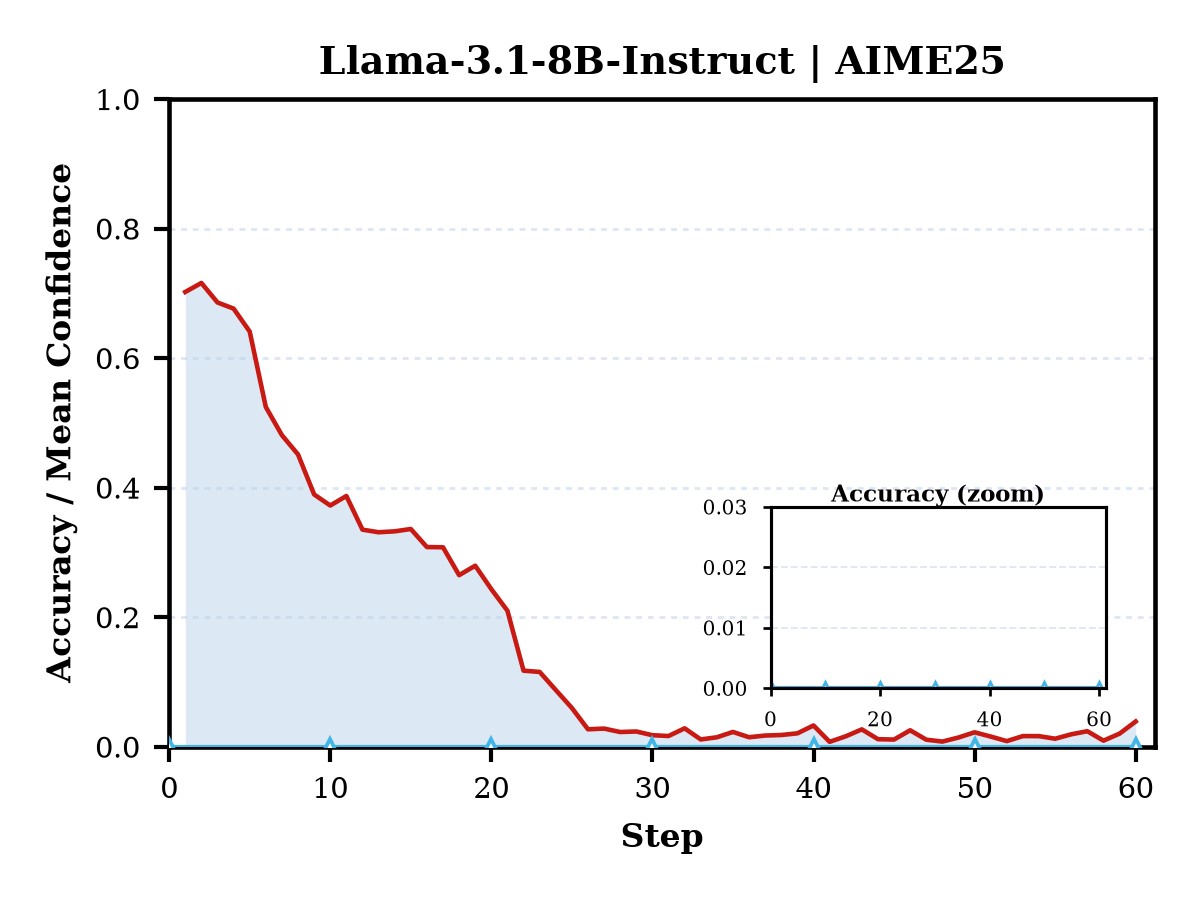}
  \vspace{-2mm}
  \caption{\textbf{More training dynamics of TTCL on base models.} The \textcolor{softred}{red curves} show mean confidence and the \textcolor{cyan}{cyan-blue curves} show empirical accuracy during TTCL training. The shaded area indicates the confidence-accuracy gap, and the inset zooms in on the accuracy trajectory. TTCL progressively aligns confidence with empirical accuracy while improving task accuracy.}
  \label{Fig:TTCL_Base_Dynamics_Appendix}
\end{figure}

\begin{figure}[t!]
  \centering
  \includegraphics[width=0.24\textwidth]{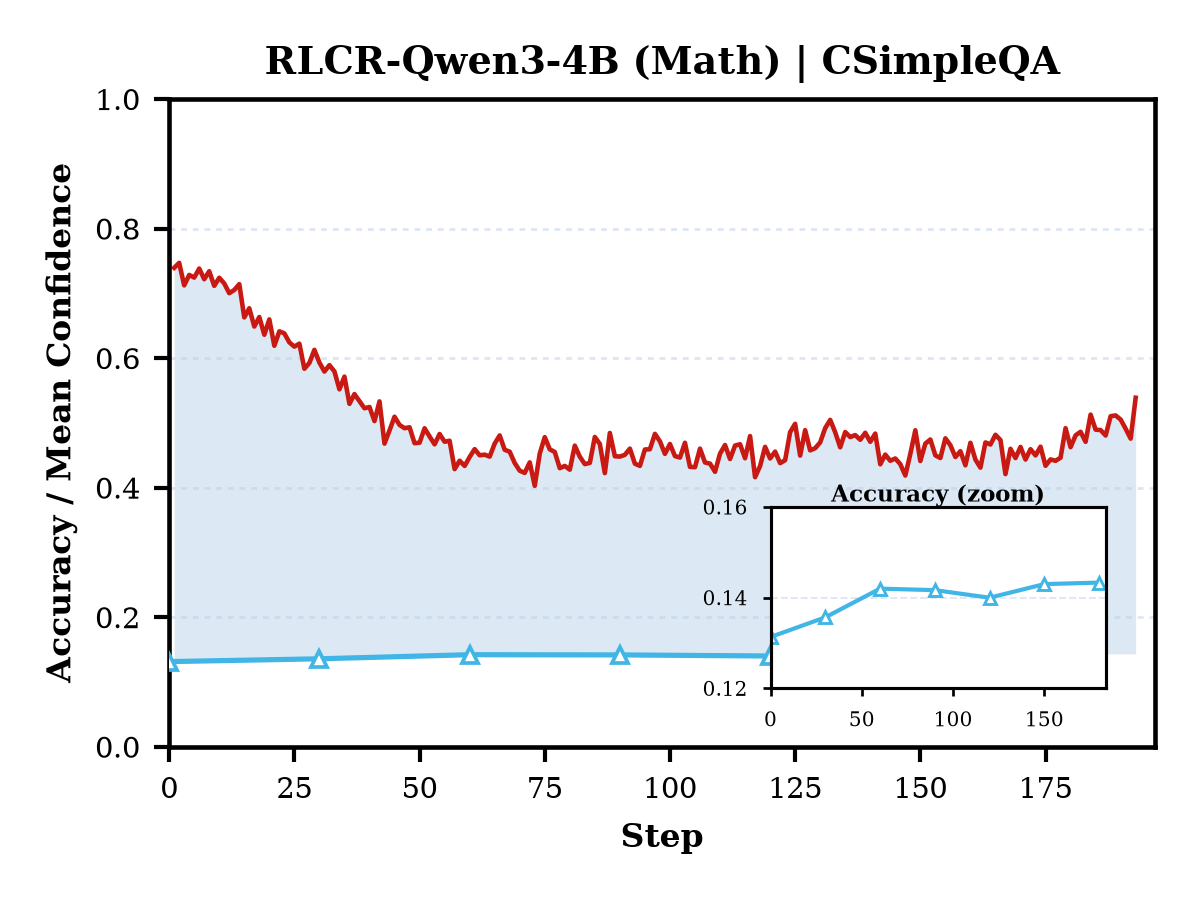}
  \includegraphics[width=0.24\textwidth]{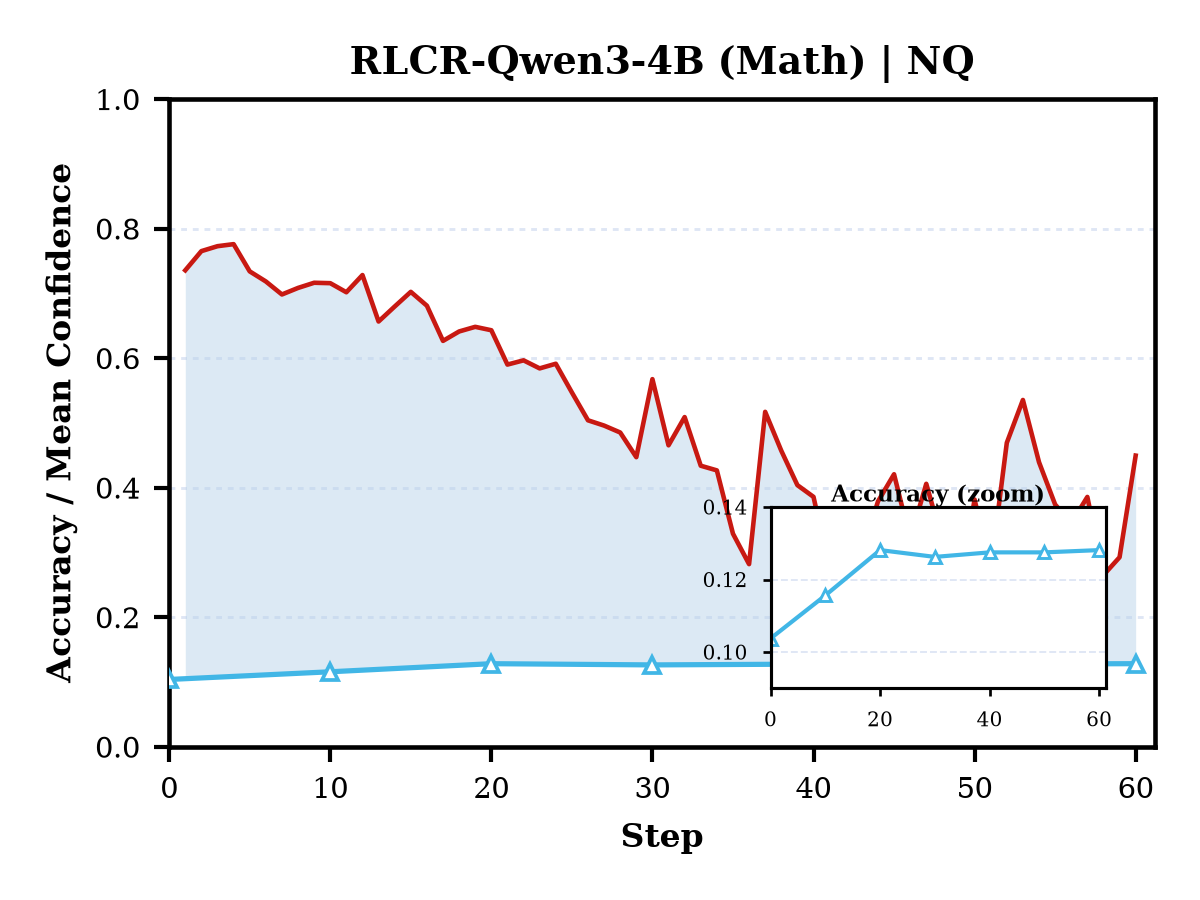}
  \includegraphics[width=0.24\textwidth]{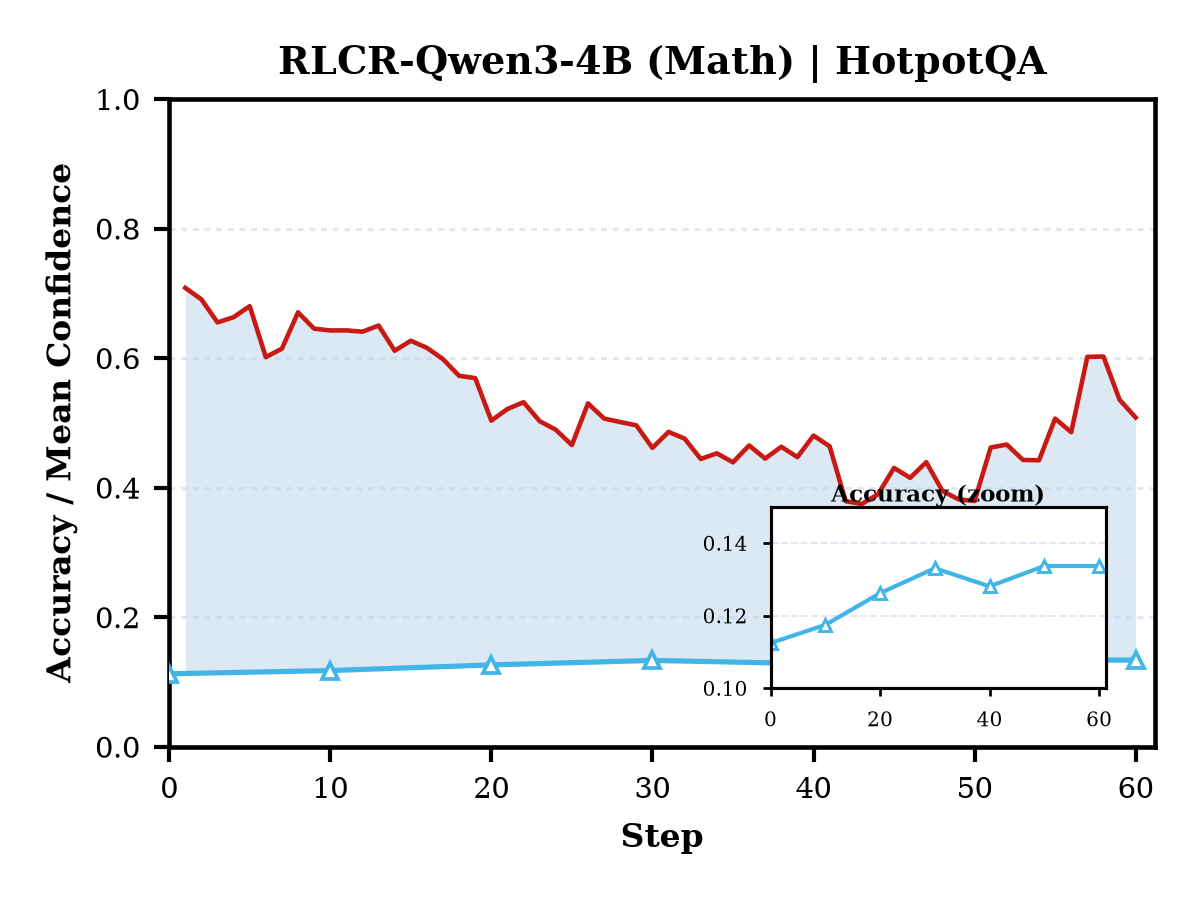}
  \includegraphics[width=0.24\textwidth]{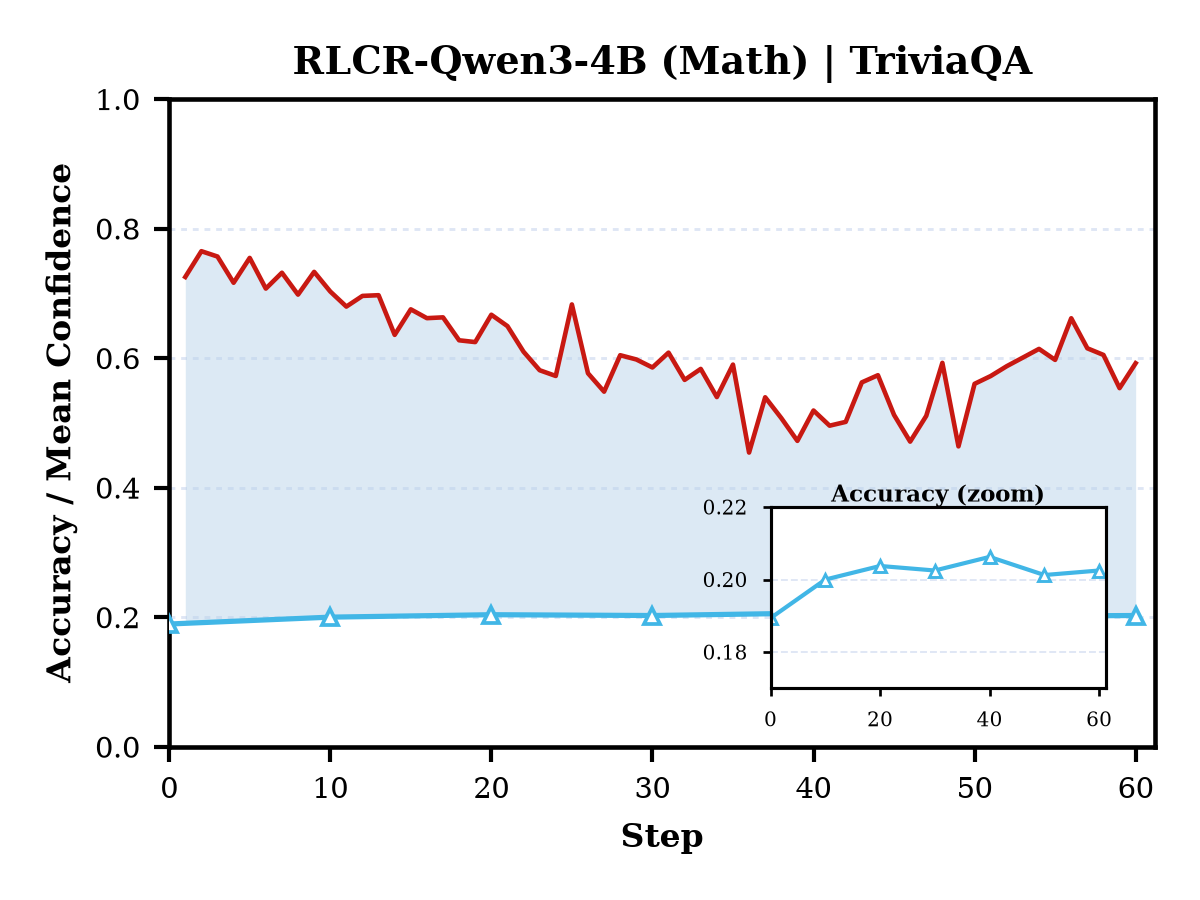} \\
  \includegraphics[width=0.24\textwidth]{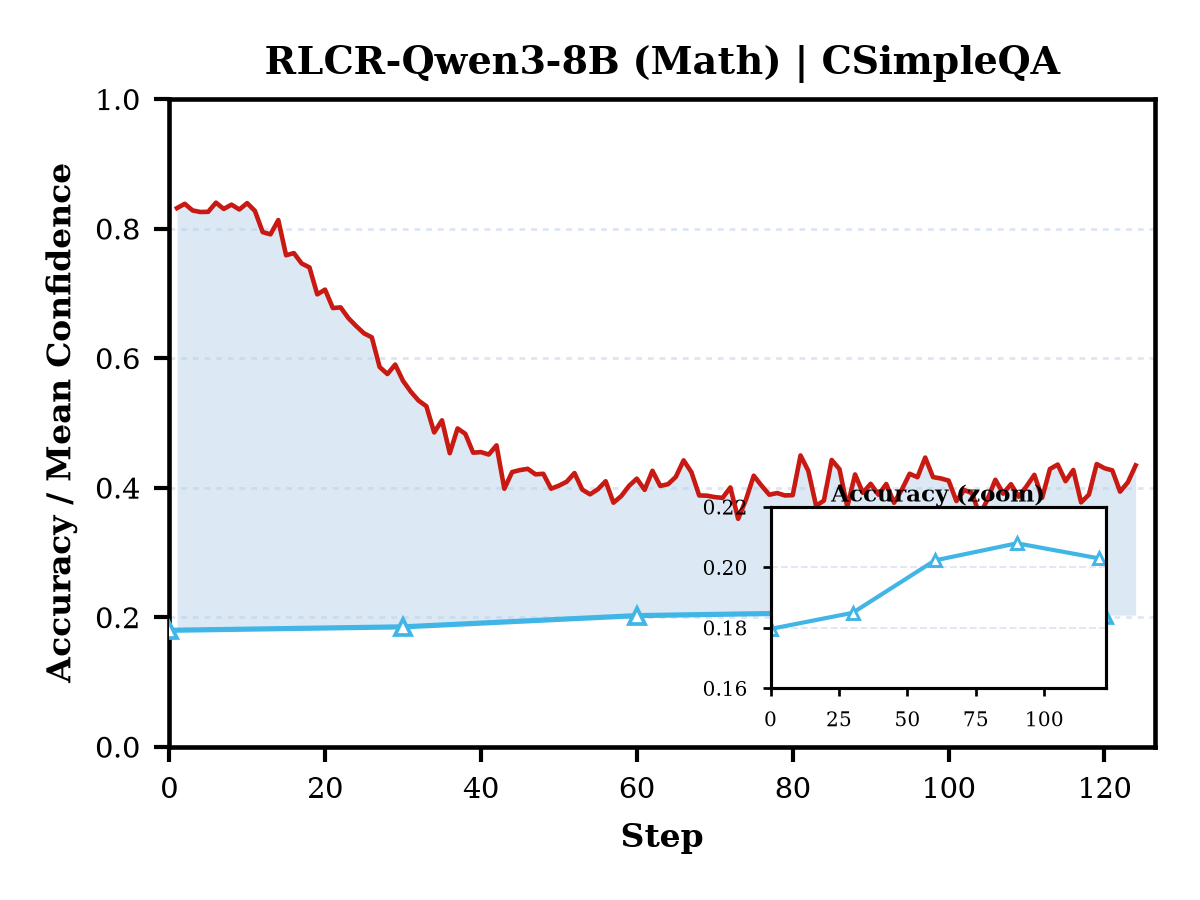}
  \includegraphics[width=0.24\textwidth]{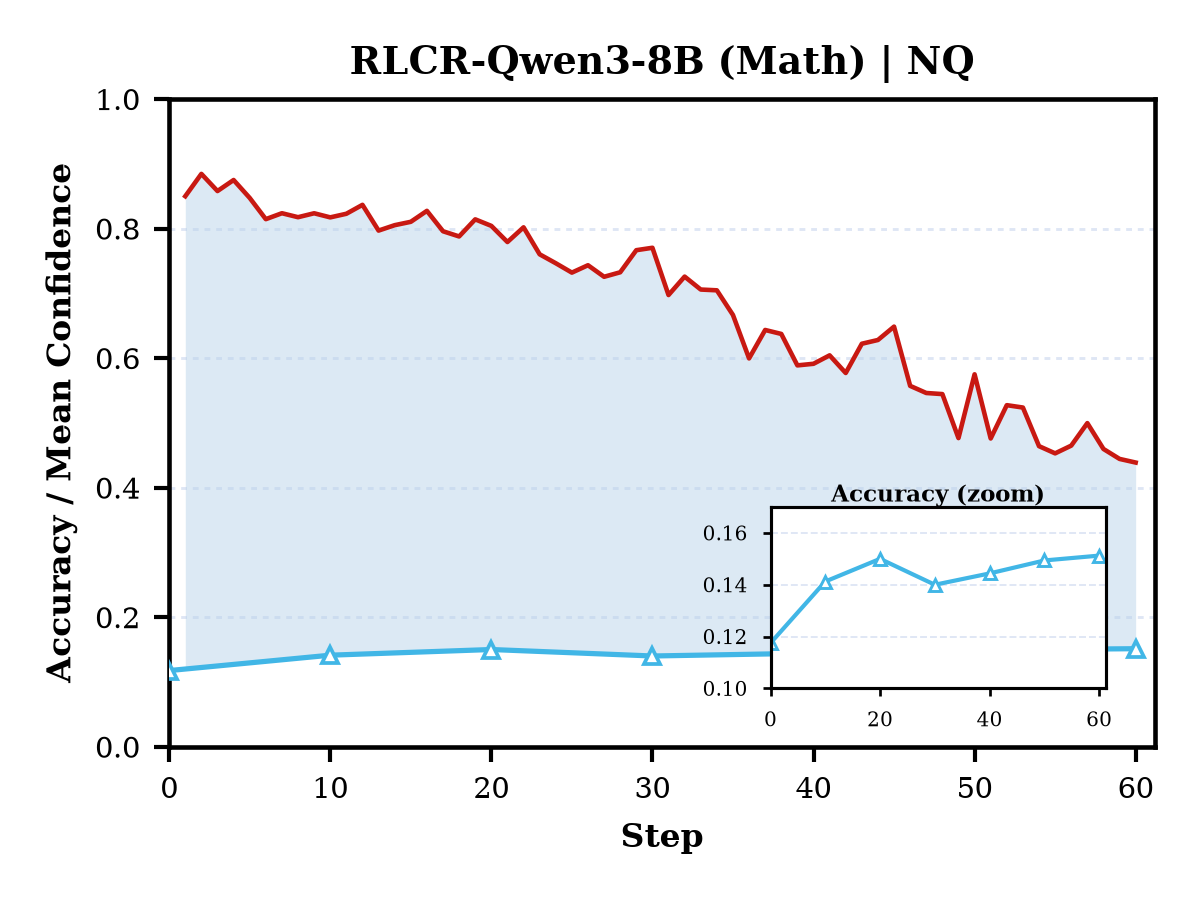}
  \includegraphics[width=0.24\textwidth]{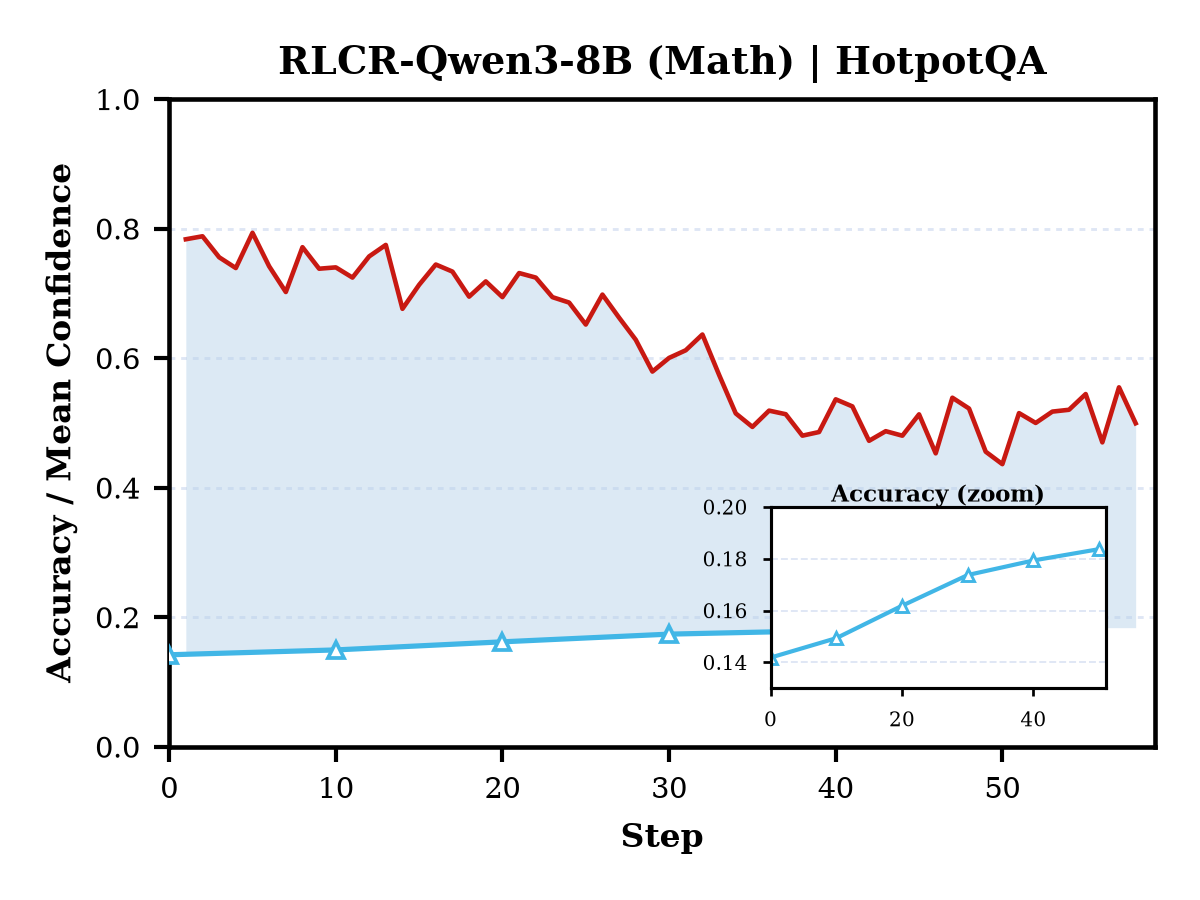}
  \includegraphics[width=0.24\textwidth]{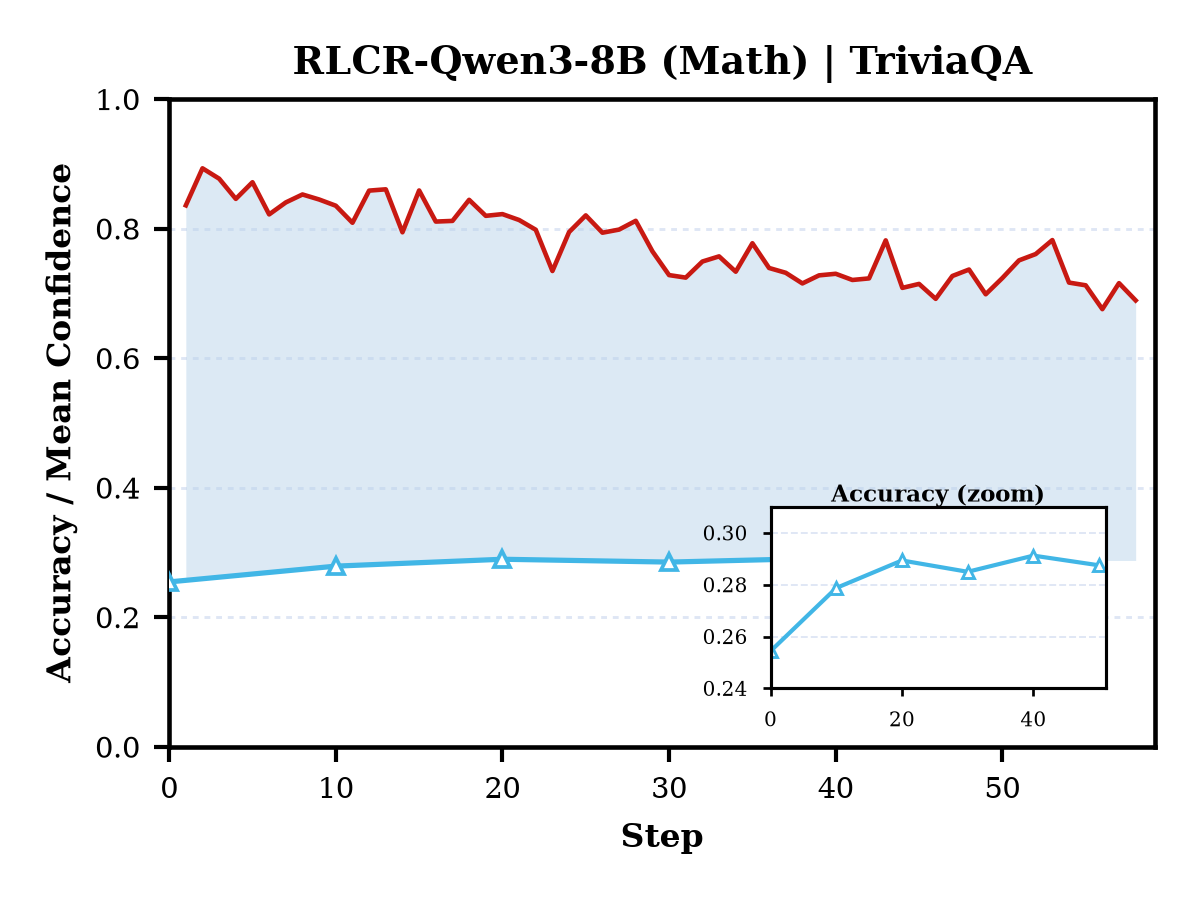}
  \vspace{-2mm}
  \caption{\textbf{More training dynamics of TTCL on calibrated models.} The \textcolor{softred}{red curves} show mean confidence and the \textcolor{cyan}{cyan-blue curves} show empirical accuracy during TTCL training. The shaded area indicates the confidence-accuracy gap, and the inset zooms in on the accuracy trajectory. TTCL progressively aligns confidence with empirical accuracy while improving task accuracy.}
  \label{Fig:TTCL_Calibrated_Dynamics_Appendix}
\end{figure}

For AMC/AIME24/AIME25/AIME26 and NQ/HotpotQA/TriviaQA, due to their relatively small dataset sizes, we sample eight responses per question to reduce the variance of the computed metrics. For MATH500, SimpleQA and ChineseSimpleQA, we sample one response per question.

\section{Additional Experimental Results}
\subsection{More Results of Calibration Performance}
\label{App:Brier_Conf}
As a supplement, Table~\ref{Tble:TTCL_Base_Math_BrierConf}, Table~\ref{Tble:TTCL_Base_Fact_BrierConf}, Table~\ref{Tble:TTCL_InDomain_BrierConf} and Table~\ref{Tble:TTCL_OutDomain_BrierConf} additionally introduce the metrics of Brier score and mean confidence. The Brier score measures the squared discrepancy between predicted confidence and actual correctness, with a lower score indicating better confidence calibration. The mean confidence, when considered together with accuracy, provides a complementary view of the model's overall confidence tendency toward over-confidence or under-confidence.

\par
\textbf{Existing base models exhibit substantial over-confidence across both mathematical reasoning and factual QA benchmarks.} 
As shown in Table~\ref{Tble:TTCL_Base_Math_BrierConf} and Table~\ref{Tble:TTCL_Base_Fact_BrierConf}, it can be observed that the mean confidence of the base models is often considerably higher than their empirical accuracy. This mismatch is especially evident on several benchmarks where the models maintain high confidence despite relatively low correctness, resulting in large Brier scores. These results further motivate the need for calibration learning of the models to provide reliable confidence estimates.

\par
\textbf{TTCL consistently lowers the Brier score and makes mean confidence more closely aligned with empirical accuracy.} 
Across all four tables, it can be observed that TTCL consistently reduces the Brier score and narrows the gap between accuracy and mean confidence, providing additional evidence that TTCL indeed performs effective calibration learning for the models. For results of the base models in Table~\ref{Tble:TTCL_Base_Math_BrierConf} and Table~\ref{Tble:TTCL_Base_Fact_BrierConf}, TTCL effectively corrects the mismatch between confidence and empirical accuracy across both mathematical reasoning and factual QA benchmarks, while simultaneously improving answer accuracy. For results of the RLCR-calibrated models in Table~\ref{Tble:TTCL_InDomain_BrierConf} and Table~\ref{Tble:TTCL_OutDomain_BrierConf}, TTCL further improves calibration on top of the already calibrated models. In particular, under out-of-domain, where the confidence learned from the source domain may transfer poorly on the target domain, TTCL effectively adapts the confidence estimates to the target distribution.

\begin{table}[t]
    \centering
    \caption{\textbf{More results of TTCL-trained models on multiple other domains.} These results examine whether TTCL improves calibration at the cost of model general capabilities on other domains. 
    }
    \label{Tble:TTCL_Other_Bench_Appendix}
    \vspace{-3mm}
    \resizebox{\textwidth}{!}{
    \begin{tabular}{l|ccccccc|l|ccccccc}
    \toprule[1.6pt]
        \multirow{2}{*}{\textbf{TTCL Target}} & \multicolumn{2}{c}{\textbf{ARC}} & \multirow{2}{*}{\textbf{GPQA-D}} & \multirow{2}{*}{\textbf{IFEval}} & \multirow{2}{*}{\textbf{MMLU-Pro}} & \multicolumn{2}{c|}{\textbf{TruthfulQA}} & \multirow{2}{*}{\textbf{TTCL Target}} & \multicolumn{2}{c}{\textbf{ARC}} & \multirow{2}{*}{\textbf{GPQA-D}} & \multirow{2}{*}{\textbf{IFEval}} & \multirow{2}{*}{\textbf{MMLU-Pro}} & \multicolumn{2}{c}{\textbf{TruthfulQA}} \\
        \cmidrule(lr){2-3} \cmidrule(lr){7-8} \cmidrule(lr){10-11} \cmidrule(lr){15-16}
        ~ & Challenge & Easy & ~ & ~ & ~ & MC1 & MC2 & ~ & Challenge & Easy & ~ & ~ & ~ & MC1 & MC2 \\
        \midrule
            \multicolumn{16}{c}{\textit{\textbf{Qwen3-8B}}} \\
        \midrule
        \textbf{Base Model} & 55.20 & 75.34 & 50.00 & 84.62 & 63.70 & 32.80 & 51.60 & \textbf{Base Model} & 55.20 & 75.34 & 50.00 & 84.62 & 63.70 & 32.80 & 51.60 \\
        \textbf{MATH500} & 55.29 & 74.83 & 51.52 & 85.47 & 64.82 & 33.29 & 51.65 & \textbf{SimpleQA} & 54.52 & 74.83 & 45.96 & 84.65 & 63.86 & 33.17 & 51.20 \\
        \textbf{AMC} & 55.29 & 74.83 & 45.96 & 85.83 & 63.87 & 33.05 & 51.35 & \textbf{NQ} & 55.12 & 74.62 & 51.52 & 84.77 & 63.12 & 33.05 & 51.54 \\ 
        \textbf{AIME24} & 54.52 & 74.07 & 48.99 & 86.20 & 62.66 & 32.68 & 50.95 & \textbf{HotpotQA} & 55.46 & 74.79 & 45.45 & 86.11 & 63.87 & 33.29 & 51.03 \\ 
        \textbf{AIME25} & 54.52 & 74.12 & 51.01 & 84.98 & 63.16 & 32.80 & 50.82 & \textbf{TriviaQA} & 54.78 & 74.75 & 48.99 & 84.56 & 63.74 & 32.56 & 50.70 \\
        \specialrule{1pt}{0.4ex}{0.4ex}
            \multicolumn{16}{c}{\textit{\textbf{Qwen3-4B}}} \\
        \midrule
        \textbf{Base Model} & 48.63 & 65.28 & 46.97 & 81.39 & 56.50 & 35.74 & 51.49 & \textbf{Base Model} & 48.63 & 65.28 & 46.97 & 81.39 & 56.50 & 35.74 & 51.49 \\
        \textbf{MATH500} & 48.98 & 65.11 & 46.46 & 84.48 & 57.46 & 35.99 & 51.66 & \textbf{SimpleQA} & 48.63 & 65.95 & 40.91 & 82.18 & 57.05 & 36.23 & 52.01 \\
        \textbf{AMC} & 48.55 & 65.45 & 45.45 & 83.29 & 57.43 & 36.72 & 51.94 & \textbf{NQ} & 48.89 & 64.39 & 44.95 & 84.57 & 57.00 & 36.23 & 51.64 \\
        \textbf{AIME24} & 47.70 & 64.81 & 43.94 & 83.57 & 56.99 & 36.23 & 52.07 & \textbf{HotpotQA} & 48.63 & 64.56 & 43.94 & 82.20 & 56.86 & 36.47 & 51.30 \\
        \textbf{AIME25} & 48.29 & 64.81 & 46.46 & 83.52 & 56.39 & 36.60 & 51.99 & \textbf{TriviaQA} & 47.95 & 64.60 & 42.42 & 84.92 & 56.79 & 35.25 & 51.70 \\
        \specialrule{1pt}{0.4ex}{0.4ex}
            \multicolumn{16}{c}{\textit{\textbf{Llama-3.1-8B-Instruct}}} \\
        \midrule
        \textbf{Base Model} & 59.13 & 81.23 & 34.85 & 77.26 & 37.73 & 38.07 & 52.77 & \textbf{Base Model} & 59.13 & 81.23 & 34.85 & 77.26 & 37.73 & 38.07 & 52.77 \\
        \textbf{MATH500} & 57.59 & 79.46 & 33.33 & 75.16 & 37.55 & 37.58 & 51.58 & \textbf{AMC} & 55.63 & 78.87 & 32.32 & 75.63 & 37.73 & 35.01 & 49.23 \\
        \textbf{AIME24} & 58.19 & 80.01 & 34.85 & 77.58 & 34.92 & 35.13 & 49.76 & \textbf{AIME25} & 58.53 & 80.05 & 32.83 & 77.53 & 36.40 & 35.74 & 50.13 \\
        \specialrule{1pt}{0.4ex}{0.4ex}
            \multicolumn{16}{c}{\textit{\textbf{RLCR-calibrated Qwen3-8B on DAPO-Math-14K}}} \\
        \midrule
        \textbf{Base Model}     & 55.20 & 75.34 & 50.00 & 84.62 & 63.70 & 32.80 & 51.60 & \textbf{Base Model}     & 55.20 & 75.34 & 50.00 & 84.62 & 63.70 & 32.80 & 51.60 \\
        \textbf{RLCR Model}     & 55.03 & 74.62 & 52.53 & 84.63 & 63.85 & 32.44 & 51.48 & \textbf{RLCR Model}     & 55.03 & 74.62 & 52.53 & 84.63 & 63.85 & 32.44 & 51.48 \\
        \textbf{MATH500}        & 54.95 & 74.33	& 52.02 & 84.01 & 64.44 & 32.93 & 51.63 & \textbf{SimpleQA}       & 55.03 & 74.03 & 52.53 & 84.28 & 64.07 & 32.93 & 50.49 \\
        \textbf{AMC}            & 54.61 & 74.41 & 57.07 & 84.58 & 64.50 & 32.80 & 51.61 & \textbf{ChineseSimpleQA} & 54.95 & 74.16 & 55.05 & 84.10 & 64.54 & 32.56 & 50.83 \\
        \textbf{AIME24}         & 54.44 & 73.53 & 52.53 & 85.28 & 64.15 & 32.68 & 50.51 & \textbf{NQ}             & 54.69 & 74.20 & 52.53 & 84.74 & 64.06 & 32.56 & 51.12 \\
        \textbf{AIME25}         & 55.03 & 74.20 & 50.00 & 83.46 & 64.39 & 33.17 & 51.32 & \textbf{HotpotQA}       & 54.86 & 74.33 & 52.53 & 85.38 & 64.40 & 32.93 & 51.22 \\
        \textbf{AIME26}         & 54.61 & 74.16 & 52.53 & 84.08 & 64.35 & 32.93 & 51.27 & \textbf{TriviaQA}       & 54.86 & 74.49 & 53.03 & 84.07 & 63.92 & 32.56 & 50.71 \\
        \specialrule{1pt}{0.4ex}{0.4ex}
            \multicolumn{16}{c}{\textit{\textbf{RLCR-calibrated Qwen3-4B on DAPO-Math-14K}}} \\
        \midrule
        \textbf{Base Model}     & 48.63 & 65.28 & 46.97 & 81.39 & 56.50 & 35.74 & 51.49 & \textbf{Base Model}     & 48.63 & 65.28 & 46.97 & 81.39 & 56.50 & 35.74 & 51.49 \\
        \textbf{RLCR Model}     & 48.63 & 65.45 & 50.51 & 82.70	& 58.75 & 36.72	& 51.84 & \textbf{RLCR Model}     & 48.63 & 65.45 & 50.51 & 82.70	& 58.75 & 36.72	& 51.84 \\
        \textbf{MATH500}        & 48.04 & 64.52 & 47.98 & 82.68 & 59.52 & 35.86 & 51.71 & \textbf{SimpleQA}       & 49.06 & 65.78 & 48.99 & 82.21 & 59.56 & 35.74 & 51.75 \\
        \textbf{AMC}            & 48.63 & 64.77 & 51.01 & 81.92 & 58.21 & 36.23 & 52.11 & \textbf{ChineseSimpleQA} & 48.46 & 64.56 & 47.47 & 83.08 & 58.30 & 36.60 & 51.79 \\
        \textbf{AIME24}         & 47.95 & 64.69 & 48.99 & 83.50 & 59.14 & 36.47 & 51.98 & \textbf{NQ}             & 48.29 & 64.81 & 48.48 & 83.61 & 59.86 & 36.35 & 51.94 \\
        \textbf{AIME25}         & 48.12 & 64.98 & 49.49 & 84.36 & 58.64 & 36.35 & 52.00 & \textbf{HotpotQA}       & 48.63 & 65.24 & 46.97 & 83.06 & 58.89 & 36.35 & 52.26 \\
        \textbf{AIME26}         & 48.04 & 65.15 & 48.99 & 83.61 & 59.00 & 36.47 & 52.03 & \textbf{TriviaQA}       & 48.38	& 65.32 & 53.03 & 82.91 & 58.88	& 36.23 & 51.46 \\
        \specialrule{1pt}{0.4ex}{0.4ex}
            \multicolumn{16}{c}{\textit{\textbf{RLCR-calibrated Qwen3-4B on FactQA-10K}}} \\
        \midrule
        \textbf{Base Model}     & 48.63 & 65.28 & 46.97 & 81.39 & 56.50 & 35.74 & 51.49 & \textbf{RLCR Model}     & 49.15 & 65.53	& 43.43	& 83.20 & 57.10 & 35.86 & 51.27 \\
        \textbf{RLCR Model}     & 49.15 & 65.53	& 43.43	& 83.20 & 57.10 & 35.86 & 51.27 & \textbf{AIME24}         & 48.55 & 65.32 & 44.44 & 83.08 & 56.85 & 36.35 & 51.70 \\
        \textbf{MATH500}        & 48.63 & 65.36 & 43.94 & 82.88 & 56.92 & 35.99 & 51.27 & \textbf{AIME25}         & 47.61 & 64.27 & 45.96 & 83.16 & 56.63 & 36.23 & 51.48 \\
        \textbf{AMC}            & 48.72 & 65.49 & 43.94 & 84.10 & 57.58 & 36.11 & 51.90 & \textbf{AIME26}         & 48.29 & 65.03 & 44.95 & 83.52 & 56.88 & 36.35 & 51.72 \\
    \bottomrule[1.6pt]
    \end{tabular}}
\end{table}

\par
\textbf{TTCL does not simply decrease confidence; instead, it adjusts confidence toward the model's empirical accuracy.}
For example, for the RLCR-calibrated Qwen3-4B on MATH500 in Table~\ref{Tble:TTCL_InDomain_BrierConf}, the mean confidence increases from 0.772 to 0.913 after TTCL, closely matching the corresponding accuracy of 90.4\%. This result indicates that the calibration improvement achieved by TTCL is not due to indiscriminately suppressing confidence, but rather to better aligning confidence with empirical correctness. This observation provides empirical evidence for our theoretical analysis that the TTCL objective serves as a surrogate for the ideal calibration learning objective.

\par
These results further support our analysis in Section~\ref{Sec:Main_Results} that TTCL can improve both accuracy and calibration on target tasks at test time, aligning confidence estimates with empirical correctness.

\subsection{More Visualizations of TTCL Calibration Effects}
\label{App:Visual_Calibrate}
As a supplement to the results in Fig.~\ref{Fig:calibration_bar}, to provide an intuitive view of the calibration effects of TTCL, we present additional visualizations under different settings in Fig.~\ref{Fig:TTCL_BaseMath_Calibration_Visualize}, Fig.~\ref{Fig:TTCL_BaseFact_Calibration_Visualize}, Fig.~\ref{Fig:TTCL_RLCR_InDomain_Visualize} and Fig.~\ref{Fig:TTCL_RLCR_OutDomain_Visualize}. Specifically, we visualize the calibration performance before and after TTCL for: (1) base models on mathematical reasoning benchmarks (Fig.~\ref{Fig:TTCL_BaseMath_Calibration_Visualize}), (2) base models on factual QA benchmarks (Fig.~\ref{Fig:TTCL_BaseFact_Calibration_Visualize}), (3) math-calibrated models on mathematical reasoning benchmarks (Fig.~\ref{Fig:TTCL_RLCR_InDomain_Visualize}), and (4) cross-domain adaptation, including math-calibrated models (DAPO-Math-14K) on factual QA benchmarks and fact-calibrated models (FactQA-10K) on mathematical reasoning benchmarks (Fig.~\ref{Fig:TTCL_RLCR_OutDomain_Visualize}). 

\par
Across these settings, TTCL consistently improves the alignment between confidence and empirical accuracy. For base models, TTCL effectively corrects their pronounced over-confidence, bringing confidence into closer alignment with empirical accuracy. For math-calibrated models evaluated in-domain, TTCL further refines calibration on top of the existing calibrated initialization. More importantly, under cross-domain transfer, calibration learned on the source domain transfers poorly to the target domain, whereas TTCL effectively adapts confidence estimates to the target benchmarks. Overall, these visualizations provide an intuitive view of the consistent calibration improvements achieved by TTCL across different model initializations and domain settings.

\begin{table}[t]
    \centering
    \caption{\textbf{Detailed IFEval results of TTCL-trained models.} IFEval reports four metrics: prompt-level strict accuracy, prompt-level loose accuracy, instruction-level strict accuracy, and instruction-level loose accuracy. In previous tables, we report the average of these four metrics. Here, we provide the detailed results as a supplement to Table~\ref{Tble:TTCL_Other_Bench} and Table~\ref{Tble:TTCL_Other_Bench_Appendix}.}
    \label{Tble:TTCL_IFEval_Appendix}
    \vspace{-3mm}
    \resizebox{\textwidth}{!}{
    \begin{tabular}{l|ccccc|l|ccccc}
    \toprule[1.6pt]
        \textbf{TTCL Target} & \textbf{Avg.} & \textbf{Prompt Strict} & \textbf{Inst. Strict} & \textbf{Prompt Loose} & \textbf{Inst. Loose} & \textbf{TTCL Target} & \textbf{Avg.} & \textbf{Prompt Strict} & \textbf{Inst. Strict} & \textbf{Prompt Loose} & \textbf{Inst. Loose} \\
        \midrule
            \multicolumn{12}{c}{\textit{\textbf{Qwen3-8B}}} \\
        \midrule
        \textbf{Base Model} & 84.62 & 80.04 & 86.45 & 83.36 & 88.61 & \textbf{Base Model}   & 84.62 & 80.04 & 86.45 & 83.36 & 88.61 \\
        \textbf{MATH500}    & 85.47 & 81.15 & 87.05 & 84.47 & 89.21 & \textbf{SimpleQA}     & 84.65 & 79.85 & 86.33 & 83.55 & 88.85 \\
        \textbf{AMC}        & 85.83 & 81.33 & 87.29 & 85.03 & 89.69 & \textbf{NQ}           & 84.77 & 80.41 & 83.36 & 86.69 & 88.61  \\
        \textbf{AIME24}     & 86.20 & 82.07 & 87.89 & 85.03 & 89.81 & \textbf{HotpotQA}     & 86.11 & 82.07 & 85.03 & 87.65 & 89.69  \\
        \textbf{AIME25}     & 84.98 & 80.59 & 86.57 & 83.92 & 88.85 & \textbf{TriviaQA}     & 84.56 & 80.22 & 83.18 & 86.33 & 88.49  \\
        \specialrule{1pt}{0.4ex}{0.4ex}
            \multicolumn{12}{c}{\textit{\textbf{Qwen3-4B}}} \\
        \midrule
        \textbf{Base Model} & 81.39 & 75.97 & 83.45 & 80.04 & 86.09 & \textbf{Base Model}   & 81.39 & 75.97 & 83.45 & 80.04 & 86.09  \\
        \textbf{MATH500}    & 84.48 & 80.04 & 85.97 & 83.55 & 88.37 & \textbf{SimpleQA}     & 82.18	& 76.89 & 84.05 & 80.96 & 86.81  \\
        \textbf{AMC}        & 83.29 & 78.37 & 84.77 & 82.26 & 87.77 & \textbf{NQ}           & 84.58 & 80.41 & 83.55 & 85.97 & 88.37 \\
        \textbf{AIME24}     & 83.57 & 78.74 & 85.13 & 82.62 & 87.77 & \textbf{HotpotQA}     & 82.20 & 77.26 & 81.15 & 83.81 & 86.57 \\
        \textbf{AIME25}     & 83.52 & 78.19 & 84.77 & 82.99 & 88.13 & \textbf{TriviaQA}     & 84.93 & 80.59 & 84.29 & 86.09 & 88.73 \\
        \specialrule{1pt}{0.4ex}{0.4ex}
            \multicolumn{12}{c}{\textit{\textbf{Llama-3.1-8B-Instruct}}} \\
        \midrule
        \textbf{Base Model} & 77.26 & 70.98 & 79.74 & 75.23 & 83.09 & \textbf{Base Model}   & 77.26 & 70.98 & 79.74 & 75.23 & 83.09 \\
        \textbf{MATH500}    & 75.16 & 68.21 & 77.10 & 73.94 & 81.41 & \textbf{AMC}          & 75.63 & 69.50 & 78.18 & 73.57 & 81.29 \\
        \textbf{AIME24}     & 77.58 & 71.35 & 79.86 & 75.79 & 83.33 & \textbf{AIME25}       & 77.53 & 70.98 & 80.10 & 75.60 & 83.45 \\
        \specialrule{1pt}{0.4ex}{0.4ex}
            \multicolumn{12}{c}{\textit{\textbf{RLCR-calibrated Qwen3-8B on DAPO-Math-14K}}} \\
        \midrule
        \textbf{Base Model} & 84.62 & 80.04 & 86.45 & 83.36 & 88.61 & \textbf{Base Model} & 84.62 & 80.04 & 86.45 & 83.36 & 88.61 \\
        \textbf{RLCR Model} & 84.63 & 80.22 & 86.69 & 82.99 & 88.61 & \textbf{RLCR Model} & 84.63 & 80.22 & 86.69 & 82.99 & 88.61 \\
        \textbf{MATH500}    & 84.01 & 79.67 & 86.09 & 82.26 & 88.01 & \textbf{SimpleQA}   & 84.28 & 79.85 & 86.21 & 82.81 & 88.25  \\
        \textbf{AMC}        & 84.58 & 79.67 & 86.45 & 83.36 & 88.85 & \textbf{ChineseSimpleQA} & 84.10 & 79.30 & 85.85 & 82.99 & 88.25 \\
        \textbf{AIME24}     & 85.28 & 80.96 & 87.17 & 83.92 & 89.09 & \textbf{NQ}         & 84.74 & 80.59 & 86.45 & 83.55 & 88.37 \\
        \textbf{AIME25}     & 83.46 & 78.74 & 85.73 & 81.70 & 87.65 & \textbf{HotpotQA}   & 85.38 & 81.15 & 87.17 & 84.10 & 89.09 \\
        \textbf{AIME26}     & 84.08 & 79.67 & 86.09 & 82.44 & 88.13 & \textbf{TriviaQA}   & 84.07 & 79.67 & 85.97 & 82.62 & 88.01 \\
        \specialrule{1pt}{0.4ex}{0.4ex}
            \multicolumn{12}{c}{\textit{\textbf{RLCR-calibrated Qwen3-4B on DAPO-Math-14K}}} \\
        \midrule
        \textbf{Base Model} & 81.39 & 75.97 & 83.45 & 80.04 & 86.09 & \textbf{Base Model} & 81.39 & 75.97 & 83.45 & 80.04 & 86.09  \\
        \textbf{RLCR Model} & 82.70 & 77.08 & 84.29 & 81.89 & 87.53 & \textbf{RLCR Model} & 82.70 & 77.08 & 84.29 & 81.89 & 87.53 \\
        \textbf{MATH500}    & 82.68 & 77.63 & 84.29 & 81.52 & 87.29 & \textbf{SimpleQA}   & 82.21 & 77.08 & 84.17 & 80.78 & 86.81  \\
        \textbf{AMC}        & 81.92 & 76.52 & 83.45 & 81.15 & 86.57 & \textbf{ChineseSimpleQA} & 83.08 & 78.19 & 84.77 & 82.07 & 87.29 \\
        \textbf{AIME24}     & 83.50 & 78.37 & 85.01 & 82.62 & 88.01 & \textbf{NQ}         & 83.61 & 78.74 & 85.13 & 82.81 & 87.77 \\
        \textbf{AIME25}     & 84.36 & 79.67 & 85.85 & 83.55 & 88.37 & \textbf{HotpotQA}   & 83.06 & 78.19 & 84.89 & 81.89 & 87.29 \\
        \textbf{AIME26}     & 83.61 & 78.74 & 85.13 & 82.81 & 87.77 & \textbf{TriviaQA}   & 82.91 & 77.26 & 83.93 & 82.81 & 87.65 \\
        \specialrule{1pt}{0.4ex}{0.4ex}
            \multicolumn{12}{c}{\textit{\textbf{RLCR-calibrated Qwen3-4B on FactQA-14K}}} \\
        \midrule
        \textbf{Base Model} & 81.39 & 75.97 & 83.45 & 80.04 & 86.09 & \textbf{RLCR Model} & 83.20 & 78.56 & 84.77 & 82.07 & 87.41 \\
        \textbf{RLCR Model} & 83.20 & 78.56 & 84.77 & 82.07 & 87.41 & \textbf{AIME24}     & 83.08 & 78.00 & 84.77 & 81.89 & 87.65 \\
        \textbf{MATH500}    & 82.88 & 77.82 & 84.53 & 81.89 & 87.29 & \textbf{AIME25}     & 83.16 & 78.19 & 84.65 & 82.26 & 87.53 \\
        \textbf{AMC}        & 84.10 & 79.48 & 85.49 & 83.18 & 88.25 & \textbf{AIME26}     & 83.52 & 78.74 & 85.13 & 82.44 & 87.77 \\
    \bottomrule[1.6pt]
    \end{tabular}}
\end{table}

\subsection{More Training Dynamics of TTCL}
\label{App:Training_Dynamics_TTCL}
As a supplement to Fig.~\ref{Fig:training_dynamics}, we provide additional training dynamics to illustrate how TTCL adjusts model confidence and accuracy throughout test-time learning. Figure~\ref{Fig:TTCL_Base_Dynamics_Appendix} shows the dynamics on base models, and Figure~\ref{Fig:TTCL_Calibrated_Dynamics_Appendix} presents the corresponding dynamics on math-calibrated models under factual QA adaptation. It can be observed that, across different models and target benchmarks, TTCL consistently drives mean confidence toward empirical accuracy while improving task accuracy. For base models, the initially pronounced over-confidence is progressively corrected during training, leading to a substantially reduced confidence-accuracy gap. For calibrated models, TTCL further adapts the existing confidence estimates to the target-domain performance, showing that calibration learned from the source setting can continue to be refined at test time. Overall, these dynamics provide a process-level view of how TTCL jointly improves accuracy and calibration during adaptation, further supporting our theoretical analysis that the TTCL objective serves as a surrogate for the ideal calibration learning objective by aligning confidence estimates with empirical correctness.

\subsection{Reward Hacking without Target Cache}
\label{App:Reward_Hacking_No_Cache}
To examine the failure mode that motivates the target cache, we train Qwen3-4B on AIME24 using supervision computed directly from the current rollout group. Fig.~\ref{Fig:Reward_Hacking_No_Cache} shows four questions for which the rollouts increasingly concentrate on an incorrect answer. As more rollouts produce that answer, its frequency-based probability rises toward one. Because this probability also serves as the confidence target, the supervision encourages the model's confidence to rise toward 1.0 despite the answer being wrong. This is a form of reward hacking: the model can obtain a higher self-generated reward by consistently producing an incorrect answer with high confidence, without improving correctness or calibration. The target cache is motivated by this failure mode; by fixing targets within each epoch, it weakens the immediate feedback between the current policy and its self-supervision signals.

\begin{figure}[t]
    \centering
    \includegraphics[width=\linewidth]{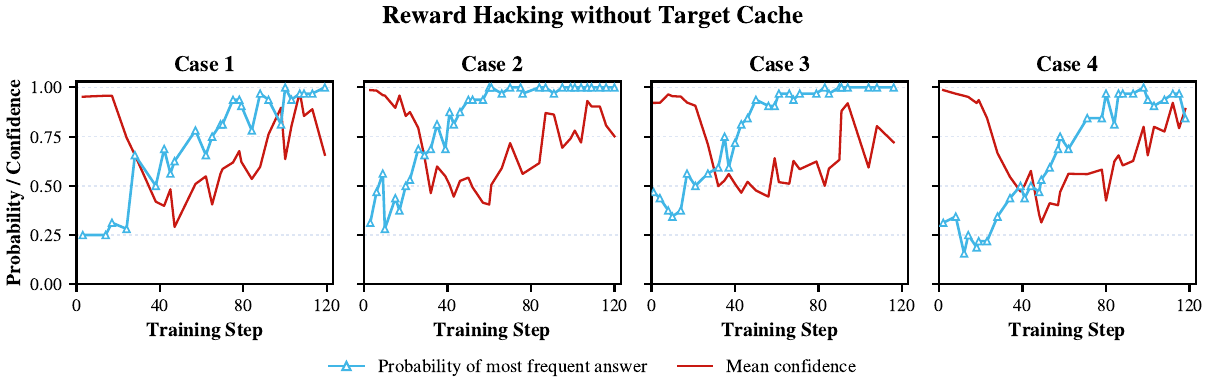}
    \vspace{-3mm}
    \caption{\textbf{Reward hacking without the target cache.} Four AIME24 examples in which the model's rollouts increasingly concentrate on an incorrect answer while mean confidence rises toward 1.0. The \textcolor{cyan}{blue lines} show the frequency-based probability of the most frequent answer, and the \textcolor{softred}{red lines} show mean confidence. When supervision signals are derived from the the rollout group of current policy, increasing agreement on a wrong answer raises its confidence target, leading to reward hacking.
    }
    \label{Fig:Reward_Hacking_No_Cache}
    \vspace{-3mm}
\end{figure}

\subsection{More Results of General Capabilities}
\label{App:Results_Other_Domains}
As a supplement to Table~\ref{Tble:TTCL_Other_Bench} in Section~\ref{Sec:Furhter_Results}, we provide more comprehensive results in Table~\ref{Tble:TTCL_Other_Bench_Appendix}. Specifically, we further report results for more TTCL-trained models on benchmarks covering multiple aspects of general capabilities, including ARC~\citep{clark2018arc}, GPQA-Diamond~\citep{rein2023gpqa}, IFEval~\citep{zhou2023ifeval}, MMLU-Pro~\citep{wang2024mmlupro}, and TruthfulQA~\citep{lin2022truthfulqa}, which evaluate capabilities across scientific reasoning, instruction following, multi-domain knowledge, and truthfulness. For IFEval, the evaluation includes four metrics: prompt-level strict, prompt-level loose, instruction-level strict, and instruction-level loose accuracy; we report their average values, with the detailed results provided in Appendix~\ref{App:Detailed_IFEval}. Across these benchmarks, TTCL-trained models generally maintain performance comparable to their corresponding models before TTCL adaptation, with only minor fluctuations across individual tasks. These results indicate that the calibration improvements brought by TTCL do not come at the cost of degrading the model's general capabilities on other domains. Instead, TTCL can improve accuracy and confidence calibration on the target tasks while largely preserving the model's general capabilities.

\begin{table}[t]
    \centering
    \caption{\textbf{Calibration results of existing label-free RL methods.} We evaluate several label-free RL methods using the released checkpoints from~\citet{zhang2026co}, and report accuracy, ECE, and mean confidence values. Mean confidence values that are substantially higher than the corresponding empirical accuracy, indicating over-confidence of calibration, are highlighted in \textcolor{softred}{red}.}
    \label{Tble:Calibration_LabelfreeRL}
    \vspace{-3mm}
    \resizebox{\textwidth}{!}{
    \begin{tabular}{l|ccc|ccc|ccc|ccc|ccc}
    \toprule[1.6pt]
        \multirow{2}{*}{\textbf{Label-free RL Methods}} & \multicolumn{3}{c}{\textbf{MATH500}} & \multicolumn{3}{c}{\textbf{AMC}} & \multicolumn{3}{c}{\textbf{AIME24}} & \multicolumn{3}{c}{\textbf{AIME25}} & \multicolumn{3}{c}{\textbf{AIME26}} \\ 
        \cmidrule(lr){2-4} \cmidrule(lr){5-7} \cmidrule(lr){8-10} \cmidrule(lr){11-13} \cmidrule(lr){14-16}
        ~ & \textbf{ACC $\uparrow$} & \textbf{ECE $\downarrow$} & \textbf{Conf} & \textbf{ACC $\uparrow$} & \textbf{ECE $\downarrow$} & \textbf{Conf} & \textbf{ACC $\uparrow$} & \textbf{ECE $\downarrow$} & \textbf{Conf} & \textbf{ACC $\uparrow$} & \textbf{ECE $\downarrow$} & \textbf{Conf} & \textbf{ACC $\uparrow$} & \textbf{ECE $\downarrow$} & \textbf{Conf}  \\ 
        \midrule
        \multicolumn{16}{c}{\textit{\textbf{Qwen3-4B-Base}}} \\
        \midrule
        Base Model  & 58.8 & 0.347 & \color{softred}0.935 & 38.44 & 0.541 & \color{softred}0.926 & 10.00 & 0.826 & \color{softred}0.926 & 4.17 & 0.876 & \color{softred}0.918 & 6.25 & 0.861 & \color{softred}0.924 \\
        \midrule
        Self-Certainty~\citep{zhao2026learning} & 61.6 & 0.342 & \color{softred}0.958 & 41.25 & 0.540 & \color{softred}0.953 & 7.50 & 0.877 & \color{softred}0.952 & 5.00 & 0.894 & \color{softred}0.944 & 5.00 & 0.900 & \color{softred}0.950 \\
        Entropy~\citep{prabhudesai2025maximizing} & 77.2 & 0.206 & \color{softred}0.978 & 51.88 & 0.462 & \color{softred}0.981 & 14.58 & 0.833 & \color{softred}0.978 & 8.75 & 0.890 & \color{softred}0.978 & 9.58 & 0.885 & \color{softred}0.980 \\
        Majority-voting~\citep{zuo2026ttrl} & 72.6 & 0.230 & \color{softred}0.956 & 43.13 & 0.509 & \color{softred}0.940 & 9.58 & 0.831 & \color{softred}0.927 & 4.58 & 0.876 & \color{softred}0.921 & 5.83 & 0.864 & \color{softred}0.923 \\
        Co-rewarding~\citep{zhang2026co} & 75.6 & 0.214 & \color{softred}0.970 & 52.81 & 0.431 & \color{softred}0.959 & 14.58 & 0.801 & \color{softred}0.947 & 9.17 & 0.852 & \color{softred}0.944 & 6.67 & 0.881 & \color{softred}0.947 \\
        \specialrule{1pt}{0.4ex}{0.4ex}
        \multicolumn{16}{c}{\textit{\textbf{Qwen3-8B-Base}}} \\
        \midrule
        Base Model  & 60.6 & 0.305 & \color{softred}0.911 & 42.50 & 0.487 & \color{softred}0.912 & 13.75 & 0.773 & \color{softred}0.911 & 13.33 & 0.778 & \color{softred}0.912 & 6.25 & 0.846 & \color{softred}0.908 \\
        \midrule
        Self-Certainty~\citep{zhao2026learning} & 83.2 & 0.132 & \color{softred}0.964 & 59.69 & 0.363 & \color{softred}0.960 & 20.00 & 0.750 & \color{softred}0.950 & 16.25 & 0.787 & \color{softred}0.949 & 13.33 & 0.820 & \color{softred}0.954 \\
        Entropy~\citep{prabhudesai2025maximizing} & 79.0 & 0.177 & \color{softred}0.967 & 57.50 & 0.388 & \color{softred}0.963 & 14.17 & 0.818 & \color{softred}0.960 & 14.17 & 0.813 & \color{softred}0.955 & 9.17 & 0.866 & \color{softred}0.957 \\
        Majority-voting~\citep{zuo2026ttrl} & 78.2 & 0.180 & \color{softred}0.962 & 57.19 & 0.380 & \color{softred}0.952 & 12.08 & 0.822 & \color{softred}0.943 & 7.50 & 0.862 & \color{softred}0.937 & 6.25 & 0.875 & \color{softred}0.937 \\
        Co-rewarding~\citep{zhang2026co} & 81.0 & 0.157 & \color{softred}0.967 & 60.63 & 0.350 & \color{softred}0.956 & 15.00 & 0.797 & \color{softred}0.947 & 11.25 & 0.834 & \color{softred}0.946 & 11.67 & 0.829 & \color{softred}0.945 \\
        \specialrule{1pt}{0.4ex}{0.4ex}
        \multicolumn{16}{c}{\textit{\textbf{Llama-3.2-3B-Instruct}}} \\
        \midrule
        Base Model  & 35.2 & 0.514 & \color{softred}0.866 & 24.06 & 0.615 & \color{softred}0.855 & 6.67 & 0.766 & \color{softred}0.832 & 0.00 & 0.828 & \color{softred}0.828 & 0.83 & 0.825 & \color{softred}0.834 \\
        \midrule
        Self-Certainty~\citep{zhao2026learning} & 47.0 & 0.428 & \color{softred}0.898 & 25.00 & 0.637 & \color{softred}0.887 & 6.67 & 0.807 & \color{softred}0.874 & 0.42 & 0.865 & \color{softred}0.869 & 0.42 & 0.863 & \color{softred}0.867 \\
        Entropy~\citep{prabhudesai2025maximizing} & 31.0 & 0.589 & \color{softred}0.899 & 15.63 & 0.730 & \color{softred}0.887 & 5.00 & 0.815 & \color{softred}0.865 & 0.42 & 0.866 & \color{softred}0.870 & 0.83 & 0.850 & \color{softred}0.858 \\
        Majority-voting~\citep{zuo2026ttrl} & 42.0 & 0.470 & \color{softred}0.890 & 25.63 & 0.619 & \color{softred}0.875 & 7.50 & 0.781 & \color{softred}0.856 & 0.42 & 0.846 & \color{softred}0.850 & 0.83 & 0.838 & \color{softred}0.847 \\
        Co-rewarding~\citep{zhang2026co} & 46.4 & 0.426 & \color{softred}0.890 & 28.44 & 0.561 & \color{softred}0.846 & 7.08 & 0.743 & \color{softred}0.814 & 0.00 & 0.813 & \color{softred}0.813 & 1.25 & 0.801 & \color{softred}0.814 \\
        \bottomrule[1.6pt]
    \end{tabular}}
    \vspace{-2mm}
\end{table}

\subsection{Detailed Results of IFEval}
\label{App:Detailed_IFEval}
In the previous results, we report the average performance over the four IFEval metrics: prompt-level strict accuracy, prompt-level loose accuracy, instruction-level strict accuracy, and instruction-level loose accuracy. As a supplement, Table~\ref{Tble:TTCL_IFEval_Appendix} further reports the detailed results for each individual metric. It can be observed that TTCL effectively preserves the model's instruction-following capability. Across different target tasks and model settings, the IFEval performance after TTCL remains generally comparable to that before TTCL training, with slight improvements observed in several cases, such as on Qwen3-4B. These results further indicate that TTCL can improve accuracy and calibration on target tasks without compromising the model's general instruction-following capability.

\subsection{Calibration Results of Existing Label-free RL Methods}
To examine whether existing label-free RL methods can improve model calibration while enhancing reasoning accuracy, we evaluate several representative approaches, including Self-certainty Maximization~\citep{zhao2026learning}, Entropy Minimization~\citep{prabhudesai2025maximizing}, Majority-voting for pseudo-label construction~\citep{zuo2026ttrl}, and Co-rewarding~\citep{zhang2026co}. These methods are primarily designed to improve reasoning performance without ground-truth labels, but lack an explicit calibration objective. Thanks to the Co-rewarding work, which releases checkpoints trained on DAPO-Math-14K with these different label-free RL methods, we directly evaluate these checkpoints for calibration performance. Since most of them are trained from pretrained models and have limited ability to follow verbalized-confidence instructions, we use the mean token probability of the generated response as the confidence estimate~\citep{malinin2020uncertainty,orgad2025llms}.

\par
As shown in Table~\ref{Tble:Calibration_LabelfreeRL}, although these label-free RL methods improve accuracy to some extent compared with base models, their calibration remains relatively poor. We observe that the ECE values are generally high, while the mean confidence consistently exceeds 0.8 and is typically above 0.9 for the Qwen models, indicating clear over-confidence. These results suggest that accuracy-oriented label-free RL does not naturally lead to well-calibrated confidence estimates when an explicit calibration objective is absent. These observations are consistent with our discussion in Appendix~\ref{App:RelatedWork} Related Work and further motivate our exploration of label-free test-time calibration learning framework.

\subsection{Case Study}
To intuitively illustrate the calibration behavior learned by TTCL, we present several cases after TTCL adaptation. In these case studies, \textcolor{blue}{blue boxes} denote the questions, \textcolor{softred}{red boxes} denote outputs that are not well calibrated (e.g., over-confidence, incorrect answer with high confidence), and \textcolor{darkgreen}{green boxes} denote the well-calibrated outputs after TTCL. These cases show that the TTCL-trained models can effectively assign high confidence to correct answers while expressing relatively low confidence on incorrect answers. When calibration learned on a source domain transfers poorly to newly encountered target tasks, TTCL can further adapt the model's confidence to the target distribution and improve its calibration on these tasks. This provides a concrete view of how TTCL aligns verbalized confidence with prediction correctness on target tasks.

\newtcolorbox{exbox}[2]{
    enhanced,
    breakable,
    colback=#2!8,
    colframe=#2!60,
    coltitle=black,
    fonttitle=\bfseries,
    title={#1},
    boxrule=0.6pt,
    arc=2mm,
    left=2mm,
    right=2mm,
    top=1.5mm,
    bottom=1.5mm
}

\begin{exbox}{Question from AMC:}{blue}
For how many integers $n$ does the expression \[\sqrt{\frac{\log (n^2) - (\log n)^2}{\log n - 3}}\] represent a real number, where log denotes the base $10$ logarithm?

\textbf{Ground-truth answer:} $\boxed{901}$
\end{exbox}

\begin{exbox}{Over-confidence of Qwen3-4B (Incorrect Answer with High Confidence)}{softred}

We are given the expression:

\[
\sqrt{\frac{\log (n^2) - (\log n)^2}{\log n - 3}}
\]

We are to find how many \textbf{integers} $n$ make this expression a
\textbf{real number}, where $\log$ denotes the \textbf{base 10 logarithm}.

\medskip
\hrule
\medskip

\textbf{Step 1: Simplify the expression}

Let’s denote $x = \log n$. Since $\log$ is base 10, the expression becomes:

\[
\sqrt{\frac{\log (n^2) - (\log n)^2}{\log n - 3}}
=
\sqrt{\frac{2\log n - (\log n)^2}{\log n - 3}}
=
\sqrt{\frac{2x - x^2}{x - 3}}
\]

We want this to be \textbf{real}, so the \textbf{expression inside the square root must be non-negative}:

\[
\frac{2x - x^2}{x - 3} \geq 0
\]

Also, the \textbf{denominator cannot be zero}, so:

\[
x \neq 3
\Rightarrow
\log n \neq 3
\Rightarrow
n \neq 10^3 = 1000
\]

So, we must have:

\[
\frac{2x - x^2}{x - 3} \geq 0,
\quad
x \neq 3
\]

\medskip
\hrule
\medskip

\textbf{Step 2: Solve the inequality}

We analyze the sign of the expression:

\[
\frac{2x - x^2}{x - 3}
=
\frac{-x^2 + 2x}{x - 3}
=
\frac{-x(x - 2)}{x - 3}
\]

Let’s find the critical points:

\begin{itemize}
    \item Numerator: $-x(x - 2) = 0 \Rightarrow x = 0, 2$
    \item Denominator: $x - 3 = 0 \Rightarrow x = 3$
\end{itemize}

So the sign chart is defined by intervals:

\begin{itemize}
    \item $x < 0$
    \item $0 < x < 2$
    \item $2 < x < 3$
    \item $x > 3$
\end{itemize}

Now test the sign of the expression in each interval:

\begin{center}
\resizebox{0.98\linewidth}{!}{
\begin{tabular}{c|c|c|c|c}
\hline
\textbf{Interval} & \textbf{Sign of $-x$} & \textbf{Sign of $x-2$} & \textbf{Sign of $x-3$} & \textbf{Sign of $\frac{-x(x-2)}{x-3}$} \\
\hline
$x < 0$ & $-$ & $-$ & $-$ & $\frac{(-)(-)}{(-)} = \frac{+}{-} = -$ \\
\hline
$0 < x < 2$ & $-$ & $-$ & $-$ & $\frac{(-)(-)}{(-)} = \frac{+}{-} = -$ \\
\hline
$2 < x < 3$ & $-$ & $+$ & $-$ & $\frac{(-)(+)}{(-)} = \frac{-}{-} = +$ \\
\hline
$x > 3$ & $-$ & $+$ & $+$ & $\frac{(-)(+)}{(+)} = \frac{-}{+} = -$ \\
\hline
\end{tabular}
}
\end{center}

So the expression is \textbf{non-negative} when:

\[
2 < x < 3
\]

But $x = \log n$, so:

\[
2 < \log n < 3
\Rightarrow
10^2 < n < 10^3
\Rightarrow
100 < n < 1000
\]

Also, $n$ must be an \textbf{integer}, and $n \neq 1000$
(since $\log n = 3$ is not allowed).

\medskip
\hrule
\medskip

\textbf{Step 3: Count the integers in the interval}

We are to count the integers $n$ such that:

\[
100 < n < 1000
\]

So $n$ ranges from 101 to 999 (inclusive). The number of such integers is:

\[
999 - 101 + 1 = 899
\]

\medskip
\hrule
\medskip

\textbf{Final Answer:}

\[
\boxed{899}
\]

\texttt{<conf>} Confidence: 0.95

\end{exbox}

\begin{exbox}{Well-calibrated Qwen3-4B after TTCL on AMC (Incorrect Answer with Low Confidence)}{darkgreen}

To determine for how many integers $n$ the expression

\[
\sqrt{\frac{\log (n^2) - (\log n)^2}{\log n - 3}}
\]

represents a real number, we need to ensure two conditions:

\begin{enumerate}
    \item The expression inside the square root must be \textbf{non-negative} (i.e., greater than or equal to 0).
    \item The denominator $\log n - 3$ must not be \textbf{zero} (to avoid division by zero).
\end{enumerate}

Let’s analyze the expression step by step.

\medskip
\hrule
\medskip

\textbf{Step 1: Simplify the numerator}

We are given:

\[
\log(n^2) - (\log n)^2
\]

Using the logarithmic identity $\log(n^2) = 2\log n$, this becomes:

\[
2\log n - (\log n)^2
\]

Let $x = \log n$. Then the numerator becomes:

\[
2x - x^2 = -x^2 + 2x
\]

So the entire expression becomes:

\[
\sqrt{\frac{-x^2 + 2x}{x - 3}}
\]

Where $x = \log n$.

\medskip
\hrule
\medskip

\textbf{Step 2: Analyze the expression inside the square root}

We need:

\[
\frac{-x^2 + 2x}{x - 3} \geq 0
\]

Also, we must have $x \neq 3$ (to avoid division by zero).

Let’s analyze the sign of the expression:

\[
\frac{-x^2 + 2x}{x - 3} = \frac{-x(x - 2)}{x - 3}
\]

We now analyze the sign of this expression.

\medskip
\hrule
\medskip

\textbf{Step 3: Find the critical points}

The numerator is zero when $x = 0$ or $x = 2$.

The denominator is zero when $x = 3$.

So the critical points are $x = 0, 2, 3$.

We now test the sign of the expression in the intervals determined by these points:

\begin{itemize}
    \item $x < 0$
    \item $0 < x < 2$
    \item $2 < x < 3$
    \item $x > 3$
\end{itemize}

Let’s analyze the sign of the expression in each interval.

\textbf{Interval 1: $x < 0$}

\begin{itemize}
    \item Numerator: $-x(x - 2)$: $x < 0$, so $-x > 0$, and $x - 2 < 0$, so overall numerator is positive.
    \item Denominator: $x - 3 < 0$
    \item Overall: $\frac{+}{-} = -$ $\rightarrow$ Negative $\rightarrow$ Not allowed.
\end{itemize}

\textbf{Interval 2: $0 < x < 2$}

\begin{itemize}
    \item Numerator: $-x(x - 2)$: $x > 0$, $x - 2 < 0$, so overall numerator is positive.
    \item Denominator: $x - 3 < 0$
    \item Overall: $\frac{+}{-} = -$ $\rightarrow$ Negative $\rightarrow$ Not allowed.
\end{itemize}

\textbf{Interval 3: $2 < x < 3$}

\begin{itemize}
    \item Numerator: $-x(x - 2)$: $x > 0$, $x - 2 > 0$, so overall numerator is negative.
    \item Denominator: $x - 3 < 0$
    \item Overall: $\frac{-}{-} = +$ $\rightarrow$ Positive $\rightarrow$ Allowed.
\end{itemize}

\textbf{Interval 4: $x > 3$}

\begin{itemize}
    \item Numerator: $-x(x - 2)$: $x > 0$, $x - 2 > 0$, so overall numerator is negative.
    \item Denominator: $x - 3 > 0$
    \item Overall: $\frac{-}{+} = -$ $\rightarrow$ Negative $\rightarrow$ Not allowed.
\end{itemize}

\medskip
\hrule
\medskip

\textbf{Step 4: Valid interval}

The expression is \textbf{non-negative} only in the interval $2 < x < 3$.

So we require:

\[
2 < \log n < 3
\]

Convert to $n$:

\[
10^2 < n < 10^3 \Rightarrow 100 < n < 1000
\]

Now, $n$ must be an \textbf{integer} in this interval.

The integers from 101 to 999 inclusive are:

\[
999 - 101 + 1 = 899
\]

But we must also ensure that $\log n \neq 3$, which means $n \neq 1000$.

So the valid integers are from 101 to 999 inclusive.

\[
999 - 101 + 1 = 899
\]

Thus, the number of integers $n$ for which the expression represents a real number is:

\[
\boxed{899}
\]

\texttt{<conf>} Confidence: 0.25

\end{exbox}

\begin{exbox}{Question from AIME24:}{blue}
Jen enters a lottery by picking $4$ distinct numbers from $S=\{1,2,3,\cdots,9,10\}.$ $4$ numbers are randomly chosen from $S.$ She wins a prize if at least two of her numbers were $2$ of the randomly chosen numbers, and wins the grand prize if all four of her numbers were the randomly chosen numbers. The probability of her winning the grand prize given that she won a prize is $\tfrac{m}{n}$ where $m$ and $n$ are relatively prime positive integers. Find $m+n$.

\textbf{Ground-truth answer:} $\boxed{116}$
\end{exbox}

\begin{exbox}{Well-calibrated Qwen3-4B after TTCL on AIME24 (Correct Answer with High Confidence)}{darkgreen}

To solve this problem, we need to calculate the probability that Jen wins the \textbf{grand prize} given that she has won a \textbf{prize} (i.e., at least two of her numbers match the randomly chosen numbers).

\medskip
\textbf{Step 1: Total number of ways to choose 4 numbers from 10}

\[
\binom{10}{4} = 210
\]

\medskip
\textbf{Step 2: Number of ways Jen can win a prize (at least 2 of her numbers match)}

This includes:
\begin{itemize}
    \item Exactly 2 matching numbers
    \item Exactly 3 matching numbers
    \item Exactly 4 matching numbers (grand prize)
\end{itemize}

Let’s calculate each case.

\medskip
\textbf{Case 1: Exactly 2 matching numbers}

\begin{itemize}
    \item Choose 2 numbers from Jen's 4: $\binom{4}{2}$
    \item Choose 2 numbers from the remaining 6: $\binom{6}{2}$
\end{itemize}

\[
\binom{4}{2} \cdot \binom{6}{2} = 6 \cdot 15 = 90
\]

\medskip
\textbf{Case 2: Exactly 3 matching numbers}

\begin{itemize}
    \item Choose 3 numbers from Jen's 4: $\binom{4}{3}$
    \item Choose 1 number from the remaining 6: $\binom{6}{1}$
\end{itemize}

\[
\binom{4}{3} \cdot \binom{6}{1} = 4 \cdot 6 = 24
\]

\medskip
\textbf{Case 3: Exactly 4 matching numbers (grand prize)}

\begin{itemize}
    \item Choose all 4 numbers from Jen's 4: $\binom{4}{4} = 1$
\end{itemize}

\medskip
\textbf{Total number of ways to win a prize:}

\[
90 + 24 + 1 = 115
\]

\medskip
\textbf{Step 3: Number of ways to win the grand prize (exactly 4 matching numbers)}

\[
1
\]

\medskip
\textbf{Step 4: Probability of winning the grand prize given that she won a prize}

\[
P(\text{grand prize} \mid \text{prize}) = \frac{1}{115}
\]

\medskip
\textbf{Step 5: Simplify the fraction $\frac{1}{115}$}

The fraction is already in simplest form, so $m = 1$, $n = 115$.

\medskip
\textbf{Step 6: Find $m + n$}

\[
1 + 115 = 116
\]

\medskip
\textbf{Final Answer:}

\[
\boxed{116}
\]

\texttt{<conf>} Confidence: 0.92

\end{exbox}

\begin{exbox}{Question from factual benchmark NQ:}{blue}
when is season 2 of jessica jones being released

\textbf{Ground-truth answer:} March 8, 2018
\end{exbox}

\begin{exbox}{RLCR-Qwen3-8B calibrated on math transfers poorly on FactQA \\ (Incorrect Answer with High Confidence)}{softred}

Let's break this down step by step:

\begin{enumerate}
    \item \textbf{Identify the show}: \textit{Jessica Jones} is a Marvel superhero TV series that aired on Netflix. It is part of the \textit{Marvel Netflix Universe} (MNLU).

    \item \textbf{Season 1 release}: \textit{Jessica Jones} Season 1 was released on \textbf{November 20, 2015}.

    \item \textbf{Season 2 release}: Season 2 was released on \textbf{March 15, 2017}.

    \item \textbf{Current status}: The show was canceled by Netflix in 2018, and no further seasons were produced. However, \textit{Jessica Jones} was later picked up by Marvel Studios for a film, \textit{Jessica Jones} (2020), which is not a continuation of the TV series.
\end{enumerate}

\textbf{Final Answer:}

Season 2 of \textit{Jessica Jones} was released on \textbf{March 15, 2017}.

\texttt{<answer>}March 15, 2017\texttt{</answer>}

\texttt{<conf>} Confidence: 0.92

\end{exbox}

\begin{exbox}{TTCL further calibrates RLCR-Qwen3-8B on the target NQ task \\ (Incorrect Answer with Low Confidence)}{darkgreen}

Jessica Jones is a Marvel Netflix series that originally aired in 2015. Season 2 of \textit{Jessica Jones} was released on \textbf{December 1, 2016}.

The series was produced by Marvel Television and released on Netflix, so the release date is specific to the streaming platform.

\texttt{<answer>}December 1, 2016\texttt{</answer>}

\texttt{<conf>} Confidence: 0.1

\end{exbox}

\begin{exbox}{Question from factual benchmark NQ:}{blue}
what's the dog's name on tom and jerry

\textbf{Ground-truth answer:} Spike
\end{exbox}

\begin{exbox}{TTCL further calibrates RLCR-Qwen3-8B on the target NQ task \\ (Correct Answer with High Confidence):}{darkgreen}

The dog's name on \textit{Tom and Jerry} is \textbf{Spike}.

Tom is the cat, and Spike is the dog. They are the main characters in the classic animated series.

\texttt{<answer>}Spike\texttt{</answer>}

\texttt{<conf>} Confidence: 0.95

\end{exbox}

\end{document}